\documentclass[review, sort&compress, numbers]{elsarticle}
\usepackage[font={footnotesize}]{caption}
\usepackage[fleqn]{amsmath}
\usepackage{graphicx}
\usepackage{lineno,hyperref}
\usepackage{latexsym}
\usepackage{amssymb}
\usepackage[center]{subfigure}
\usepackage{float}
\usepackage{color}
\usepackage{geometry}
\usepackage{diagbox}
\usepackage[section]{placeins}
\usepackage{bm}
\usepackage{algorithm}
\usepackage{algorithmicx}
\usepackage{algpseudocode}
\usepackage{lipsum}
\makeatletter
\usepackage{array}
\usepackage{multirow}
\usepackage{longtable}
\usepackage{rotating}
\newdefinition{definition}{Definition}
\modulolinenumbers[5]
\newproof{proof}{Proof}
\newtheorem{theorem}{Theorem}

\newtheorem{lemma}{Lemma}
\graphicspath{{./figs/},{./logo/}}

\makeatletter
\renewcommand{\@thesubfigure}{\hskip\subfiglabelskip}
\makeatother
\journal{}

\begin{document}
\begin{frontmatter}
\title{Low-rank tensor completion via tensor joint rank with logarithmic composite norm}

\author[mymainaddress]{Hongbing~Zhang\corref{mycorrespondingauthor}}
\cortext[mycorrespondingauthor]{Corresponding author}
\ead{zhb123abc@163.com}
\address[mymainaddress]{School of Mathematics and Statistics, Lanzhou University, Lanzhou, Gansu 730000, China}
\date{}

\begin{abstract}
Low-rank tensor completion (LRTC) aims to recover a complete low-rank tensor from incomplete observed tensor, attracting extensive attention in various practical applications such as image processing and computer vision. However, current methods often perform well only when there is a sufficient of observed information, and they perform poorly or may fail when the observed information is less than 5\%. In order to improve the utilization of observed information, a new method called the tensor joint rank with logarithmic composite norm (TJLC) method is proposed. This method simultaneously exploits two types of tensor low-rank structures, namely tensor Tucker rank and tubal rank, thereby enhancing the inherent correlations between known and missing elements. To address the challenge of applying two tensor ranks with significantly different directly to LRTC, a new tensor Logarithmic composite norm is further proposed. 
Subsequently, the TJLC model and algorithm for the LRTC problem are proposed. Additionally, theoretical convergence guarantees for the TJLC method are provided. Experiments on various real datasets demonstrate that the proposed method outperforms state-of-the-art methods significantly. Particularly, the proposed method achieves satisfactory recovery even when the observed information is as low as 1\%, and the recovery performance improves significantly as the observed information increases.
\end{abstract}

\begin{keyword}
Low-rank tensor completion, tensor Tucker rank, tensor tubal rank, tensor joint rank, tensor logarithmic composite norm.
\end{keyword}
\end{frontmatter}

\section{Introduction}
As the dimension of real data increases and its structure becomes more complex, tensor, as high-order generalizations of vector and matrix, has received widespread attention from researchers. Currently, tensor play an increasingly important role in various applications such as magnetic resonance image (MRI) data recovery \cite{LIAO2023109624,YANG2022108311}, multispectral/hyperspectral image (HSI/MSI) processing \cite{doi:10.1137/22M1531907,PAN2023109699,li2022nonlinear}, color image/ color video (CV) processing \cite{Lu_2019_CVPR,jiang2023robust,10145070}, face data recovery \cite{song2020robust}, background subtraction \cite{KONG2023109545,5032019109}, and signal reconstruction \cite{ding2022tensor}.

Low-rank tensor completion (LRTC) aims to recover a complete low-rank tensor from severely incomplete observed tensor, and LRTC problem is an important issue in image processing and computer vision research. The low-rank tensor completion problem expresses as follows:
\begin{eqnarray}
\min_{\mathcal{X}}~{\tt rank}(\mathcal{X})~{\tt s.t}.~\mathcal{P}_{\Omega}(\mathcal{T})=\mathcal{P}_{\Omega}(\mathcal{X}),\label{lLRTC}
\end{eqnarray}
where $\mathcal{T}\in\mathbb{R}^{\mathit{I}_{1}\times\mathit{I}_{2}\times\cdots\times\mathit{I}_{N}}$ is the $N$-th order observed tensor; $\mathcal{X}$ is the complete tensor; $\mathcal{P}_{\Omega}(\mathcal{X})$ is a projection operator that keeps the entries of $\mathcal{X}$ in $\Omega$ and sets all others to zero. A crucial issue of problem (\ref{lLRTC}) is how to define the tensor rank. Different from the definition of matrix rank, the definition of tensor rank is not unique. The mainstream definitions of tensor rank are the CANDECOMP/PARAFAC (CP) rank based on CP decomposition \cite{harshman1970foundations}, Tucker rank based on Tucker decomposition \cite{tucker1966some}, and tubal rank \cite{doi:10.1137/110837711} induced by tensor singular value decomposition (t-SVD) \cite{6416568}. Nevertheless, directly solving the CP rank of the given tensor is NP-hard \cite{139201360}. The calculation of Tucker rank requires data to be folded and unfolded, which will cause structural damage to data. The tubal rank can better maintain the data structure compared with CP rank and Tucker rank. Subsequently, Zhang et al. \cite{6909886} defined the tensor nuclear norm based on t-SVD and tensor tubal rank to solve the LRTC problem and obtained the advanced tensor recovery results. 

Currently, researchers have successfully recovered observed tensors using various tensor methods. For example, Xue et al. proposed a multilayer sparsity-based tensor decomposition method for LRTC problem based on the tensor CP rank in \cite{9460114}. Xue et al. proposed the Laplacian scale mixture based on three-layer transform method for LRTC problem based on the tensor Tucker rank in \cite{9694520}. Yu et al. proposed the tensor $\ell^{p}_{r}$ pseudo-norm method for tensor recovery problem based on the tensor tubal rank in \cite{YU2023109343}. Kong et al. \cite{24120002} proposed a Low-tubal-rank tensor completion via local and nonlocal knowledge. Although the above methods have achieved considerable effectiveness, they have not deeply considered cases with very limited observation information, namely when the known information rate is below 5\%. These methods often demonstrate their advantages only when the known information rate is relatively high. Most low-rank tensor completion methods perform poorly or even fail when the observation information is below 5\%. How to achieve better recovery results with fewer observed information has always been an important research topic in LRTC. In order to improve the utilization of observation information, that is, to enhance the intrinsic correlation between known and missing elements, a new method called the tensor joint rank with logarithmic composite norm (TJLC) method is proposed. A new tensor joint rank is proposed, which can simultaneously represent the tensor Tucker rank and tubal rank. The tensor joint rank utilizes the correlation characteristics of tensors from multiple perspectives, including the correlation information between different dimensions of the tensor and between any two dimensions, thereby enhancing the intrinsic correlation between known and missing elements. Although the tensor joint rank greatly enhances the utilization of information, the significant differences in the structures of different tensor ranks make it impossible to directly apply them to the low-rank tensor completion problem. To address this issue, a new tensor logarithmic composite norm is further proposed. It processes singular values in a multiple way, effectively balancing the differences between the tensor Tucker rank and tubal rank.

The main contributions of this paper are summarized below:


Firstly, a new rank representation called the tensor joint rank is proposed, which can simultaneously describe both the tensor Tucker rank and tensor tubal rank. By utilizing the feature information of tensors from multiple perspectives, it enhances the intrinsic correlation between known and missing elements. However, this rank representation method cannot be directly applied to the LRTC problem due to the differences in the two low-rank tensor structures. To address this issue, a new tensor logarithmic composite norm is further proposed, which performs multiple processing on singular values to balance the differences between two tensor ranks. 

Secondly, a new TJLC model is proposed for LRTC problem, ensuring that the recovered complete tensor has both tensor Tucker rank and tensor tubal rank with low-rank tensor structures. The solution algorithm for the model is provided using the alternating direction multipliers method. Furthermore, the convergence guarantee of the proposed method is proved under some assumptions.

Thirdly, Experiments on various real datasets demonstrate that the TJLC method significantly outperforms the state-of-the-art methods. Furthermore, experimental results with extremely limited observation information further validate the strong utilization of observation information by the proposed method. Particularly, the proposed method achieves good recovery even when observation information is as low as 1\%, and the recovery effectiveness significantly improves with the increase of observation information.

The rest of this article is organized as follows. Some preliminary knowledge and background of the tensor are given in Section 2. The tensor joint rank and the tensor logarithmic composite norm and corresponding theorem are presented in Section 3. The proposed model and algorithm are shown in Section 4. Then, in section 5, we study the convergence of the proposed method. In Section 6, the results of extensive experiments are presented. Conclusion and future work are drawn in Section 7.
\section{Preliminaries}
In this section, we list some basic notations and briefly introduce some definitions used throughout the paper. Generally, a lowercase letter and an uppercase letter denote a vector $x$ and a matrix $X$, respectively. An $N$th-order tensor is denoted by a calligraphic upper case letter $\mathcal{X}\in \mathbb{R}^{\mathit{I}_{1}\times\mathit{I}_{2}\times\cdots\times\mathit{I}_{N}}$ and $\mathit{x}_{i_{1},i_{2},\cdots,i_{N}}$ is its $(i_{1},i_{2},\cdots,i_{N})$-th element. The Frobenius norm of a tensor is defined as $\|\mathcal{X}\|_{F}=(\sum_{i_{1},i_{2},\cdots,i_{N}}\mathit{x}_{i_{1},i_{2},\cdots,i_{N}}^{2})^{1/2}$. For a third-order tensor $\mathcal{X}\in\mathbb{R}^{\mathit{I}_{1}\times\mathit{I}_{2}\times\mathit{I}_{3}}$, we use $\bar{\mathcal{X}}$ to denote the tensor generated by performing discrete Fourier transform (DFT) along each tube of $\mathcal{X}$, i.e., $\bar{\mathcal{X}}={\tt fft}(\mathcal{X},[~],3)$. The inverse DFT is computed by command ifft satisfying $\mathcal{X}={\tt ifft}(\bar{\mathcal{X}},[~],3)$), where fft and ifft are MATLAB
commands. More often, the frontal slice $\mathcal{X}(:,:,i)$ is denoted compactly as $\mathcal{X}^{(i)}$.
\begin{definition}[Tensor mode-$n$ unfolding and folding \cite{12345152009}]
	The mode-$n$ unfolding of a tensor $\mathcal{X}\in \mathbb{R}^{\mathit{I}_{1}\times\mathit{I}_{2}\times\cdots\times\mathit{I}_{N}}$ is denoted as a matrix $\mathcal{X}_{(n)}\in\mathbb{R}^{\mathit{I}_{n}\times\mathit{I}_{1}\cdots\mathit{I}_{n-1}\mathit{I}_{n+1}\cdots\mathit{I}_{N}} $. Tensor element $(i_{1}, i_{2},...,i_{N} )$ maps to matrix element $(i_{n}, j)$, where
	\begin{eqnarray*}
		j=1+\sum_{k=1,k\neq n}^{N}(i_{k}-1)\mathit{J}_{k}\quad \text{with}\quad \mathit{J}_{k}=\prod_{m=1,m\neq n}^{k-1}\mathit{I}_{m}. 
	\end{eqnarray*}
	The mode-$n$ unfolding operator and its inverse are respectively denoted as ${\tt unfold}_{n}$ and ${\tt fold}_{n}$, and they satisfy $\mathcal{X}={\tt fold}_{n}(\mathcal{X}_{(n)})={\tt fold}_{n}({\tt unfold}_{n}(\mathcal{X}))$.
\end{definition}
\begin{definition}[Tensor Tucker rank \cite{12345152009}]
	The tensor Tucker rank is defined as a vector, the $i$-th element of which is the rank of the mode-$i$ unfolding matrix, i.e., 
	\begin{eqnarray*}
		{\tt rank}_{tc}:=({\tt rank}(\mathcal{X}_{(1)}), {\tt rank}(\mathcal{X}_{(2)}),\cdots, {\tt rank}(\mathcal{X}_{(N)})),
	\end{eqnarray*}
	where $\mathcal{X}$ is an $N$th-order tensor and $\mathcal{X}_{(i)}~(i=1,2,\dots,N)$ is the mode-$i$ unfolding of $\mathcal{X}$.
\end{definition}
For a third-order tensor $\mathcal{X}\in\mathbb{R}^{\mathit{I}_{1}\times\mathit{I}_{2}\times\mathit{I}_{3}}$, the block circulation operation is defined as
\begin{eqnarray}
	{\tt bcirc}(\mathcal{X}):=
	\begin{pmatrix}
		\mathcal{X}^{(1)}& \mathcal{X}^{(\mathit{I}_{3})}&\dots& \mathcal{X}^{(2)}&\\
		\mathcal{X}^{(2)}& \mathcal{X}^{(1)}&\dots& \mathcal{X}^{(3)}&\\
		\vdots&\vdots&\ddots&\vdots&\\
		\mathcal{X}^{(\mathit{I}_{3})}& \mathcal{X}^{(\mathit{I}_{3}-1)}&\dots& \mathcal{X}^{(1)}&
	\end{pmatrix}\in\mathbb{R}^{\mathit{I}_{1}\mathit{I}_{3}\times\mathit{I}_{2}\mathit{I}_{3}}.\nonumber
\end{eqnarray}
The block diagonalization operation and its inverse operation are respectively determined by 
\begin{eqnarray}
	{\tt bdiag}(\mathcal{X}):=\begin{pmatrix}
		\mathcal{X}^{(1)} & & &\\
		& \mathcal{X}^{(2)} & &\\
		& & \ddots &\\
		& & & \mathcal{X}^{(\mathit{I}_{3})}
	\end{pmatrix} \in\mathbb{R}^{\mathit{I}_{1}\mathit{I}_{3}\times\mathit{I}_{2}\mathit{I}_{3}},\quad {\tt bdfold}({\tt bdiag}(\mathcal{X})):=\mathcal{X}.\nonumber
\end{eqnarray}
The block vectorization operation and its inverse operation are respectively defined as 
\begin{eqnarray}
	{\tt bvec}(\mathcal{X}):=\begin{pmatrix}
		\mathcal{X}^{(1)}\\\mathcal{X}^{(2)}\\\vdots\\\mathcal{X}^{(\mathit{I}_{3})}
	\end{pmatrix}\in\mathbb{R}^{\mathit{I}_{1}\mathit{I}_{3}\times\mathit{I}_{2}},\quad {\tt bvfold}({\tt bvec}(\mathcal{X})):=\mathcal{X}.\nonumber
\end{eqnarray}
\begin{definition}[t-product \cite{6416568}]
	Let $\mathcal{A}\in\mathbb{R}^{\mathit{I}_{1}\times\mathit{I}_{2}\times\mathit{I}_{3}}$ and $\mathcal{B}\in\mathbb{R}^{\mathit{I}_{2}\times\mathit{J}\times\mathit{I}_{3}}$. Then the t-product $\mathcal{A}\ast\mathcal{B}$ is defined to be a tensor of size $\mathit{I}_{1}\times\mathit{J}\times\mathit{I}_{3}$,
	\begin{eqnarray}
		\mathcal{A}\ast\mathcal{B}:={\tt bvfold}({\tt bcirc}(\mathcal{A}){\tt bvec}(\mathcal{B})).\nonumber
	\end{eqnarray}
\end{definition}
	The \textbf{tensor conjugate transpose} of a tensor $\mathcal{A}\in\mathbb{C}^{\mathit{I}_{1}\times\mathit{I}_{2}\times\mathit{I}_{3}}$ is the tensor $\mathcal{A}^{H}\in\mathbb{C}^{\mathit{I}_{2}\times\mathit{I}_{1}\times\mathit{I}_{3}}$ obtained by conjugate transposing each of the frontal slices and then reversing the order of transposed frontal slices 2 through $\mathit{I}_{3}$.
	The \textbf{identity tensor} $\mathcal{I}\in\mathbb{R}^{\mathit{I}_{1}\times\mathit{I}_{1}\times\mathit{I}_{3}}$ is the tensor whose first frontal slice is the $\mathit{I}_{1}\times\mathit{I}_{1}$ identity matrix, and whose other frontal slices are all zeros.
It is clear that ${\tt bcirc}(\mathcal{I})$ is the  $\mathit{I}_{1}\mathit{I}_{3}\times\mathit{I}_{1}\mathit{I}_{3}$ identity matrix. So it is easy to get $\mathcal{A}\ast\mathcal{I}=\mathcal{A}$ and $\mathcal{I}\ast\mathcal{A}=\mathcal{A}$. 
	A tensor $\mathcal{Q}\in\mathbb{R}^{\mathit{I}_{1}\times\mathit{I}_{1}\times\mathit{I}_{3}}$ is \textbf{orthogonal tensor} if it satisfies
		$\mathcal{Q}\ast\mathcal{Q}^{H}=\mathcal{Q}^{H}\ast\mathcal{Q}=\mathcal{I}.$
	A tensor is called \textbf{f-diagonal} if each of its frontal slices is a diagonal matrix.
\begin{theorem}[t-SVD \cite{8606166}]
	Let $\mathcal{X}\in\mathbb{R}^{\mathit{I}_{1}\times\mathit{I}_{2}\times\mathit{I}_{3}}$ be a third-order tensor, then it can be factored as 
	\begin{eqnarray}
		\mathcal{X}=\mathcal{U}\ast\mathcal{S}\ast\mathcal{V}^{H},\nonumber
	\end{eqnarray}
	where $\mathcal{U}\in\mathbb{R}^{\mathit{I}_{1}\times\mathit{I}_{1}\times\mathit{I}_{3}}$ and $\mathcal{V}\in\mathbb{R}^{\mathit{I}_{2}\times\mathit{I}_{2}\times\mathit{I}_{3}}$ are orthogonal tensors, and $\mathcal{S}\in\mathbb{R}^{\mathit{I}_{1}\times\mathit{I}_{2}\times\mathit{I}_{3}}$ is an f-diagonal tensor. 

\end{theorem}
\begin{definition}[Tensor tubal-rank \cite{doi:10.1137/110837711}]
The tubal-rank of a tensor $\mathcal{X}\in\mathbb{R}^{\mathit{I}_{1}\times\mathit{I}_{2}\times\mathit{I}_{3}}$, denoted as ${\tt rank}_{t}(\mathcal{X})$, is defined to be the number of non-zero singular tubes of $\mathcal{S}$, where $\mathcal{S}$ comes from the t-SVD of $\mathcal{X}:\mathcal{Y}=\mathcal{U}\ast\mathcal{S}\ast\mathcal{V}^{H}$. That is 
\begin{eqnarray*}
	{\tt rank}_{t}(\mathcal{X})=\#\{i:\mathcal{S}(i,:,:)\neq0\}.
\end{eqnarray*}
\end{definition}
When $\mathit{I}_{3}=1$, tensor tubal-rank is equivalent to the matrix rank.

\section{Tensor joint rank and logarithmic composite norm}
In this section, we propose the definitions of the tensor joint rank. Then we further propose the tensor logarithmic composite norm as a guarantee for applying tensor joint rank to the LRTC problem. To facilitate the representation of the tensor joint rank, we first provide a new comprehensive definition of tensor mode-$l_{1}l_2$ unfolding and folding.
\begin{definition}[Tensor mode-$l_{1}l_2$ unfolding and folding]
	For an $N$th-order tensor $\mathcal{X}$, its mode-$l_1 l_2$ unfolding is denoted by
	\begin{eqnarray}
			\mathcal{X}_{(l_1 l_2)}={\tt unfold}(\mathcal{X},l_1,l_2)=\left\{\begin{array}{l}
			{\tt unfold}(\mathcal{X},l_1)\in\mathbb{R}^{\mathit{I}_{l_1}\times\mathit{I}_{1}\cdots\mathit{I}_{l_1-1}\mathit{I}_{l_1+1}\cdots\mathit{I}_{N}},\:l_1=l_2,
			\\{\tt unfold}(\mathcal{X},l_1,l_2)\in\mathbb{R}^{\mathit{I}_{l_{1}}\times\mathit{I}_{l_{2}}\times\prod_{s\neq l_{1},l_{2}}\mathit{I}_{s}},  l_1\neq l_2.
		\end{array}
		\right.
	\end{eqnarray}
When $l_1=l_2$, $\mathcal{X}_{(l_1 l_2)}$ is a matrix. Mathematically, tensor element $(i_{1}, i_{2},...,i_{N} )$ maps to matrix element $(i_{l_1}, j)$, where $	j=1+\sum_{k=1,k\neq l_1}^{N}(i_{k}-1)\mathit{J}_{k}~ \text{with}~ \mathit{J}_{k}=\prod_{m=1,m\neq n}^{k-1}\mathit{I}_{m}$. When $l_1\neq l_2$, $\mathcal{X}_{(l_1 l_2)}$ is a third-order tensor. Mathematically, the  $(i_{1}, i_{2},...,i_{N} )$-th element of $\mathcal{X}$ maps to the $(i_{l_{1}},i_{l_{2}},j)$-th element of $\mathcal{X}_{(l_{1}l_{2})}$, where $j=1+\sum_{s=1,s\neq l_{1},s\neq l_{2}}^{N}(i_{s}-1)\mathit{J}_{s}~ \text{with}~ \mathit{J}_{s}=\prod_{m=1,m\neq l_{1},m\neq l_{2}}^{s-1}\mathit{I}_{m}. $

The tensor mode-$l_{1}l_{2}$ unfolding operator and its inverse operation are respectively denoted as $\mathcal{X}_{(l_{1}l_{2})}:={\tt unfold}(\mathcal{X},l_{1},l_{2})$ and $\mathcal{X}:={\tt fold}(\mathcal{X}_{(l_{1}l_{2})},l_{1},l_{2})$.
\end{definition}
Next, based on tensor mode-$l_{1}l_2$ unfolding and folding, we propose the definition of tensor joint rank.
\begin{definition}[Tensor joint rank]
	The tensor joint rank is a vector consisting of the tubal ranks of all mode-$l_1l_2$ ($l_{1}\leqslant l_{2}$) unfolding tensors, i.e.,
	\begin{eqnarray*}
		&&{\tt rank}_{TJ}(\mathcal{X}):=({\tt rank}_{t}(\mathcal{X}_{(11)}),{\tt rank}_{t}(\mathcal{X}_{(12)}),\cdots, {\tt rank}_{t}(\mathcal{X}_{(1N)}),\\&&\qquad\qquad\qquad~~{\tt rank}_{t}(\mathcal{X}_{(22)}),\cdots,{\tt rank}_{t}(\mathcal{X}_{(N-1N)}),{\tt rank}_{t}(\mathcal{X}_{(NN)}))\in\mathbb{R}^{N(N+1)/2}.
	\end{eqnarray*}
\end{definition}

It is not difficult to see that tensor joint rank possesses two kinds of low-rank tensor structures, namely, tensor Tucker rank and tensor tubal rank. Therefore, the proposed tensor joint rank combines the advantages of tensor Tucker rank and tensor tubal rank. On one hand, the mode-$l_1l_2$ unfolding of the tensor ($l_1=l_2$) can capture the global correlations in multidimensional data. On the other hand, the mode-$l_1l_2$ unfolding of the tensor ($l_1\neq l_2$) describes the pairwise correlations between each pair of dimensions in the multidimensional data. In summary, the tensor joint rank can more comprehensively utilize the correlations of multiple low-rank structures between known and missing elements, thereby ensuring high-quality tensor completion. However, considering the significant structural differences between the two tensor ranks, directly applying tensor joint rank to LRTC problem may weaken the recovery performance. Next, to solve the above problem, a new tensor logarithmic composite norm is further proposed.
\begin{definition}[Tensor logarithmic composite norm]
	The tensor logarithmic composite norm of $\mathcal{X}\in\mathbb{R}^{\mathit{I}_{1}\times\mathit{I}_{2}\times\mathit{I}_{3}}$, denoted by $\|\mathcal{X}\|_{LC}$, is defined as follows:
		\begin{eqnarray}
		&&\|\mathcal{X}\|_{LC}=\frac{1}{\mathit{I}_{3}}\sum_{i=1}^{\mathit{I}_{3}}\sum_{j=1}^{R}\omega_{j,i}\log(g(\sigma_{j}(\bar{\mathcal{X}}^{(i)}))+1) ,\label{LC}		
			\end{eqnarray}
		where
				\begin{eqnarray*}
	&&g(\sigma_{j}(\bar{\mathcal{X}}^{(i)}))=\left\{\begin{array}{l}
			 \nu\sigma_{j}(\bar{\mathcal{X}}^{(i)})-\sigma_{j}(\bar{\mathcal{X}}^{(i)})^{2}/(2\vartheta),\:\sigma_{j}(\bar{\mathcal{X}}^{(i)})\leqslant \nu\vartheta,
			\\ \nu^{2}\vartheta/2 ,\qquad\qquad\qquad\qquad\quad\,\:\sigma_{j}(\bar{\mathcal{X}}^{(i)})> \nu\vartheta,
		\end{array}
		\right.
		\\&&\omega_{j,i}=1/(c+e^{-w_{R-j+1,i}}),
	\end{eqnarray*}
and $\nu>0,\vartheta>0,R=\min\{\mathit{I}_{1},\mathit{I}_{2}\}$; $\sigma_{j}(\bar{\mathcal{X}}^{(i)})$ is $j$-th singular value of  $\bar{\mathcal{X}}^{(i)}$.
\end{definition}

In tensor recovery problem, larger singular values typically correspond to significant information such as contours, sharp edges, and smooth regions, while smaller singular values are mainly composed of noise or outliers \cite{Gu_2014_CVPR}. The tensor logarithmic composite norm includes multiple singular value processing. The first layer $g(\cdot)$ of processing mainly aims to separate large singular values and protect them during the solution process. The second layer involves the logarithmic penalty, which enhances the penalty for smaller singular values. The final layer is a weighting scheme applied to the overall singular values, increasing sensitivity to different singular values of the same tensor. When approximating the tensor joint rank using the tensor logarithmic composite norm, it effectively brings each rank component to the same level, which solves the problem of significant differences in different low-rank tensor structures. Furthermore, to facilitate the solution of the proposed method in LRTC problem, a proximal operator for tensor logarithmic composite norm is proposed.
\begin{theorem}[Proximal operator for tensor logarithmic composite norm]
Consider the tensor logarithmic composite norm given in (\ref{LC}). Its proximal operator denoted by $\mathit{S}:\mathbb{R}^{\mathit{I}_{1}\times\mathit{I}_{2}\times\mathit{I}_{3}}\to\mathbb{R}^{\mathit{I}_{1}\times\mathit{I}_{2}\times\mathit{I}_{3}}$, $\nu>0,\vartheta>0,$ and $R=\min\{\mathit{I}_{1},\mathit{I}_{2}\}$ and defined as follows:
	\begin{eqnarray}
		\mathit{S}(\mathcal{Y})=\arg\min_{\mathcal{L}}\{\frac{\rho}{2}\|\mathcal{L}-\mathcal{Y}\|_{F}^{2}+\|\mathcal{L}\|_{LC}\},\label{prox}
	\end{eqnarray}
	is given by
	\begin{eqnarray}
		\mathit{S}(\mathcal{Y})=\mathcal{U}\ast\mathcal{S}_{1}\ast\mathcal{V}^{H},
		\label{opewtgn}
	\end{eqnarray}
	where $\mathcal{U}$ and $\mathcal{V}$ are derived from the t-SVD of $\mathcal{Y}=\mathcal{U}\ast\mathcal{S}_{2}\ast\mathcal{V}^{H}$. More importantly, the $i$-th front slice of DFT of $\mathcal{S}_{1}$ and $\mathcal{S}_{2}$, i.e., $\bar{\mathcal{S}}^{(i)}_{1}=\sigma(\bar{\mathcal{L}}^{(i)})$ and $\bar{\mathcal{S}}^{(i)}_{2}=\sigma(\bar{\mathcal{Y}}^{(i)})$, have the following relationship $
	\sigma_{j}(\bar{\mathcal{L}}^{(i)})=S_{LC}(\sigma_{j}(\bar{\mathcal{Y}}^{(i)}))$.
	\label{thPTLC}
\end{theorem}
\begin{proof}
	Let $\mathcal{Y}=\mathcal{U}\ast\mathcal{S}_{2}\ast\mathcal{V}^{H}$ and $\mathcal{L}=\mathcal{W}\ast\mathcal{S}_{1}\ast\mathcal{R}^{H}$ be the t-SVD of $\mathcal{Y}$ and $\mathcal{L}$, respectively. Consider
	\begin{eqnarray}
		&&\mathit{S}(\mathcal{Y})=\arg\min_{\mathcal{L}}\frac{\rho}{2}\|\mathcal{L}-\mathcal{Y}\|_{F}^{2}+\|\mathcal{L}\|_{LC}\nonumber
		\\&&\qquad\;=\arg\min_{\mathcal{L}}\frac{\rho}{2}\|\mathcal{W}\ast\mathcal{S}_{1}\ast\mathcal{R}^{H}-\mathcal{U}\ast\mathcal{S}_{2}\ast\mathcal{V}^{H}\|_{F}^{2}\nonumber
		+\|\mathcal{L}\|_{LC}\nonumber
		\\&&\qquad\;=\arg\min_{\bar{\mathcal{L}}}\frac{1}{\mathit{I}_{3}}(\sum_{i=1}^{\mathit{I}_{3}}\frac{\rho}{2}\|\bar{\mathcal{W}}^{(i)}\ast\bar{\mathcal{S}_{1}}^{(i)}\ast\bar{\mathcal{R}}^{(i)H}-\bar{\mathcal{U}}^{(i)}\ast\bar{\mathcal{S}_{2}}^{(i)}\ast\bar{\mathcal{V}}^{(i)H}\|_{F}^{2}+\|\bar{\mathcal{L}}^{(i)}\|_{LC}).\label{proflc}
	\end{eqnarray}
	It can be found that (\ref{proflc}) is separable and can be divided into $\mathit{I}_{3}$ sub-problems. For the $i$-th sub-problem:
	\begin{eqnarray}
		&&\arg\min_{\bar{\mathcal{L}}^{(i)}}\frac{\rho}{2}\|\bar{\mathcal{W}}^{(i)}\ast\bar{\mathcal{S}_{1}}^{(i)}\ast\bar{\mathcal{R}}^{(i)H}-\bar{\mathcal{U}}^{(i)}\ast\bar{\mathcal{S}_{2}}^{(i)}\ast\bar{\mathcal{V}}^{(i)H}\|_{F}^{2}\nonumber
		+\|\bar{\mathcal{L}}^{(i)}\|_{LC}\nonumber
		\\&&=\arg\min_{\bar{\mathcal{L}}^{(i)}}\frac{\rho}{2}{\tt Tr}(\bar{\mathcal{S}_{1}}^{(i)}\bar{\mathcal{S}_{1}}^{(i)H})+\frac{\rho}{2}{\tt Tr}(\bar{\mathcal{S}_{2}}^{(i)}\bar{\mathcal{S}_{2}}^{(i)H})\nonumber
		+ \rho {\tt Tr}(\bar{\mathcal{L}}^{(i)H}\bar{\mathcal{Y}}^{(i)})+\|\bar{\mathcal{L}}^{(i)}\|_{LC}.
	\end{eqnarray}
	Invoking von Neumann's trace inequality \cite{mirsky1975trace}, we can write
	\begin{eqnarray}
		&&\arg\min_{\bar{\mathcal{L}}^{(i)}}\frac{\rho}{2}\|\bar{\mathcal{W}}^{(i)}\ast\bar{\mathcal{S}_{1}}^{(i)}\ast\bar{\mathcal{R}}^{(i)H}-\bar{\mathcal{U}}^{(i)}\ast\bar{\mathcal{S}_{2}}^{(i)}\ast\bar{\mathcal{V}}^{(i)H}\|_{F}^{2}\nonumber
		+\|\bar{\mathcal{L}}^{(i)}\|_{LC}\nonumber
		\\&&\geqslant\arg\min_{\bar{\mathcal{S}_{1}}^{(i)}}\frac{\rho}{2}{\tt Tr}(\bar{\mathcal{S}_{1}}^{(i)}\bar{\mathcal{S}_{1}}^{(i)H})+\frac{\rho}{2}{\tt Tr}(\bar{\mathcal{S}_{2}}^{(i)}\bar{\mathcal{S}_{2}}^{(i)H})\nonumber
		+\rho {\tt Tr}(\bar{\mathcal{S}_{2}}^{(i)}\bar{\mathcal{S}_{1}}^{(i)H})+\|\bar{\mathcal{L}}^{(i)}\|_{LC}\nonumber\nonumber
		\\&&=\arg\min_{\sigma(\bar{\mathcal{L}}^{(i)})}\frac{\rho}{2}\|\sigma(\bar{\mathcal{L}}^{(i)})-\sigma(\bar{\mathcal{Y}}^{(i)})\|_{F}^{2}+\|\bar{\mathcal{L}}^{(i)}\|_{LC}.\nonumber
		\\&&=\sum_{j=1}^{R}\arg\min_{\sigma_{j}(\bar{\mathcal{L}}^{(i)})}\frac{\rho}{2}(\sigma_{j}(\bar{\mathcal{L}}^{(i)})-\sigma_{j}(\bar{\mathcal{Y}}^{(i)}))^{2}+\omega_{j,i}\log(g(\sigma_{j}(\bar{\mathcal{L}}^{(i)}))+1).
	\end{eqnarray}
	The equality holds when $\bar{\mathcal{W}}^{(i)}=\bar{\mathcal{U}}^{(i)}$ and $\bar{\mathcal{R}}^{(i)}=\bar{\mathcal{V}}^{(i)}$. Hence, the optimal solution to (\ref{proflc}) is obtained by solving the	problem below:
		\begin{eqnarray}
		\sigma_{j}(\bar{\mathcal{L}}^{(i)})=\mathit{S}_{LC}(\sigma_{j}(\bar{\mathcal{Y}}^{(i)}))=\arg\min_{\sigma_{j}(\bar{\mathcal{L}}^{(i)})}\frac{\rho}{2}(\sigma_{j}(\bar{\mathcal{L}}^{(i)})-\sigma_{j}(\bar{\mathcal{Y}}^{(i)}))^{2}+\omega_{j,i}\log(g(\sigma_{j}(\bar{\mathcal{L}}^{(i)}))+1).\label{prox1}
			\end{eqnarray}
For simplicity, let $l:=\sigma_{j}(\bar{\mathcal{L}}^{(i)})$ and $y:=\sigma_{j}(\bar{\mathcal{Y}}^{(i)})$. Then, let 
\begin{eqnarray*}
	&&h(l)=\frac{\rho}{2}(l-y)^{2}+\omega_{j,i}\log(g(l)+1)
	\\&&~~~~\,~=\left\{\begin{array}{l}
		\frac{\rho}{2}(l-y)^{2}+\omega_{j,i}\log \left( \nu l-l^{2}/(2\vartheta)+1\right),\: l\leqslant \nu\vartheta,
		\\\frac{\rho}{2}(l-y)^{2}+\omega_{j,i}\log \left( \nu^{2}\vartheta/2+1 \right),\qquad\:\:\,\: l> \nu\vartheta.
	\end{array}
	\right.\nonumber
\end{eqnarray*}
According to the definition of $h(l)$, when $ l> \nu\vartheta$, $\mathit{S}_{LC}(y) = y$.
 
Next, we consider the case $ l\leqslant \nu\vartheta$. Let the derivatives of the objective function $h(l)$ with respect to $l$ be zero. Therefore, we have
\begin{eqnarray}
	\rho(l-y)+2\omega_{j,i}(\nu\vartheta-l)/(2\nu\vartheta l- l^{2}+2\vartheta)=0.\label{Dgl}
\end{eqnarray}
Due to $2\nu\vartheta l- l^{2}+2\vartheta>0,\rho>0$, equation (\ref{Dgl}) can be transformed into the following form:
\begin{eqnarray}
	-l^{3}+(2\nu\vartheta+y) l^{2}+(-(2\omega_{j,i})/\rho+2\vartheta-2\nu\vartheta y) l+((2\omega_{j,i}\nu\vartheta)/\rho-2\vartheta y)=0.\label{SDgl}
\end{eqnarray}
Let $a=-1,b=2\nu\vartheta+y,c=-2\omega_{j,i}/\rho+2\vartheta-2\nu\vartheta y,d=(2\omega_{j,i}\nu\vartheta)/\rho-2\vartheta y,A=b^{2}-3ac,B=bc-9ad,C=c^{2}-3bd,\Delta=B^{2}-4AC$, $l_{1}, l_{2},$ and $l_{3}$ are the three solutions of equation (\ref{SDgl}). These results are derived by considering different values of $A,B$, and $\Delta$.
\\ 1) Case-1: $A=B=0$. The solution to equation (\ref{SDgl}) in this case are $l_{1}=l_{2}=l_{3}=-c/b$. 
\\ 2) Case-2: $\Delta>0$. The solution to equation (\ref{SDgl}) in this case are as follows:
\begin{eqnarray}
	&&l_{1}=\frac{-b-(\sqrt[3]{K_{1}}+\sqrt[3]{K_{2}})}{3a},\nonumber
	\\&&l_{2}=\frac{-b+0.5(\sqrt[3]{K_{1}}+\sqrt[3]{K_{2}})+0.5\sqrt{3}(\sqrt[3]{K_{1}}-\sqrt[3]{K_{2}})i}{3a},\nonumber
	\\&&l_{3}=\frac{-b+0.5(\sqrt[3]{K_{1}}+\sqrt[3]{K_{2}})-0.5\sqrt{3}(\sqrt[3]{K_{1}}-\sqrt[3]{K_{2}})i}{3a},\nonumber
\end{eqnarray}
where $K_{1}=Ab+1.5a(-B+\sqrt{B^{2}-4AC}),~K_{2}=Ab+1.5a(-B-\sqrt{B^{2}-4AC})$. Since equation (\ref{prox1}) is in the real number domain, only $l_1$ is retained as the solution in this case.
\\ 3) Case-3: $\Delta=0$. The solution to equation (\ref{SDgl}) in this case are $l_{1}=B/A-b/a,~ l_{2}=l_{3}=-B/(2A)$.
\\ 4) Case-4: $\Delta<0$. The solution to equation (\ref{SDgl}) in this case are as follows:
\begin{eqnarray}
	&&l_{1}=(-b-2\sqrt{A}\cos{\frac{\theta}{3}})/(3a),\nonumber
	\\&&l_{2}=(-b+\sqrt{A}(\cos{\frac{\theta}{3}}-\sqrt{3}\sin{\frac{\theta}{3}}))/(3a),\nonumber
	\\&&l_{3}=(-b+\sqrt{A}(\cos{\frac{\theta}{3}}-\sqrt{3}\sin{\frac{\theta}{3}}))/(3a),\nonumber
\end{eqnarray}
where $\theta=\arccos{T},T=(2Ab-3aB)/(2A\sqrt{A})$.

In addition, we need to further consider whether $l_{i}~(i=1,2,3)$ are within the domain, as well as what the optimal solution is. Therefore, we will take the following steps:
\\Step 1: $l_{i}=\min\{\max\{l_{i},0\},\nu\vartheta\}~(i=1,2,3)$. 
\\Step 2: $h(l^{\star})=\min\{h(l_{i}),h(0),h(\nu\vartheta)\}~(i=1,2,3)$.
\\Step 3: The optimal solution at this moment is as follows:
\begin{eqnarray*}
	l^{\star}=\left\{\begin{array}{l}
		l_{i},\quad h(l^{\star})=h(l_{i}),
		\\0,\quad h(l^{\star})=h(0),
		\\y,\quad h(l^{\star})=h(\nu\vartheta).
	\end{array}
	\right.
\end{eqnarray*}
To sum up:
\begin{eqnarray}
	\sigma_{j}(\bar{\mathcal{L}}^{(i)})=\mathit{S}_{LC}(y)=\left\{\begin{array}{l}
		l^{\star},\: l\leqslant \nu\vartheta,
		\\y~,\:\,\: l> \nu\vartheta.
	\end{array}
	\right.\nonumber
\end{eqnarray}
\end{proof}

\section{TJLC model and solving algorithm}
LRTC problem aims at estimating the missing elements from an incomplete observation tensor. First, let's represent the rank in problem (\ref{lLRTC}) in terms of the tensor joint rank.
\begin{eqnarray}
	\min_{\mathcal{X}}\,{\tt rank}_{TJ}(\mathcal{X})~{\tt s.t.}~\mathcal{P}_{\Omega}(\mathcal{T})=\mathcal{P}_{\Omega}(\mathcal{X}),
\end{eqnarray}
where $\mathcal{T}\in\mathbb{R}^{\mathit{I}_{1}\times\mathit{I}_{2}\times\cdots\times\mathit{I}_{N}}$ is the $N$-th order observed tensor; $\mathcal{X}$ is the complete tensor; $\mathcal{P}_{\Omega}(\mathcal{X})$ is a projection operator that keeps the entries of $\mathcal{X}$ in $\Omega$ and sets all others to zero. 
Subsequently, we apply the tensor logarithmic composite norm to each rank component of the tensor joint rank. The proposed TJLC model is formulated as follow
\begin{eqnarray}
	\min_{\mathcal{X}}\sum_{1\leqslant l_1\leqslant l_2\leqslant N}\beta_{l_1l_2}\|\mathcal{X}_{(l_1l_2)}\|_{LC}~ {\tt s.t.}~ \mathcal{P}_{\Omega}(\mathcal{T})=\mathcal{P}_{\Omega}(\mathcal{X}),\label{TJLC}
\end{eqnarray}
where parameter $\beta_{l_1l_2}$ further balances the weight of different rank components, and $\sum_{1\leqslant l_1\leqslant l_2\leqslant N}\beta_{l_1l_2}=1$. Then, we introduce auxiliary tensor set $\{\mathcal{Z}_{l_1l_2}=\mathcal{X}\}_{1\leqslant l_1\leqslant l_2\leqslant N}$, i.e., $ \{\mathcal{Z}_{l_1l_2(l_1l_2)}=\mathcal{X}_{(l_1l_2)}\}_{1\leqslant l_1\leqslant l_2\leqslant N}$, and consider the augmented Lagrangian function of (\ref{TJLC})
\begin{eqnarray}
&&L(\mathcal{X},\mathcal{Z}_{l_1l_2},\mathcal{Q}_{l_1l_2};\mu_{l_1l_2})
=\sum_{1\leqslant l_1\leqslant l_2\leqslant N} ( \beta_{l_1l_2}\|\mathcal{Z}_{l_1l_2(l_1l_2)}\|_{LC}\nonumber
\\&&\qquad\qquad\qquad\qquad\qquad\quad+\frac{\mu_{l_1l_2}}{2}\|\mathcal{X}-\mathcal{Z}_{l_1l_2}+\mathcal{Q}_{l_1l_2}/\mu_{l_1l_2}\|_{F}^{2}) ~ {\tt s.t.}~ \mathcal{P}_{\Omega}(\mathcal{T}-\mathcal{X})=\mathbf{0},\label{TJLC2}
\end{eqnarray}
where $\{\mathcal{Q}_{l_1l_2}\}_{1\leqslant l_1\leqslant l_2\leqslant N}$ are Lagrangian multipliers; $\{\mu_{l_1l_2}>0\}_{1\leqslant l_1\leqslant l_2\leqslant N}$ are the augmented Lagrangian parameters. We omit the specific number of iterations, and denote the variable updated by the iteration as $(\cdot)^{+}$. The update equations are derived in the following. 

\textbf{Update $\{\mathcal{Z}_{l_1l_2}\}_{1\leqslant l_1\leqslant l_2\leqslant N}$}: In a single iteration, the update of each variable $\mathcal{Z}_{l_1l_2}$ is independent, so we can solve the subproblem of each $\mathcal{Z}_{l_1l_2}$ individually. Fix other variables, and the corresponding optimization is as follows: 
\begin{eqnarray}
	&&\mathcal{Z}_{l_1l_2(l_1l_2)}^{+}=\arg\min_{\mathcal{Z}_{l_1l_2(l_1l_2)}}\beta_{l_1l_2}\|\mathcal{Z}_{l_1l_2(l_1l_2)}\|_{LC}
	+\frac{\mu_{l_1l_2}}{2}\|\mathcal{X}-\mathcal{Z}_{l_1l_2}+\mathcal{Q}_{l_1l_2}/\mu_{l_1l_2}\|_{F}^{2}\nonumber
	\\&&~~=\arg\min_{\mathcal{Z}_{l_1l_2(l_1l_2)}}\beta_{l_1l_2}\|\mathcal{Z}_{l_1l_2(l_1l_2)}\|_{LC}
	+\frac{\mu_{l_1l_2}}{2}\|\mathcal{X}_{(l_1l_2)}-\mathcal{Z}_{l_1l_2(l_1l_2)}+\mathcal{Q}_{l_1l_2(l_1l_2)}/\mu_{l_1l_2}\|_{F}^{2},
\end{eqnarray}
where the second equality is due to the fact that the unfolding operation only rearranges the elements without changing the value of $\|\cdot\|_{F}^2$. Recalling Theorem \ref{thPTLC}, the solution to the above optimization is given by:
\begin{eqnarray}
	\mathcal{Z}_{l_1l_2(l_1l_2)}^{+}=\mathit{S}(\mathcal{X}_{(l_1l_2)}+\mathcal{Q}_{l_1l_2(l_1l_2)}/\mu_{l_1l_2}),\label{UPZ2}
\end{eqnarray}
where $\mathit{S}$ denotes the proximal operator defined in (\ref{opewtgn}).
\begin{algorithm}[!h]
	\caption{TJLC} 
	\hspace*{0.02in} {\bf Input:} 
	An incomplete tensor $\mathcal{T}$, the index set of the known elements $\Omega$, convergence criteria $\epsilon=10^{-4}$, maximum iteration number $K$. \\
	\hspace*{0.02in} {\bf Initialization:} 
	$\mathcal{X}^{0}=\mathcal{T}_{\Omega}$, $\{\mathcal{Z}_{l_1l_2}^{0}=\mathcal{X}^{0},\mu_{l_1l_2}^{0}>0\}_{1\leqslant l_1\leqslant l_2\leqslant N}$, $\eta>1$.
	\begin{algorithmic}
		\While{not converged and $k<K$} 
		\State Updating $\mathcal{Z}_{l_1l_2}^{k}$ via (\ref{UPZ2});
		\State Updating $\mathcal{X}^{k}$ via (\ref{UPX2});	
		\State Updating the multiplier $\mathcal{Q}_{l_1l_2}^{k}$ via (\ref{UPQ2});
		\State $\mu_{l_1l_2}^{k}=\eta\mu_{l_1l_2}^{k-1}$, $k=k+1$;
		\State Check the convergence conditions $\|\mathcal{X}^{k+1}-\mathcal{X}^{k}\|^{2}_{F}/\|\mathcal{X}^{k}\|^{2}_{F}\leq\epsilon$.
		\EndWhile
		\State \Return $\mathcal{X}^{k+1}$.
	\end{algorithmic}
	\hspace*{0.02in} {\bf Output:} 
	$\mathcal{X}$.
	\label{TC}\end{algorithm}

\textbf{Update $\mathcal{X}$}: The closed form of $\mathcal{X}$ can be derived by setting the derivative of (\ref{TJLC2}) to zero. The corresponding optimization is as follows: 
\begin{eqnarray}
	&&\mathcal{X}^{+}=\arg\min_{\mathcal{X}}\sum_{1\leqslant l_1\leqslant l_2\leqslant N}(\frac{\mu_{l_1l_2}}{2}\|\mathcal{X}-\mathcal{Z}_{l_1l_2}+\mathcal{Q}_{l_1l_2}/\mu_{l_1l_2}\|_{F}^{2}) ~ {\tt s.t.}~ \mathcal{P}_{\Omega}(\mathcal{T}-\mathcal{X})=\mathbf{0},
\end{eqnarray}
Then, we can now update $\mathcal{X}$ by the following equation:
\begin{eqnarray}
	\mathcal{X}^{+}=\mathcal{P}_{\Omega^{\perp}}(\frac{\sum_{1\leqslant l_1\leqslant l_2\leqslant N}\mu_{l_1l_2}(\mathcal{Z}_{l_1l_2}-\mathcal{Q}_{l_1l_2}/\mu_{l_1l_2})}{\sum_{1\leqslant l_1\leqslant l_2\leqslant N}\mu_{l_1l_2}})+\mathcal{P}_{\Omega}(\mathcal{T}).\label{UPX2}
\end{eqnarray}
where the vertical symbol $^{\perp}$ in the upper right corner of $\Omega$ indicates the complement set $\Omega^{\perp}$ of $\Omega$.

\textbf{Update $\{\mathcal{Q}_{l_1l_2}\}_{1\leqslant l_1\leqslant l_2\leqslant N}$}: Finally, multiplier $\mathcal{Q}_{l_1l_2}$ is updated as follows:
\begin{eqnarray}
\mathcal{Q}_{l_1l_2}^{+}=\mathcal{Q}_{l_1l_2}+\mu_{l_1l_2}(\mathcal{X}^{+}-\mathcal{Z}_{l_1l_2}^{+}).\label{UPQ2}
\end{eqnarray}
The TJLC model computation is given in Algorithm \ref{TC}. The main per-iteration cost lies in the update of $\mathcal{Z}_{l_1l_2}$, which requires computing t-SVD . The per-iteration complexity is $O(LE(\sum_{1\leqslant l_{1}< l_{2}\leqslant N}[log(le_{l_{1}l_{2}})+\min(\mathit{I}_{l_{1}},\mathit{I}_{l_{2}})]+\sum_{1\leqslant l_{1} \leqslant N}\mathit{I}_{l_{1}}))$, where $LE=\prod_{i=1}^{N}\mathit{I}_{i}$ and $le_{l_{1}l_{2}}=LE/(\mathit{I}_{l_{1}}\mathit{I}_{l_{2}})$.

\section{Convergence analysis}
To prove the theoretical convergence analysis of the proposed algorithm \ref{TC}, we first have the following lemma.
\begin{lemma}\label{le2}
Sequences $\{\mathcal{X}^{k}\},\{\mathcal{Z}_{l_1l_2}^{k}\}_{1\leqslant l_1\leqslant l_2\leqslant N}$ are bounded if the sequence $\{\mathcal{Q}_{l_1l_2}^{k}\}_{1\leqslant l_1\leqslant l_2\leqslant N}$ are bounded and $\{\sum_{j=1}^{\infty}(\mu_{l_1l_2}^{j}+\mu_{l_1l_2}^{j-1})/(\mu_{l_1l_2}^{j-1})^2<\infty\}_{1\leqslant l_1\leqslant l_2\leqslant N}$.
\end{lemma}
\begin{proof}
	By simple manipulation, we can get,
	\begin{eqnarray}
		&&L(\mathcal{X}^{k},\mathcal{Z}_{l_1l_2}^{k},\mathcal{Q}_{l_1l_2}^{k};\mu_{l_1l_2}^{k})\nonumber
		\\&&=L(\mathcal{X}^{k},\mathcal{Z}_{l_1l_2}^{k},\mathcal{Q}_{l_1l_2}^{k-1};\mu_{l_1l_2}^{k-1})+\sum_{1\leqslant l_{1}< l_{2}\leqslant N}(\mu_{l_1l_2}^{k}+\mu_{l_1l_2}^{k-1})/(2(\mu_{l_1l_2}^{k-1})^2)\|\mathcal{Q}_{l_1l_2}^{k}-\mathcal{Q}_{l_1l_2}^{k-1}\|_{F}^{2}\nonumber.
	\end{eqnarray}
Then, it follows that,
\begin{eqnarray}
	&&L(\mathcal{X}^{k+1},\mathcal{Z}_{l_1l_2}^{k+1},\mathcal{Q}_{l_1l_2}^{k};\mu_{l_1l_2}^{k})\nonumber
	\\&&~~\leqslant L(\mathcal{X}^{k},\mathcal{Z}_{l_1l_2}^{k},\mathcal{Q}_{l_1l_2}^{k};\mu_{l_1l_2}^{k})\nonumber
	\\&&~~\leqslant L(\mathcal{X}^{k},\mathcal{Z}_{l_1l_2}^{k},\mathcal{Q}_{l_1l_2}^{k-1};\mu_{l_1l_2}^{k-1})+\sum_{1\leqslant l_{1}< l_{2}\leqslant N}(\mu_{l_1l_2}^{k}+\mu_{l_1l_2}^{k-1})/(2(\mu_{l_1l_2}^{k-1})^2)\|\mathcal{Q}_{l_1l_2}^{k}-\mathcal{Q}_{l_1l_2}^{k-1}\|_{F}^{2}\nonumber
	\\&&~~\leqslant L(\mathcal{X}^{1},\mathcal{Z}_{l_1l_2}^{1},\mathcal{Q}_{l_1l_2}^{0};\mu_{l_1l_2}^{0})+\sum_{j=1}^{k}\sum_{1\leqslant l_{1}< l_{2}\leqslant N}(\mu_{l_1l_2}^{j}+\mu_{l_1l_2}^{j-1})/(2(\mu_{l_1l_2}^{j-1})^2)\|\mathcal{Q}_{l_1l_2}^{j}-\mathcal{Q}_{l_1l_2}^{j-1}\|_{F}^{2}\nonumber.	
\end{eqnarray}
By the bounded property of $\|\mathcal{Q}_{l_1l_2}^{j}-\mathcal{Q}_{l_1l_2}^{j-1}\|_{F}^{2}$, as well as under the given condition on $\{\mu_{l_1l_2}^{k}\}_{1\leqslant l_{1}< l_{2}\leqslant N}$, the right hand side of the inequality is bounded, so $L(\mathcal{X}^{k+1},\mathcal{Z}_{l_1l_2}^{k+1},\mathcal{Q}_{l_1l_2}^{k};\mu_{l_1l_2}^{k})$ is bounded.
\begin{eqnarray}
L(\mathcal{X}^{k+1},\mathcal{Z}_{l_1l_2}^{k+1},\mathcal{Q}_{l_1l_2}^{k};\mu_{l_1l_2}^{k})=\sum_{1\leqslant l_1\leqslant l_2\leqslant N} ( \beta_{l_1l_2}\|\mathcal{Z}^{k+1}_{l_1l_2(l_1l_2)}\|_{LC}+\frac{\mu_{l_1l_2}}{2}\|\mathcal{X}^{k+1}-\mathcal{Z}^{k+1}_{l_1l_2}+\mathcal{Q}^{k}_{l_1l_2}/\mu^{k}_{l_1l_2}\|_{F}^{2}).\nonumber
\end{eqnarray}
The terms on the right side of the equation are nonnegative and the terms on the left side of the equation are bounded, so $\{\mathcal{Z}_{l_1l_2}^{k}\}_{1\leqslant l_{1}< l_{2}\leqslant N}$ are bounded. By observing the last regular term
on the right side of the equation, $\mathcal{X}^{k}$ is bounded. Therefore, $\{\mathcal{X}^{k}\}$, and $\{\mathcal{Z}_{l_1l_2}^{k}\}_{1\leqslant l_{1}< l_{2}\leqslant N}$ are all bounded. The proof is completed.
\end{proof}
\begin{theorem}
	The sequence $L(\mathcal{X}^{k},\mathcal{Z}_{l_1l_2}^{k},\mathcal{Q}_{l_1l_2}^{k})$ generated by algorithm \ref{TC} has at least one accumulation point $L(\mathcal{X}^{\star},\mathcal{Z}_{l_1l_2}^{\star},,\mathcal{Q}_{l_1l_2}^{\star})$. Then, $L(\mathcal{X}^{\star},\mathcal{Z}_{l_1l_2}^{\star})$ is a stationary point of optimization problem (\ref{TJLC}) under the condition that $\{\lim\limits_{k\to\infty}\mu_{l_1l_2}^{k}(\mathcal{Z}_{l_1l_2}^{k+1}-\mathcal{Z}_{l_1l_2}^{k})=0\}_{1\leqslant l_{1}< l_{2}\leqslant N}$, the sequence $\{\mathcal{Q}_{l_1l_2}^{k}\}_{1\leqslant l_1\leqslant l_2\leqslant N}$ are bounded and $\{\sum_{j=1}^{\infty}(\mu_{l_1l_2}^{j}+\mu_{l_1l_2}^{j-1})/(\mu_{l_1l_2}^{j-1})^2<\infty\}_{1\leqslant l_1\leqslant l_2\leqslant N}$.
\end{theorem}
\begin{proof}
The sequence $L(\mathcal{X}^{k},\mathcal{Z}_{l_1l_2}^{k},\mathcal{Q}_{l_1l_2}^{k})$ generated by algorithm \ref{TC} is bounded as proven in lemma \ref{le2}. By the Bolzano–Weierstrass theorem, the sequence
has at least one accumulation point $(\mathcal{X}^{\star},\mathcal{Z}_{l_1l_2}^{\star},\mathcal{Q}_{l_1l_2}^{\star})$. Without loss of generality, we can assume that the sequence $(\mathcal{X}^{k},\mathcal{Z}_{l_1l_2}^{k},\mathcal{Q}_{l_1l_2}^{k})$ converges to $(\mathcal{X}^{\star},\mathcal{Z}_{l_1l_2}^{\star},\mathcal{Q}_{l_1l_2}^{\star})$. Actually, as $\sum_{j=1}^{\infty}1/(\mu_{l_1l_2}^{j})<\sum_{j=1}^{\infty}(\mu_{l_1l_2}^{j}+\mu_{l_1l_2}^{j-1})/(\mu_{l_1l_2}^{j-1})^2<\infty$ also satisfies. It follows the date rule of $\mathcal{Q}_{l_1l_2}^{k}$ that $\lim\limits_{k\to\infty}\mathcal{X}^{k}-\mathcal{Z}_{l_1l_2}^{k}=\lim\limits_{k\to\infty}(\mathcal{Q}_{l_1l_2}^{k}-\mathcal{Q}_{l_1l_2}^{k-1})/\mu^{k-1}=0$, that is, $\mathcal{X}^{\star}=\mathcal{Z}_{l_1l_2}^{\star}$. Therefore, the feasible condition is satisfied.

First, we list the Karush–Kuhn–Tucker (KKT) conditions the $L(\mathcal{X}^{\star},\mathcal{Z}_{l_1l_2}^{\star},\mathcal{Q}_{l_1l_2}^{\star})$ satisfies,
\begin{eqnarray*}
	\left\{ \begin{array}{l}
		0\in\dfrac{\partial(\|\mathcal{Z}_{l_1l_2}\|_{LC})}{\partial\mathcal{Z}_{l_1l_2}}\lvert_{\mathcal{Z}_{l_1l_2}^{\star}}-\mathcal{Q}_{l_1l_2}^{\star},
		\\\sum_{1\leqslant l_{1}< l_{2}\leqslant N}\mathcal{Q}_{l_1l_2}^{\star}=0.
	\end{array}
	\right.
\end{eqnarray*}

For $\mathcal{Z}_{l_1l_2}^{k+1}$, we have 
\begin{eqnarray}
	&&\dfrac{\partial L(\mathcal{X}^{k},\mathcal{Z}_{l_1l_2},\mathcal{Q}_{l_1l_2}^{k})}{\partial\mathcal{Z}_{l_1l_2}}\lvert_{\mathcal{Z}_{l_1l_2}^{k+1}}=\dfrac{\partial(\|\mathcal{Z}_{l_1l_2}\|_{LC})}{\partial\mathcal{Z}_{l_1l_2}}\lvert_{\mathcal{Z}_{l_1l_2}^{k+1}}-\mathcal{Q}_{l_1l_2}^{k}-\mu_{l_1l_2}^{k}(\mathcal{X}^{k+1}-\mathcal{Z}_{l_1l_2}^{k})\nonumber
	\\&&\qquad\qquad\qquad\,~~\qquad\qquad=\dfrac{\partial(\|\mathcal{Z}_{l_1l_2}\|_{LC})}{\partial\mathcal{Z}_{l_1l_2}}\lvert_{\mathcal{Z}_{l_1l_2}^{k+1}}-\mathcal{Q}_{l_1l_2}^{k+1}-\mu_{l_1l_2}^{k}(\mathcal{Z}_{l_1l_2}^{k+1}-\mathcal{Z}_{l_1l_2}^{k}).\nonumber
\end{eqnarray}
In terms of provided condition $\lim\limits_{k\to\infty}\mu_{l_1l_2}^{k}(\mathcal{Z}_{l_1l_2}^{k+1}-\mathcal{Z}_{l_1l_2}^{k})=0$, and the bounded
properties of $\mathcal{Z}_{l_1l_2}^{k+1},\mathcal{Q}_{l_1l_2}^{k+1}$, and $\dfrac{\partial(\|\mathcal{Z}_{l_1l_2}\|_{LC})}{\partial\mathcal{Z}_{l_1l_2}}\lvert_{\mathcal{Z}_{l_1l_2}^{\star}}$, and the formula of $\partial(\|\mathcal{Z}_{l_1l_2}\|_{LC})$, we can obtain that $0\in\dfrac{\partial(\|\mathcal{Z}_{l_1l_2}\|_{LC})}{\partial\mathcal{Z}_{l_1l_2}}\lvert_{\mathcal{Z}_{l_1l_2}^{\star}}-\mathcal{Q}_{l_1l_2}^{\star}$.

Similarly, for $\mathcal{X}^{k+1}$, it is noted that,
\begin{eqnarray*}
	&&\nabla_{\mathcal{X}}L(\mathcal{X},\mathcal{Z}_{l_1l_2}^{k},\mathcal{Q}_{l_1l_2}^{k})\lvert_{\mathcal{X}^{k+1}}
	\\&&~~=\sum_{1\leqslant l_{1}< l_{2}\leqslant N}\mathcal{Q}_{l_1l_2}^{k}+\mu_{l_1l_2}^{k}(\mathcal{X}^{k+1}-\mathcal{Z}_{l_1l_2}^{k+1})
	\\&&~~=\sum_{1\leqslant l_{1}< l_{2}\leqslant N}\mathcal{Q}_{l_1l_2}^{k+1}.
\end{eqnarray*}
By the bounded condition of $\{\mathcal{Q}_{l_1l_2}^{k}\}_{1\leqslant l_{1}< l_{2}\leqslant N}$, we can obtain that $\sum_{1\leqslant l_{1}< l_{2}\leqslant N}\mathcal{Q}_{l_1l_2}^{\star}=0$.

Through the above-mentioned, $L(\mathcal{X}^{\star},\mathcal{Z}_{l_1l_2}^{\star},\mathcal{Q}_{l_1l_2}^{\star})$ satisfies the KKT conditions of problem (\ref{TJLC}). Therefore, $L(\mathcal{X}^{\star},\mathcal{Z}_{l_1l_2}^{\star})$ is a stationary point of optimization problem (\ref{TJLC}). This completes our proof.
\end{proof}

\section{Experiments}
In this section, we evaluate the performance of the proposed TJLC method. We use three kinds real-world of tensor data: magnetic resonance image (MRI), multispectral image (MSI), and color video (CV), with different scenes and environments. 
\subsection{Experimental setup}
We employ four picture quality indices (PQIs) to measure the quality of the recovered results, i.e., the peak signal-to-noise rate (PSNR) value, the structural similarity (SSIM) value \cite{1284395}, the feature similarity (FSIM) value \cite{5705575}, and erreur relative globale adimensionnelle de synth$\grave{e}$se (ERGAS) value \cite{2432002352}. The ERGAS value is the smaller the better, and the PSNR, SSIM, and FSIM values are the bigger the better. The optimal results will be highlighted in bold. All tests are implemented on the Windows 11 platform and MATLAB (R2019a) with an 13th Gen Intel Core i5-13600K 3.50 GHz and 32 GB of RAM.

Suppose $I^{c}$ and $I^{r}$ stand for the complete and reference images, respectively; $I_{1}$ and $I_{2}$ represent the spatial size of the image.
\\(a) PSNR is formulated as:
	\begin{eqnarray*}
	\text{PSNR}:=10\log_{10}(\frac{255^{2}\times I_{1}I_{2}}{\|I^{r}-I^{c}\|_F^{2}})
		\end{eqnarray*}
(b) SSIM is defined as:
	\begin{eqnarray*}
	\text{SSIM}:=\dfrac{(2\mu_{c}\mu_{r}+C_{1})(2\sigma_{cr}+C_{2})}{(\mu^{2}_{c}+\mu^{2}_{r}+C_{1})(\sigma_{c}^{2}+\sigma_{r}^{2}+C_{2})},
		\end{eqnarray*}
where $\mu_{c},\mu_{r}$ are average of $I^{c},I^{r}$ respectively; $\sigma_{c},\sigma_{r}$ are the variance of $I^{c},I^{r}$ respectively; $\sigma_{cr}$ is the covariance of $I^{c}$ and $I^{r}$. 
\\(c) FSIM is set as:
	\begin{eqnarray*}
	\text{FSIM}:=(\sum_{x\in\Omega} S_{L}(x)\cdot PC_{m}(x))/(\sum_{x\in\Omega} PC_{m}(x)),
		\end{eqnarray*}
where $S_{L}(x)$ is derived from the phase congruency and the image gradient magnitude of
$I^{c}$ and $I^{r}$ ; $PC_{m}(x)$ is the maximum phase congruency of $PC_{c}$ (for $I^{c}$) and $PC_{r}$ (for
$I^{r}$ ); and $\Omega$ represents the entire airspace of the image. 

When calculating the PSNR, SSIM and FSIM values of the tensor image, the slices of the tensor are calculated first, and the average value is taken at last.
\\(d) For $\mathcal{X},\mathcal{Y}\in \mathbb{R}^{\mathit{I}_{1}\times\mathit{I}_{2}\times\mathit{I}_{3}}$, ERGAS is defined as:
\begin{eqnarray*}
	\text{ERGAS}:=100\sqrt{\dfrac{\sum_{i=1}^{\mathit{I}_{3}}({\tt mse}(\mathcal{Y}(:,:,i)-\mathcal{X}(:,:,i))/({\tt mean2}(\mathcal{Y}(:,:,i))))}{\mathit{I}_{3}}},\nonumber
\end{eqnarray*}
where ${\tt mse}$ is mean squared error performance function, and ${\tt mean2}$ is average of matrix elements.
\\(e) For a tensor $\mathcal{X}\in \mathbb{R}^{\mathit{I}_{1}\times\mathit{I}_{2}\times\cdots
	\times\mathit{I}_{N}}$, its missing rate (MR) is defined as $(1-{\tt num}(\Omega)/(\mathit{I}_{1}\times\mathit{I}_{2}\times\cdots
	\times\mathit{I}_{N}))\times100\%$, where $\Omega$ denotes an index set of observed data; ${\tt num}(\Omega)$ is the total number of elements in index set $\Omega$. 

The compared LRTC methods are as follows: HaLRTC \cite{6138863}, LRTCTV \cite{3120171}, representing state-of-the-art for the Tucker-decomposition-based method; and TNN \cite{7782758}, PSTNN \cite{1122020112680}, FTNN \cite{9115254}, WSTNN \cite{2020170}, BEMCP \cite{ZHANG2024110253} representing state-of-the-art for the t-SVD-based method. All comparison methods use the corresponding optimal parameters for the experiment

\subsection{MRI data}
We test the performance of the proposed TJLC method and the the compared method on MRI\footnote{http://brainweb.bic.mni.mcgill.ca/brainweb/selection\_normal.html} data with the size of $181\times217\times181$. Firstly, we demonstrate the visual effect recovered by MRI data at missing rates of 90\% and 95\% in Fig. \ref{MRIF}. The proposed TJLC method is clearly superior to the compared methods. From Fig. \ref{MRIF}, it can be seen that the TJLC method proposed in this study produces clearer texture recovery in images compared to the suboptimal BEMCP method. Then, we list the average quantitative results of frontal slice of MRI restored by all methods at different missing rates in Table \ref{MRIT}. At the missing rate of 95\%, the PSNR value of the proposed TJLC method is 1.322 dB higher than that of the suboptimal BEMCP method, and the values of SSIM, FSIM, and ERGAS are also better than the suboptimal BEMCP method. For missing rate of 90\%, the proposed TJLC method also achieves a PSNR value that is 1.064 dB higher than the suboptimal BEMCP method. Furthermore, the proposed TJLC method outperforms the suboptimal method in terms of other quantitative metrics.
\begin{table}[!h]
	\centering
	\caption{The PSNR, SSIM, FSIM and ERGAS values for MRI data tested by observed and the eight utilized LRTC methods.}
	\label{MRIT}
	{\footnotesize 
		\begin{tabular}{|c|cccc|cccc|}
			\hline
			MR       & \multicolumn{4}{c|}{95\%}                                            & \multicolumn{4}{c|}{90\%}                                           \\ \hline
			PQIs     & PSNR            & SSIM           & FSIM           & ERGAS            & PSNR            & SSIM           & FSIM           & ERGAS           \\ \hline
			Observed & 11.398          & 0.310          & 0.530          & 1021.184         & 11.633          & 0.323          & 0.565          & 994.004         \\
			HaLRTC   & 17.282          & 0.297          & 0.636          & 538.555          & 20.112          & 0.438          & 0.725          & 390.982         \\
			TNN      & 22.701          & 0.470          & 0.743          & 304.037          & 26.064          & 0.642          & 0.812          & 205.606         \\
			LRTCTV & 19.370          & 0.597          & 0.702          & 432.969          & 22.880          & 0.749          & 0.805          & 293.425         \\
			PSTNN    & 17.043          & 0.243          & 0.640          & 544.389          & 22.854          & 0.484          & 0.755          & 297.299         \\
			FTNN     & 24.792          & 0.689          & 0.835          & 234.158          & 28.330          & 0.826          & 0.895          & 151.795         \\
			WSTNN    & 25.534          & 0.708          & 0.825          & 211.914          & 29.063          & 0.837          & 0.887          & 139.068         \\
			BEMCP    & 29.772          & 0.835          & 0.884          & 126.076          & 33.383          & 0.913          & 0.928          & 83.053          \\
			TJLC     & \textbf{31.094} & \textbf{0.886} & \textbf{0.906} & \textbf{107.356} & \textbf{34.447} & \textbf{0.942} & \textbf{0.942} & \textbf{72.728} \\ \hline
		\end{tabular}}%
\end{table}
\begin{figure}[!h] 
	\centering  
	\vspace{0cm} 
\subfigtopskip=2pt 
\subfigbottomskip=2pt 
\subfigure[Original]{
	\begin{minipage}[b]{0.095\linewidth}
		\includegraphics[width=1\linewidth]{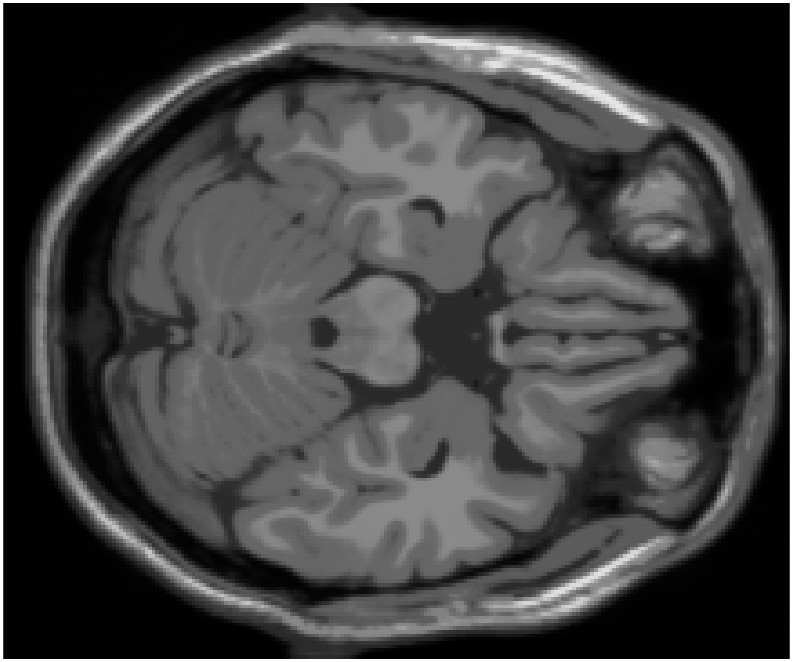}\\
		\includegraphics[width=1\linewidth]{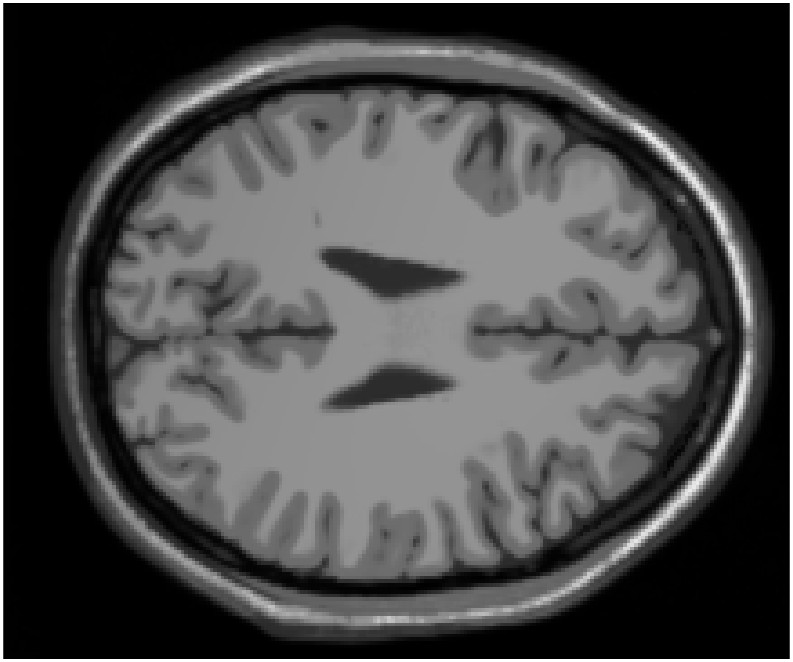}
\end{minipage}}\subfigure[Observed]{
	\begin{minipage}[b]{0.095\linewidth}
		\includegraphics[width=1\linewidth]{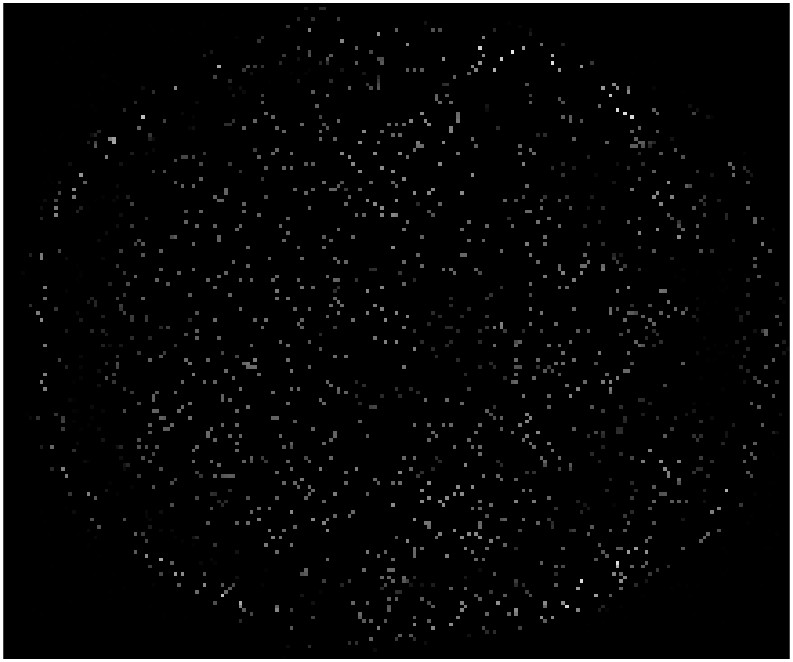}\\
		\includegraphics[width=1\linewidth]{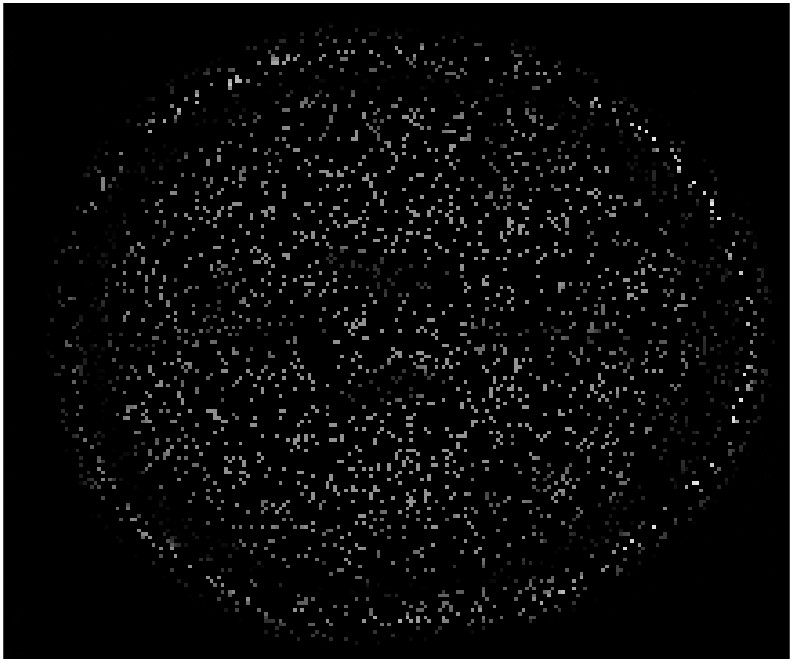}
\end{minipage}}\subfigure[HaLRTC]{
	\begin{minipage}[b]{0.095\linewidth}
		\includegraphics[width=1\linewidth]{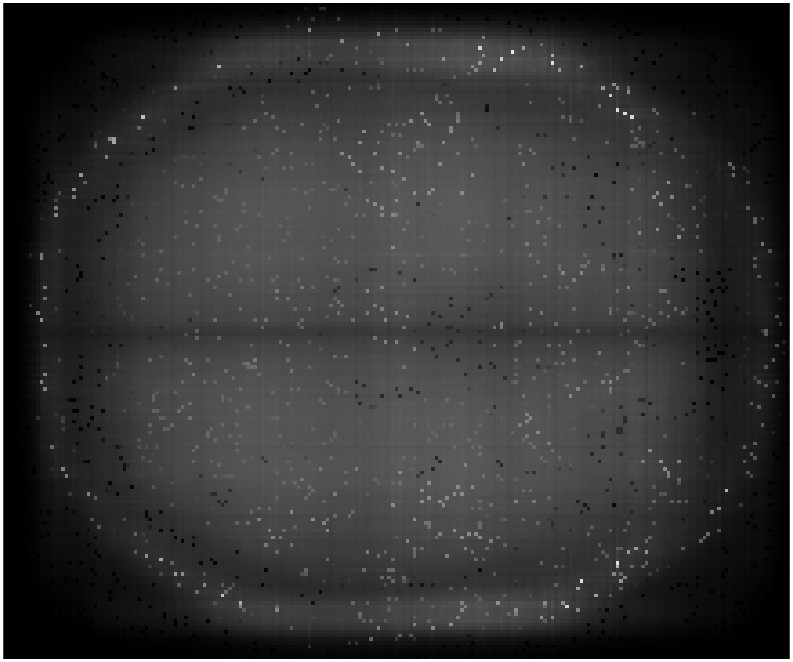}\\
		\includegraphics[width=1\linewidth]{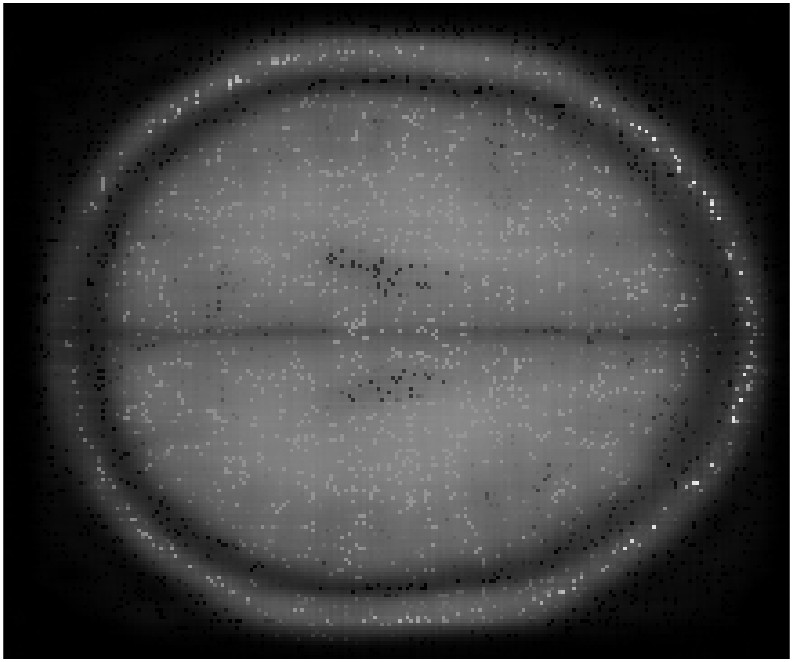}
\end{minipage}}\subfigure[TNN]{
	\begin{minipage}[b]{0.095\linewidth}
		\includegraphics[width=1\linewidth]{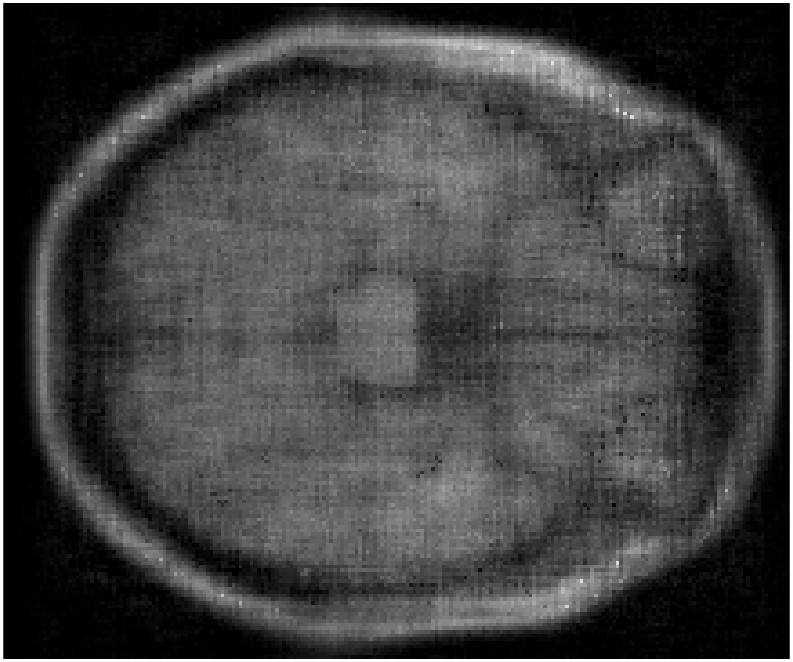}\\
		\includegraphics[width=1\linewidth]{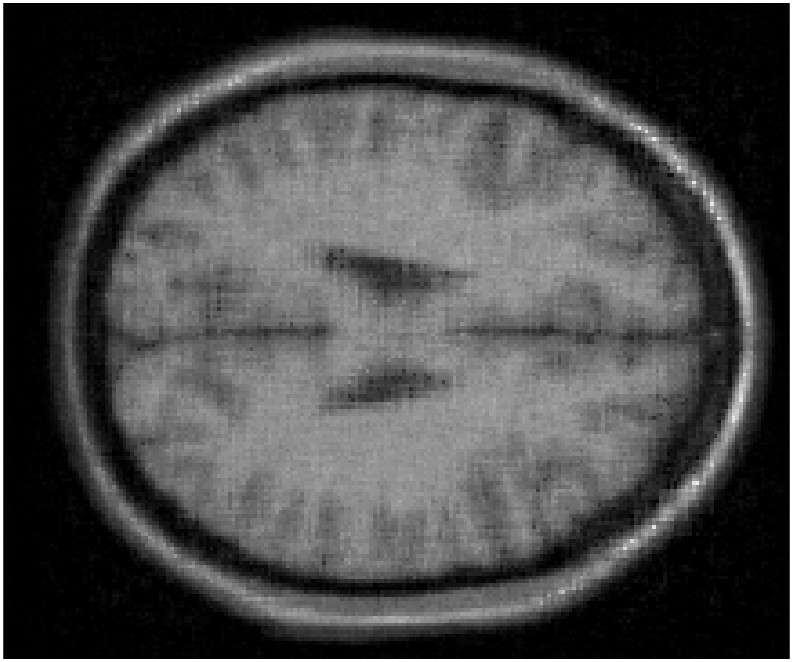}
\end{minipage}}\subfigure[LRTCTV]{
	\begin{minipage}[b]{0.095\linewidth}
		\includegraphics[width=1\linewidth]{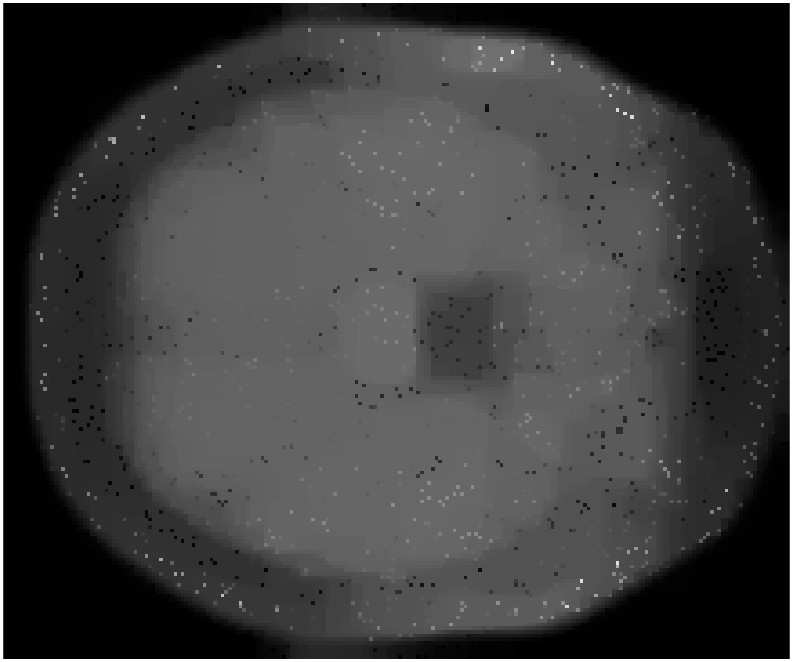}\\
		\includegraphics[width=1\linewidth]{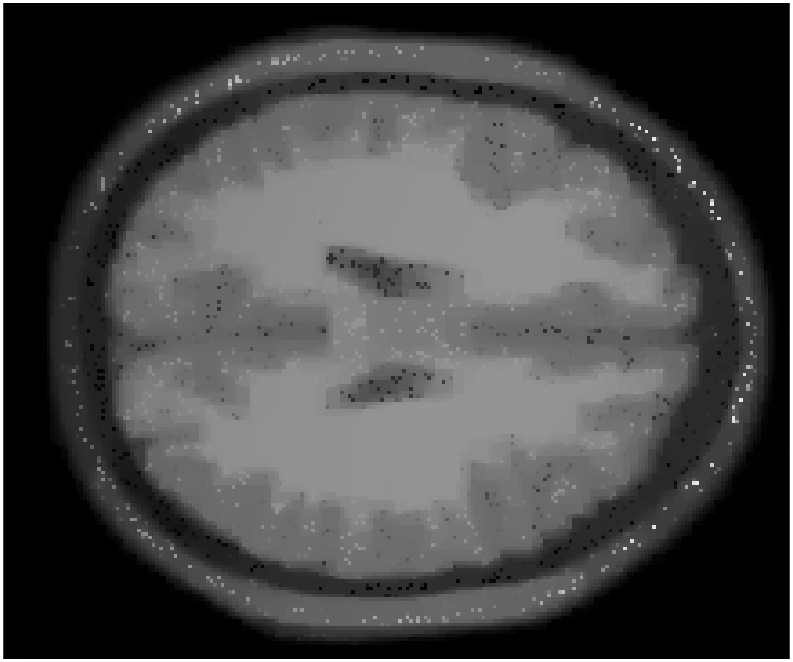}
\end{minipage}}\subfigure[PSTNN]{
	\begin{minipage}[b]{0.095\linewidth}
		\includegraphics[width=1\linewidth]{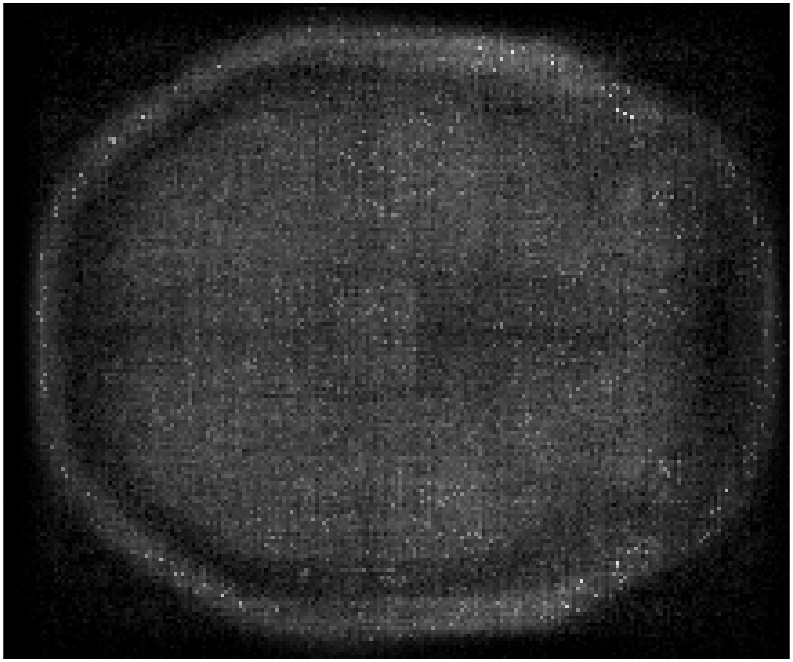}\\
		\includegraphics[width=1\linewidth]{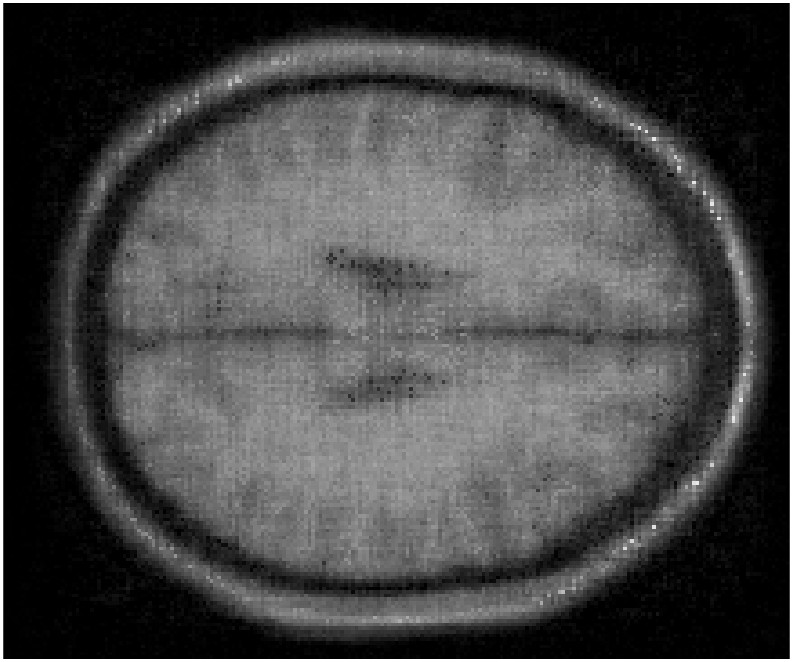}
\end{minipage}}\subfigure[FTNN]{
	\begin{minipage}[b]{0.095\linewidth}
		\includegraphics[width=1\linewidth]{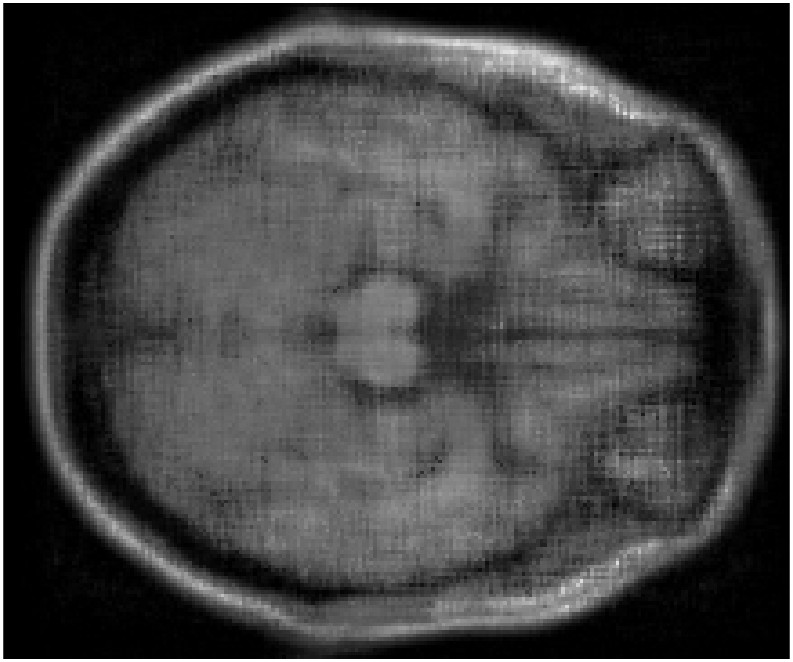}\\
		\includegraphics[width=1\linewidth]{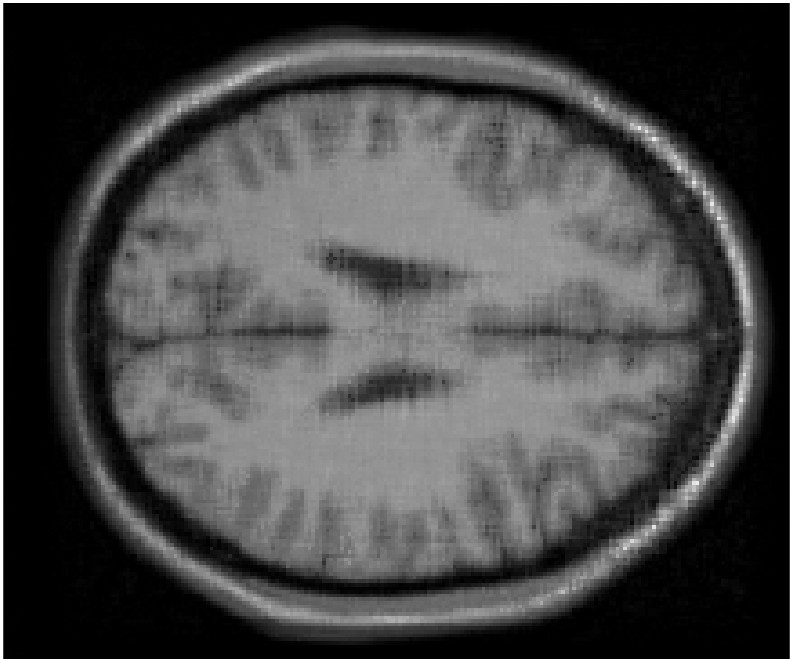}
\end{minipage}}\subfigure[WSTNN]{
	\begin{minipage}[b]{0.095\linewidth}
		\includegraphics[width=1\linewidth]{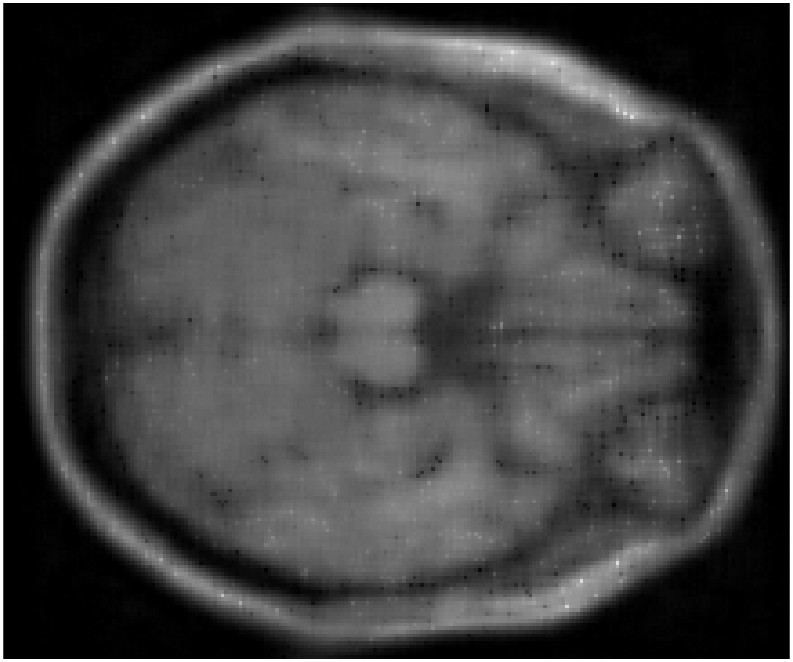}\\
		\includegraphics[width=1\linewidth]{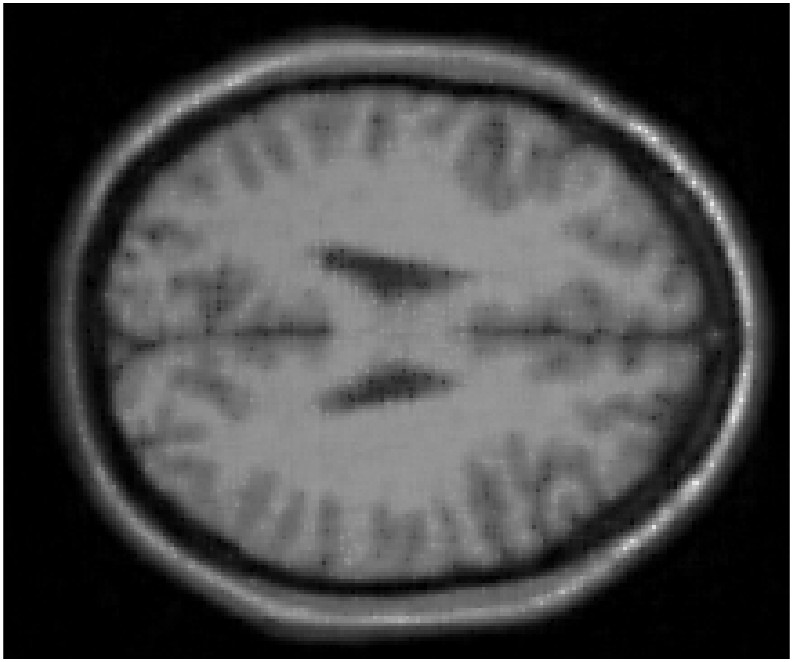}
\end{minipage}}\subfigure[BEMCP]{
	\begin{minipage}[b]{0.095\linewidth}
		\includegraphics[width=1\linewidth]{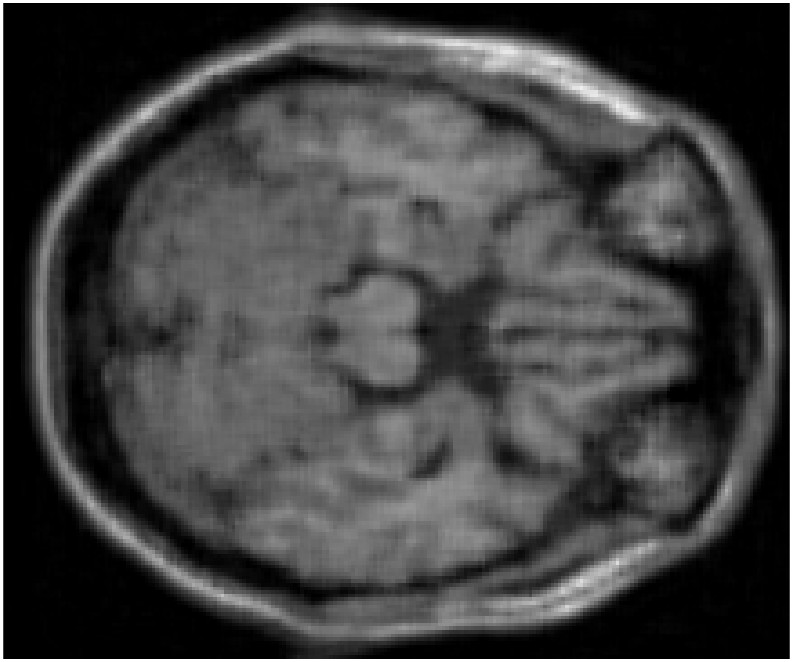}\\
		\includegraphics[width=1\linewidth]{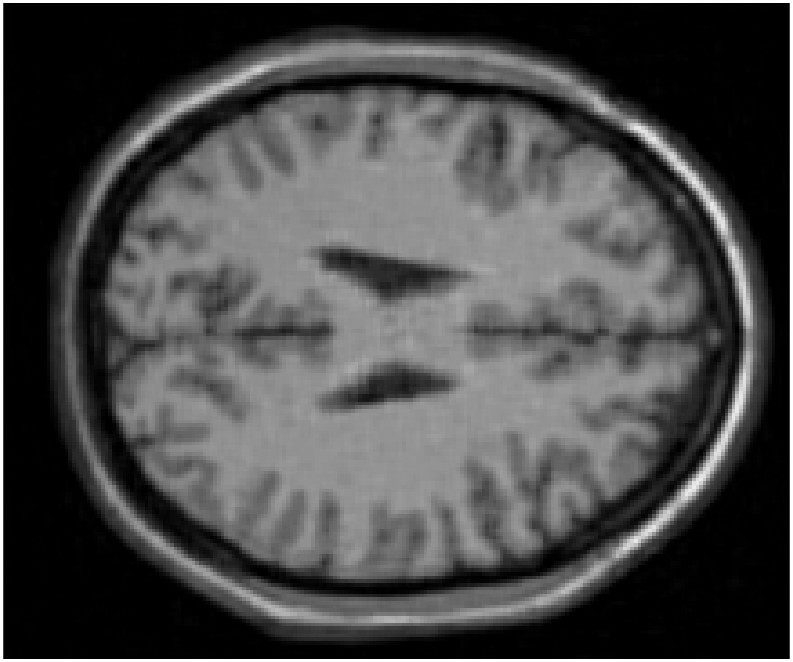}
\end{minipage}}\subfigure[TJLC]{
	\begin{minipage}[b]{0.095\linewidth}
		\includegraphics[width=1\linewidth]{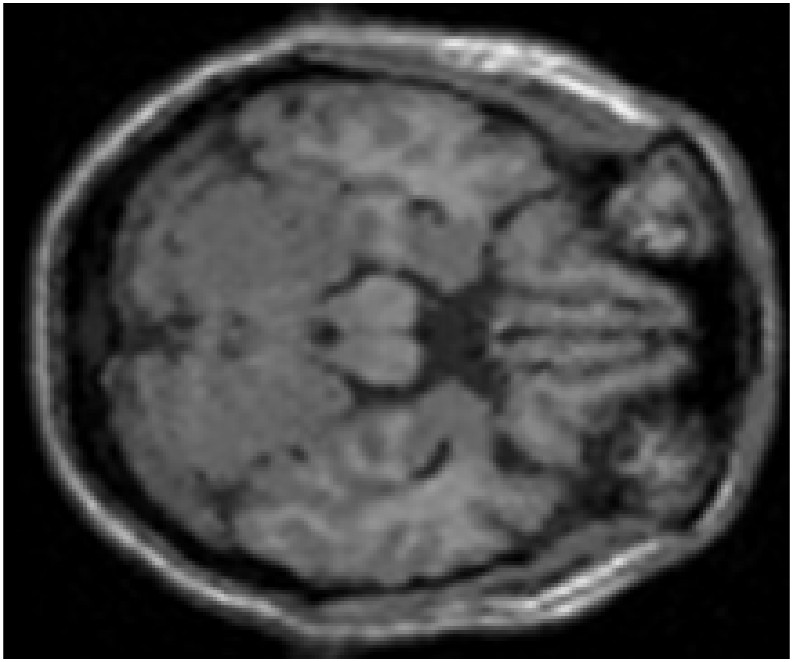}\\
		\includegraphics[width=1\linewidth]{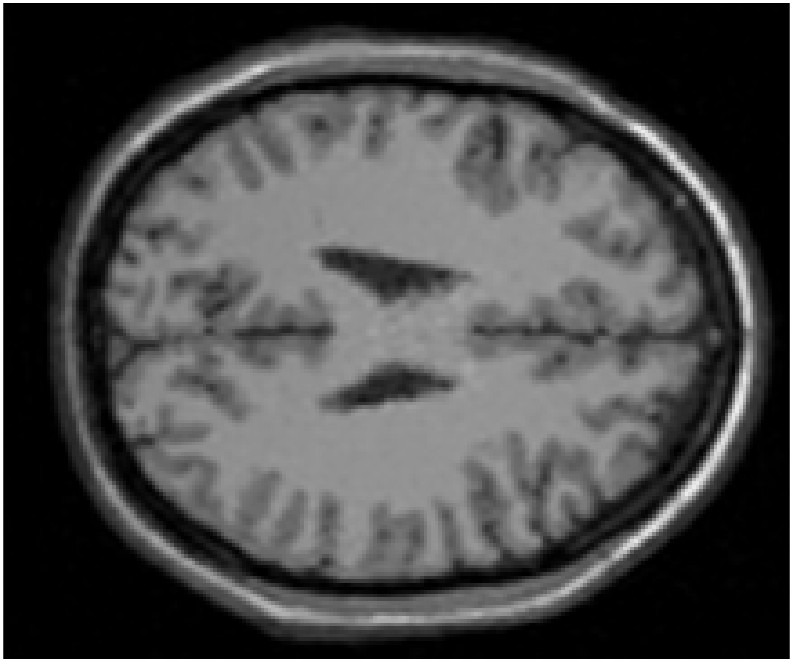}
\end{minipage}}
	\caption{Visual results for MRI data. MR: top row is 95\%, and last  row is 90\%. The corresponding slices in each row are: 50, 100.}
	\label{MRIF}
\end{figure}

\subsection{MSI data}
We test four MSIs in the dataset CAVE\footnote{http://www.cs.columbia.edu/CAVE/databases/multispectral/} (respectively named clay, chart\_and\_stuffed\_toy, balloons, cd). All testing data are of size $256\times256\times31$. Fig. \ref{MSIF} demonstrates the visual results for different missing rates and different spectral bands. From Fig. \ref{MSIF}, the visual effect of the proposed TJLC method is better than the compared method at all missing rates. Specifically, as shown in Fig. \ref{MSIF}, it can be observed that the face of the toy in the image of `` chart\_and\_stuffed\_toy " restored by the proposed TJLC method appears very smooth and almost identical to the original image. In contrast, the image restored by the suboptimal BEMCP method exhibit many small black dots and lack smoothness. To further highlight the superiority of the proposed TJLC method, the average quantitative results of four MSIs are listed in Table \ref{MSIT}. It can be seen that the proposed TJLC method has a great improvement compared to the suboptimal method. The PSNR value at 90\% missing rate is 3.212 dB higher than the suboptimal BEMCP method, and even reaches 3.398 dB at 95\% missing rate.
\begin{figure}[!h] 
	\centering  
	\vspace{0cm} 
\subfigtopskip=2pt 
\subfigbottomskip=2pt 
\subfigure[Original]{
	\begin{minipage}[b]{0.095\linewidth}
		\includegraphics[width=1\linewidth]{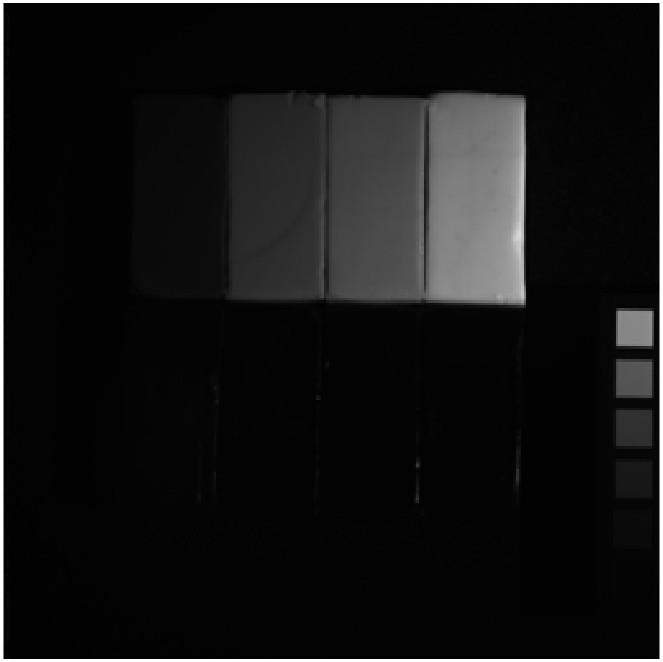}\\
		\includegraphics[width=1\linewidth]{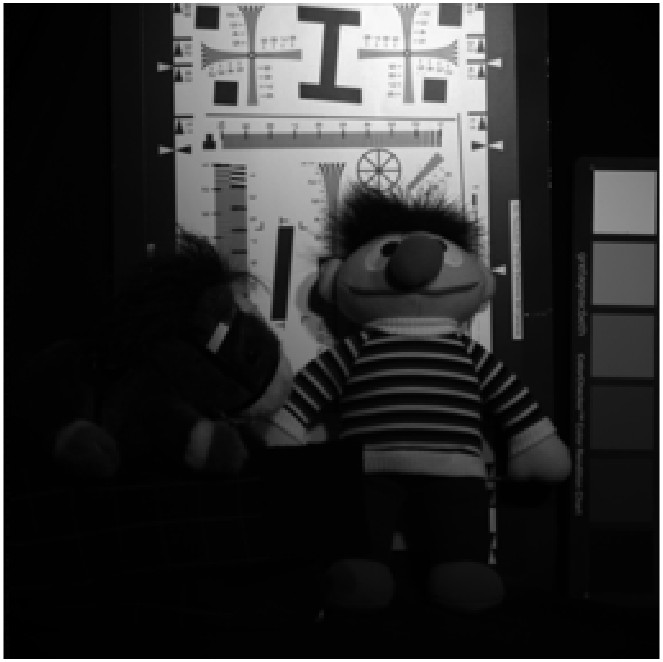}\\
		\includegraphics[width=1\linewidth]{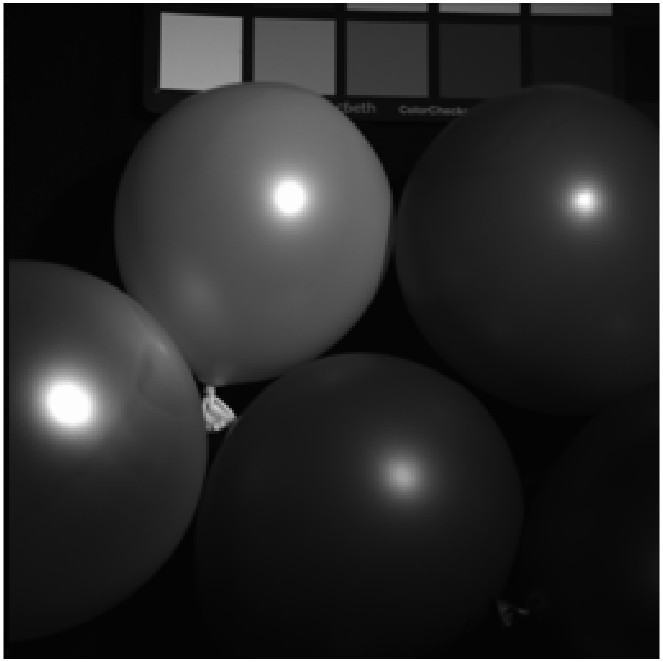}\\
		\includegraphics[width=1\linewidth]{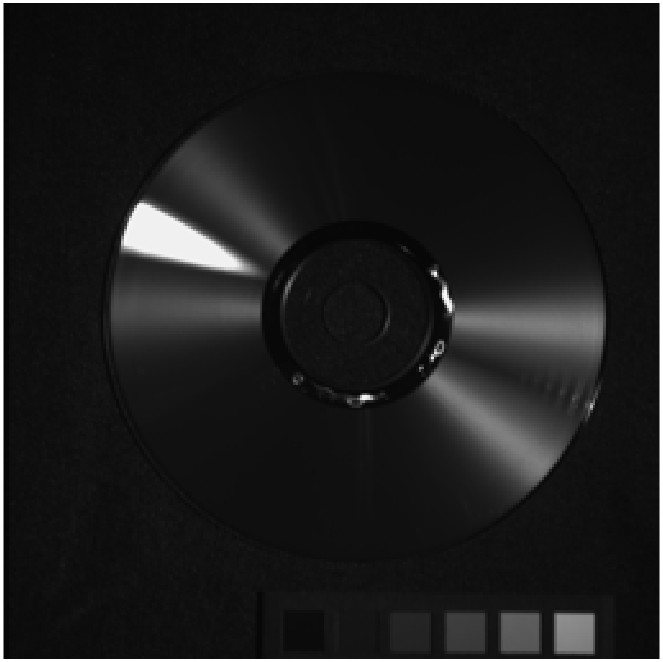}
\end{minipage}}\subfigure[Observed]{
	\begin{minipage}[b]{0.095\linewidth}
		\includegraphics[width=1\linewidth]{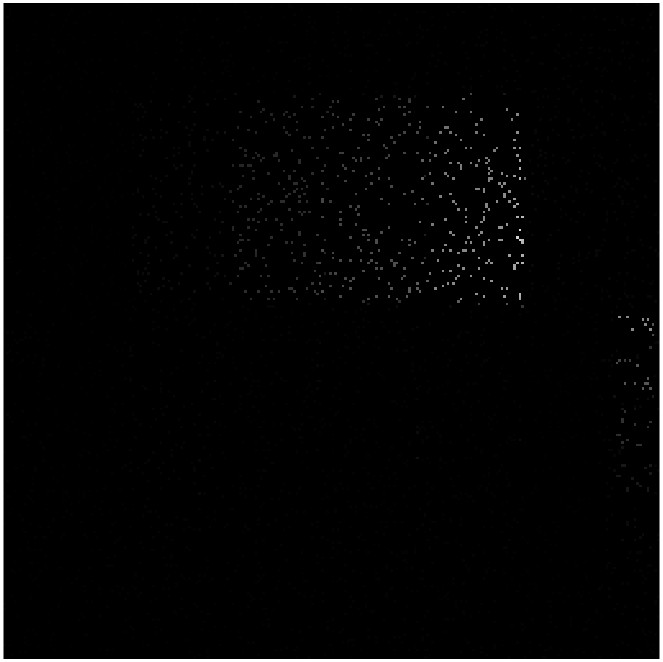}\\
		\includegraphics[width=1\linewidth]{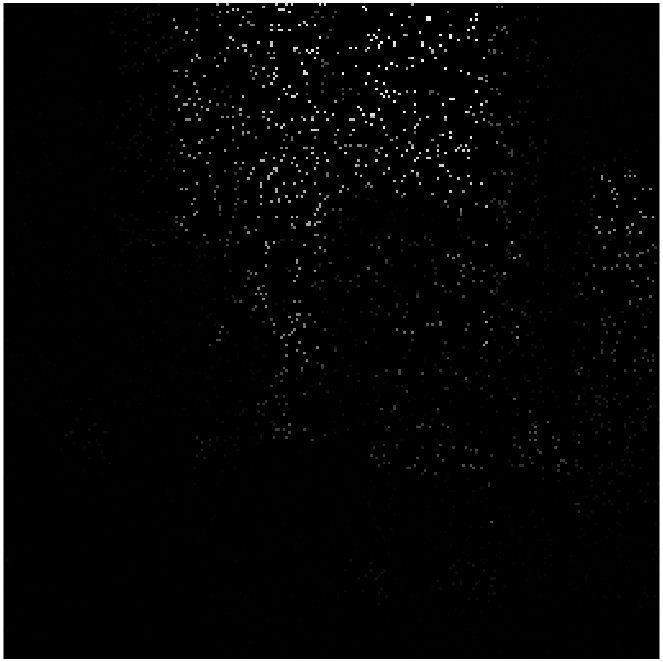}\\
		\includegraphics[width=1\linewidth]{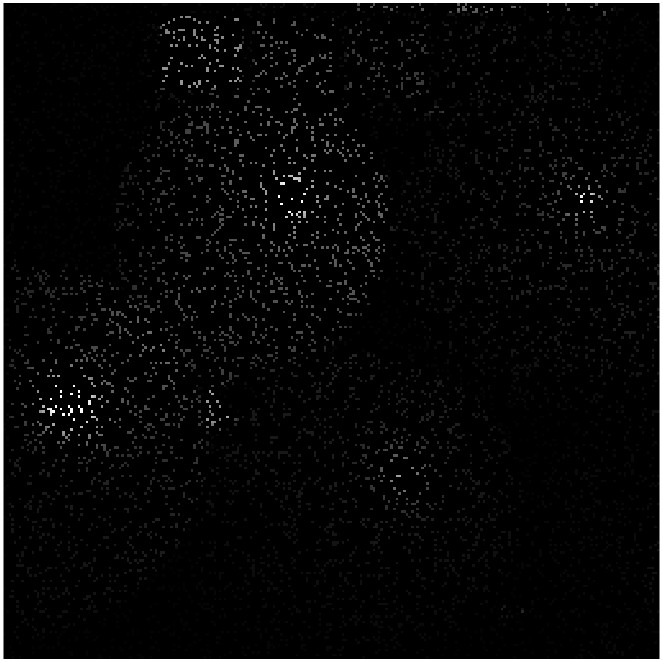}\\
		\includegraphics[width=1\linewidth]{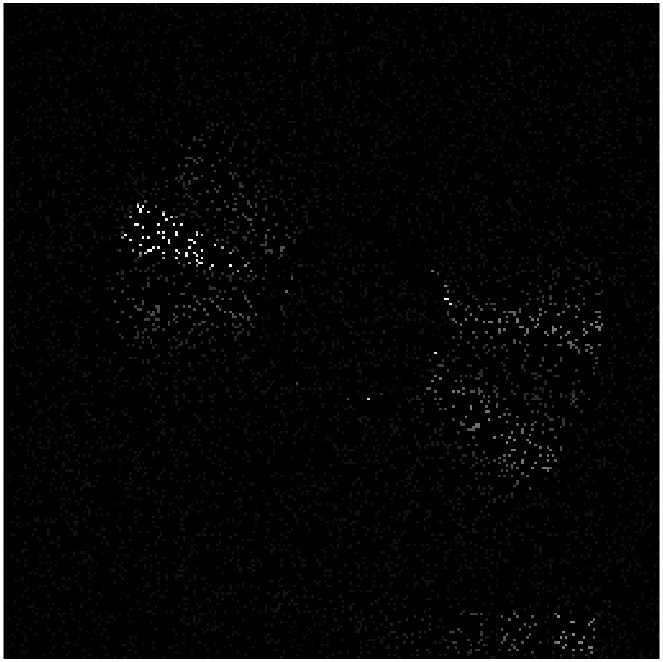}
\end{minipage}}\subfigure[HaLRTC]{
	\begin{minipage}[b]{0.095\linewidth}
		\includegraphics[width=1\linewidth]{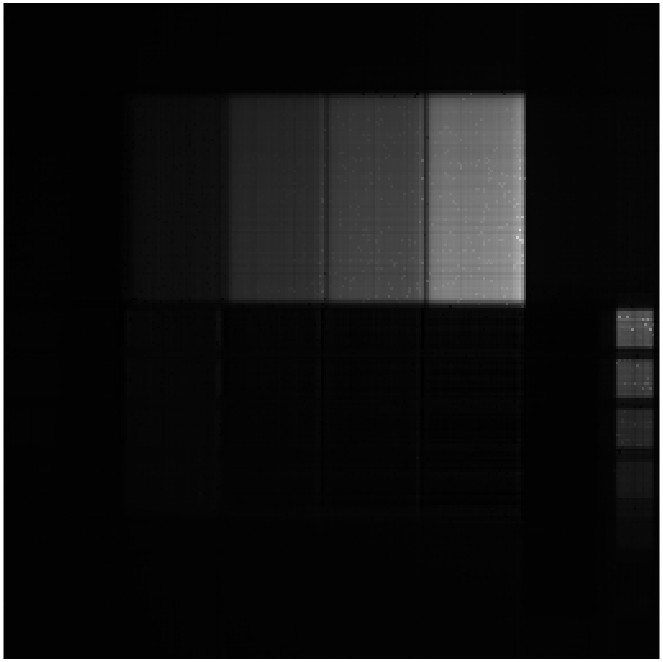}\\
		\includegraphics[width=1\linewidth]{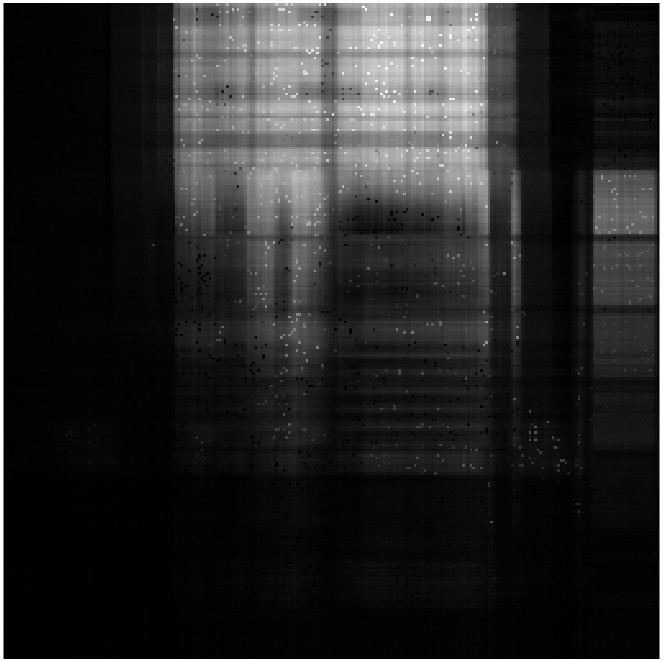}\\
		\includegraphics[width=1\linewidth]{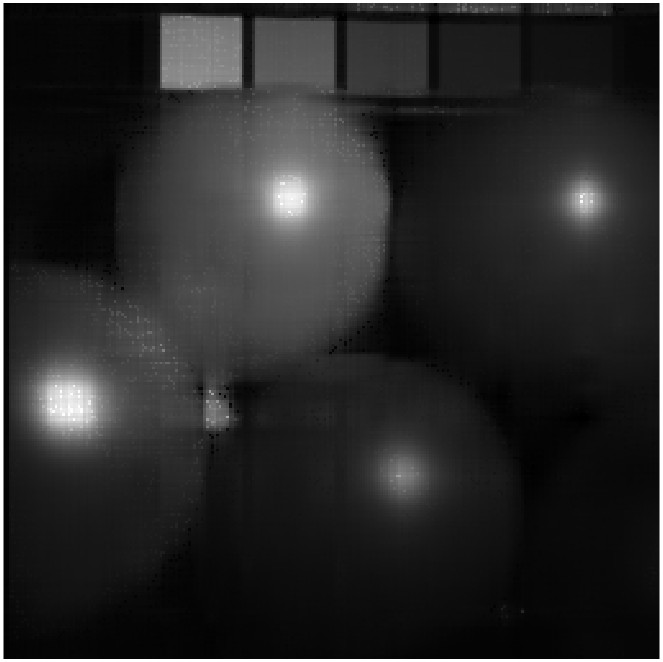}\\
		\includegraphics[width=1\linewidth]{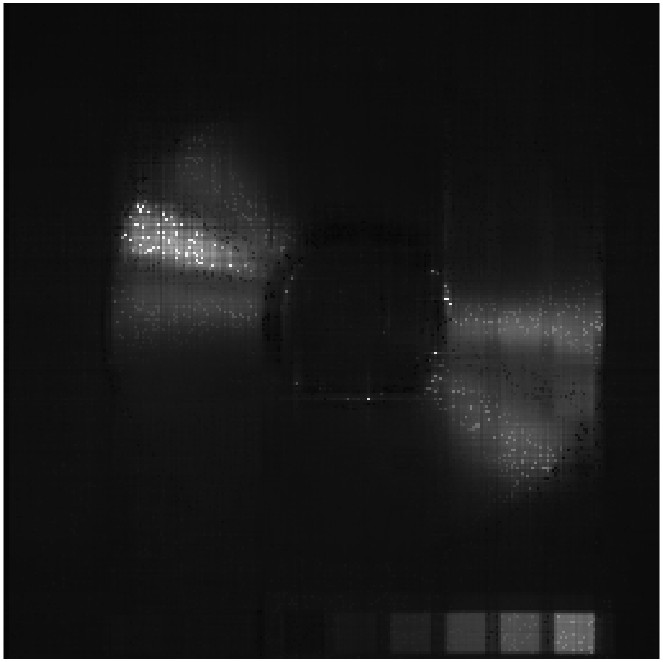}
\end{minipage}}\subfigure[TNN]{
	\begin{minipage}[b]{0.095\linewidth}
		\includegraphics[width=1\linewidth]{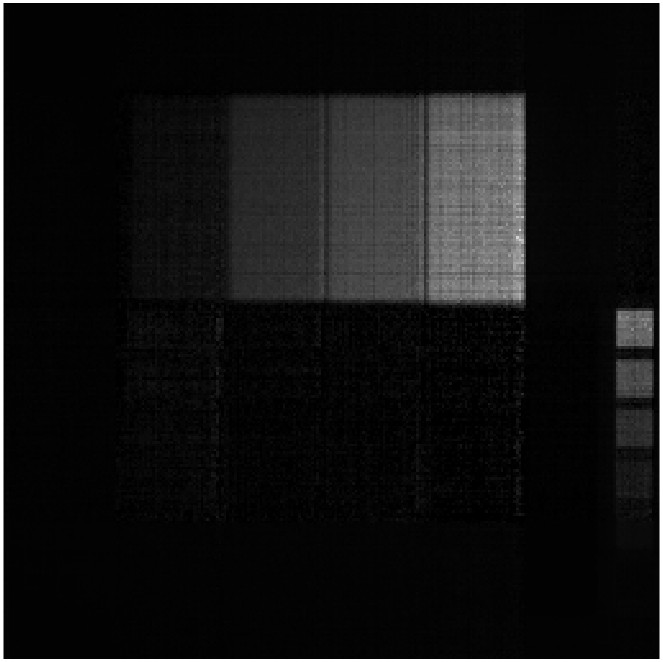}\\
		\includegraphics[width=1\linewidth]{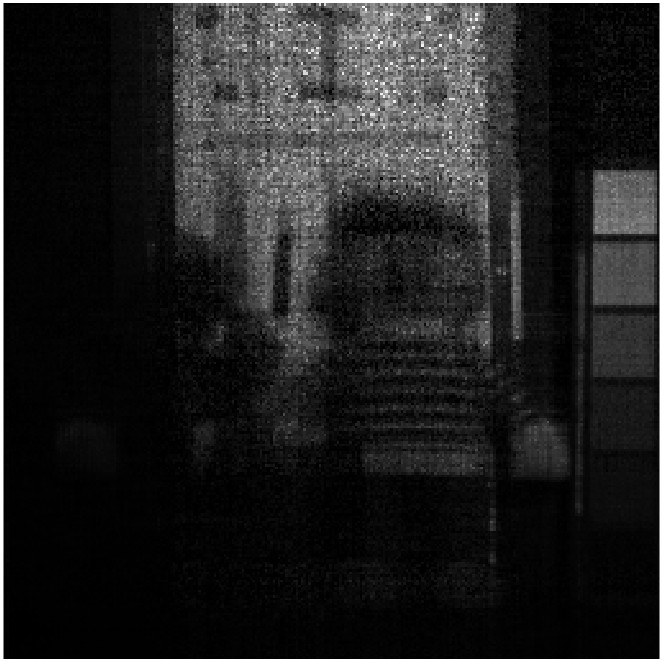}\\
		\includegraphics[width=1\linewidth]{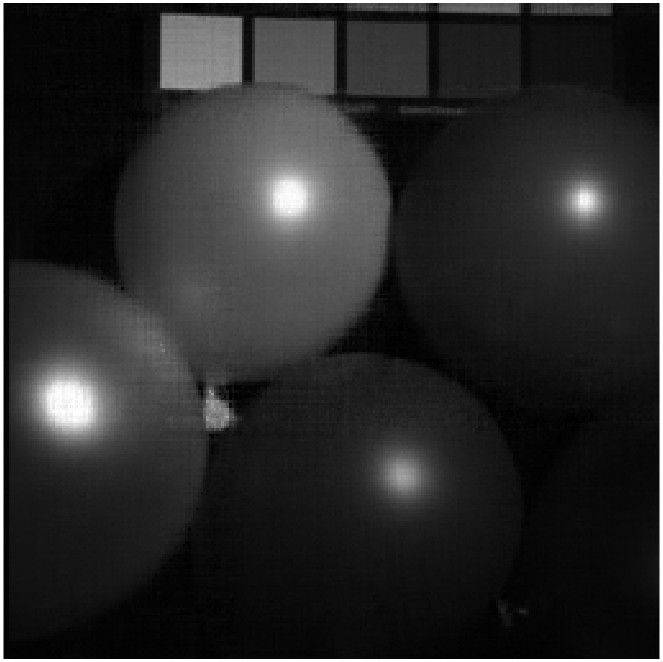}\\
		\includegraphics[width=1\linewidth]{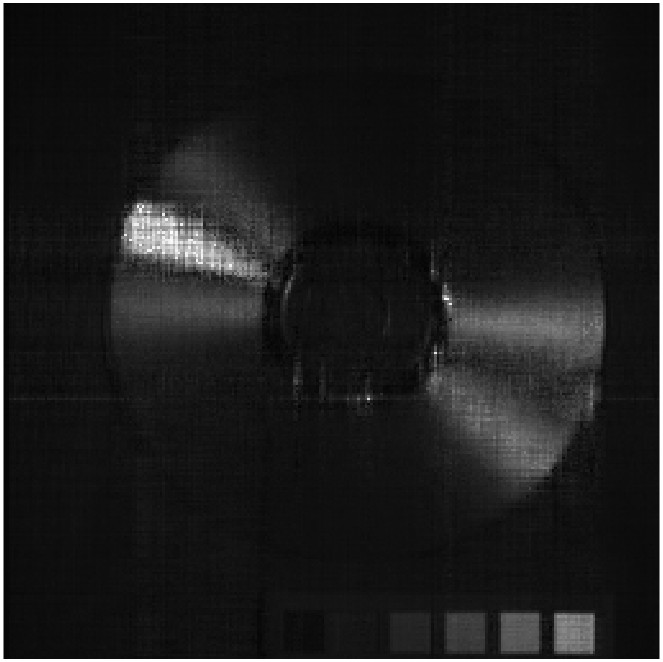}
\end{minipage}}\subfigure[LRTCTV]{
	\begin{minipage}[b]{0.095\linewidth}
		\includegraphics[width=1\linewidth]{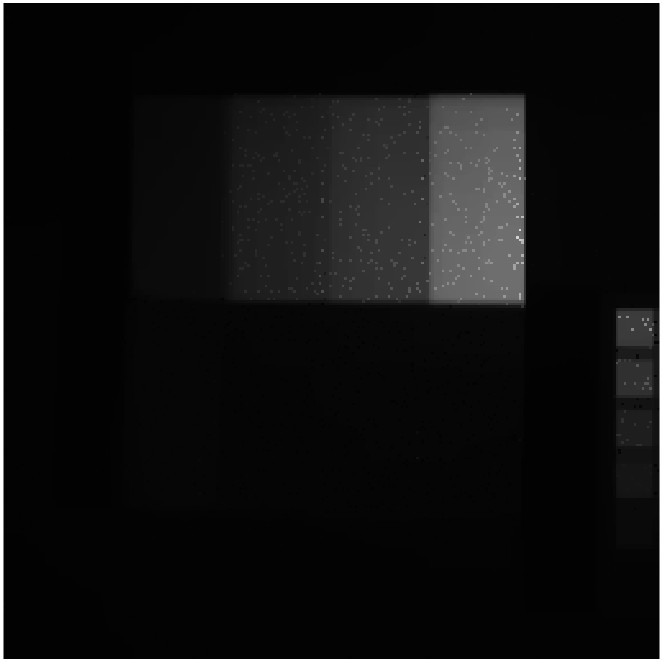}\\
		\includegraphics[width=1\linewidth]{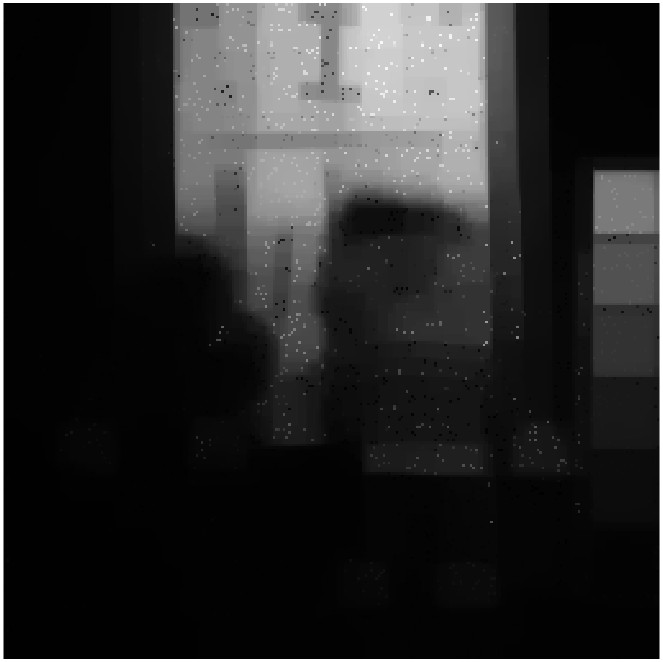}\\
		\includegraphics[width=1\linewidth]{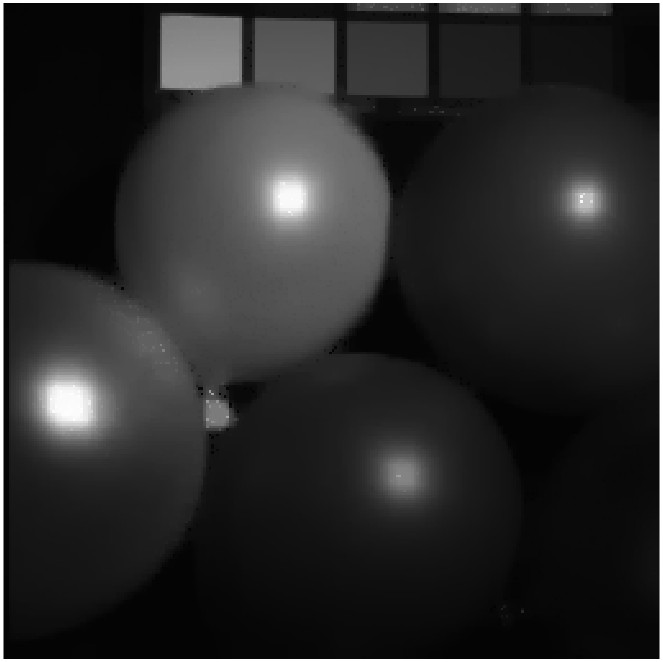}\\
		\includegraphics[width=1\linewidth]{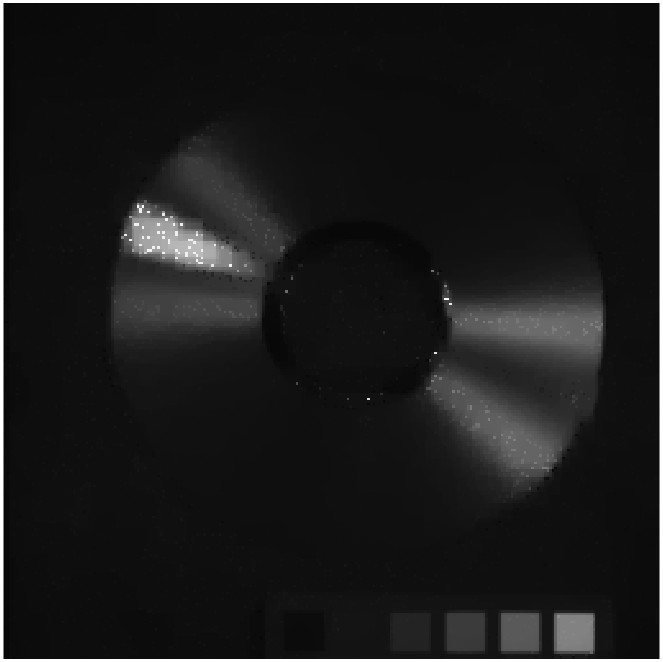}
\end{minipage}}\subfigure[PSTNN]{
	\begin{minipage}[b]{0.095\linewidth}
		\includegraphics[width=1\linewidth]{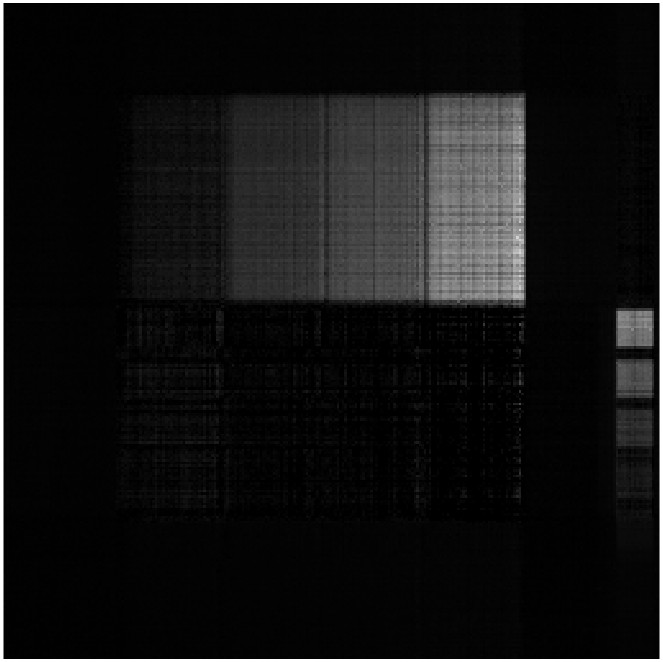}\\
		\includegraphics[width=1\linewidth]{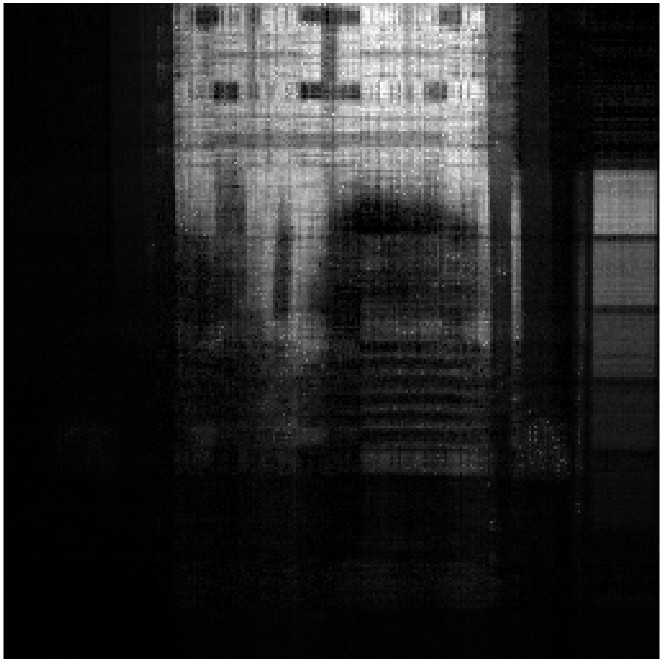}\\
		\includegraphics[width=1\linewidth]{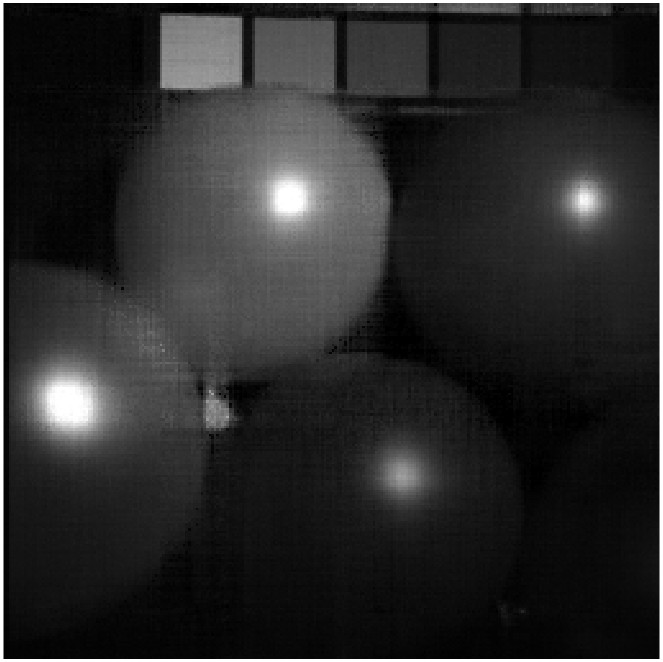}\\
		\includegraphics[width=1\linewidth]{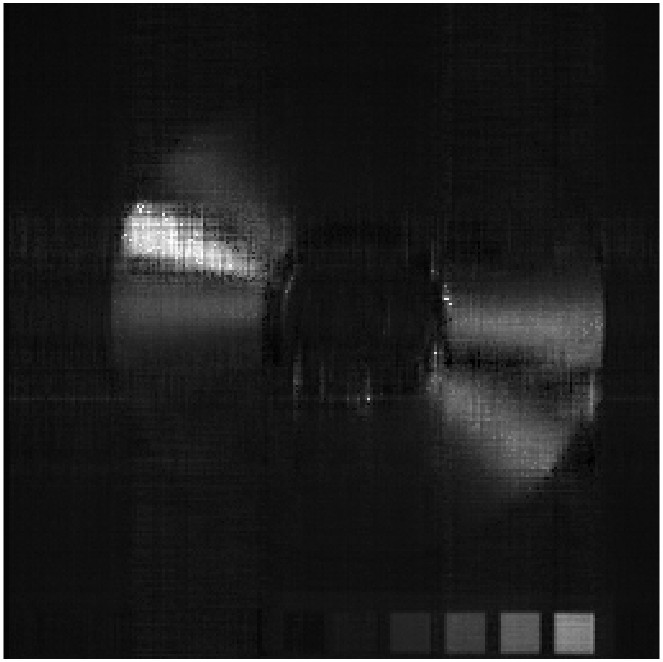}
\end{minipage}}\subfigure[FTNN]{
	\begin{minipage}[b]{0.095\linewidth}
		\includegraphics[width=1\linewidth]{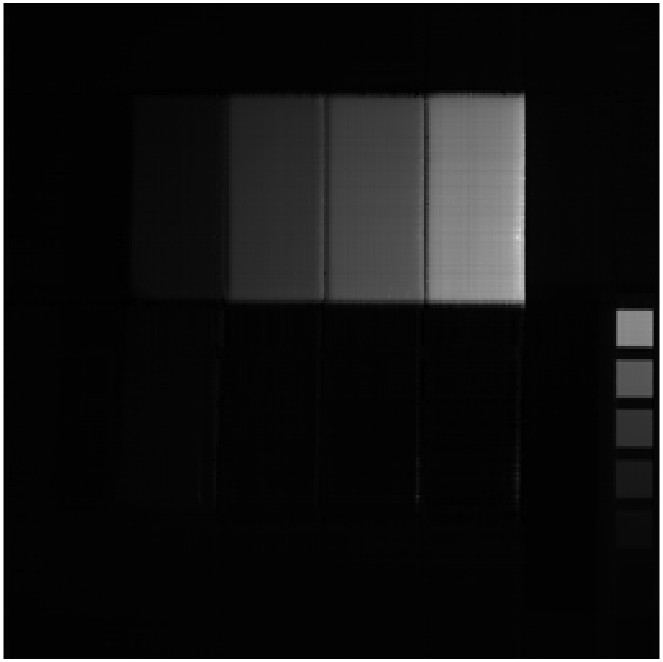}\\
		\includegraphics[width=1\linewidth]{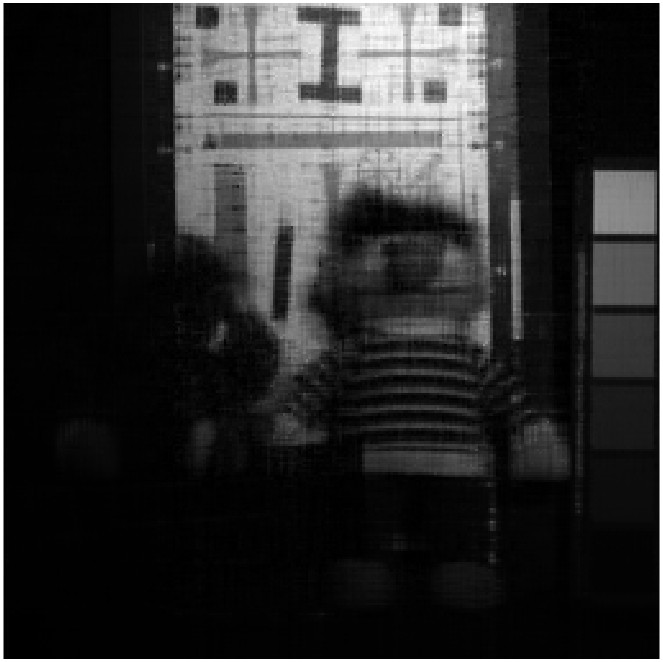}\\
		\includegraphics[width=1\linewidth]{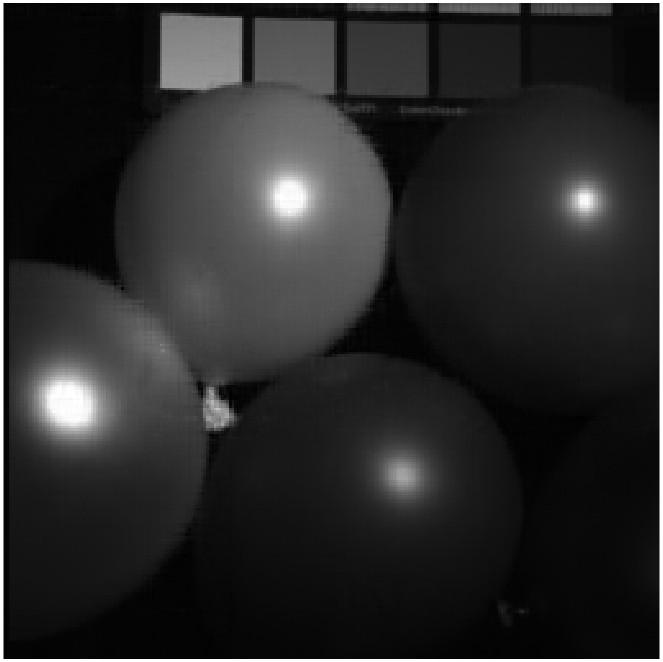}\\
		\includegraphics[width=1\linewidth]{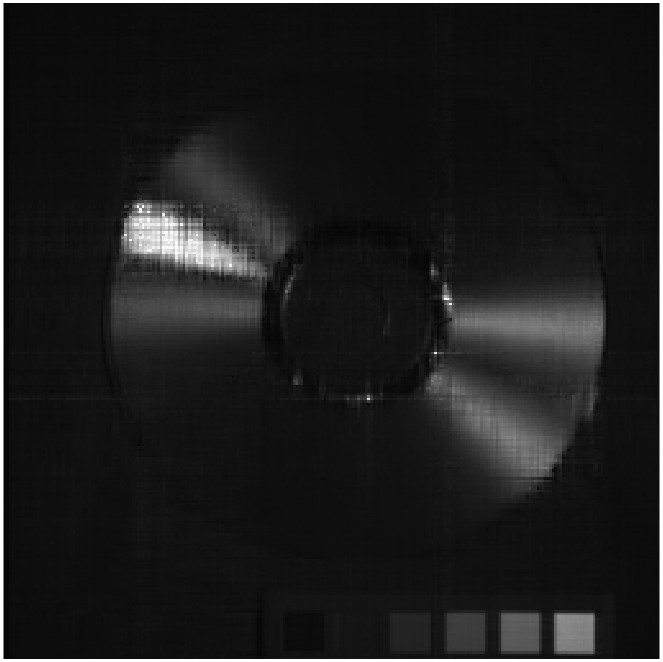}
\end{minipage}}\subfigure[WSTNN]{
	\begin{minipage}[b]{0.095\linewidth}
		\includegraphics[width=1\linewidth]{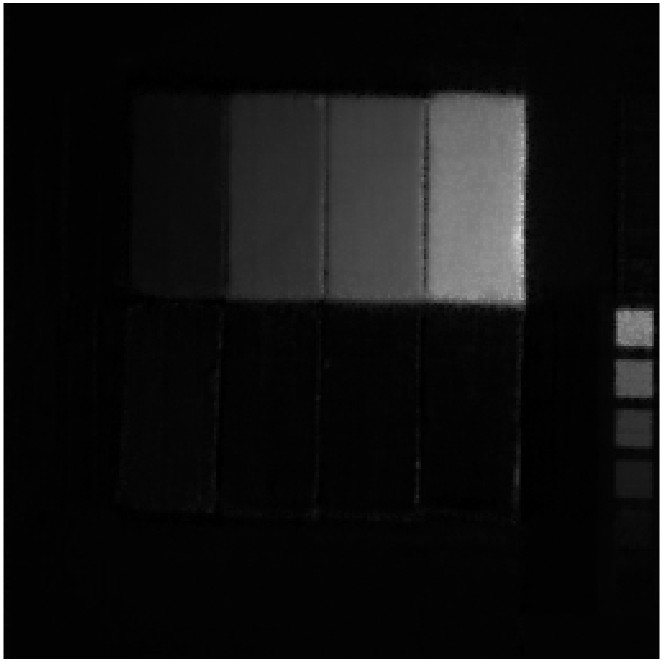}\\
		\includegraphics[width=1\linewidth]{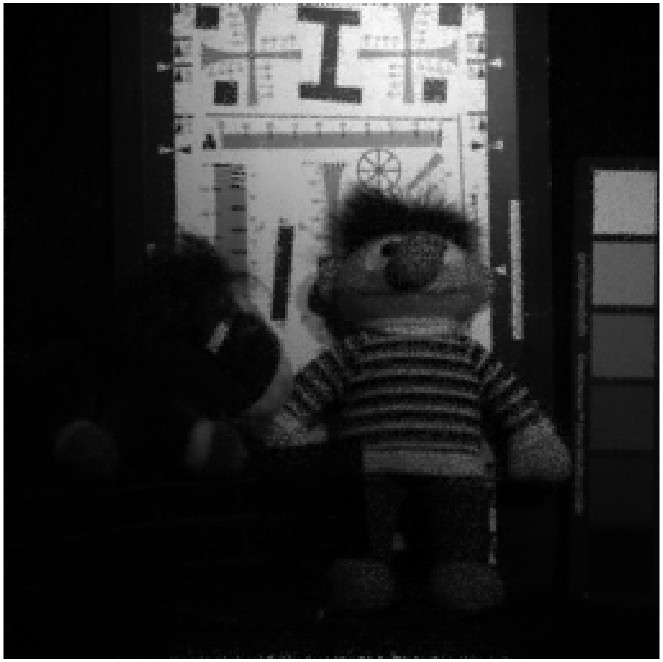}\\
		\includegraphics[width=1\linewidth]{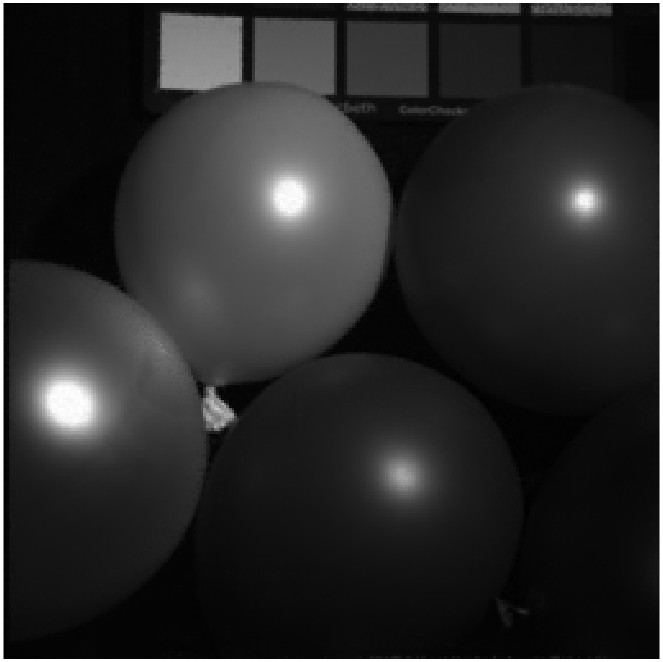}\\
		\includegraphics[width=1\linewidth]{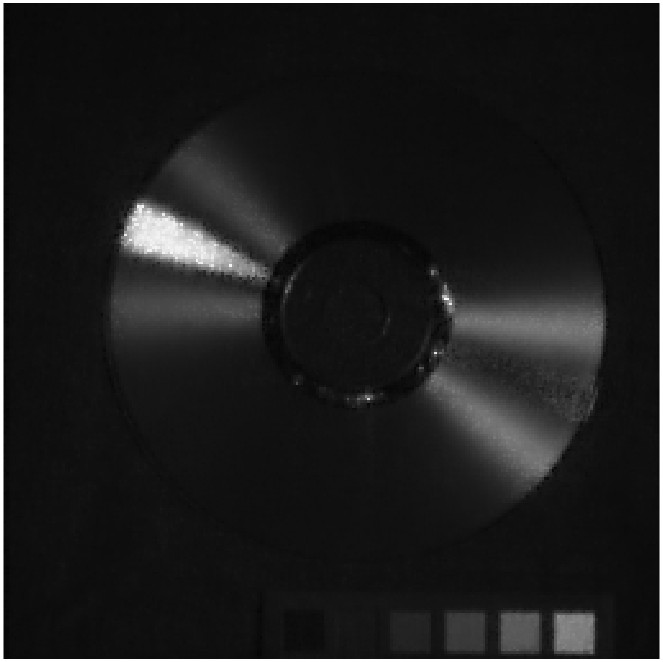}
\end{minipage}}\subfigure[BEMCP]{
	\begin{minipage}[b]{0.095\linewidth}
		\includegraphics[width=1\linewidth]{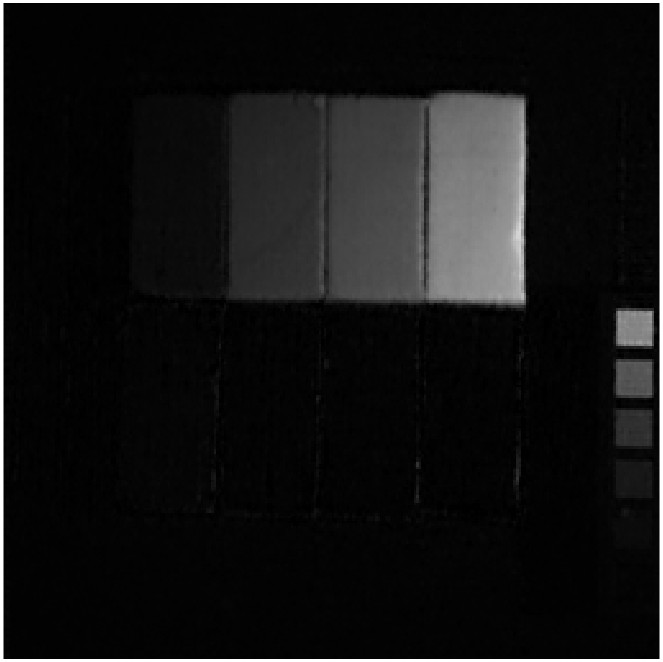}\\
		\includegraphics[width=1\linewidth]{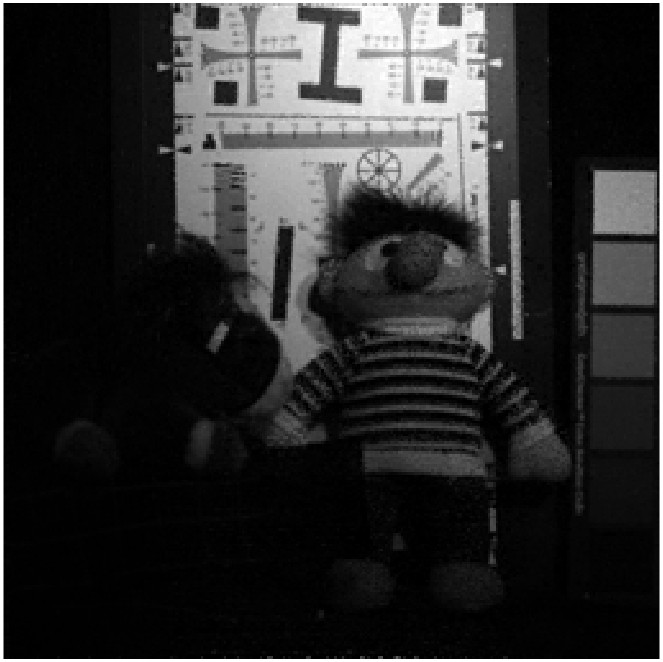}\\
		\includegraphics[width=1\linewidth]{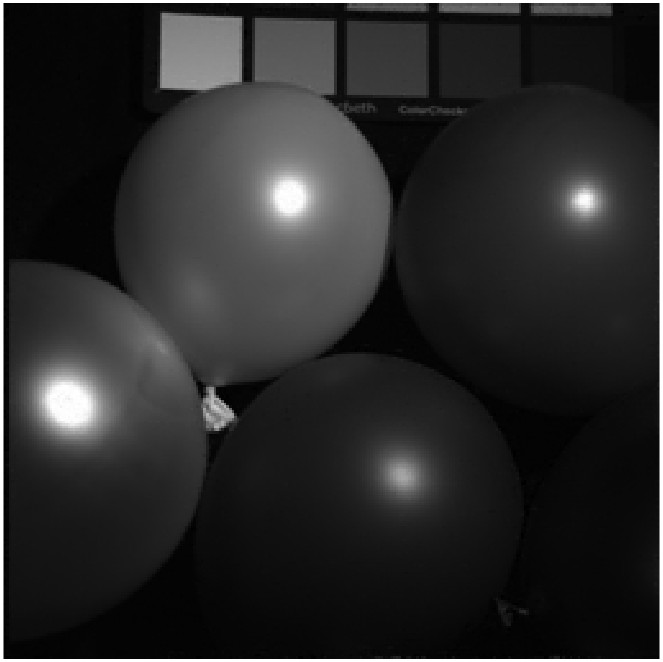}\\
		\includegraphics[width=1\linewidth]{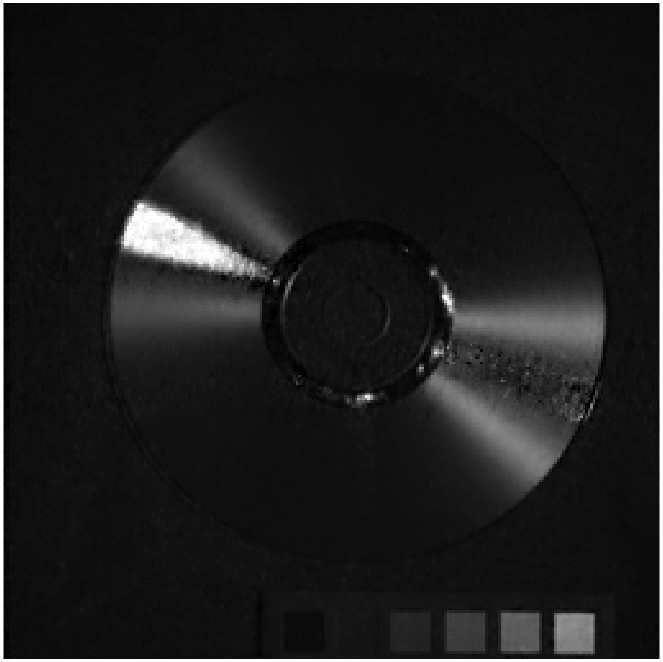}
\end{minipage}}\subfigure[TJLC]{
	\begin{minipage}[b]{0.095\linewidth}
		\includegraphics[width=1\linewidth]{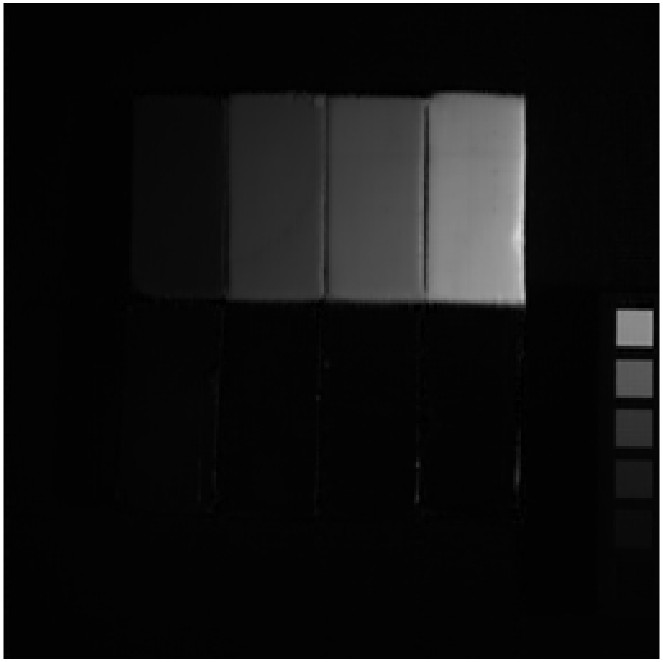}\\
		\includegraphics[width=1\linewidth]{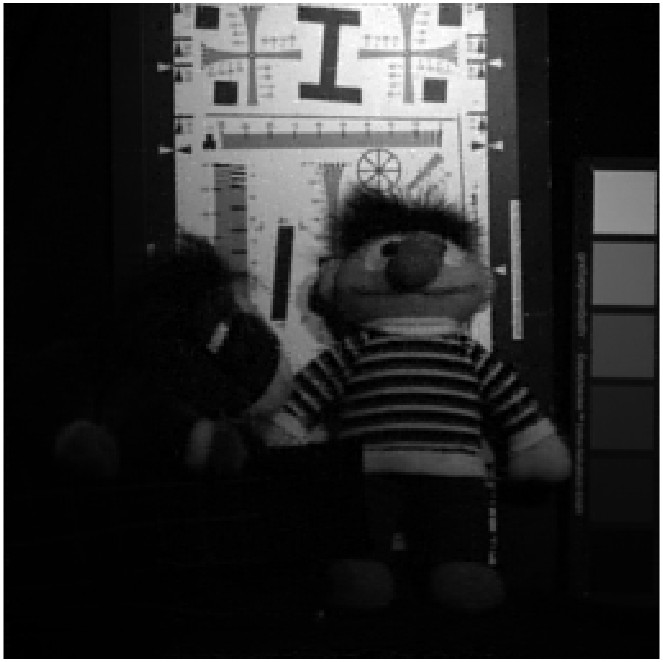}\\
		\includegraphics[width=1\linewidth]{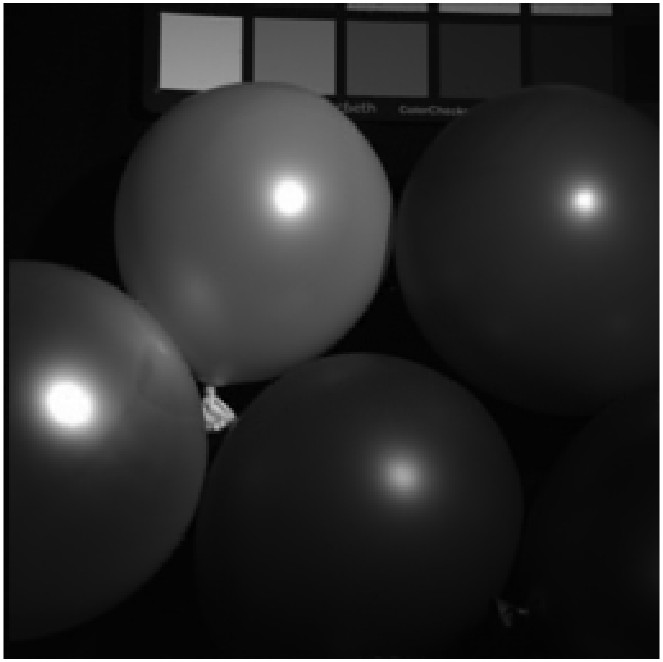}\\
		\includegraphics[width=1\linewidth]{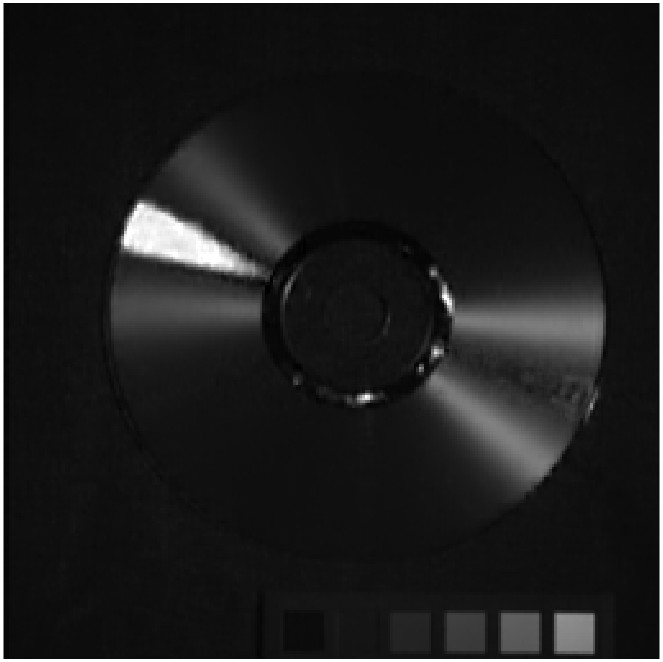}
\end{minipage}}
	\caption{Visual results for MSI data. The rows of MSIs are in order: clay, chart\_and\_stuffed\_toy, balloons, cd. MR: top two rows are 95\%, and last two rows are 90\%. The corresponding bands in each row are: 15, 20, 25, 30.}
	\label{MSIF}
\end{figure}

\begin{table}[!h]
	\centering
	\caption{The average PSNR, SSIM, FSIM and ERGAS values for four MSIs tested by observed and the eight utilized LRTC methods.}
	\label{MSIT}
{\footnotesize 		
	\begin{tabular}{|c|cccc|cccc|}
			\hline
			MR       & \multicolumn{4}{c|}{95\%}                                           & \multicolumn{4}{c|}{90\%}                                           \\ \hline
			PQIs     & PSNR            & SSIM           & FSIM           & ERGAS           & PSNR            & SSIM           & FSIM           & ERGAS           \\ \hline
			Observed & 15.216          & 0.161          & 0.687          & 874.799         & 15.452          & 0.203          & 0.668          & 851.367         \\
			HaLRTC   & 25.720          & 0.848          & 0.872          & 276.368         & 30.850          & 0.912          & 0.920          & 164.226         \\
			TNN      & 24.799          & 0.766          & 0.821          & 306.299         & 33.086          & 0.912          & 0.920          & 132.621         \\
			LRTCTV & 25.662          & 0.872          & 0.887          & 273.522         & 32.079          & 0.941          & 0.940          & 139.816         \\
			PSTNN    & 26.071          & 0.742          & 0.803          & 252.326         & 31.526          & 0.883          & 0.900          & 145.272         \\
			FTNN     & 33.307          & 0.936          & 0.936          & 124.535         & 37.550          & 0.970          & 0.967          & 76.659          \\
			WSTNN    & 34.934          & 0.960          & 0.955          & 94.610          & 40.012          & 0.984          & 0.980          & 56.094          \\
			BEMCP    & 38.005          & 0.964          & 0.956          & 68.354          & 43.376          & 0.984          & 0.981          & 39.328          \\
			TJLC     & \textbf{41.403} & \textbf{0.984} & \textbf{0.980} & \textbf{47.987} & \textbf{46.588} & \textbf{0.994} & \textbf{0.991} & \textbf{28.066} \\ \hline
		\end{tabular}
	}%
\end{table}

\subsection{CV data}
We test six color videos\footnote{http://trace.eas.asu.edu/yuv/} (respectively named akiyo, hall, foreman, news, highway, container) of size $144 \times 176 \times 50$. Firstly, we demonstrate the visual results in our experiment in Fig. \ref{CVF}. It is not hard to see from Fig. \ref{CVF} that the recovery of the proposed TJLC method on the visual effect is more better. Specifically, as shown in Fig. \ref{CVF}, it can be observed that the suboptimal BEMCP method produces noticeable anomalies in the background of the `` news " image, while the proposed TJLC method exhibits almost no anomalies and the overall background is clear. Furthermore, we list the average quantitative results of six CVs in Table \ref{CVT}. When the missing rate is 90\%, the PSNR value of the proposed TJLC method is 0.939 dB higher than the suboptimal BEMCP method. In addition, at the missing rate of 95\%, the PSNR value of the proposed TJLC method is at least 0.883 dB higher than that of the suboptimal method.

\begin{table}[!h]
	\centering
	\caption{The average PSNR, SSIM, FSIM and ERGAS values for six CVs tested by observed and the eight utilized LRTC methods.}
	\label{CVT}
{\footnotesize 		
	\begin{tabular}{|c|cccc|cccc|}
			\hline
			MR       & \multicolumn{4}{c|}{95\%}                                           & \multicolumn{4}{c|}{90\%}                                           \\ \hline
			PQIs     & PSNR            & SSIM           & FSIM           & ERGAS           & PSNR            & SSIM           & FSIM           & ERGAS           \\ \hline
			Observed & 5.628           & 0.011          & 0.425          & 1207.399        & 5.863           & 0.019          & 0.421          & 1175.185        \\
			HaLRTC   & 17.345          & 0.510          & 0.709          & 328.433         & 21.391          & 0.647          & 0.788          & 207.663         \\
			TNN      & 27.832          & 0.809          & 0.902          & 100.813         & 31.469          & 0.886          & 0.941          & 67.719          \\
			LRTCTV & 19.515          & 0.606          & 0.716          & 274.267         & 21.235          & 0.677          & 0.788          & 230.001         \\
			PSTNN    & 14.749          & 0.306          & 0.664          & 422.958         & 27.650          & 0.816          & 0.904          & 101.060         \\
			FTNN     & 25.561          & 0.781          & 0.879          & 133.283         & 29.046          & 0.871          & 0.925          & 87.222          \\
			WSTNN    & 30.109          & 0.901          & 0.937          & 77.051          & 33.679          & 0.942          & 0.963          & 52.736          \\
			BEMCP    & 32.191          & 0.915          & 0.953          & 61.269          & 35.675          & 0.947          & 0.972          & 42.018          \\
			TJLC     & \textbf{33.074} & \textbf{0.927} & \textbf{0.958} & \textbf{55.429} & \textbf{36.614} & \textbf{0.957} & \textbf{0.975} & \textbf{38.075} \\ \hline
		\end{tabular}}%
\end{table}

\begin{figure}[!h] 

	\vspace{0cm} 
		\subfigtopskip=2pt 
		\subfigbottomskip=2pt 
	\subfigure[Original]{
	\begin{minipage}[b]{0.095\linewidth}
		\includegraphics[width=1\linewidth]{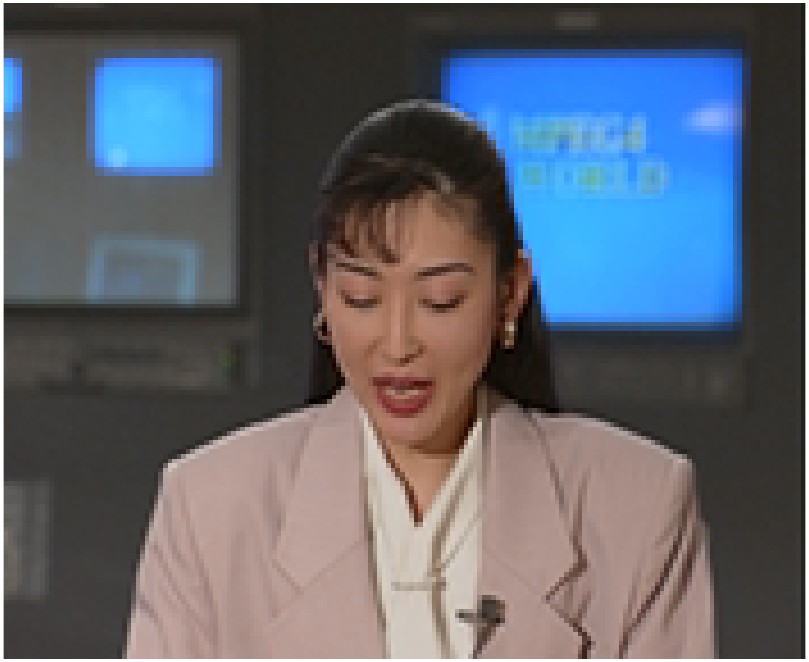}\\
		\includegraphics[width=1\linewidth]{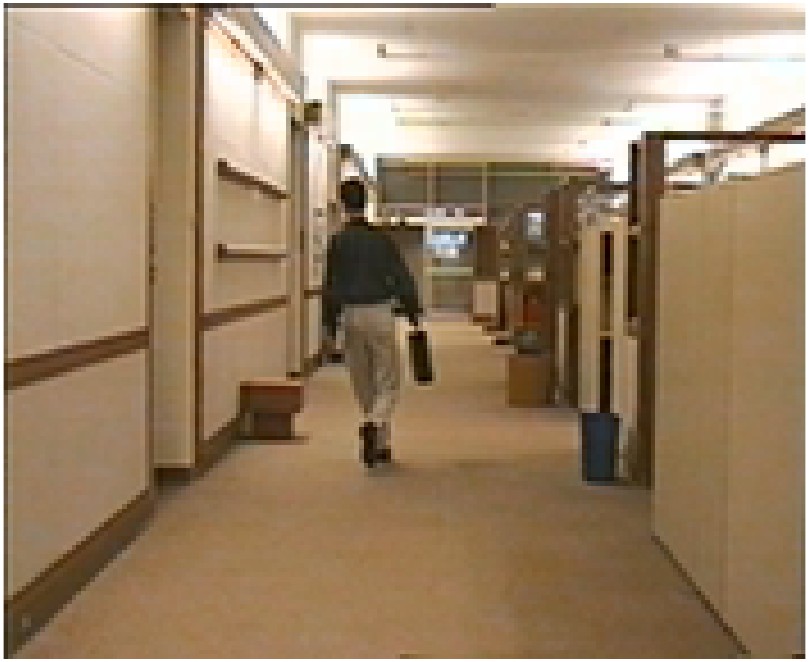}\\
		\includegraphics[width=1\linewidth]{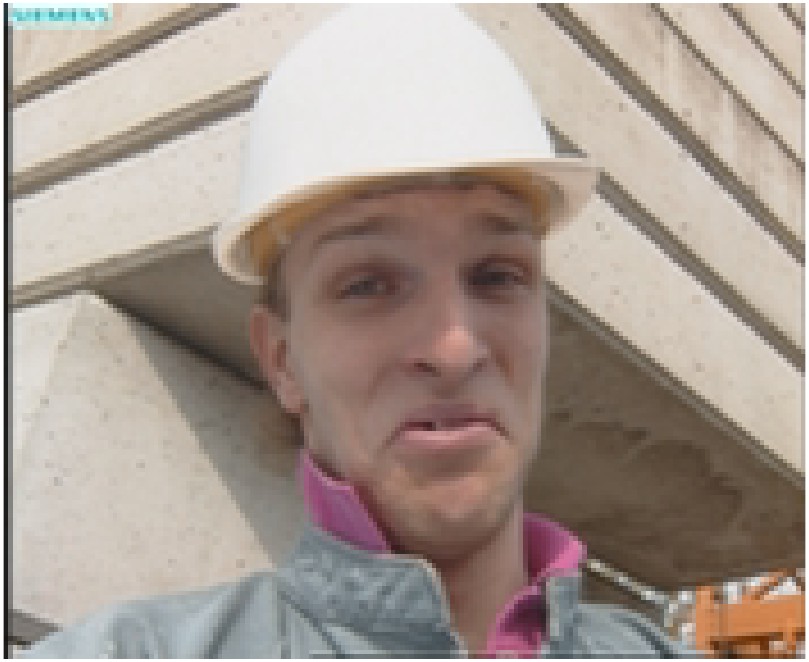}\\
		\includegraphics[width=1\linewidth]{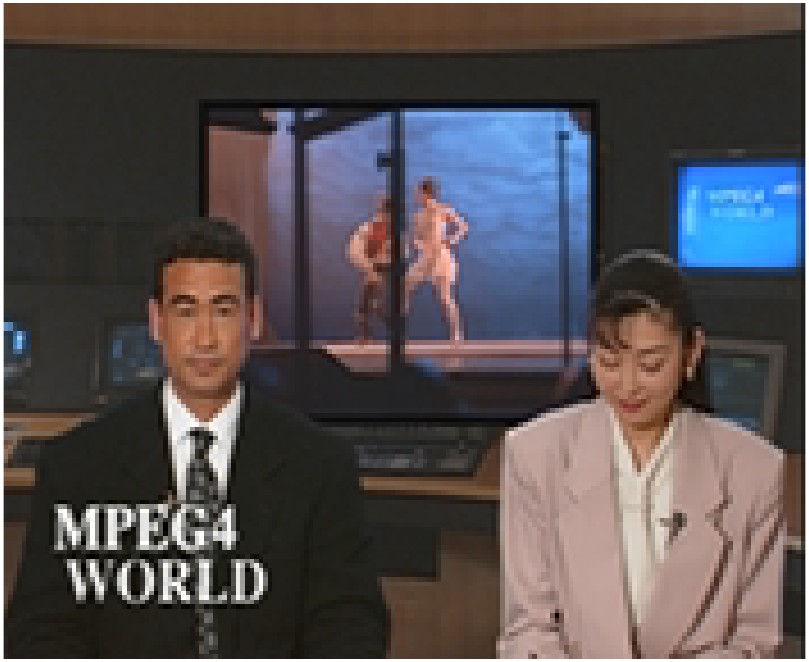}\\
		\includegraphics[width=1\linewidth]{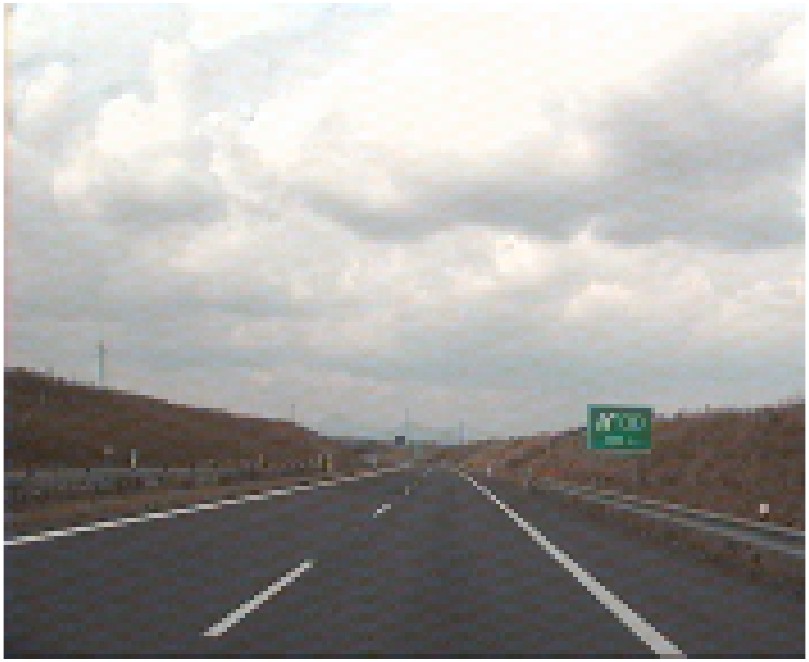}\\
		\includegraphics[width=1\linewidth]{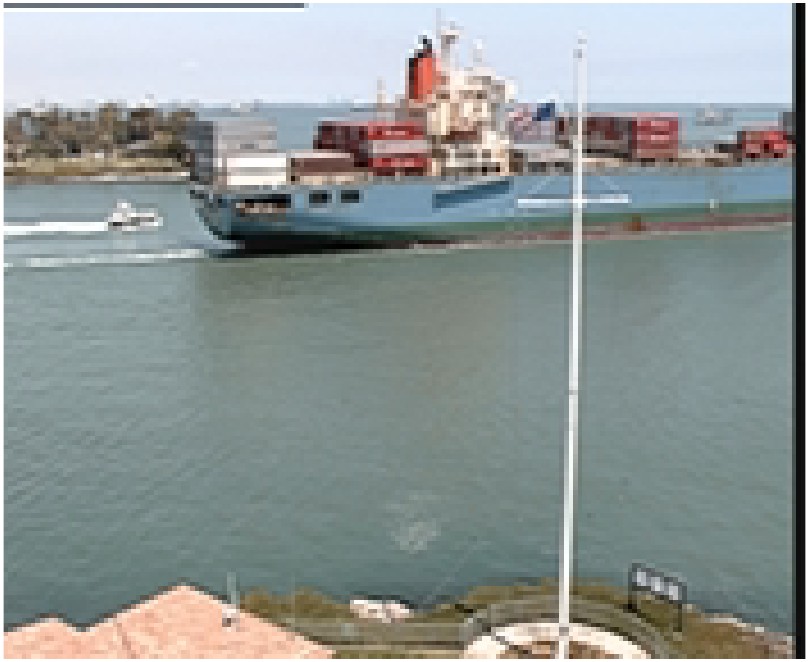} 
\end{minipage}}\subfigure[Observed]{
	\begin{minipage}[b]{0.095\linewidth}
		\includegraphics[width=1\linewidth]{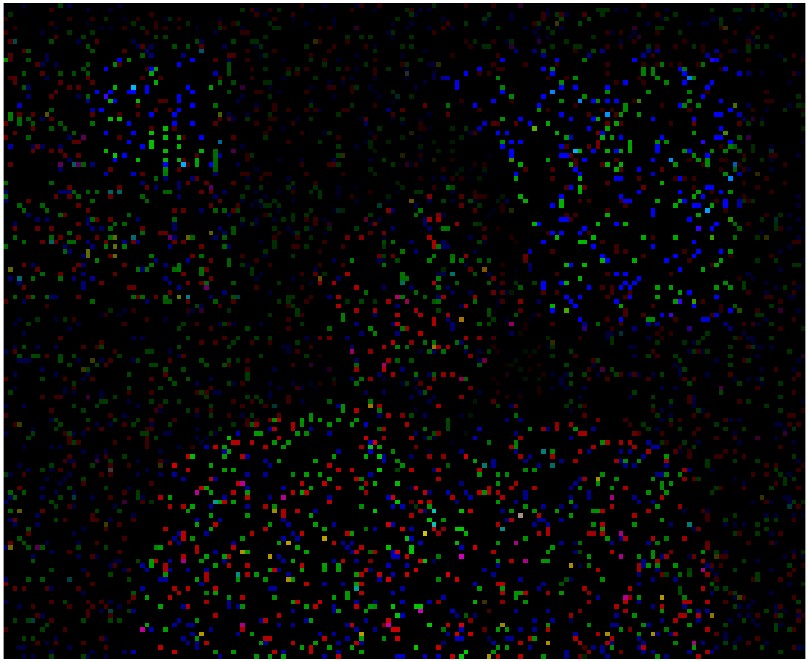}\\
		\includegraphics[width=1\linewidth]{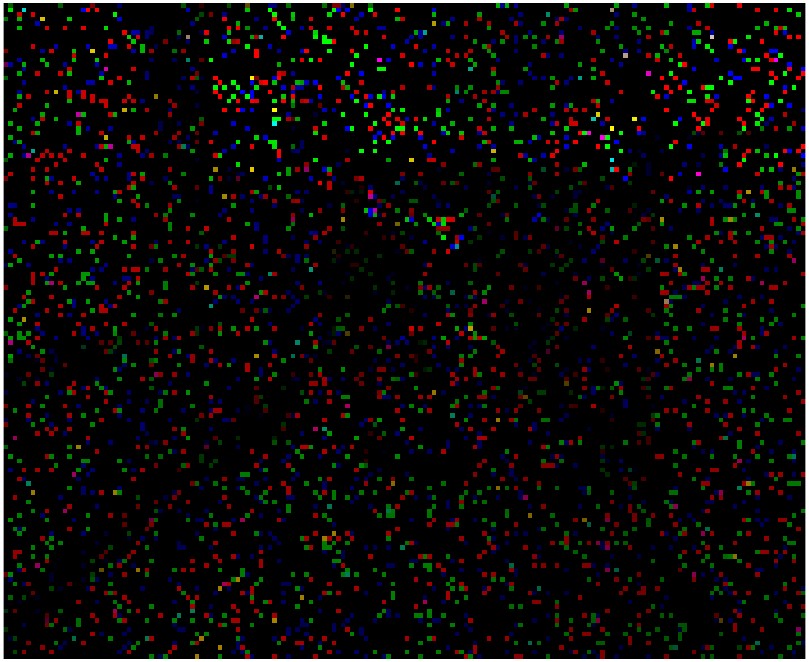}\\
		\includegraphics[width=1\linewidth]{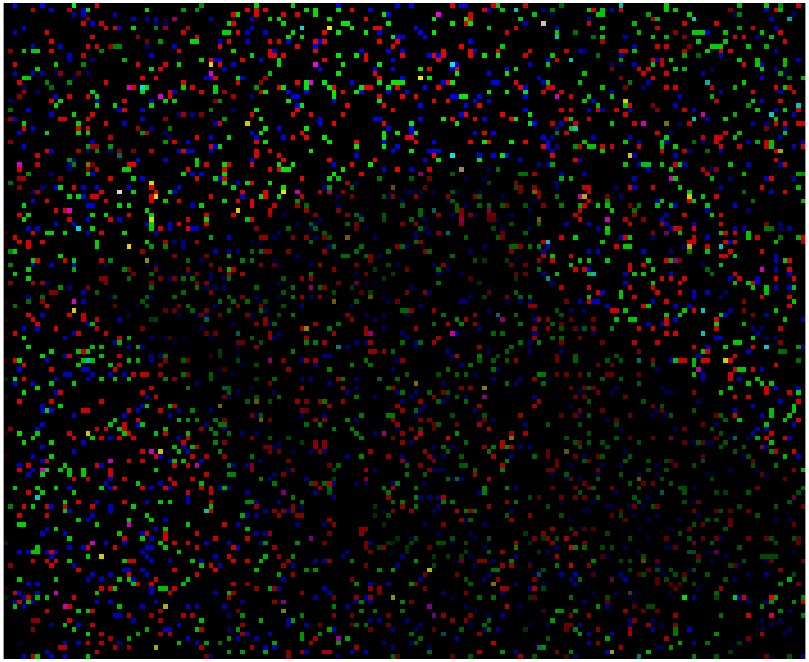}\\
		\includegraphics[width=1\linewidth]{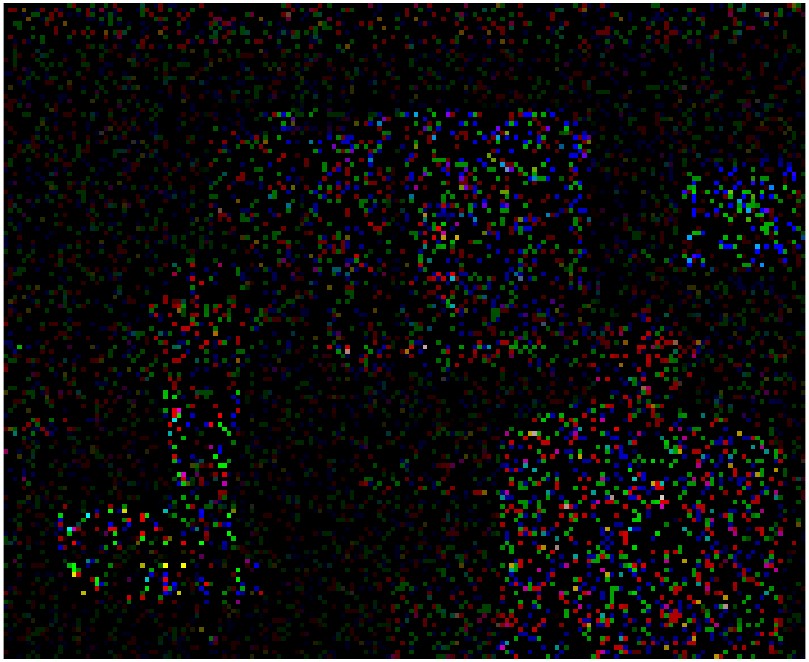}\\
		\includegraphics[width=1\linewidth]{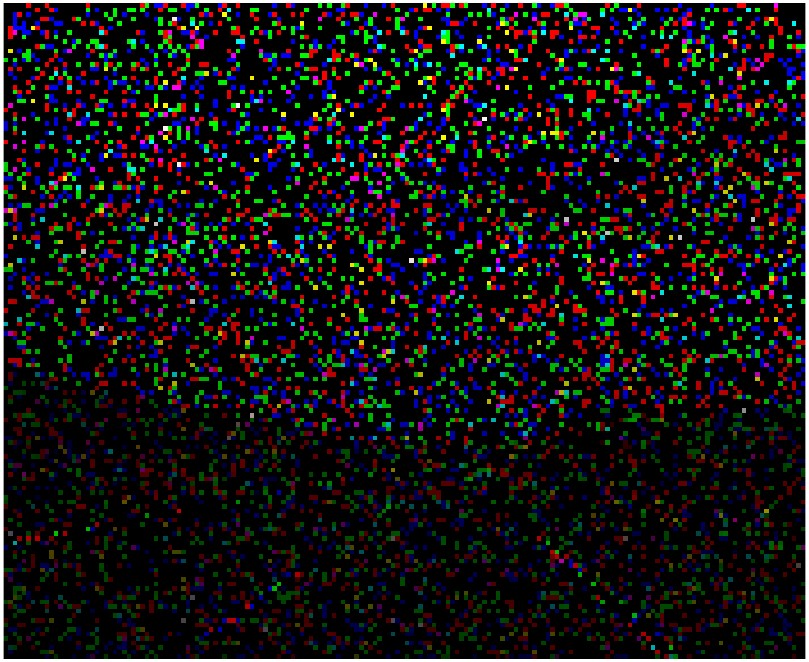}\\
		\includegraphics[width=1\linewidth]{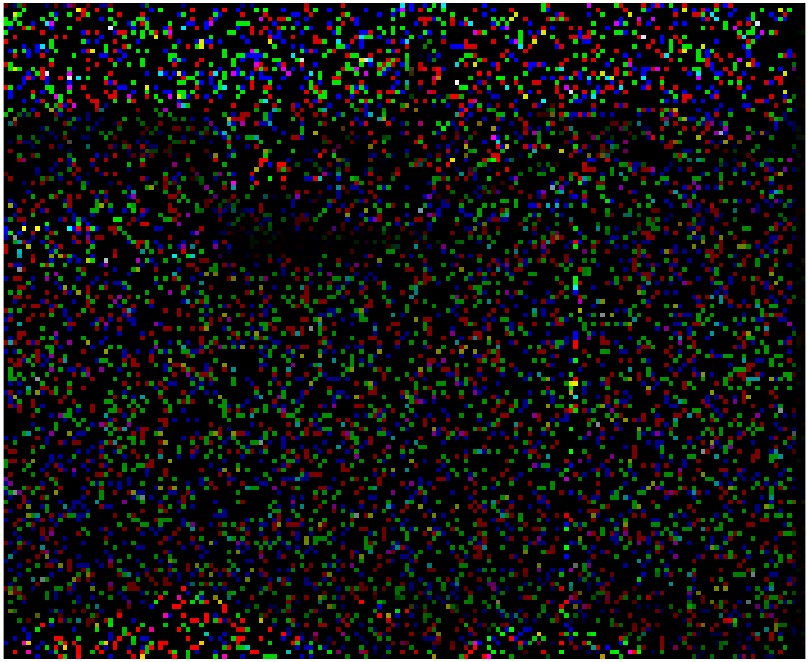} 
\end{minipage}}\subfigure[HaLRTC]{
\begin{minipage}[b]{0.095\linewidth}
\includegraphics[width=1\linewidth]{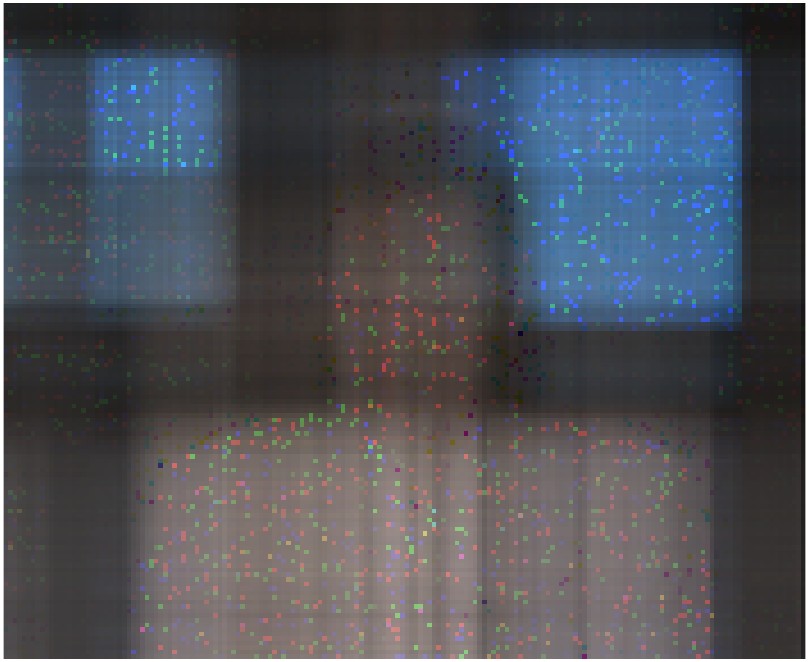}\\
\includegraphics[width=1\linewidth]{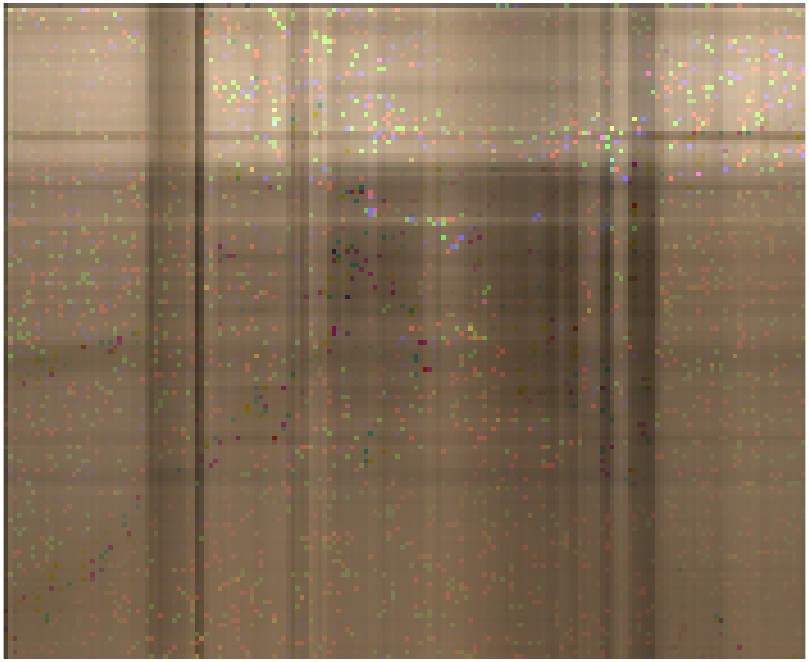}\\
\includegraphics[width=1\linewidth]{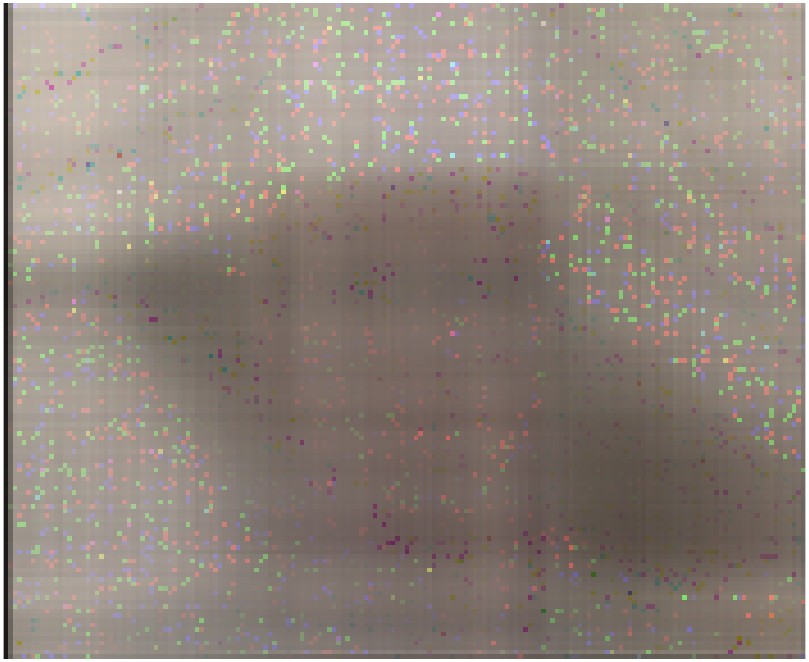}\\
\includegraphics[width=1\linewidth]{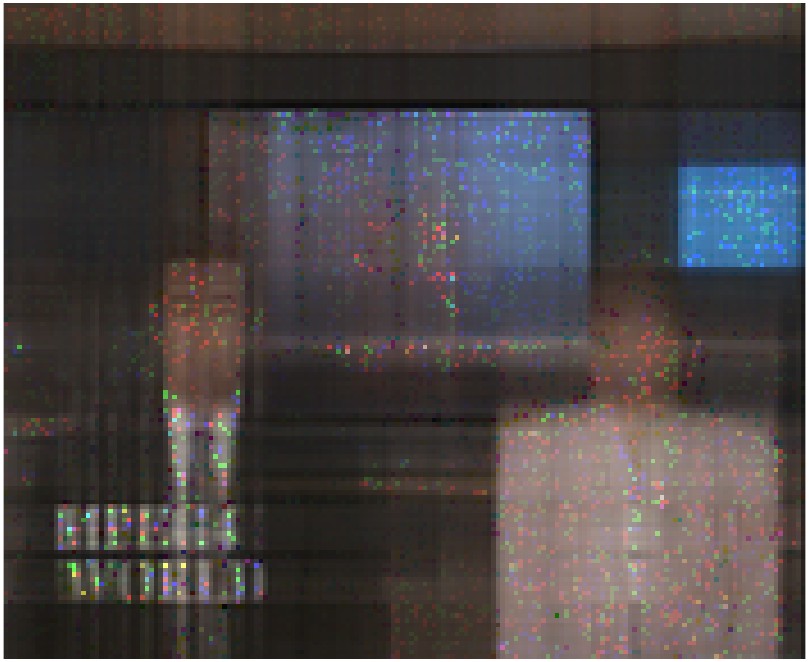}\\
\includegraphics[width=1\linewidth]{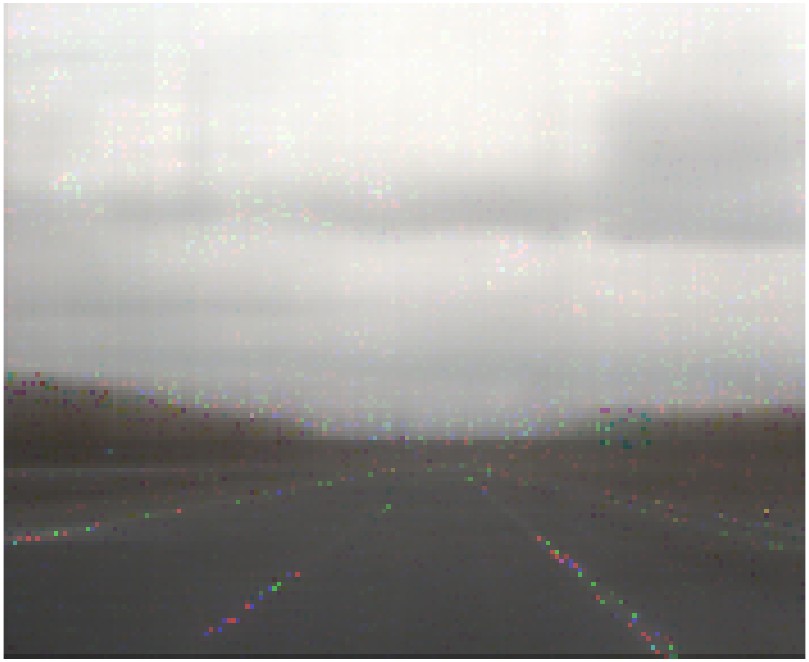}\\
\includegraphics[width=1\linewidth]{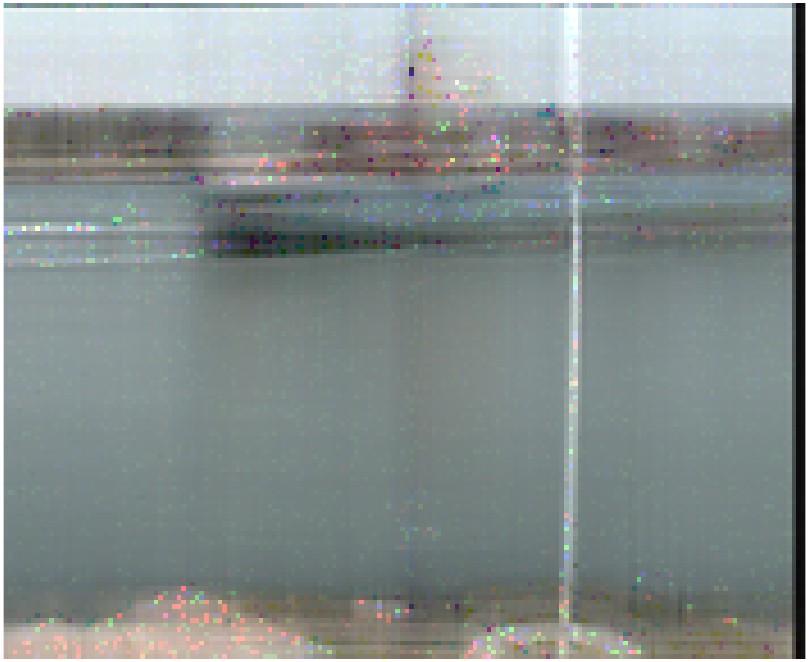} 
\end{minipage}}\subfigure[TNN]{
\begin{minipage}[b]{0.095\linewidth}
\includegraphics[width=1\linewidth]{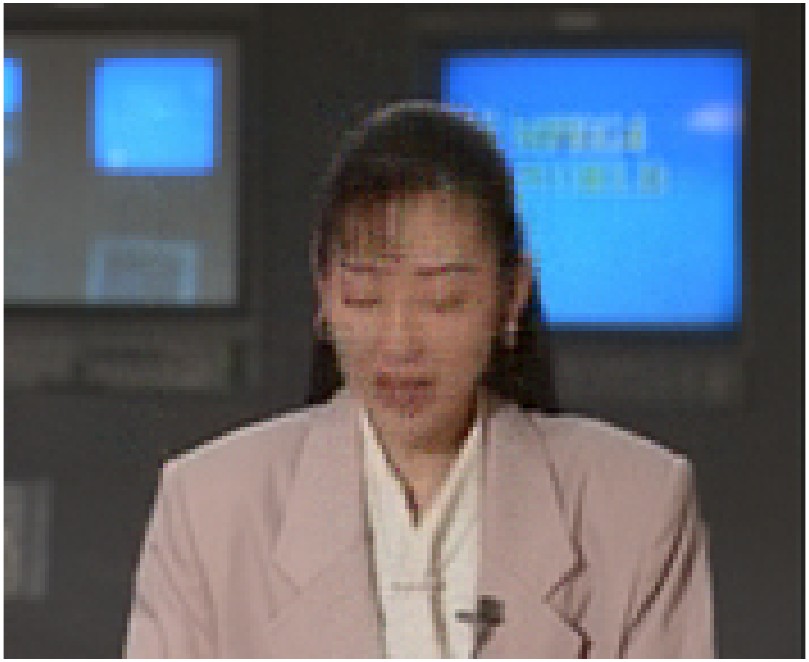}\\
\includegraphics[width=1\linewidth]{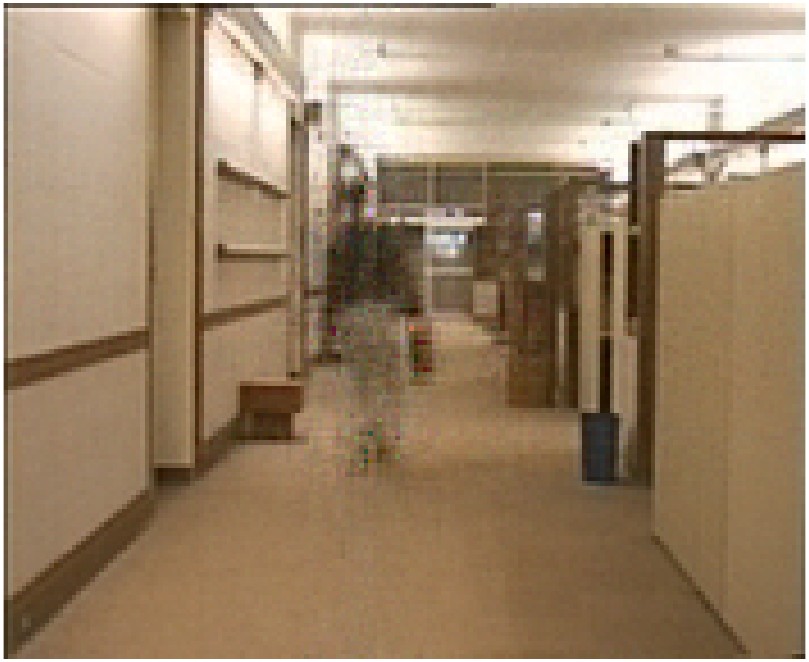}\\
\includegraphics[width=1\linewidth]{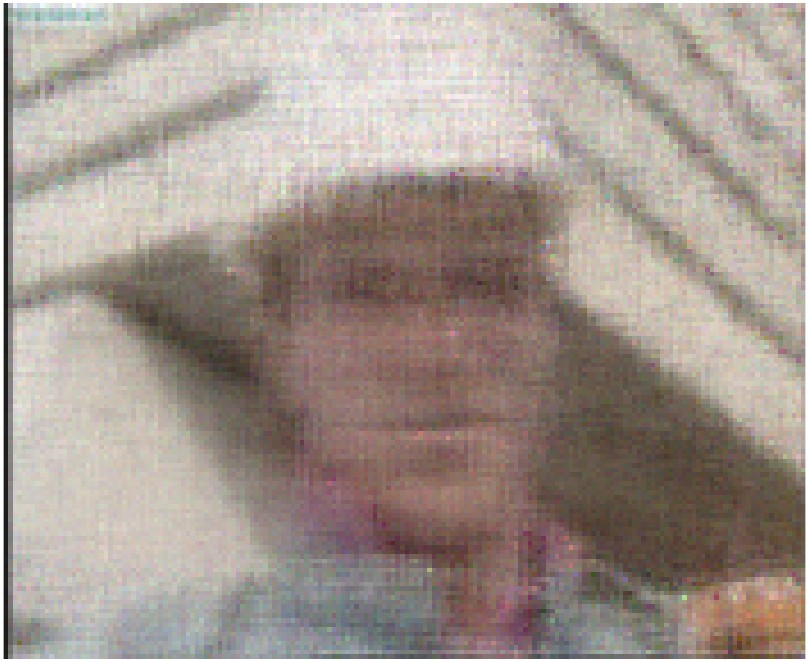}\\
\includegraphics[width=1\linewidth]{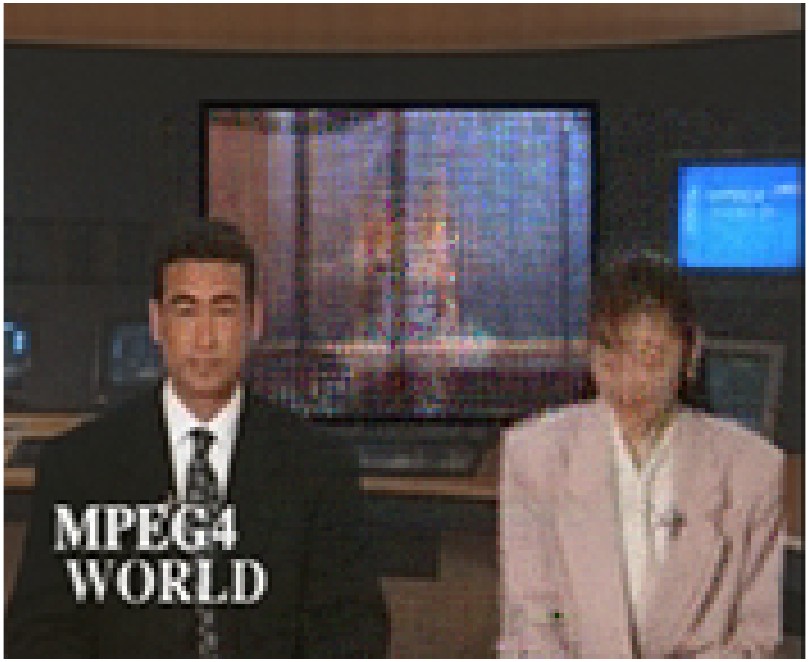}\\
\includegraphics[width=1\linewidth]{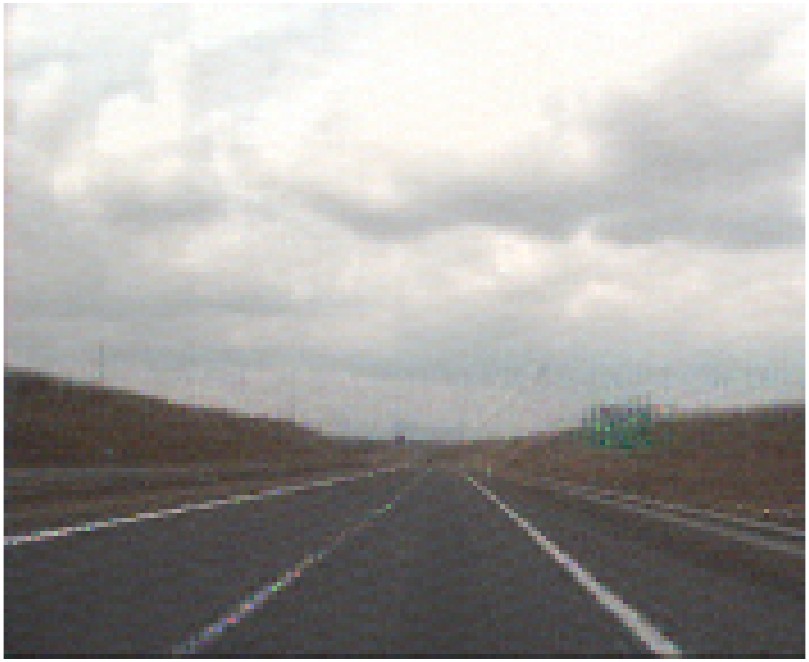}\\
\includegraphics[width=1\linewidth]{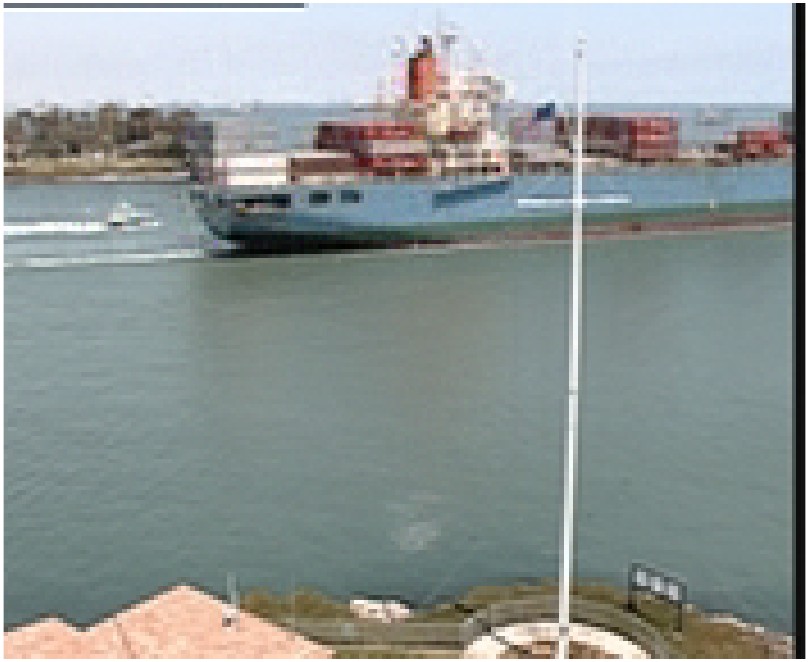} 
\end{minipage}}\subfigure[LRTCTV]{
\begin{minipage}[b]{0.095\linewidth}
\includegraphics[width=1\linewidth]{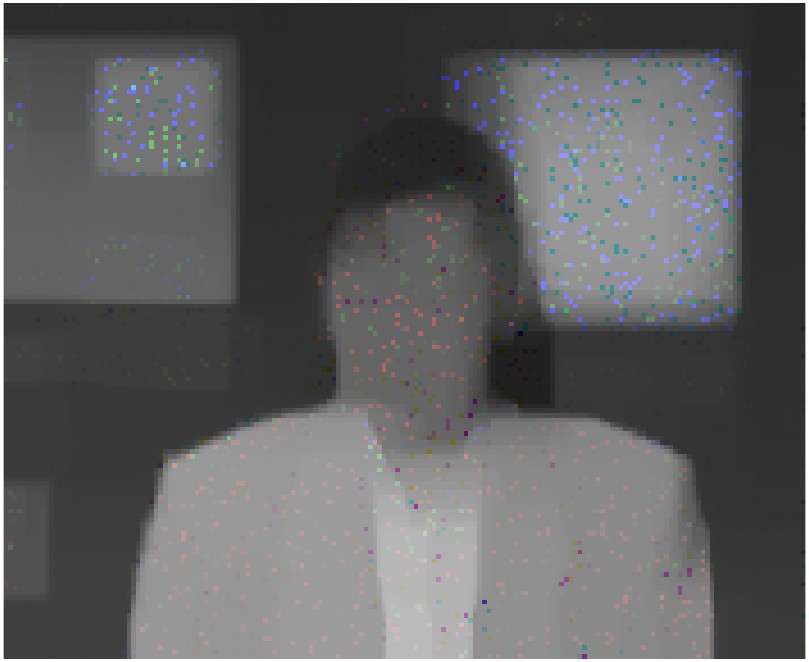}\\
\includegraphics[width=1\linewidth]{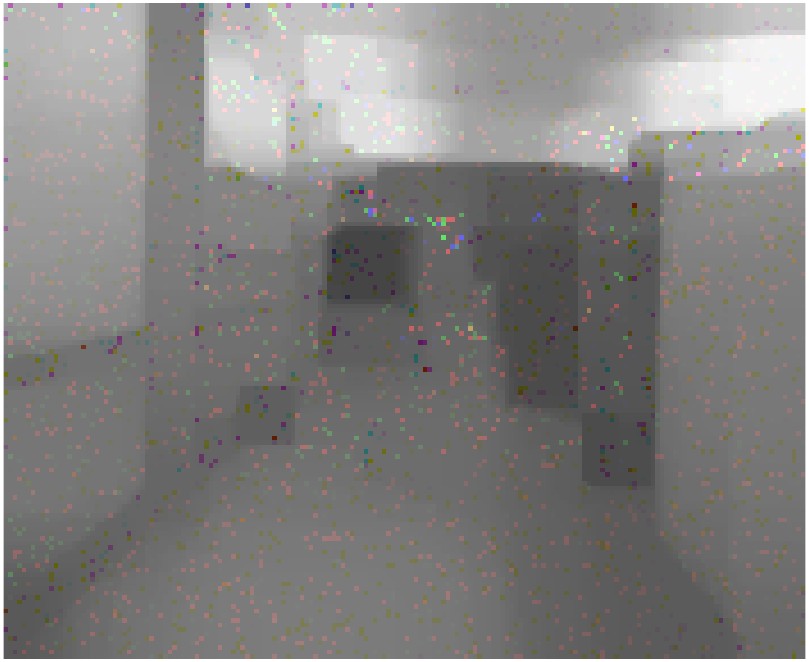}\\
\includegraphics[width=1\linewidth]{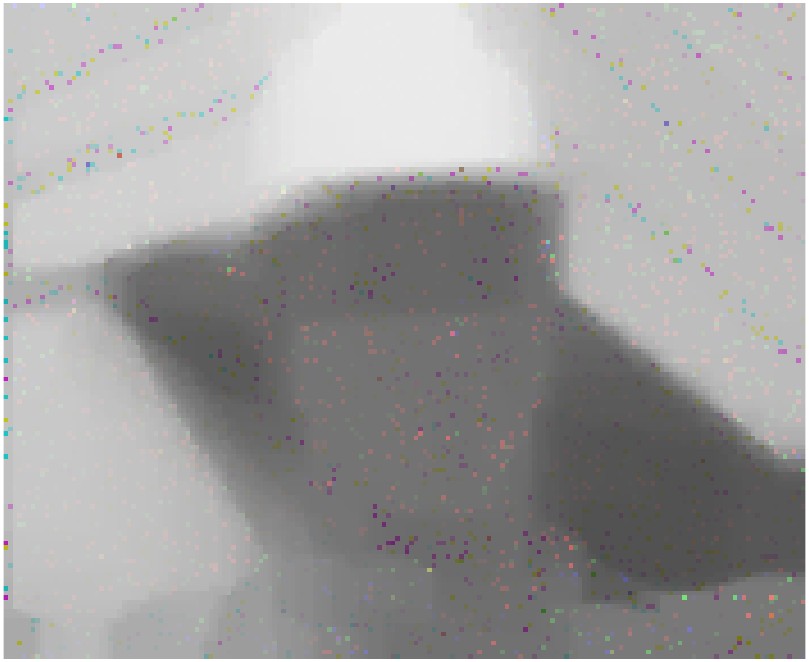}\\
\includegraphics[width=1\linewidth]{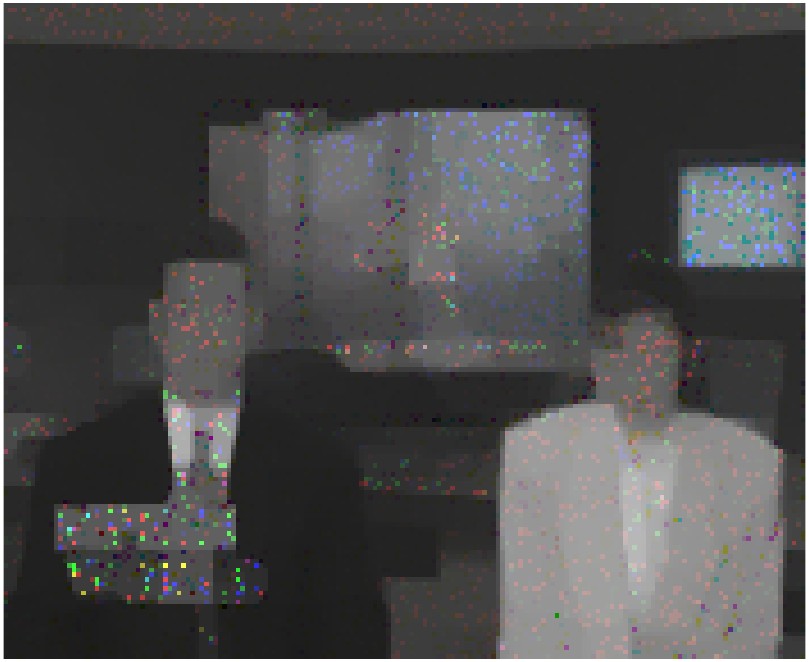}\\
\includegraphics[width=1\linewidth]{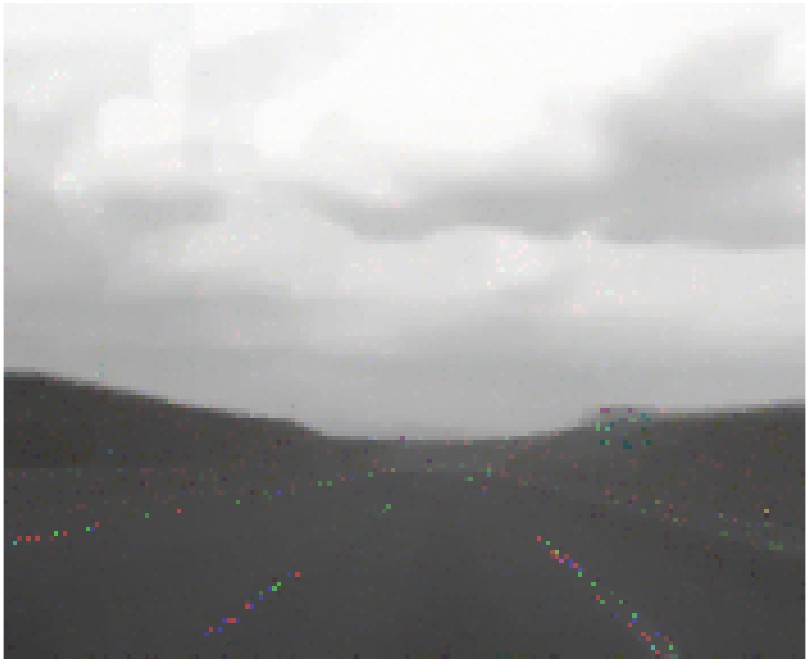}\\
\includegraphics[width=1\linewidth]{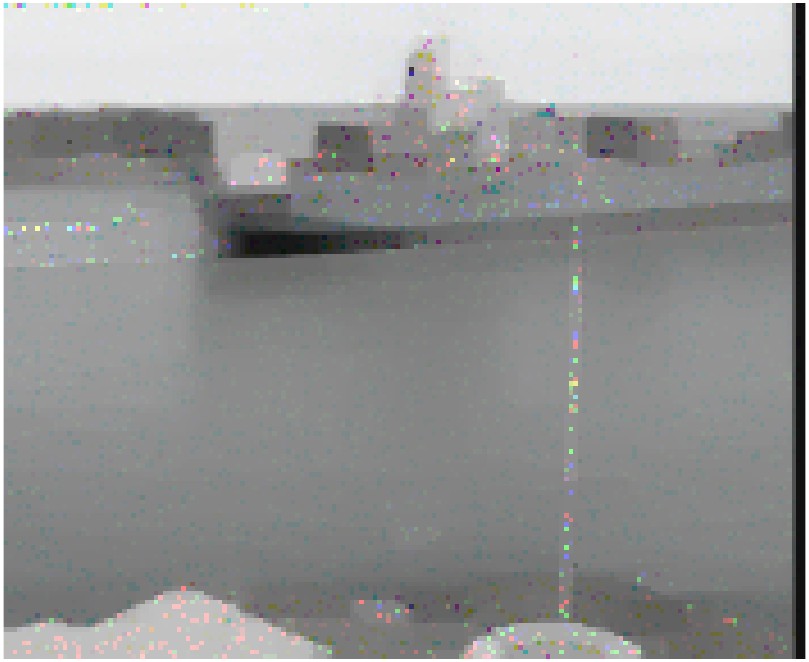} 
\end{minipage}}\subfigure[PSTNN]{
\begin{minipage}[b]{0.095\linewidth}
\includegraphics[width=1\linewidth]{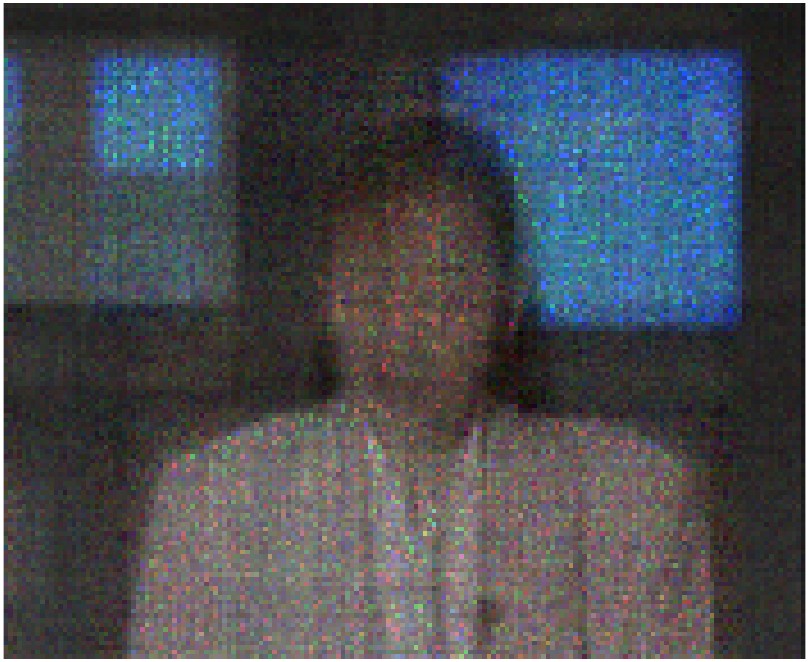}\\
\includegraphics[width=1\linewidth]{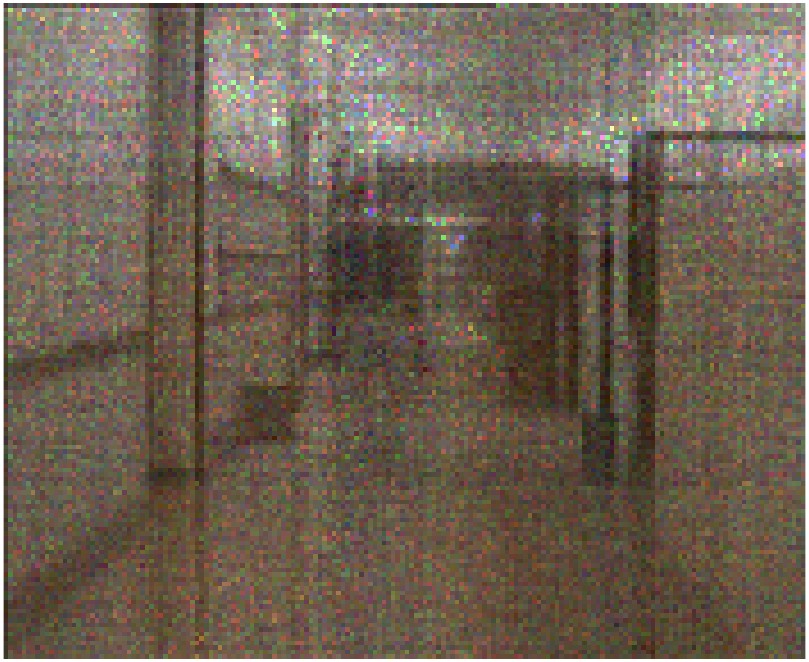}\\
\includegraphics[width=1\linewidth]{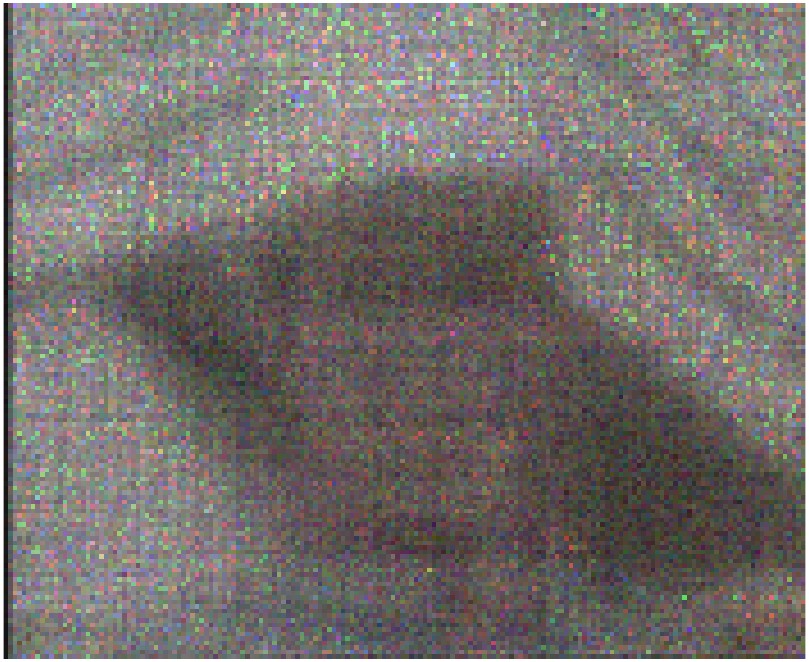}\\
\includegraphics[width=1\linewidth]{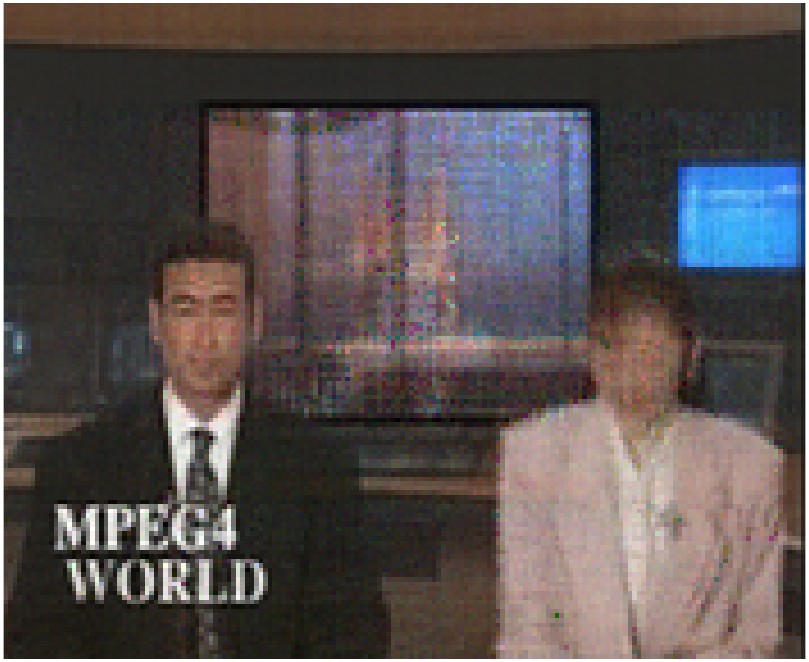}\\
\includegraphics[width=1\linewidth]{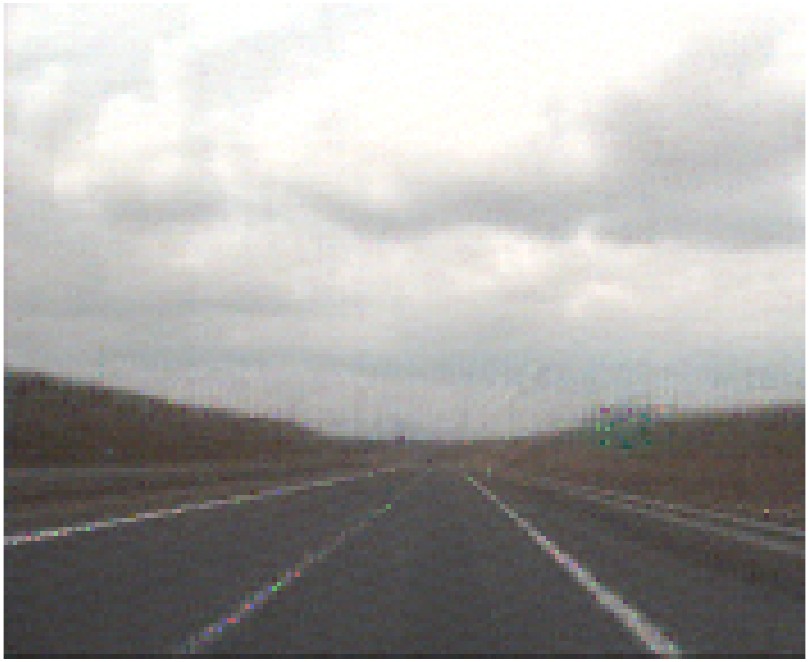}\\
\includegraphics[width=1\linewidth]{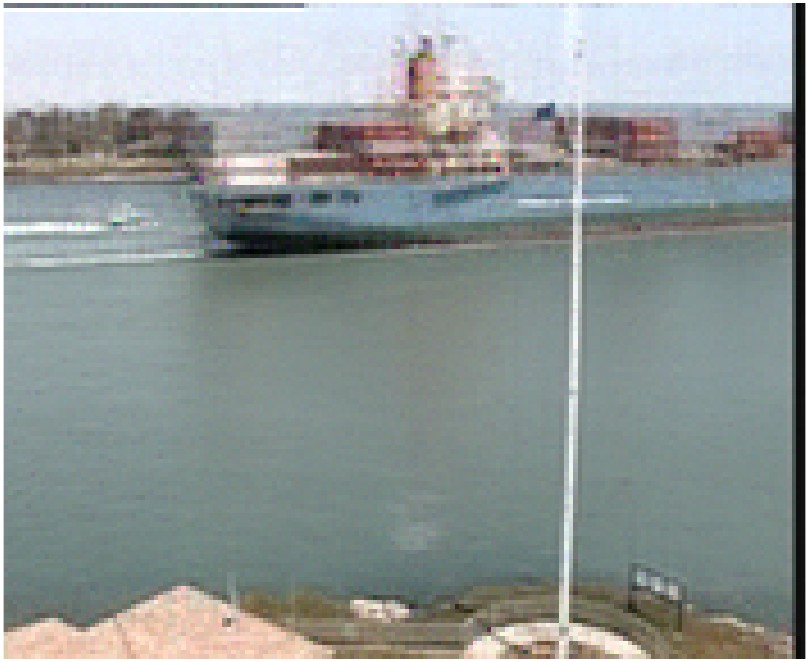} 
\end{minipage}}\subfigure[FTNN]{
\begin{minipage}[b]{0.095\linewidth}
\includegraphics[width=1\linewidth]{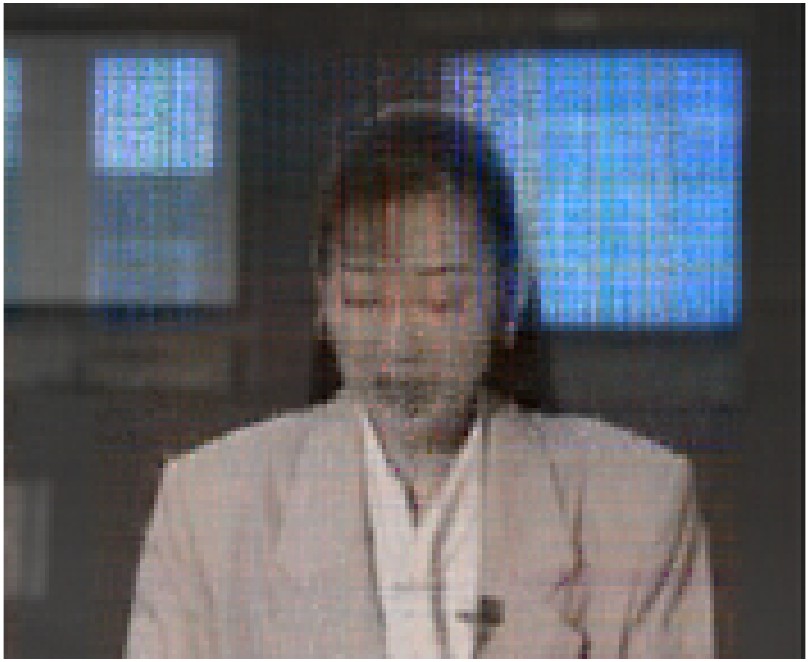}\\
\includegraphics[width=1\linewidth]{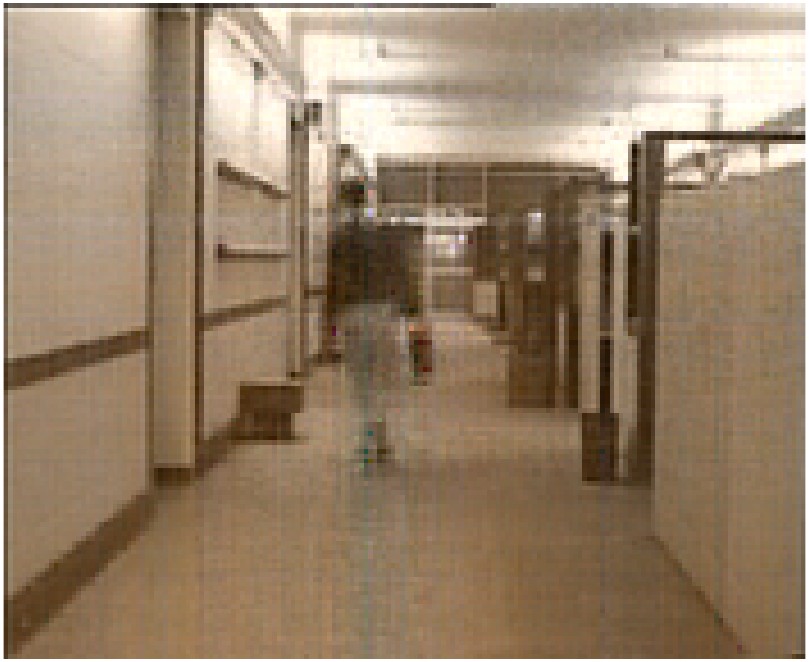}\\
\includegraphics[width=1\linewidth]{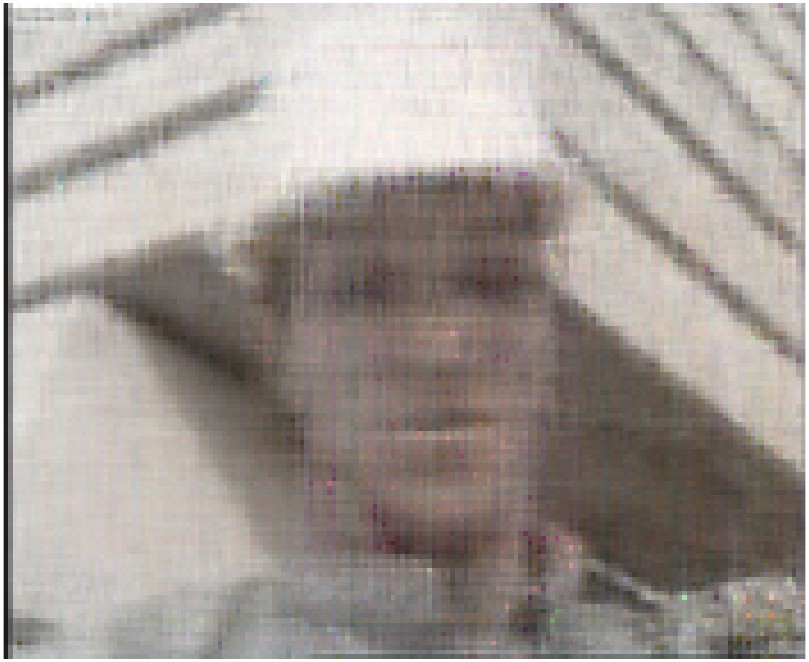}\\
\includegraphics[width=1\linewidth]{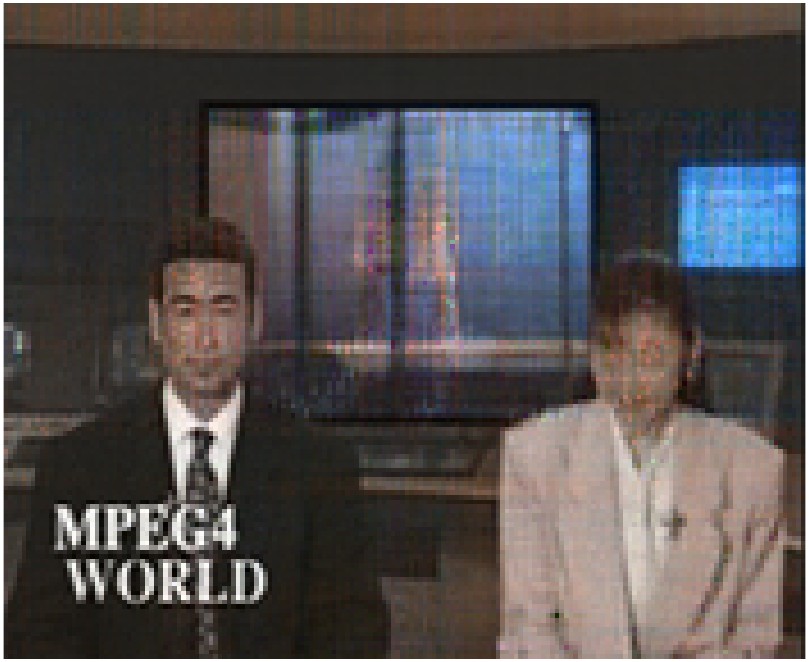}\\
\includegraphics[width=1\linewidth]{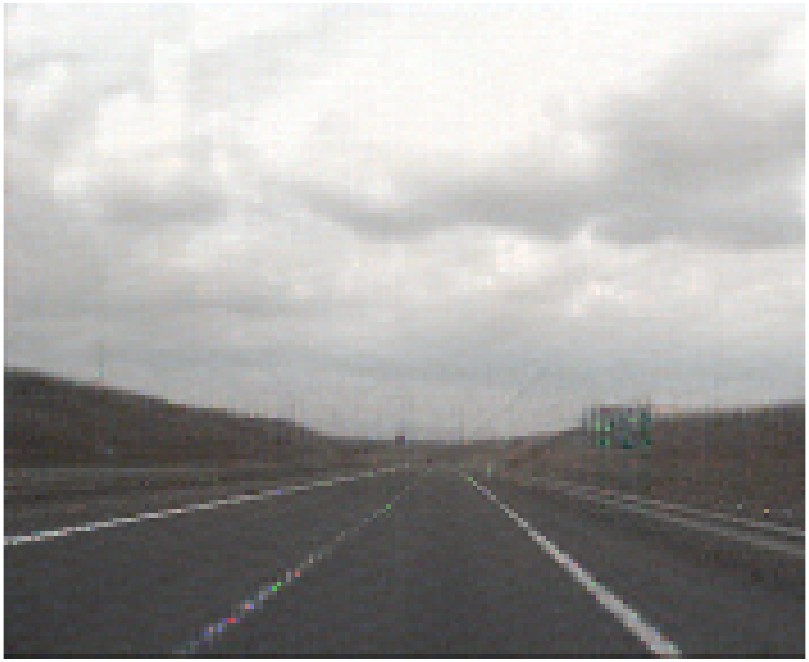}\\
\includegraphics[width=1\linewidth]{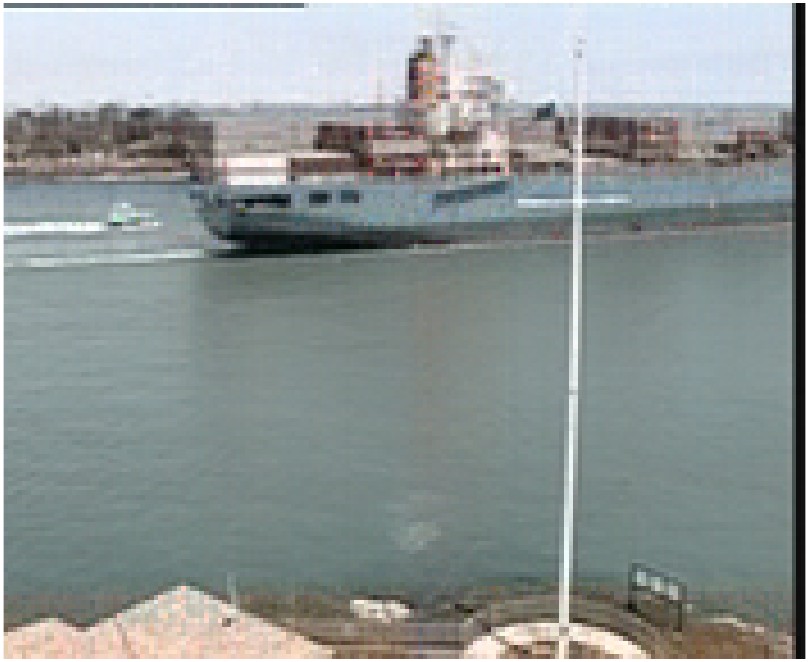} 
\end{minipage}}\subfigure[WSTNN]{
\begin{minipage}[b]{0.095\linewidth}
\includegraphics[width=1\linewidth]{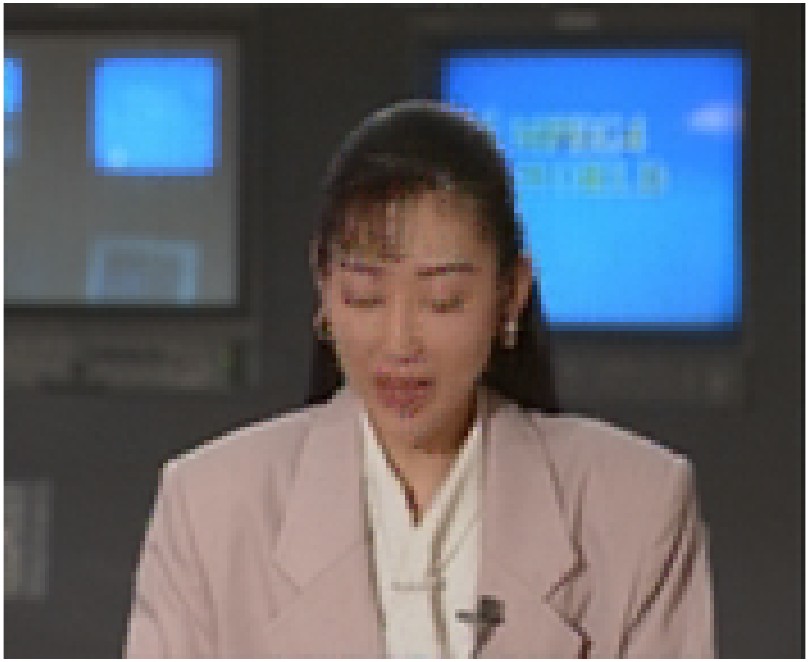}\\
\includegraphics[width=1\linewidth]{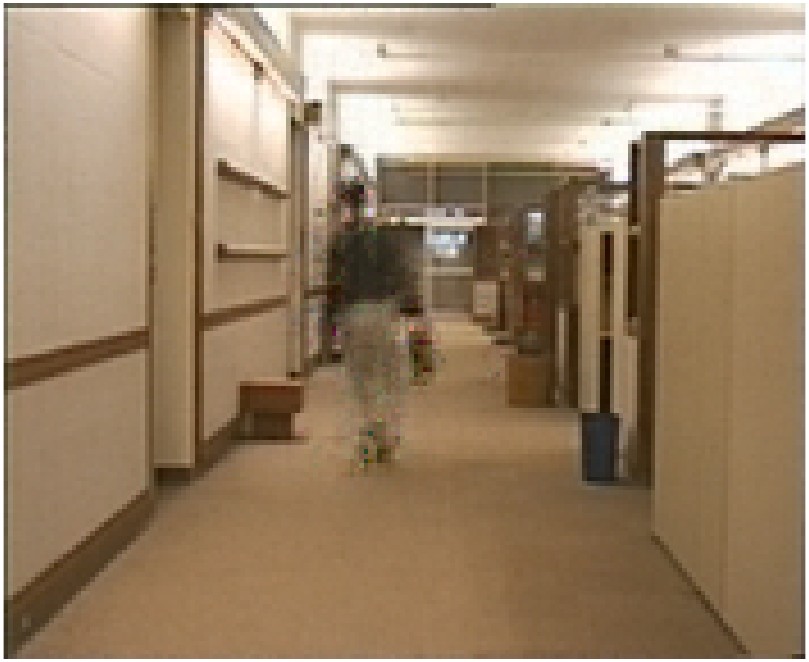}\\
\includegraphics[width=1\linewidth]{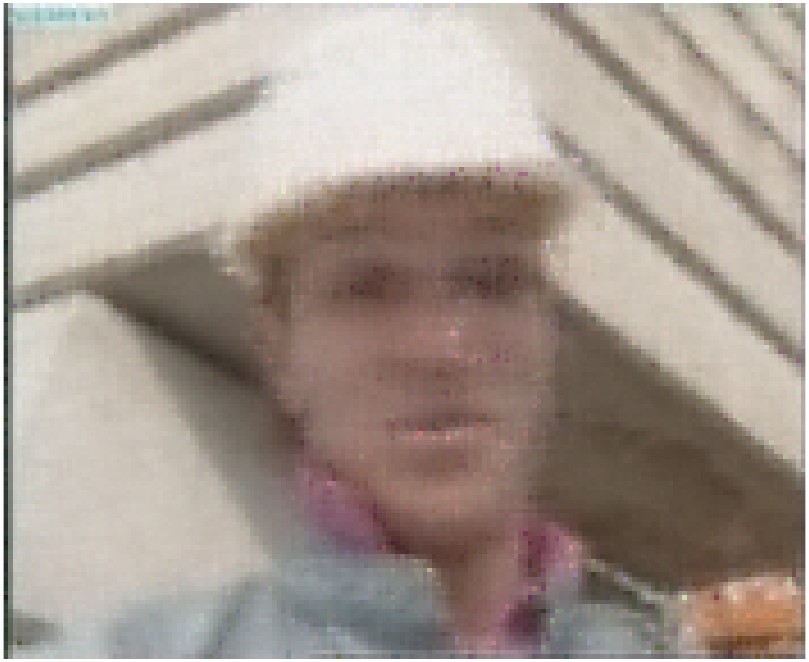}\\
\includegraphics[width=1\linewidth]{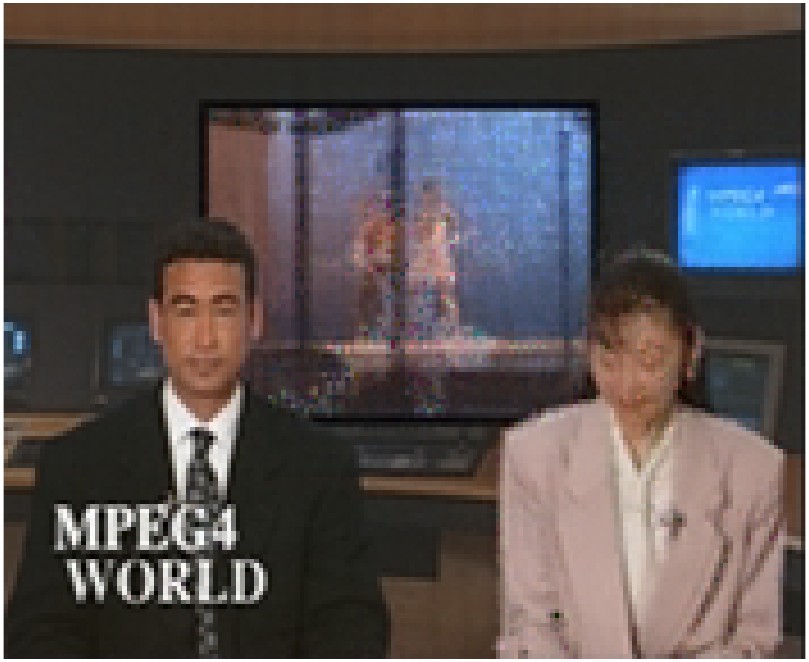}\\
\includegraphics[width=1\linewidth]{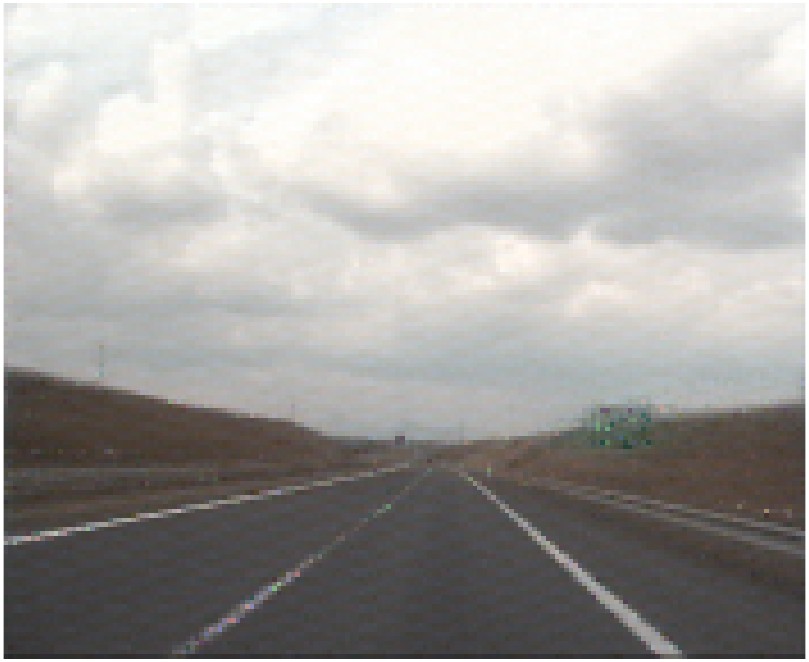}\\
\includegraphics[width=1\linewidth]{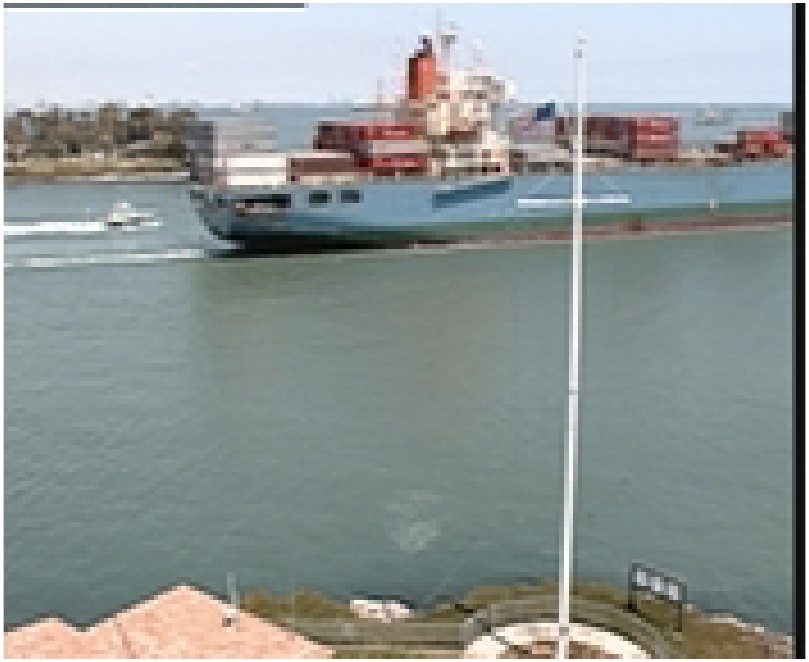} 
\end{minipage}}\subfigure[BEMCP]{
\begin{minipage}[b]{0.095\linewidth}
\includegraphics[width=1\linewidth]{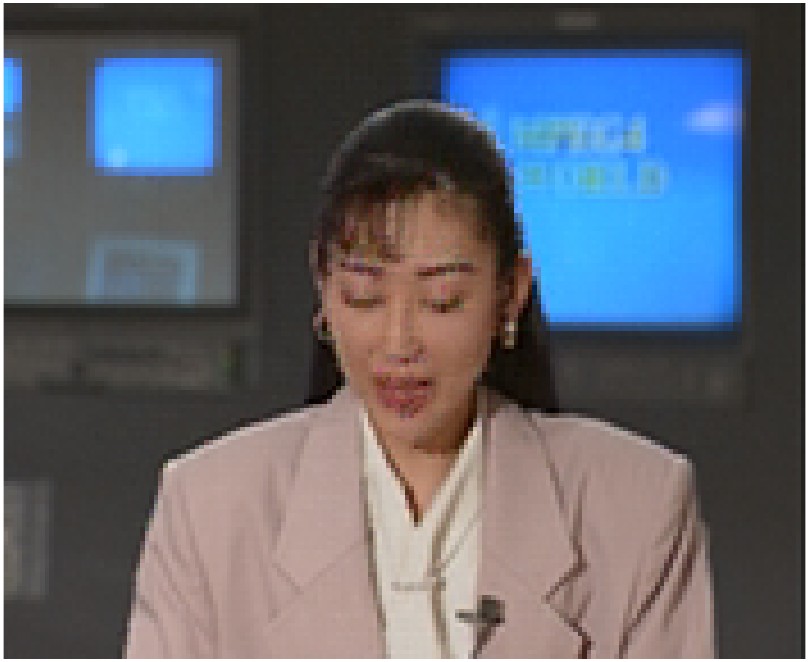}\\
\includegraphics[width=1\linewidth]{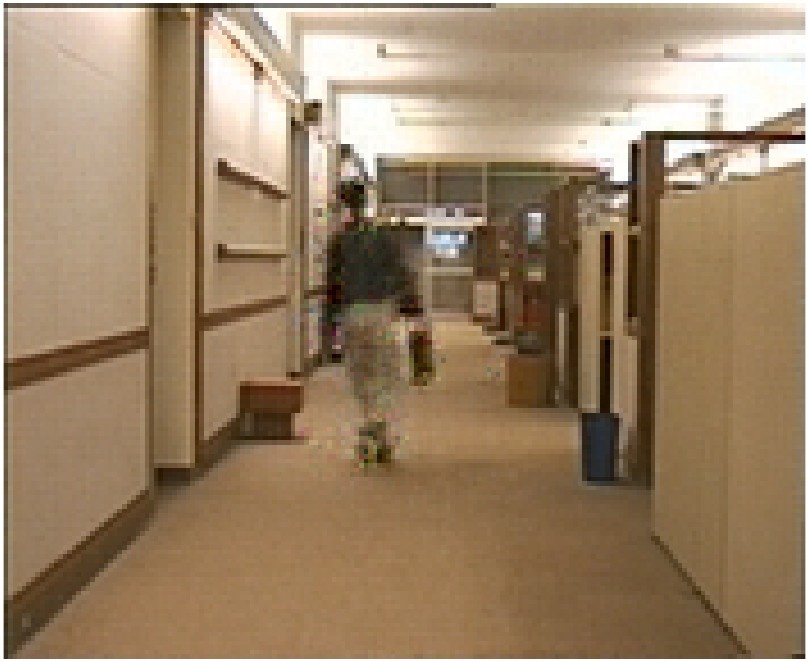}\\
\includegraphics[width=1\linewidth]{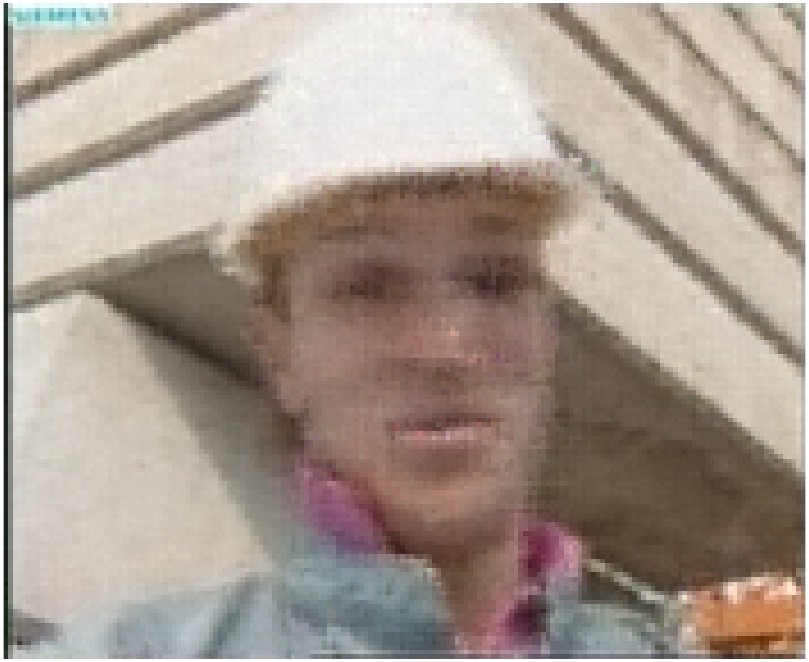}\\
\includegraphics[width=1\linewidth]{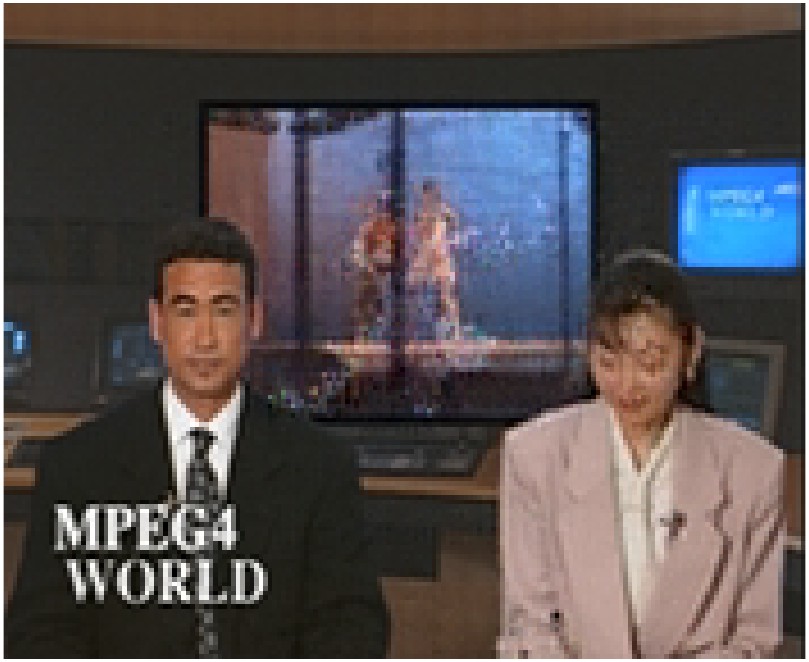}\\
\includegraphics[width=1\linewidth]{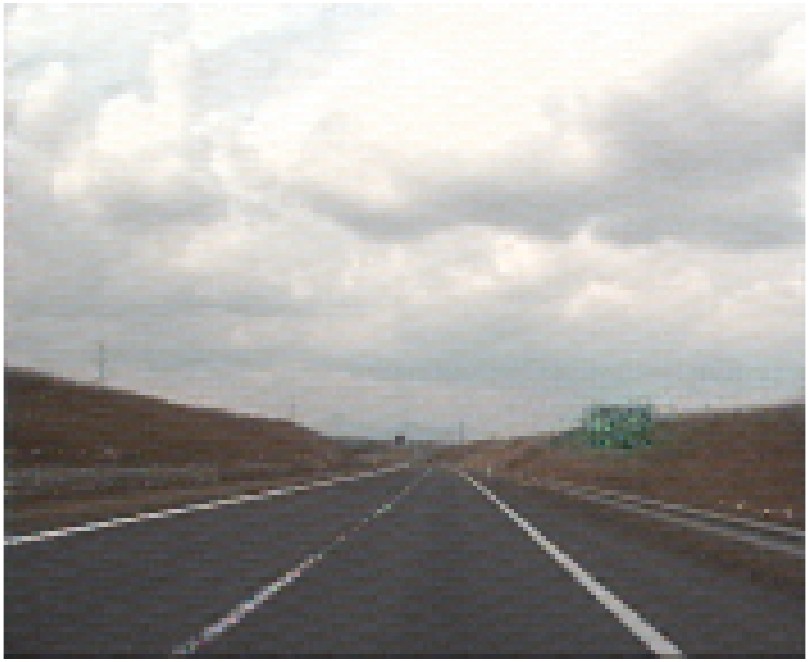}\\
\includegraphics[width=1\linewidth]{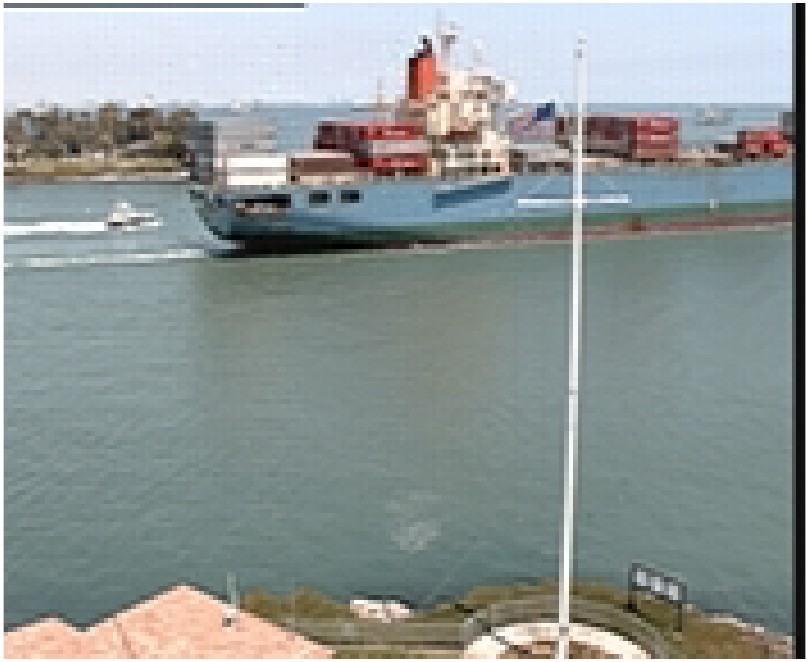} 
\end{minipage}}\subfigure[TJLC]{
\begin{minipage}[b]{0.095\linewidth}
\includegraphics[width=1\linewidth]{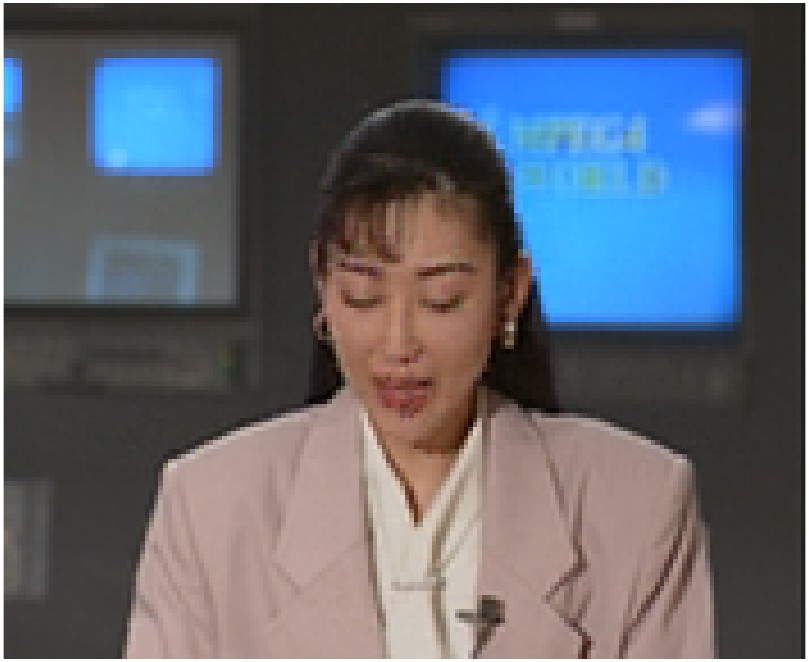}\\
\includegraphics[width=1\linewidth]{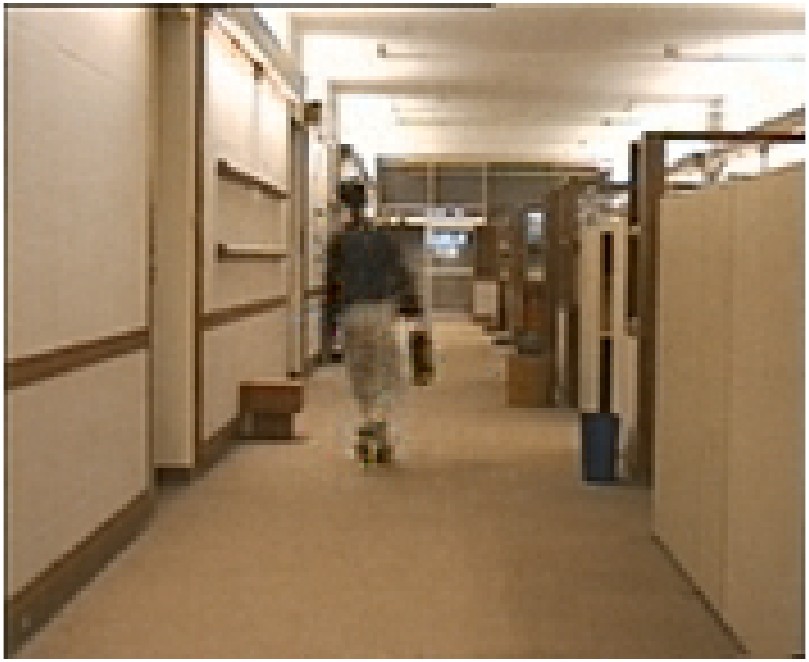}\\
\includegraphics[width=1\linewidth]{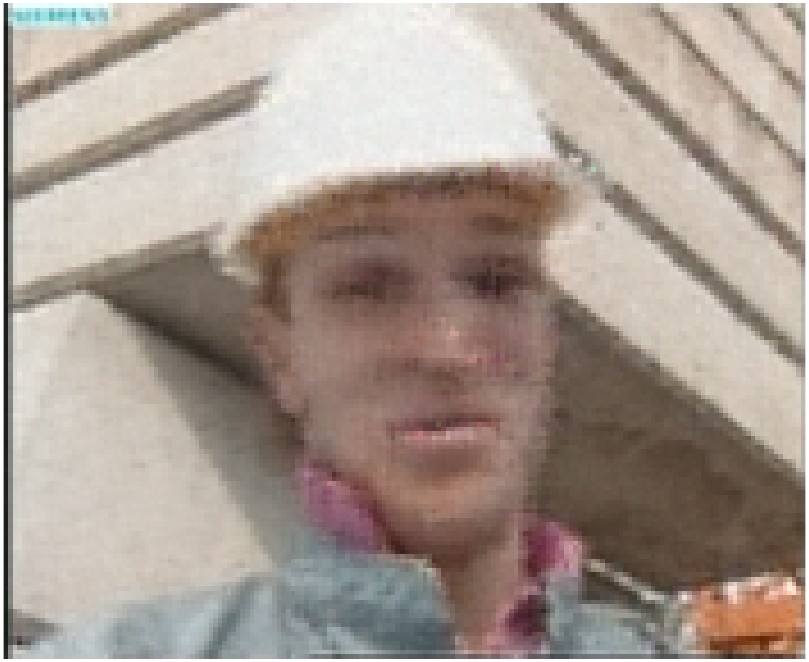}\\
\includegraphics[width=1\linewidth]{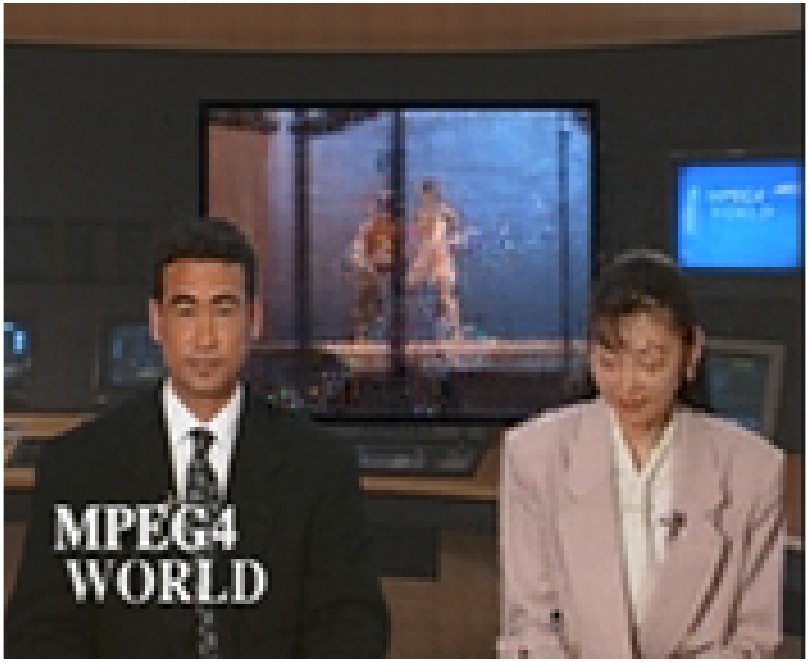}\\
\includegraphics[width=1\linewidth]{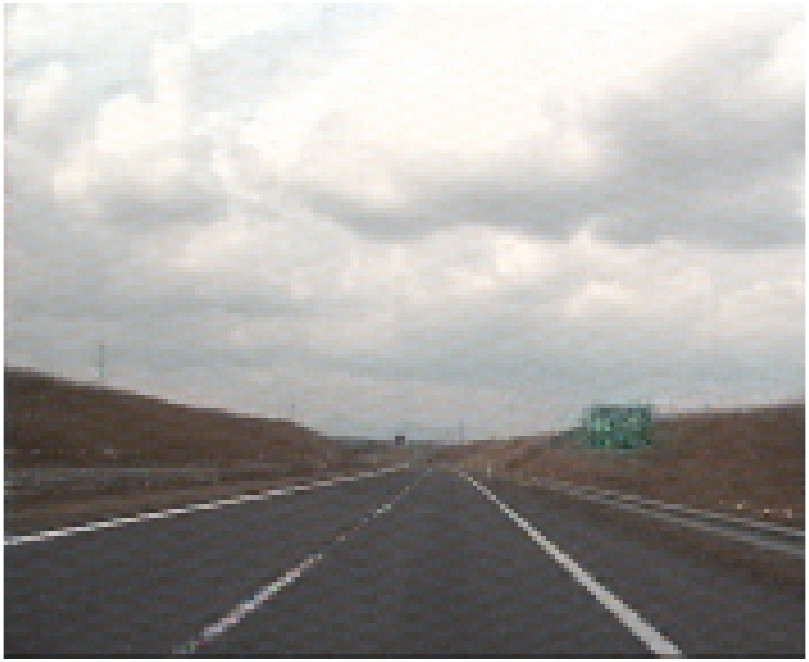}\\
\includegraphics[width=1\linewidth]{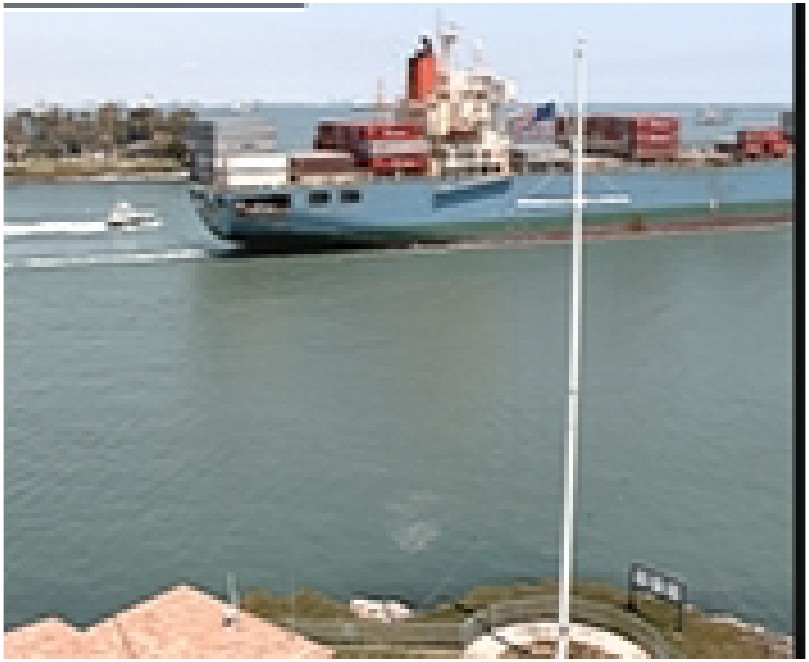} 
\end{minipage}}

	\caption{Visual results for CV data. The rows of CVs are in order: akiyo, hall, foreman, news, highway, container. MR: top three rows are 95\%, and last three rows are 90\%. The corresponding bands in each row are: 5, 15, 20, 40, 30, 45.}

	\label{CVF}
\end{figure}

\subsection{Parameters setting}
The proposed TJLC method includes parameters $\beta$, $\mu$, $\eta$, $\nu$, $\vartheta$, and $\omega$. Firstly, $\beta=(\beta_{11},\cdots,\beta_{NN})$ and $\sum_{1\leqslant l_1\leqslant l_2\leqslant N}\beta_{l_1l_2}=1$. To obtain the desired value for $\beta$ that meets the conditions, we first set each component individually, and then divide it by the sum of the overall components, i.e., $\beta=(\alpha_{11},\alpha_{12},\cdots,\alpha_{NN})/(\sum_{1\leqslant l_1\leqslant l_2\leqslant N}\alpha_{l_1l_2})$.  Furthermore, for convenience, let's $\mu = \beta/\tau$. Table \ref{TCPARA} shows the parameters of parameters $\beta$, $\mu$, $\eta$, $\nu$, and $\vartheta$ on different experiments. Due to the significant differences in MSI data, the optimal parameter settings also differ greatly. Therefore, we specifically list the parameters for the four MSIs. Besides, for MRI and MSI datas, the weights are set to $\mathcal{Z}_{l_1l_2(l_1l_2)}\in\mathbb{R}^{\mathit{I}_{L1}\times\mathit{I}_{L2}\times\mathit{I}_{L3}}$: $\omega_{(l1l2)j,i}=1/(c+e^{-w_{(l1l2)R-j+1,i}})$, where $c=0.8, R=\min\{\mathit{I}_{L1},\mathit{I}_{L2}\}$, $w_{(l1l2)R-j+1,i}=(R\times\sigma_{j}(\bar{\mathcal{W}}_{l1l2}^{(i)}))/(m_{i})$, $\sigma_{j}(\bar{\mathcal{W}}_{l1l2}^{(i)})$ is the $(j,j,i)$-th singular value of $\bar{\mathcal{W}}_{l1l2}$, $\mathcal{W}_{l1l2}=\mathcal{X}_{(l1l2)}+\mathcal{Q}_{l1l2(l1l2)}/\mu_{l1l2}$ and $m_{i}=\max\{\sigma_{j}(\bar{\mathcal{W}}_{l1l2}^{(i)}), j=1,2,\dots,R\}$.
For CV data, the weights are set to $\mathcal{Z}_{l_1l_2(l_1l_2)}\in\mathbb{R}^{\mathit{I}_{L1}\times\mathit{I}_{L2}\times\mathit{I}_{L3}}$: $\omega_{j,i}=1/(c+e^{-w_{R-j+1,i}})$, where $c=0.8, R=\min\{\mathit{I}_{L1},\mathit{I}_{L2}\}$, $w_{R-j+1,i}=\sigma_{j}(\bar{\mathcal{W}}_{l1l2}^{(i)})$, $\sigma_{j}(\bar{\mathcal{W}}_{l1l2}^{(i)})$ is the $(j,j,i)$-th singular value of $\bar{\mathcal{W}}_{l1l2}$, $\mathcal{W}_{l1l2}=\mathcal{X}_{(l1l2)}+\mathcal{Q}_{l1l2(l1l2)}/\mu_{l1l2}$.
\begin{table}[!h]
	\centering
	\caption{Parameters under different experiments of the proposed TJLC method.}
	\label{TCPARA}
{\footnotesize
	\begin{tabular}{|c|ccccc|}
			\hline
			Data                     & $\alpha  $     & $\tau$ & $\eta$ & $\nu$ & $\vartheta$ \\ \hline
			MRI                      & (1,1,1,1,1,1)               & 10000               & 1.1                 & 1                  & 500                      \\ \hline
			clay                     & (0.01,0.001,1,0.1,1,0.001)  & 10000               & 1.1                 & 2.5                & 500                      \\
			chart\_and\_stuffed\_toy & (0.1,0.001,1,0.1,1,0.001)   & 10000               & 1.1                 & 1                  & 500                      \\
			balloons                 & (0.1,0.001,1,0.1,1,0.01)    & 10000               & 1.1                 & 2.5                & 500                      \\
			cd                       & (0.1,0.01,1,0.1,1,0.01)     & 10000               & 1.1                 & 0.5                & 500                      \\ \hline
			CV                       & (0.1,1,1,1,0.1,1,1,1,1,0.1) & 100000              & 1.1                 & 0.1                & 1000                     \\ \hline
		\end{tabular}}%
\end{table}

\subsection{Convergence behavior}
In order to observe the convergence behaviors of the TJLC algorithm, the relative change (RE) of two successive recovered tensors is defined: $RE:=\|\mathcal{X}^{k+1}-\mathcal{X}^{k}\|^{2}_{F}/\|\mathcal{X}^{k}\|^{2}_{F}$. Fig. \ref{convergence} shows the iterative RE of various data, i.e. MRI, MSI, and CV, by the proposed TJLC method. Here, MSI1 and MSI2 refer to the two multispectral images `` balloons " and `` cd " respectively. CV1 and CV2 refer to the two color videos `` news " and `` container " respectively. As the iteration progresses, the RE can become smaller, which guarantees the convergence of the proposed algorithms. The decrease in the RE further verifies the theory.
\begin{figure}[!h]
	\centering
	\subfigure[MR=95\%]{
		\includegraphics[width=0.5\linewidth]{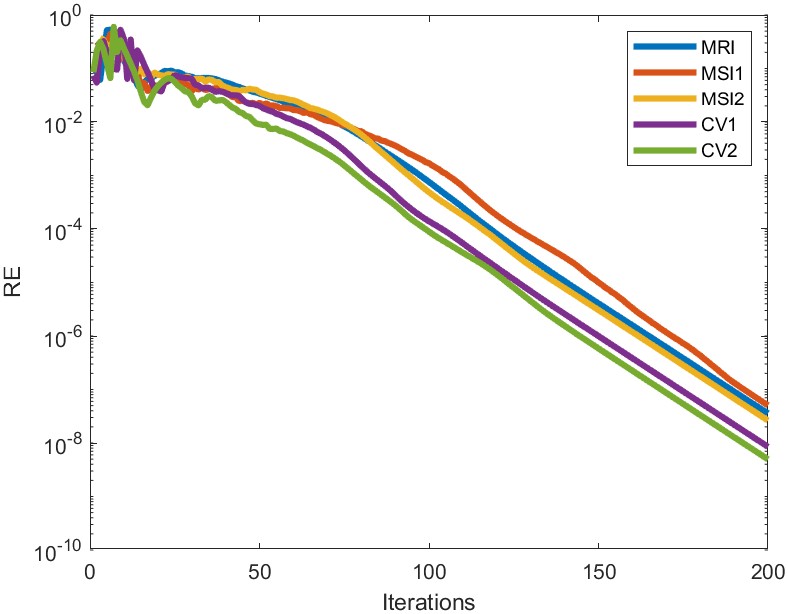}
	}\subfigure[MR=90\%]{
		\includegraphics[width=0.5\linewidth]{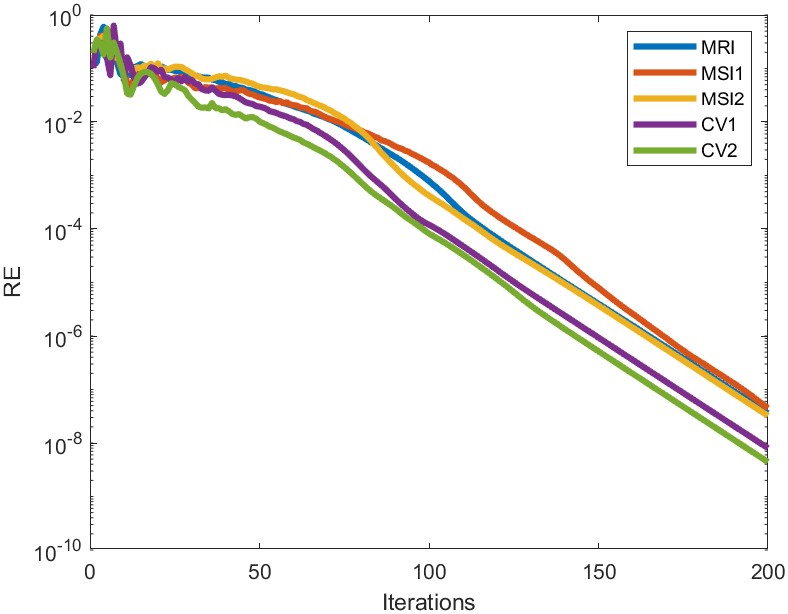}
	}
	\caption{The convergence behaviours of TJLC algorithm for MRI, MSI and CV datas.}
	\label{convergence}
\end{figure}
\subsection{Discussion}

\begin{figure}[!h] 

	\vspace{0cm} 
	\subfigtopskip=2pt 
	\subfigbottomskip=2pt 
	\subfigure[Original]{
		\begin{minipage}[b]{0.105\linewidth}
			\includegraphics[width=1\linewidth]{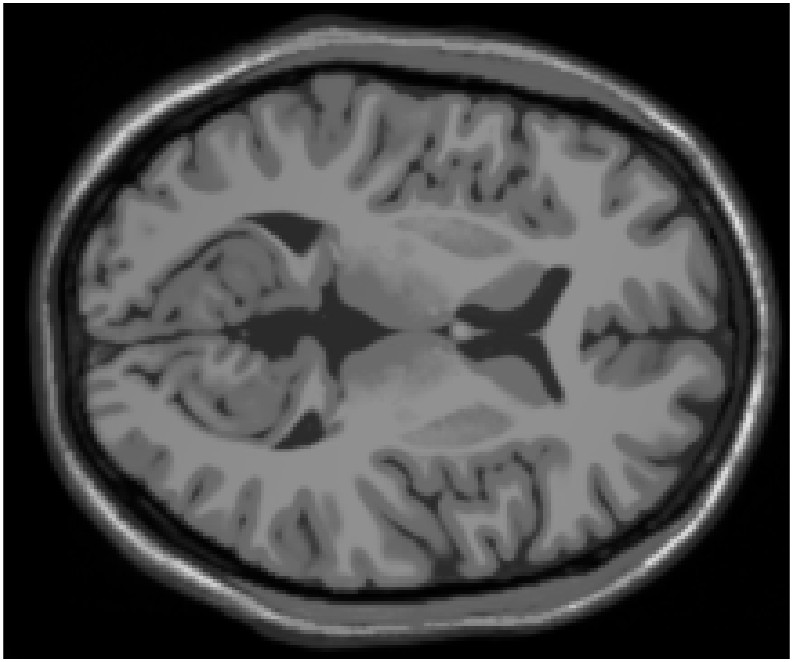}\\
			\includegraphics[width=1\linewidth]{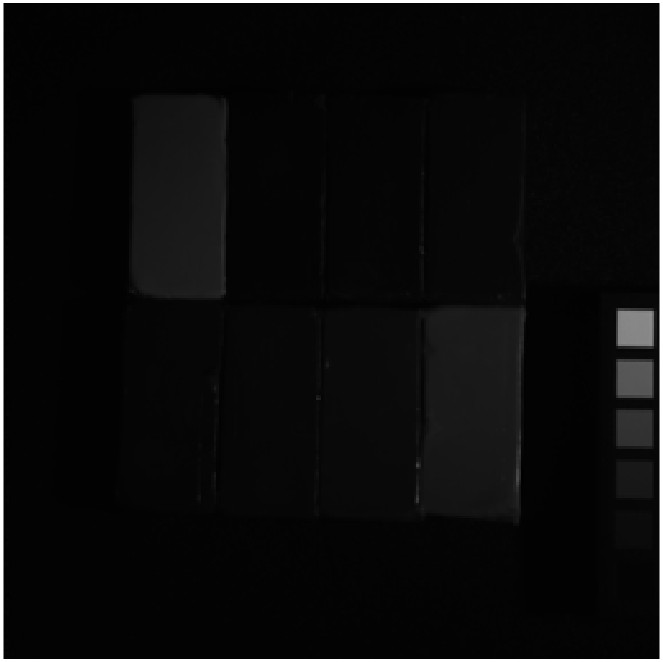}\\
			\includegraphics[width=1\linewidth]{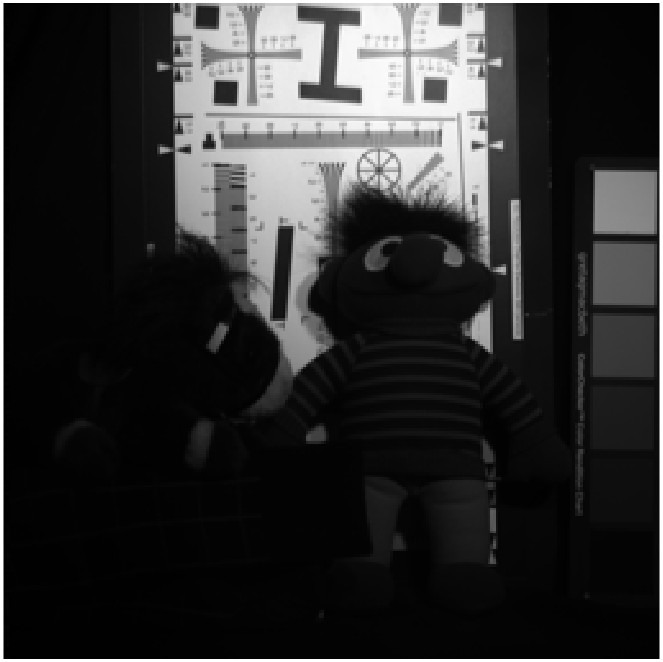}\\
			\includegraphics[width=1\linewidth]{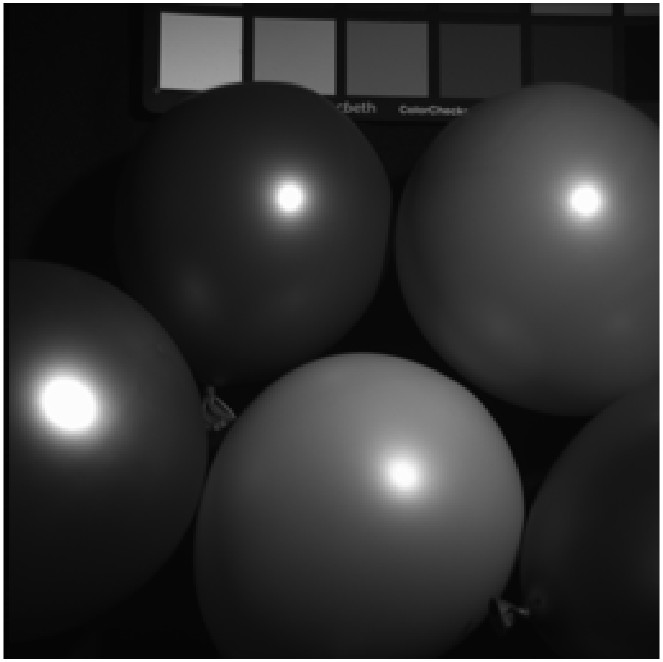}\\
			\includegraphics[width=1\linewidth]{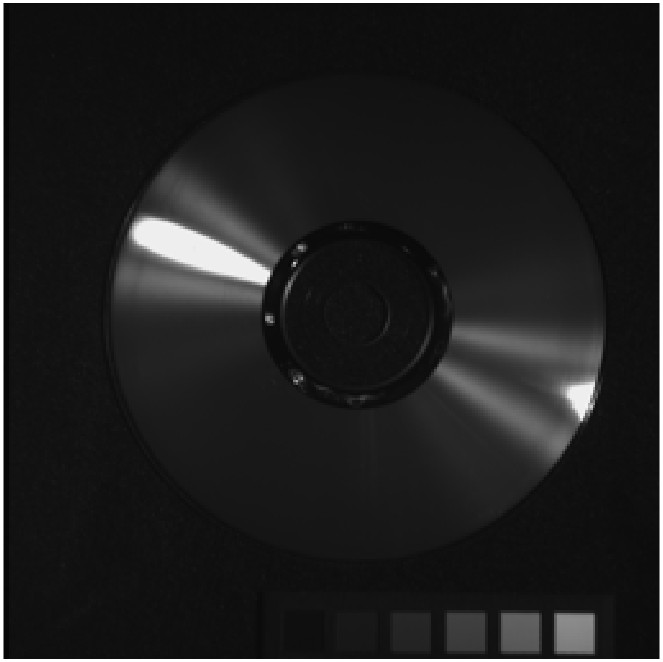}\\
			\includegraphics[width=1\linewidth]{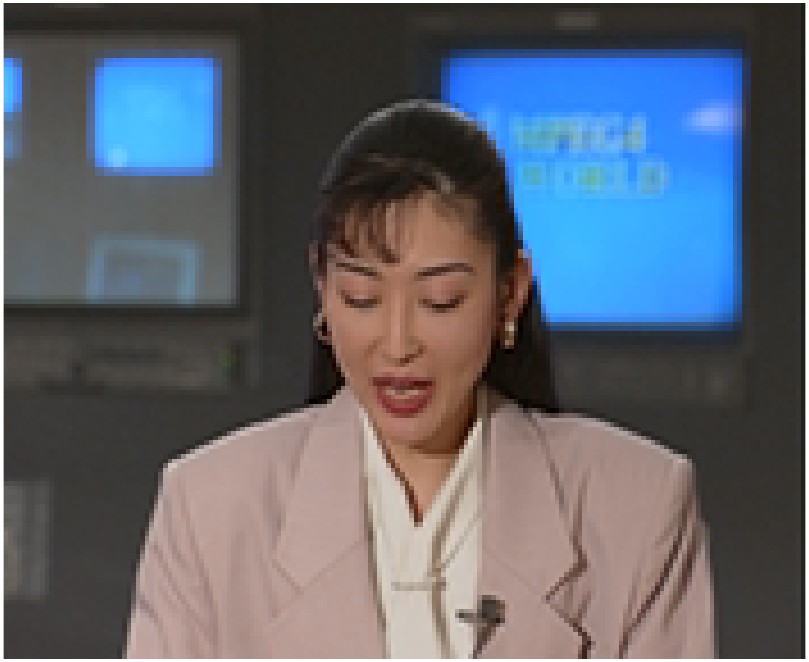}\\
			\includegraphics[width=1\linewidth]{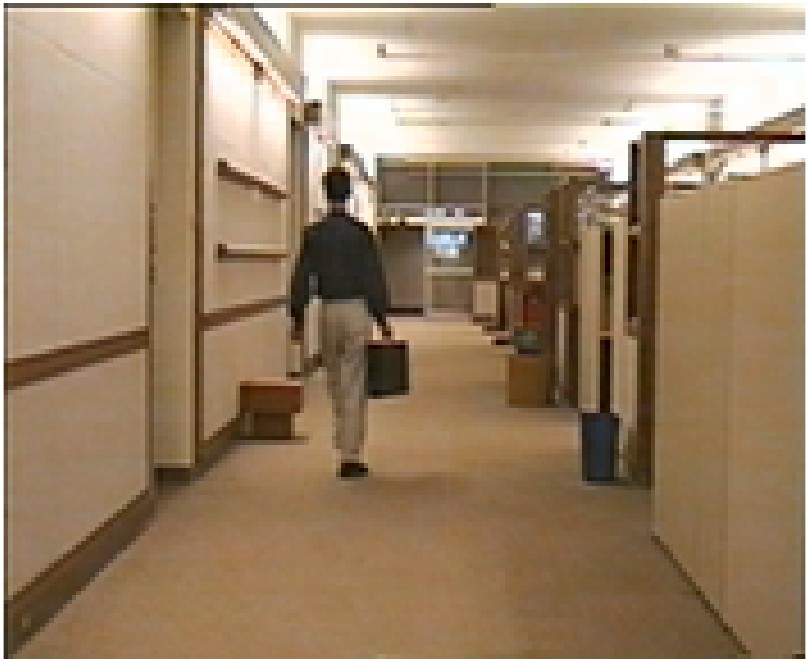}\\
			\includegraphics[width=1\linewidth]{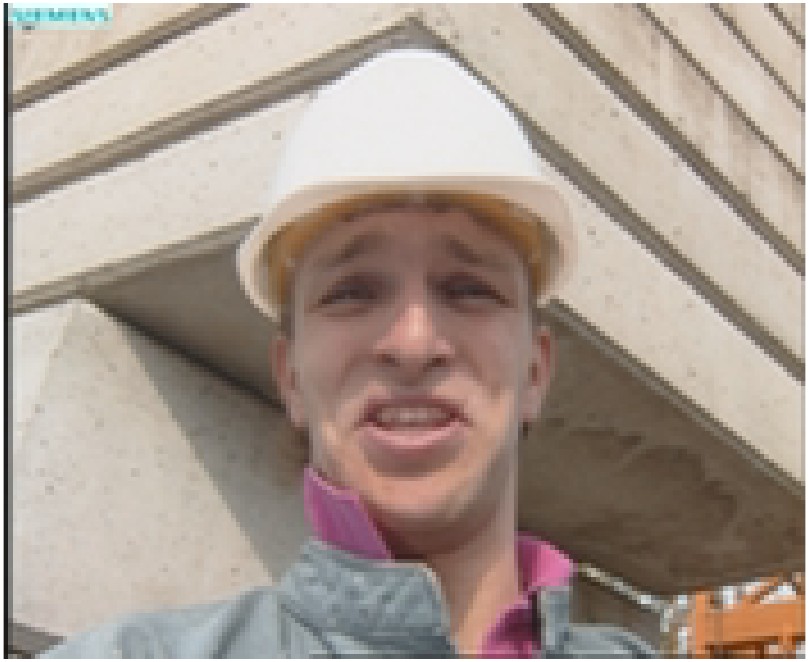}\\
			\includegraphics[width=1\linewidth]{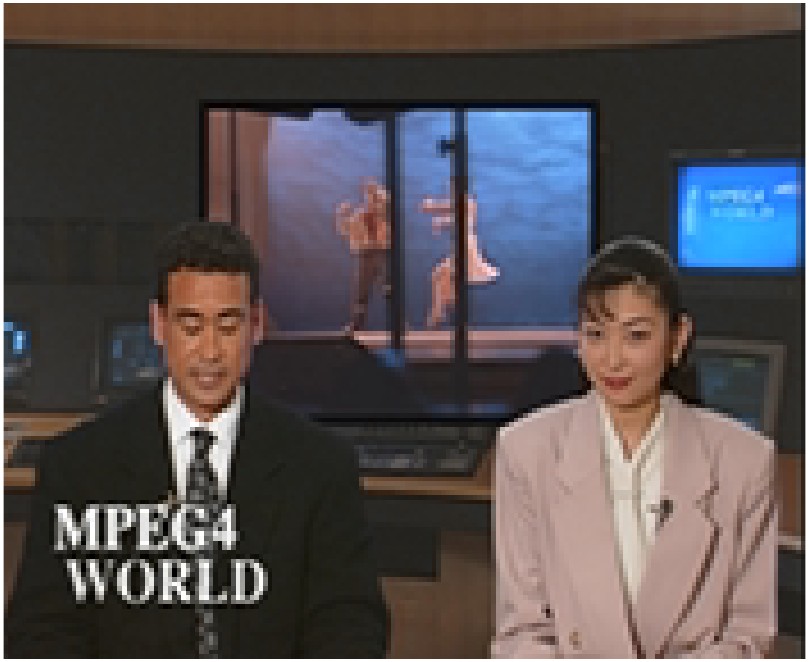}\\
			\includegraphics[width=1\linewidth]{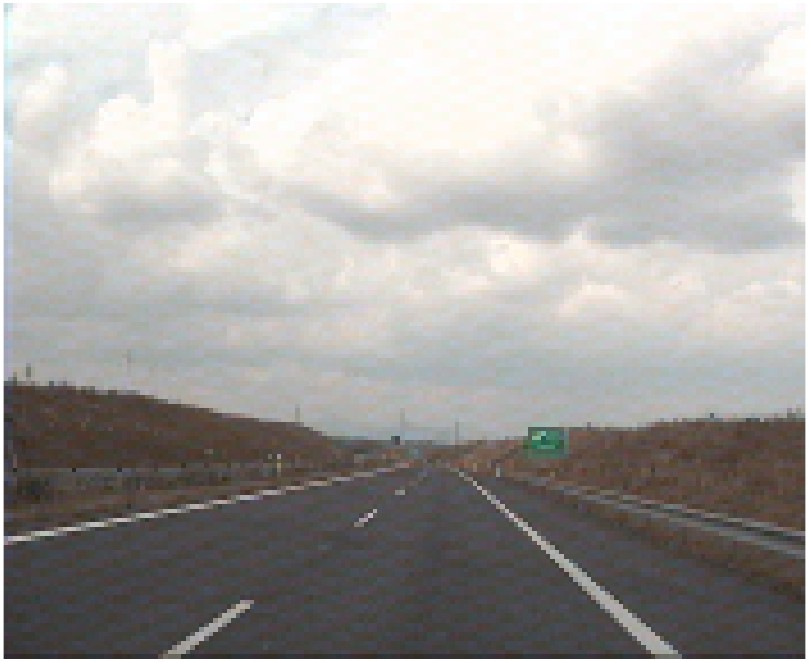}\\
			\includegraphics[width=1\linewidth]{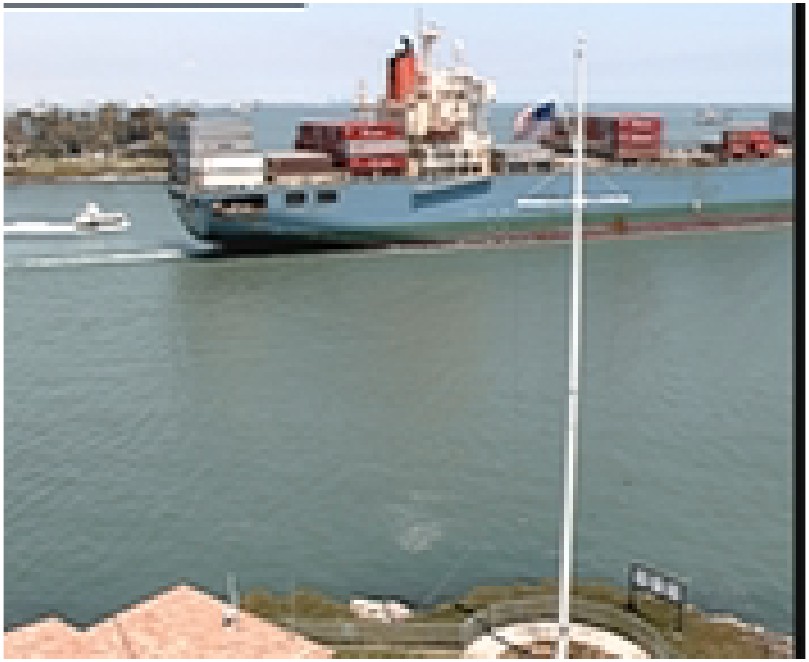} 
	\end{minipage}}\subfigure[MR=96\%]{
		\begin{minipage}[b]{0.105\linewidth}
			\includegraphics[width=1\linewidth]{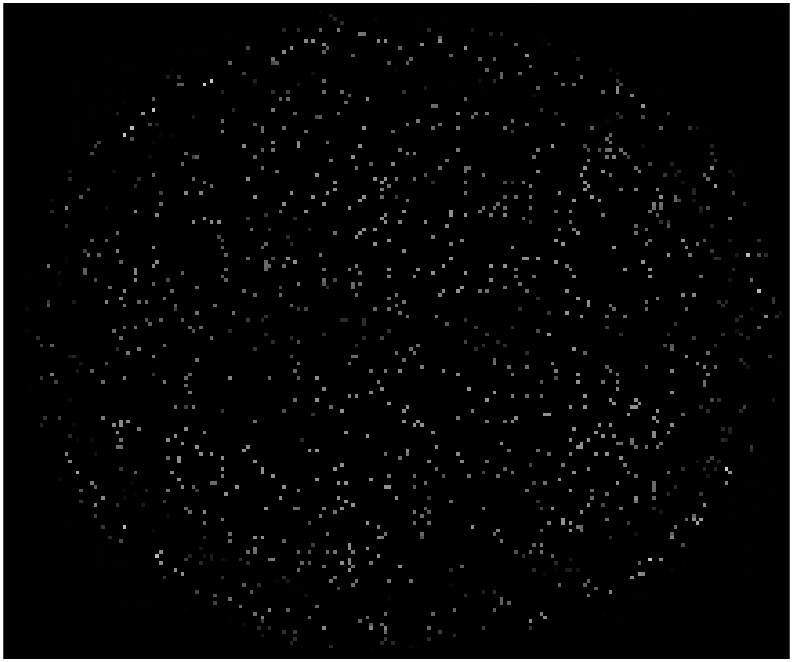}\\
			\includegraphics[width=1\linewidth]{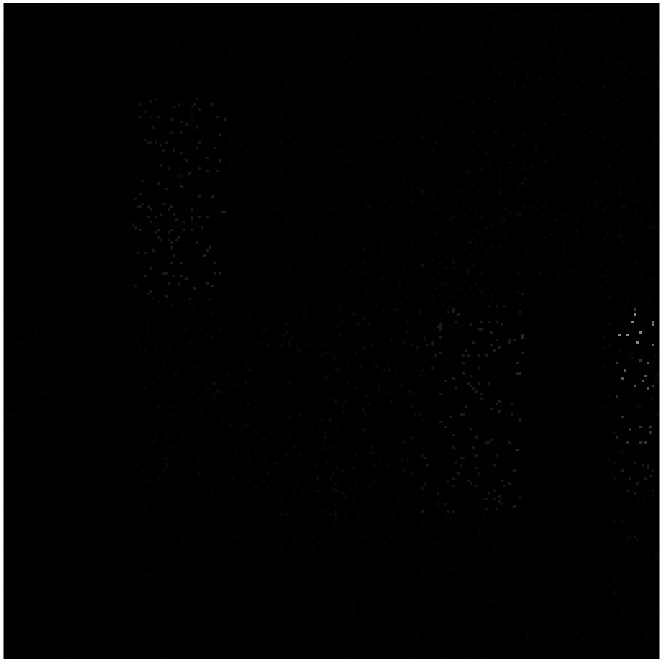}\\
			\includegraphics[width=1\linewidth]{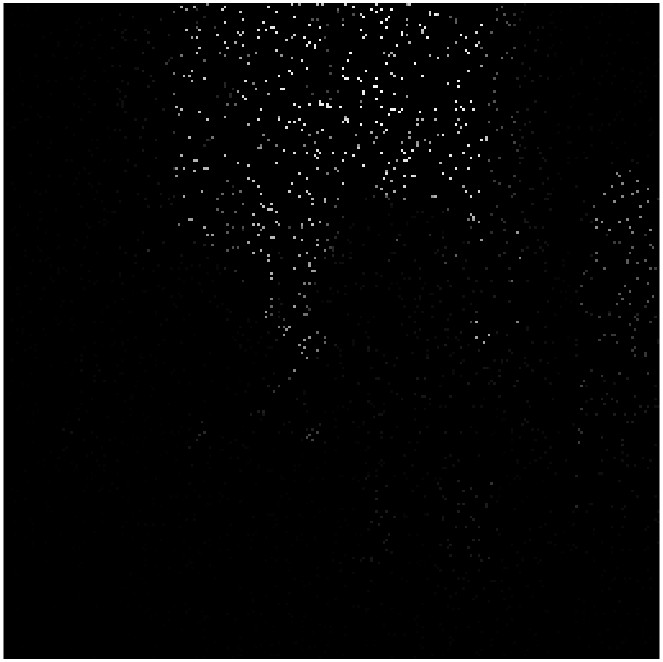}\\
			\includegraphics[width=1\linewidth]{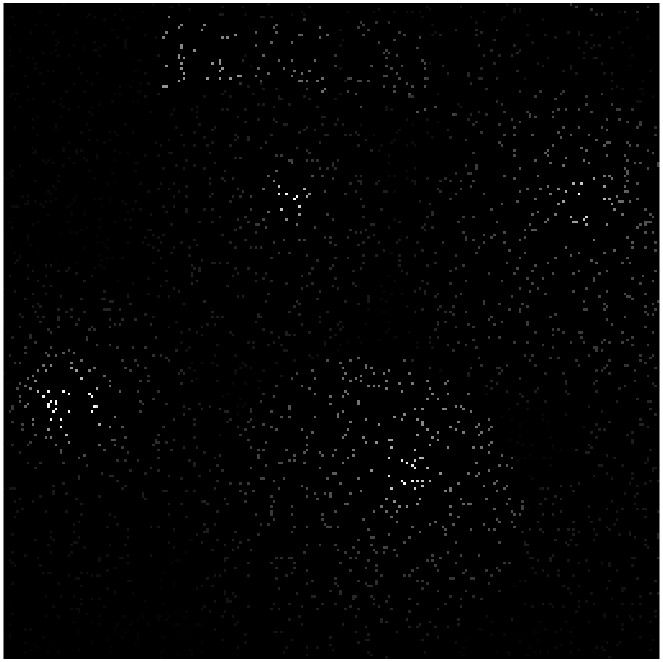}\\
			\includegraphics[width=1\linewidth]{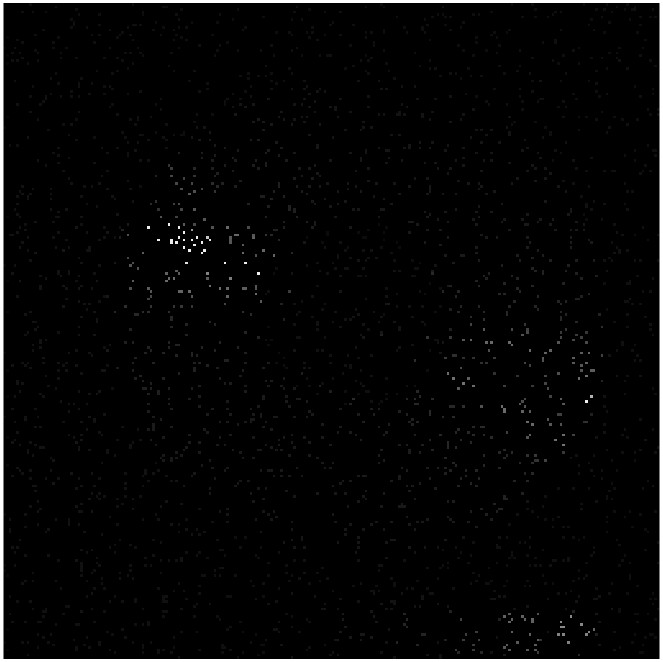}\\
			\includegraphics[width=1\linewidth]{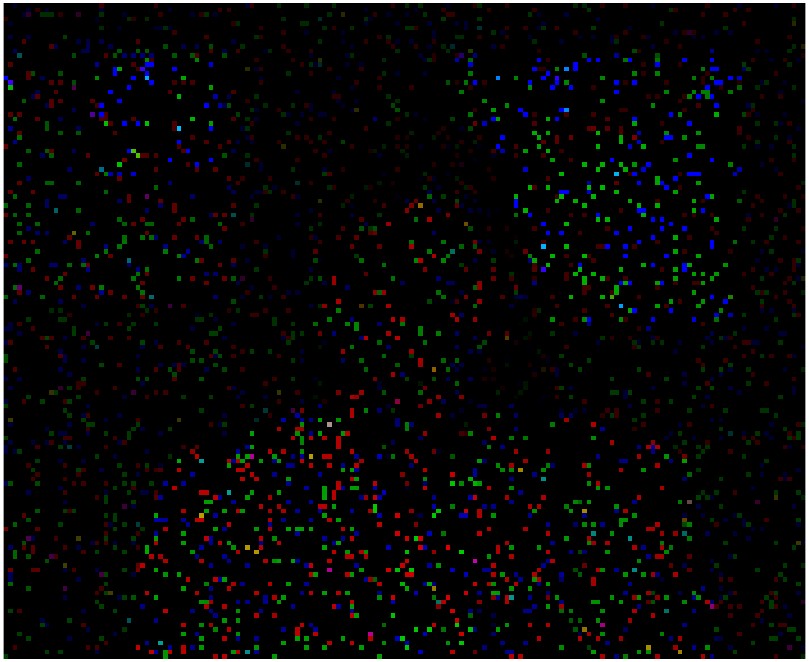}\\
			\includegraphics[width=1\linewidth]{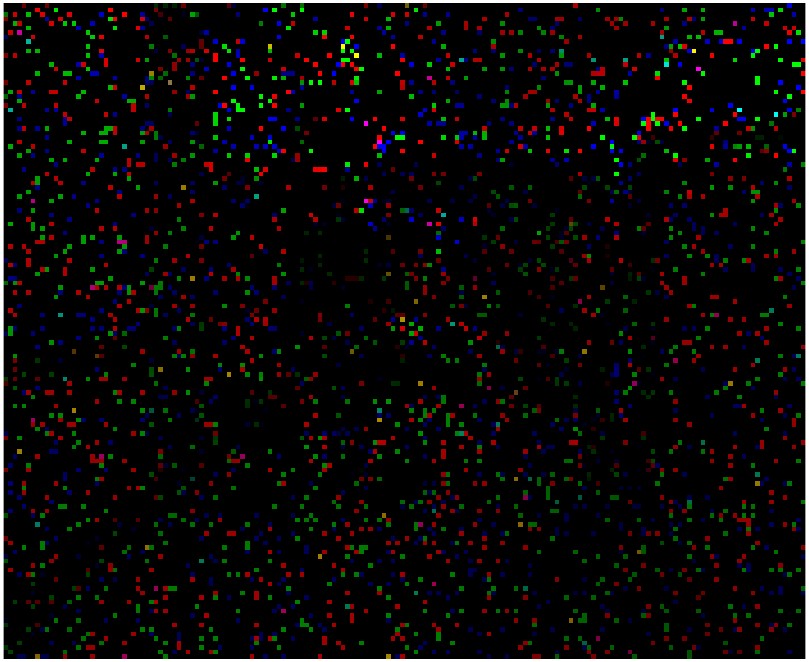}\\
			\includegraphics[width=1\linewidth]{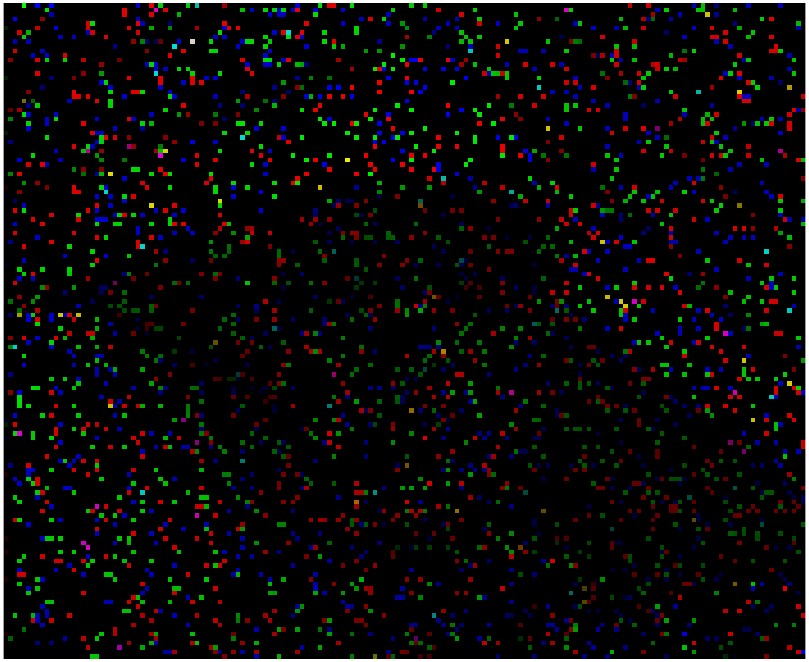}\\
			\includegraphics[width=1\linewidth]{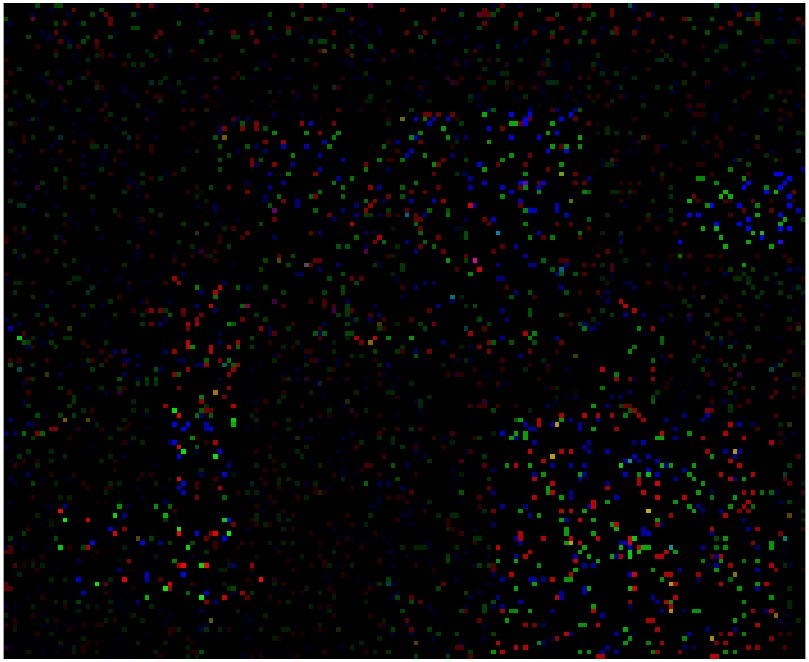}\\
			\includegraphics[width=1\linewidth]{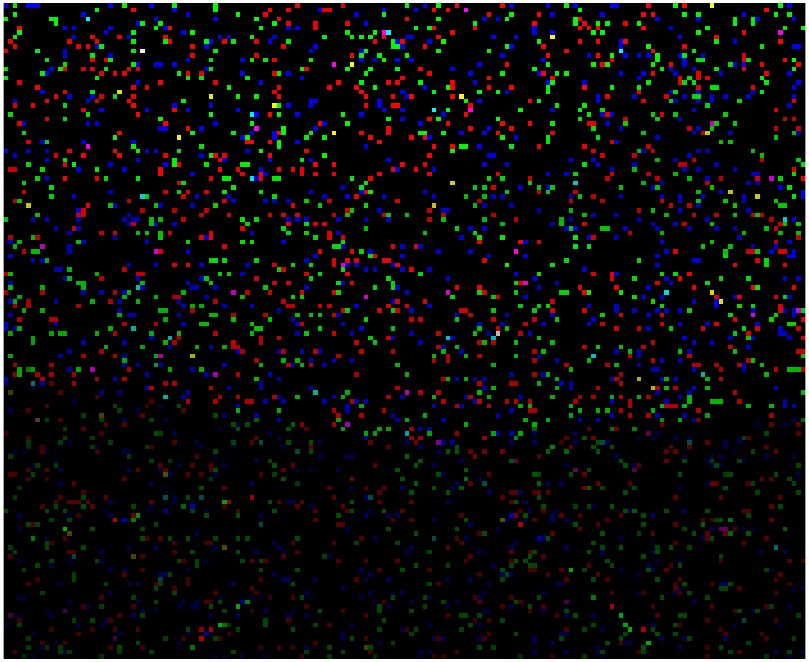}\\
			\includegraphics[width=1\linewidth]{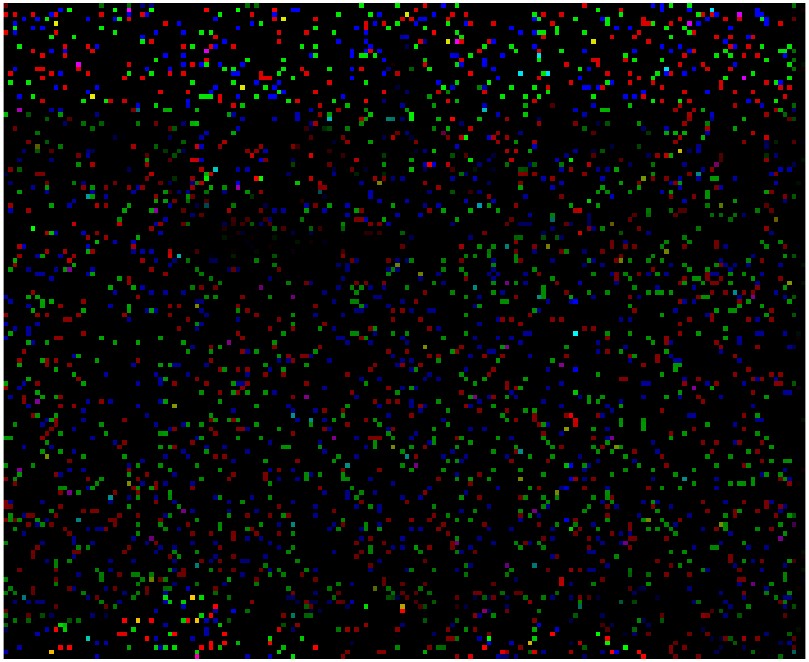}  
	\end{minipage}}\subfigure[TJLC]{
		\begin{minipage}[b]{0.105\linewidth}
			\includegraphics[width=1\linewidth]{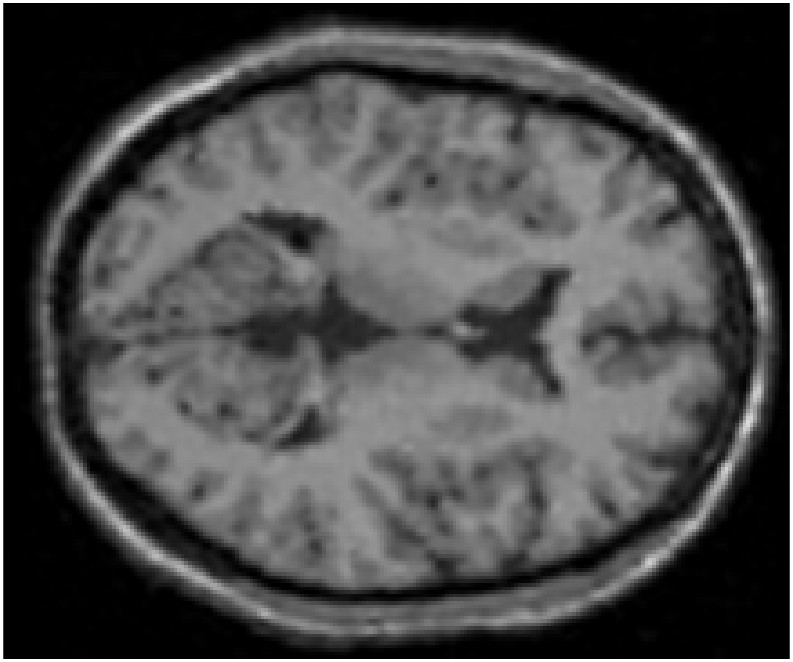}\\
			\includegraphics[width=1\linewidth]{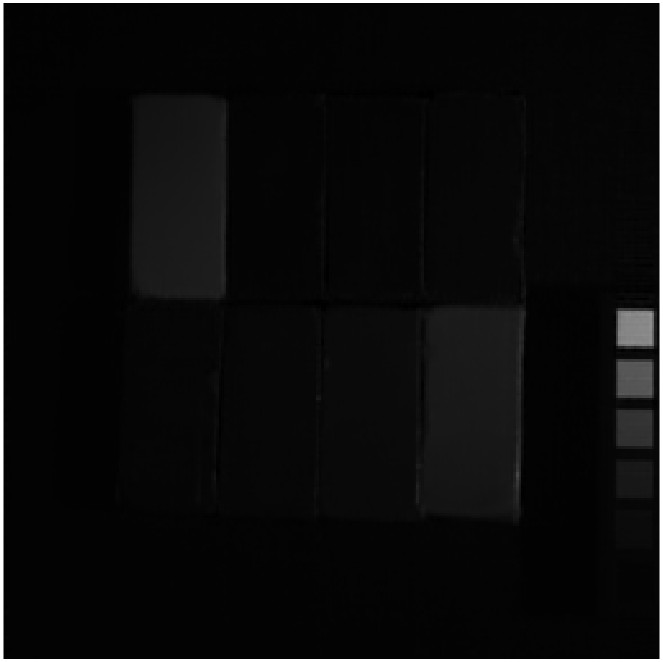}\\
			\includegraphics[width=1\linewidth]{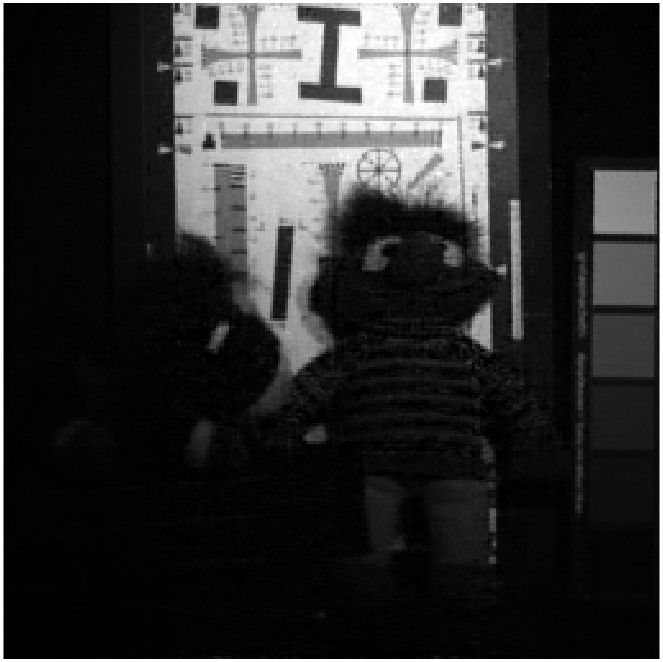}\\
			\includegraphics[width=1\linewidth]{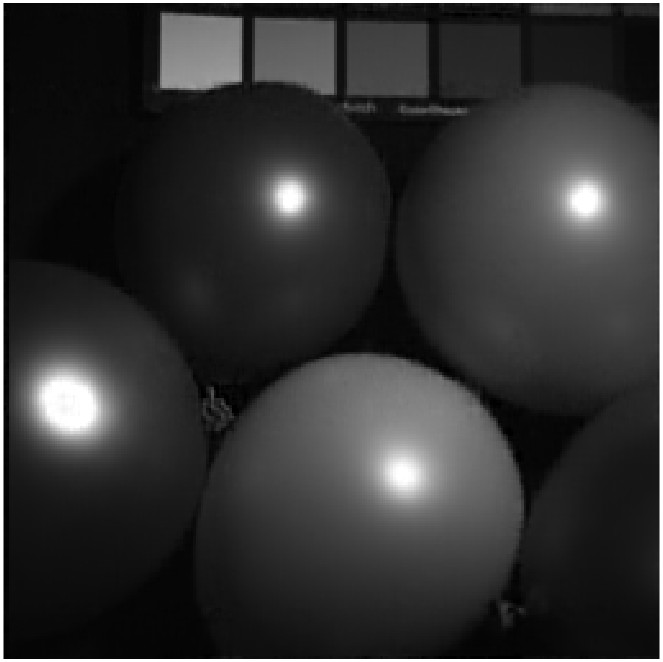}\\
			\includegraphics[width=1\linewidth]{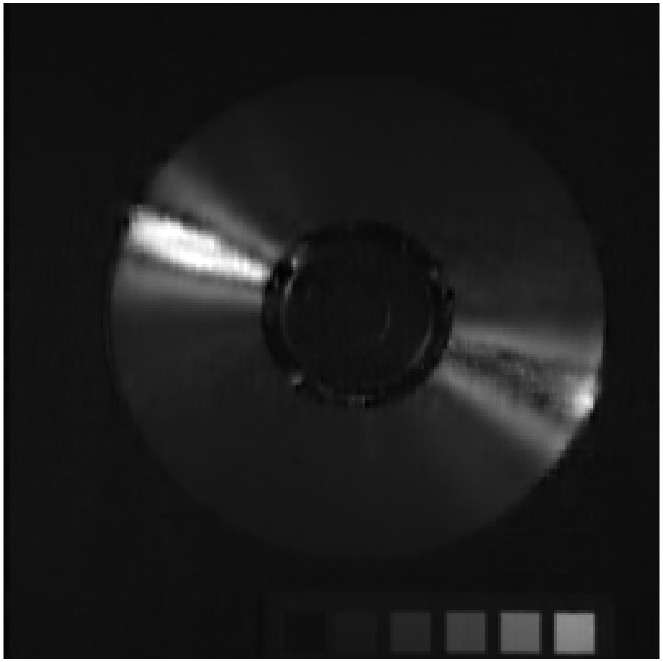}\\
			\includegraphics[width=1\linewidth]{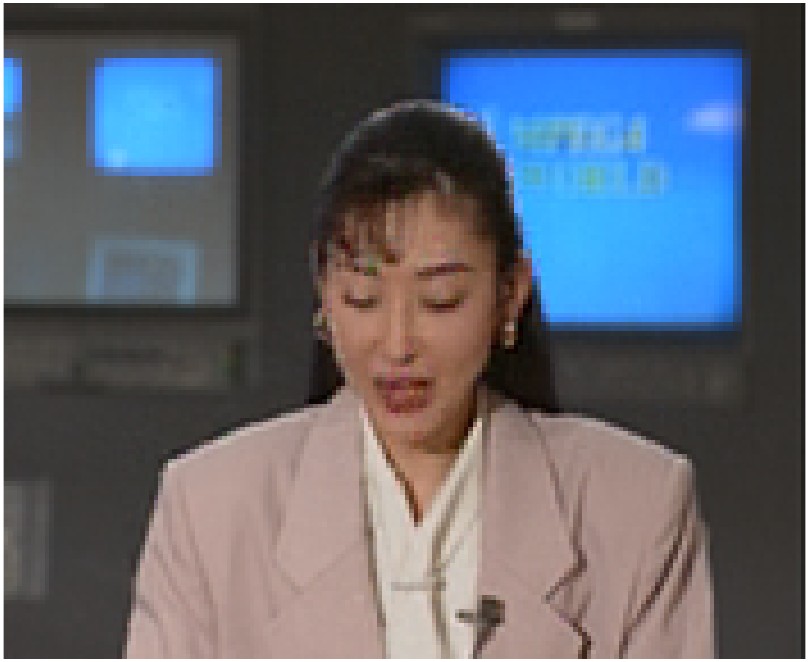}\\
			\includegraphics[width=1\linewidth]{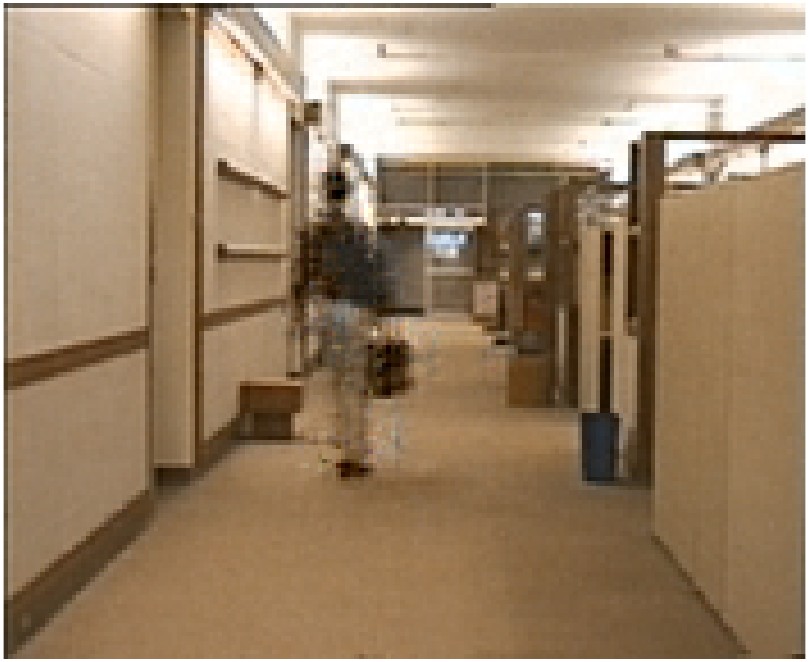}\\
			\includegraphics[width=1\linewidth]{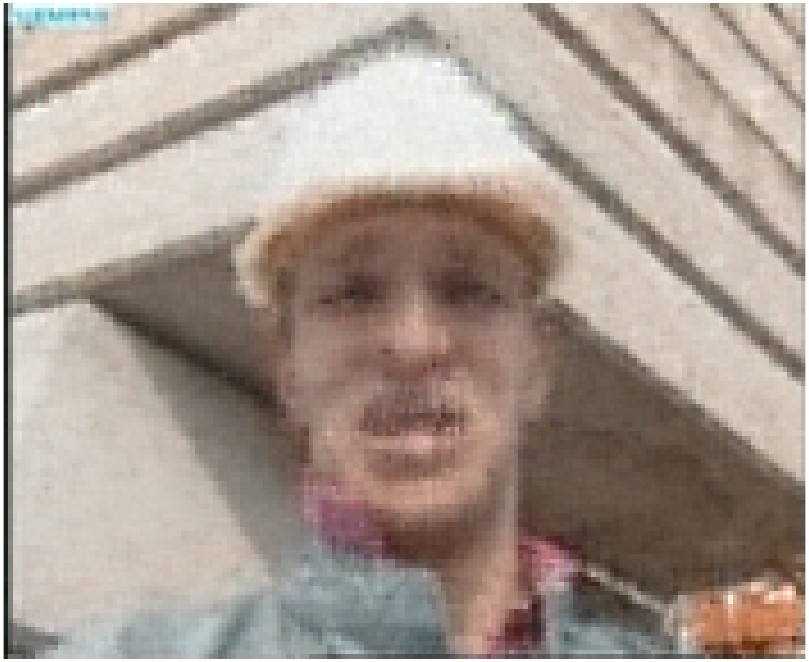}\\
			\includegraphics[width=1\linewidth]{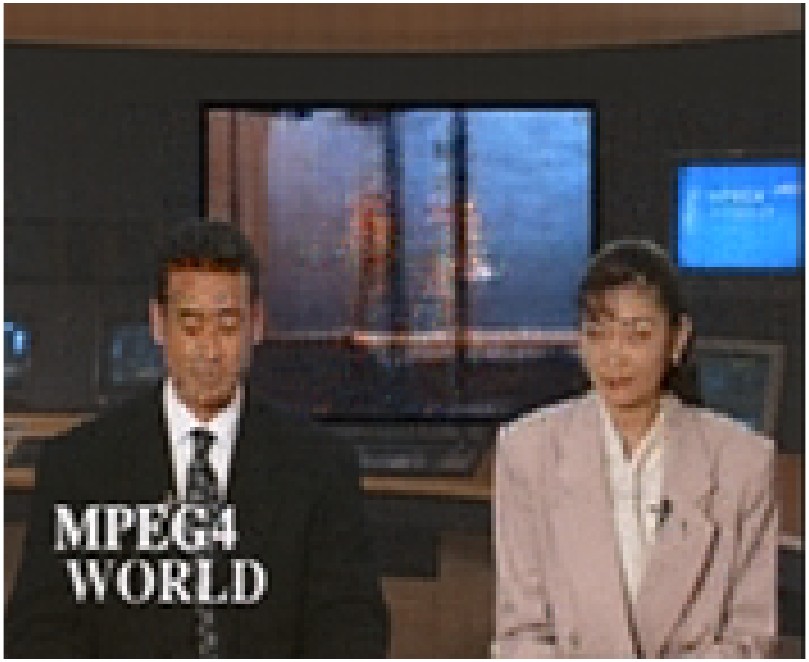}\\
			\includegraphics[width=1\linewidth]{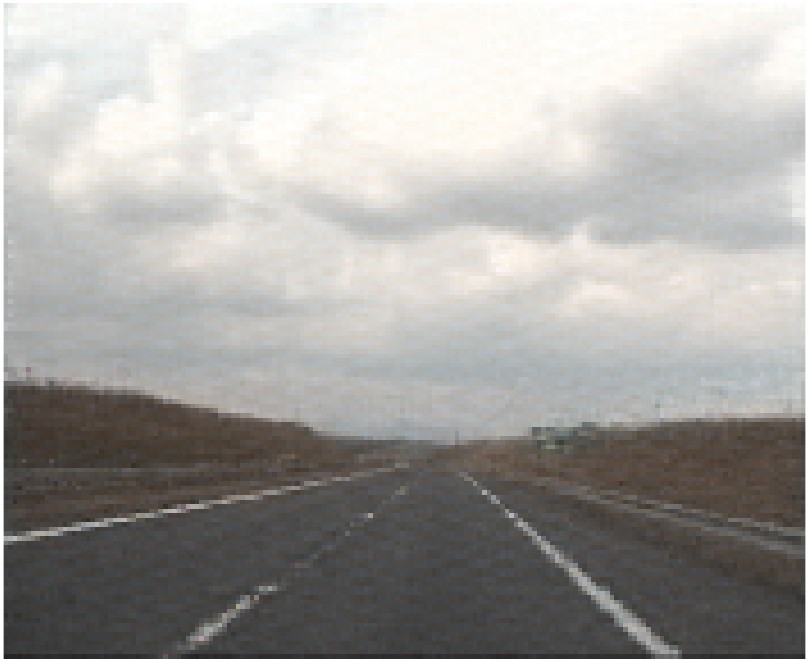}\\
			\includegraphics[width=1\linewidth]{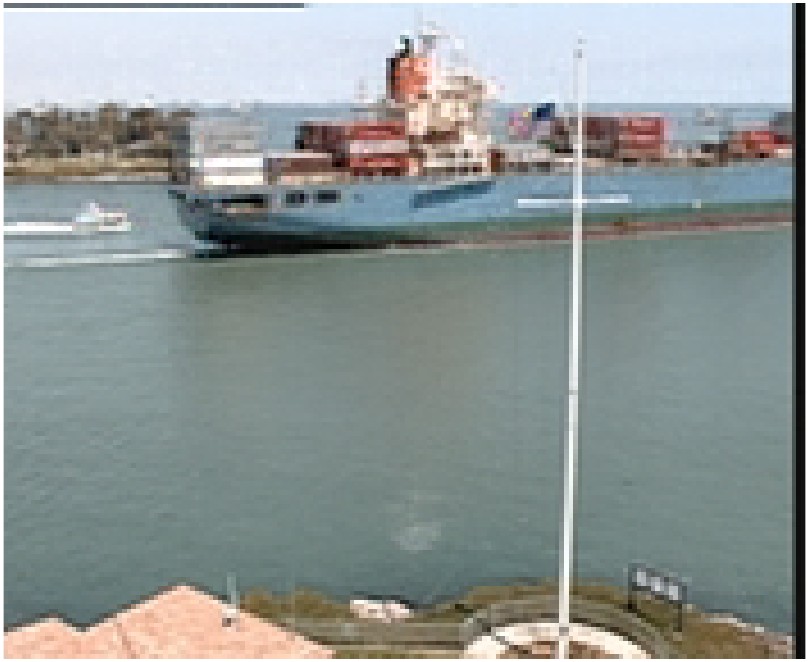} 
	\end{minipage}}\subfigure[Original]{
		\begin{minipage}[b]{0.105\linewidth}
			\includegraphics[width=1\linewidth]{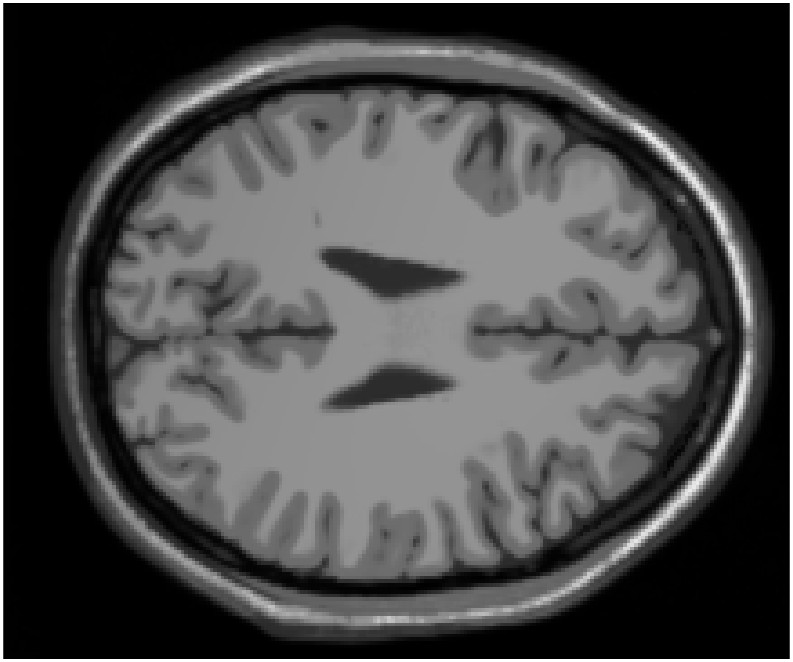}\\
			\includegraphics[width=1\linewidth]{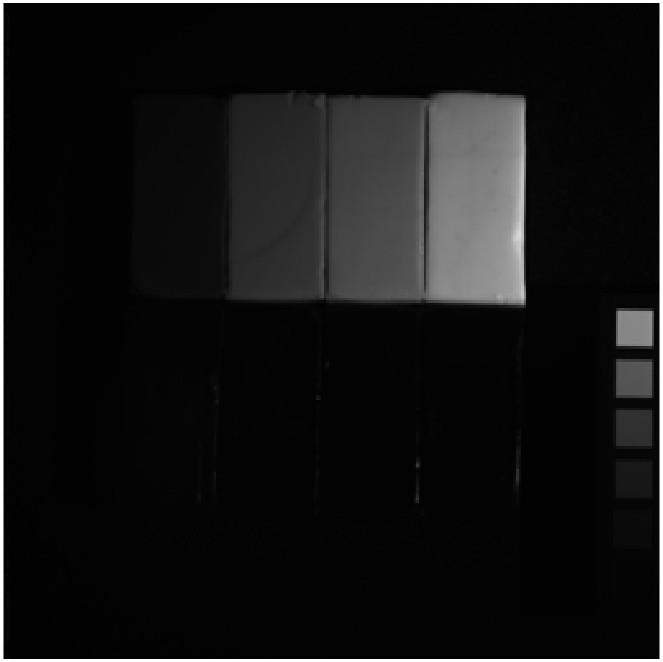}\\
			\includegraphics[width=1\linewidth]{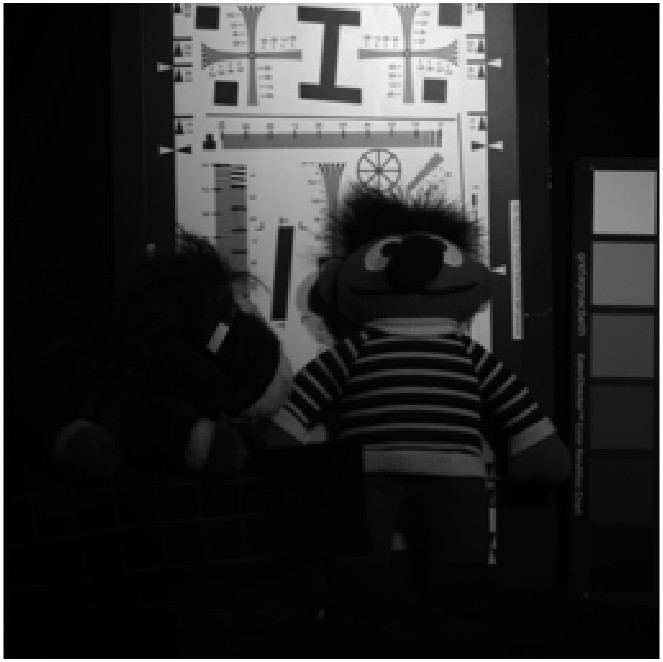}\\
			\includegraphics[width=1\linewidth]{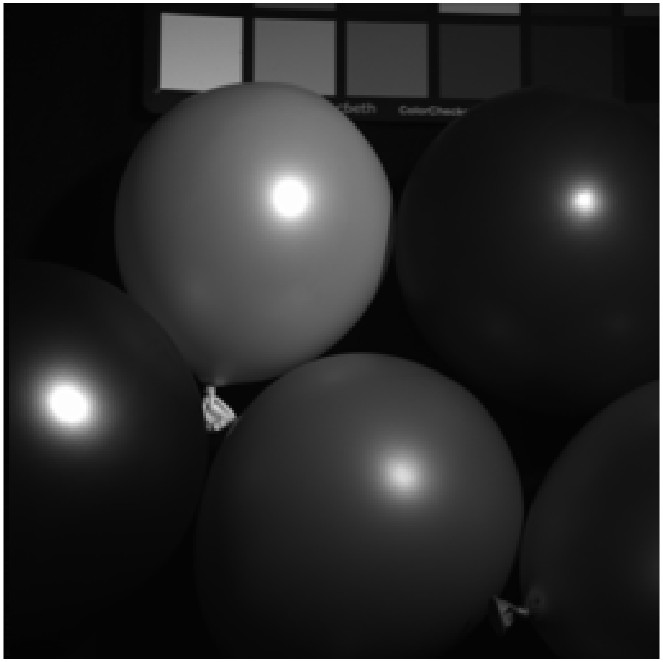}\\
			\includegraphics[width=1\linewidth]{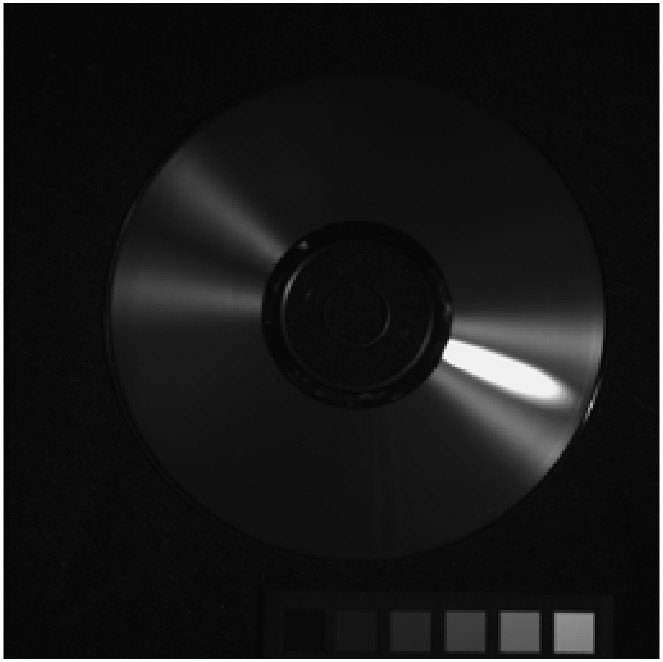}\\
			\includegraphics[width=1\linewidth]{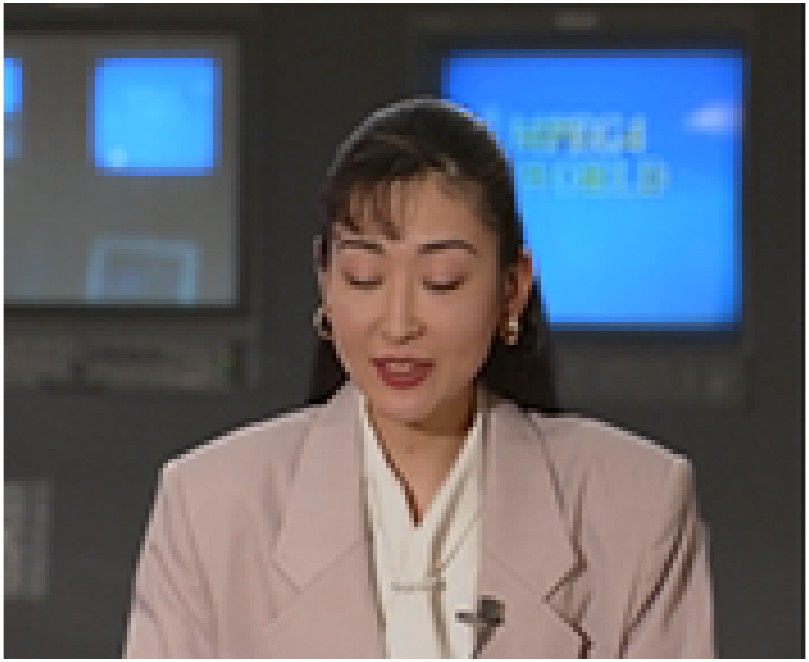}\\
			\includegraphics[width=1\linewidth]{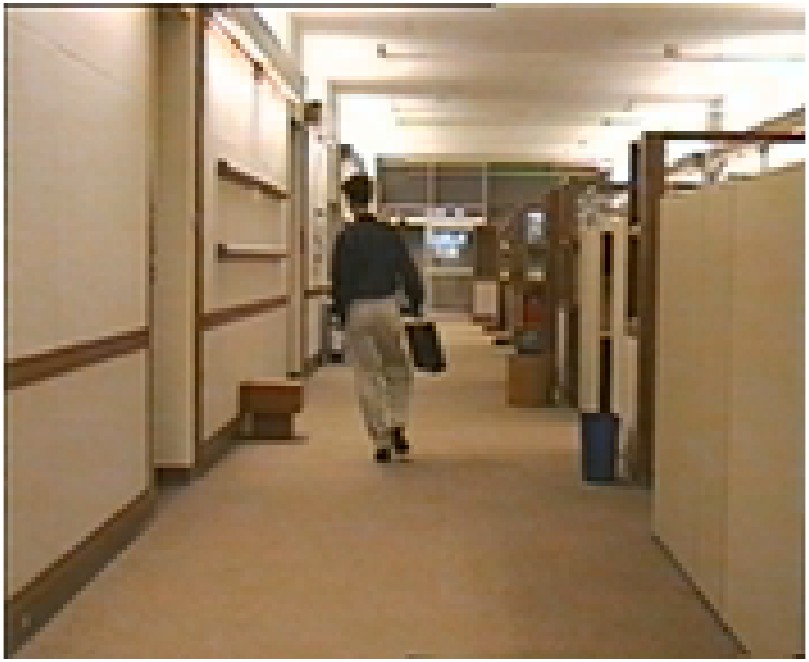}\\
			\includegraphics[width=1\linewidth]{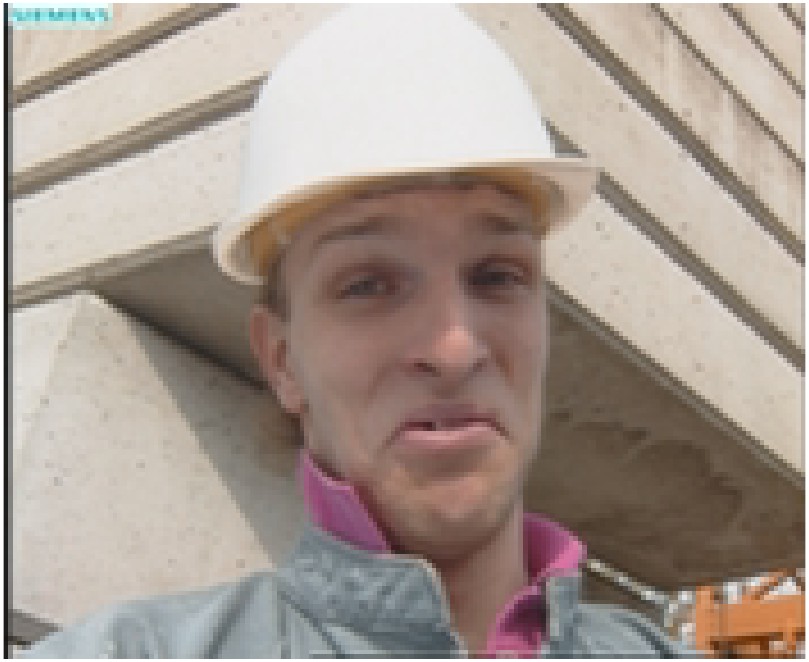}\\
			\includegraphics[width=1\linewidth]{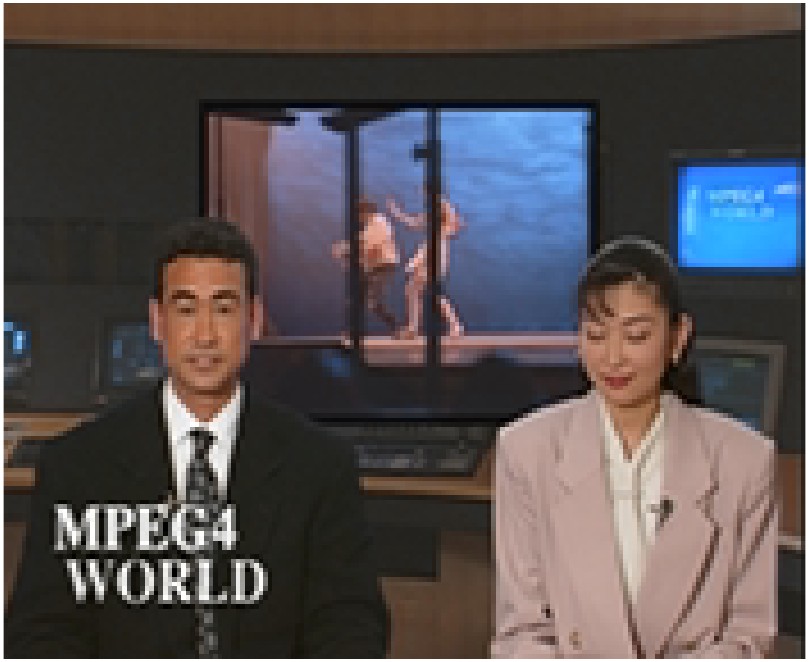}\\
			\includegraphics[width=1\linewidth]{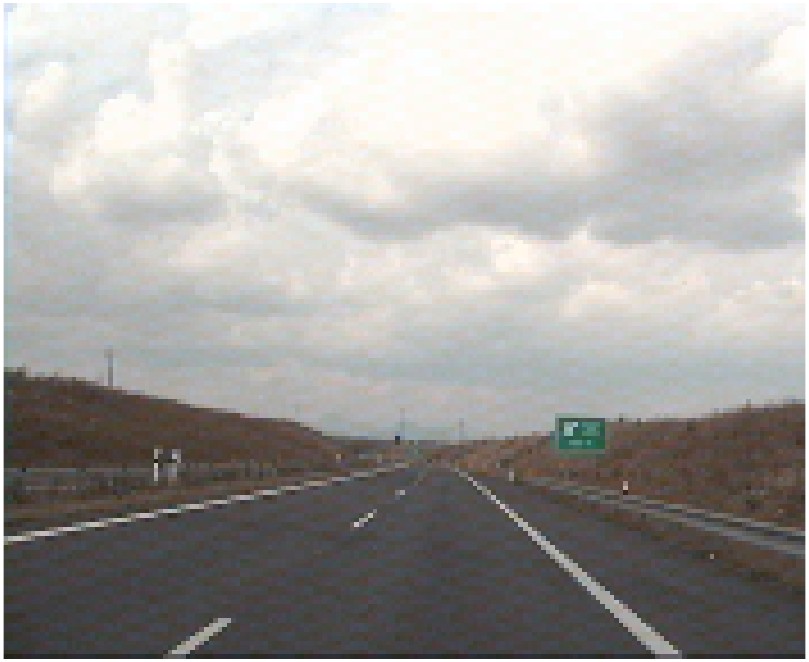}\\
			\includegraphics[width=1\linewidth]{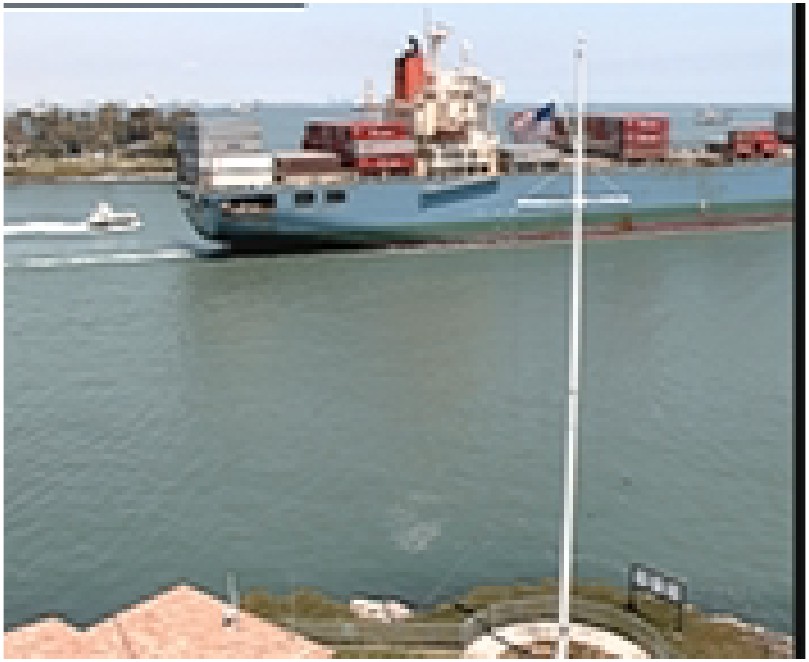} 
	\end{minipage}}\subfigure[MR=98\%]{
		\begin{minipage}[b]{0.105\linewidth}
			\includegraphics[width=1\linewidth]{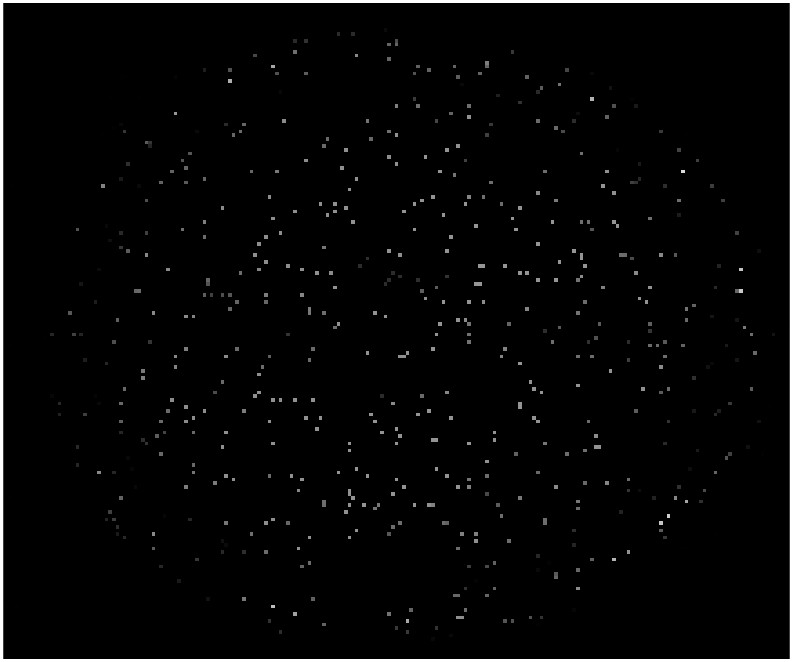}\\
			\includegraphics[width=1\linewidth]{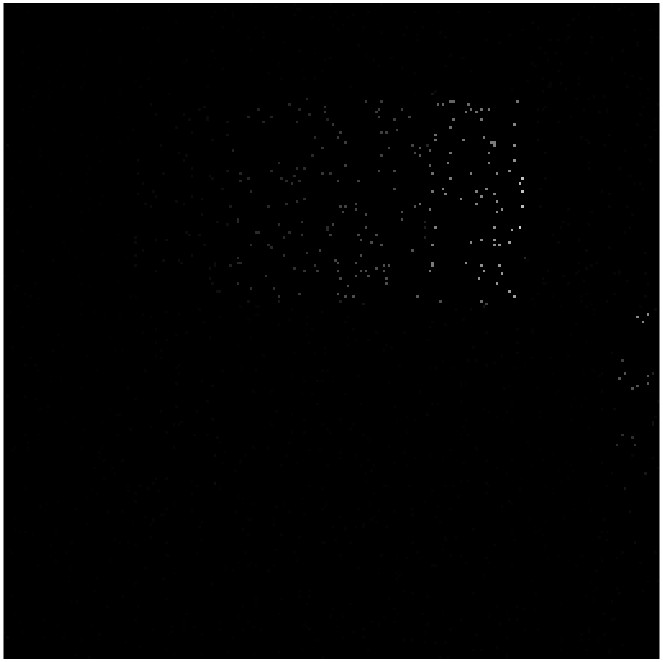}\\
			\includegraphics[width=1\linewidth]{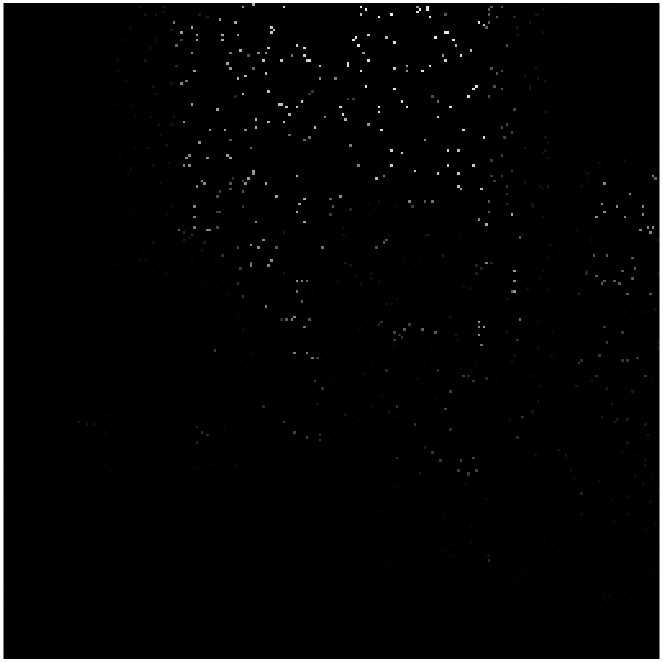}\\
			\includegraphics[width=1\linewidth]{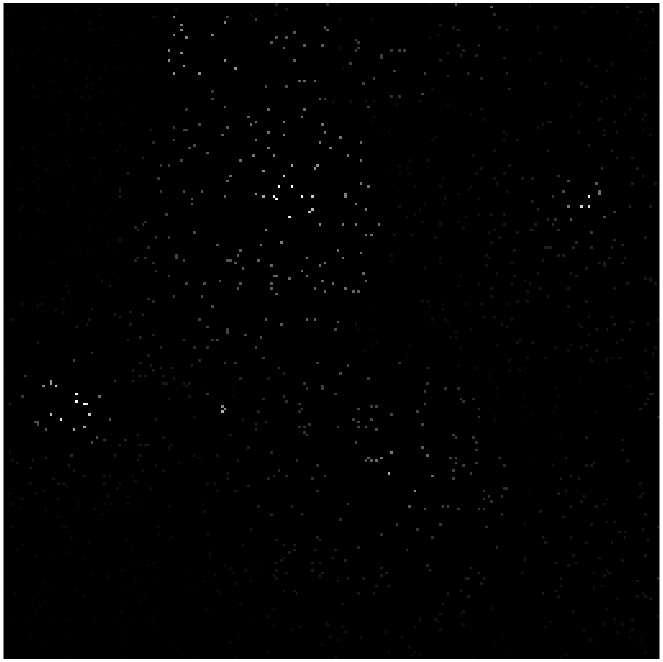}\\
			\includegraphics[width=1\linewidth]{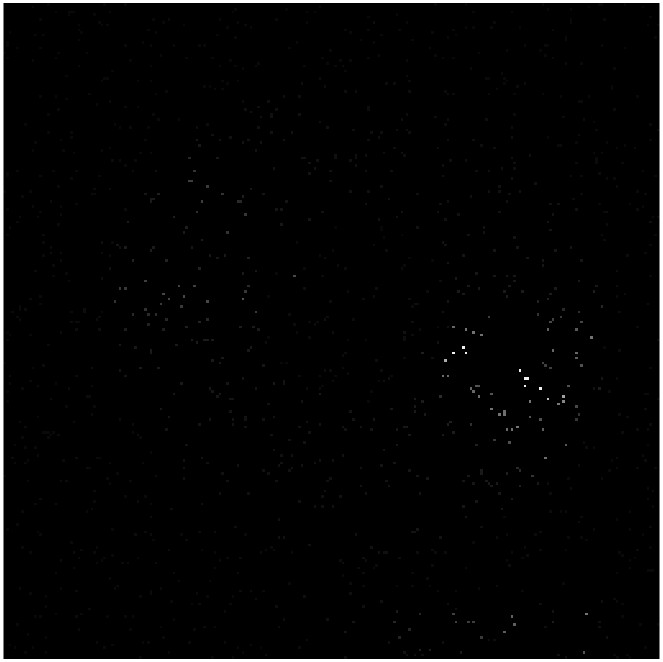}\\
			\includegraphics[width=1\linewidth]{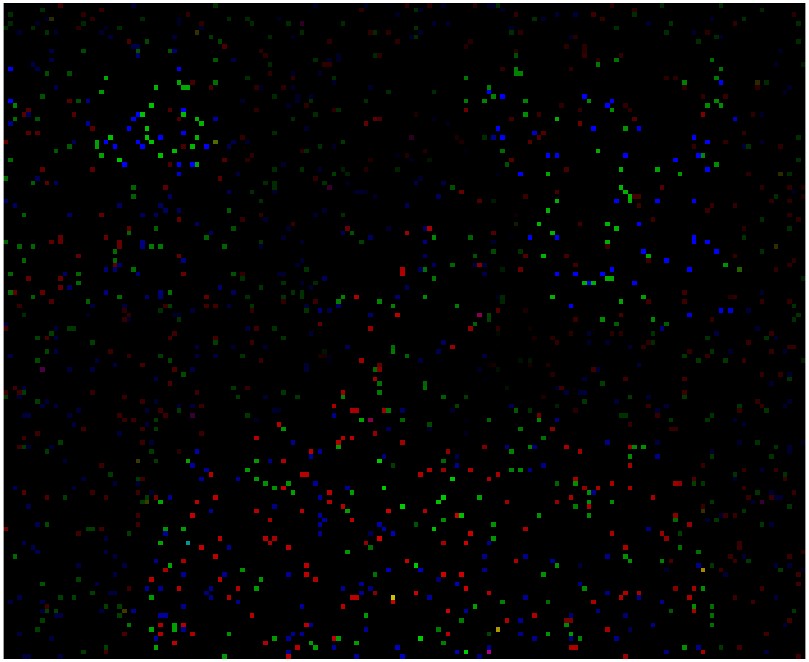}\\
			\includegraphics[width=1\linewidth]{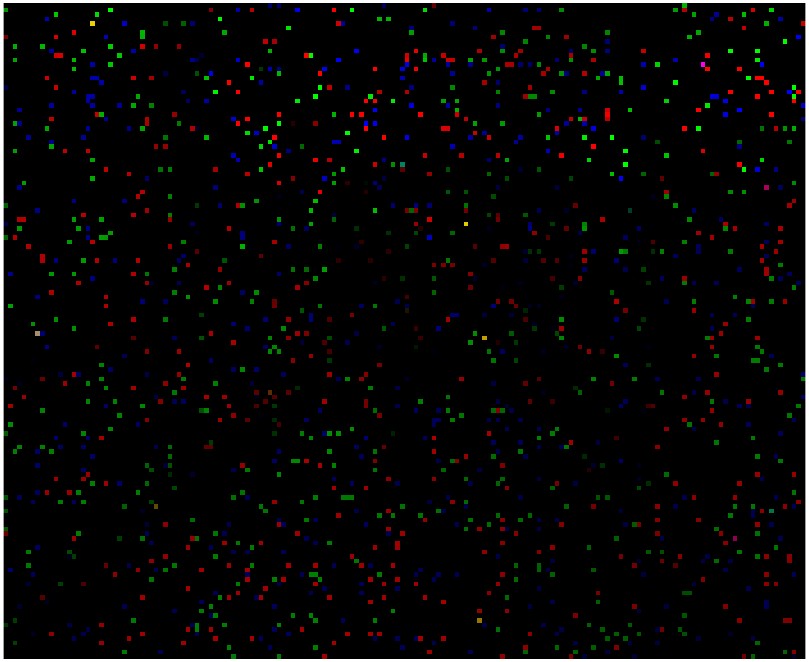}\\
			\includegraphics[width=1\linewidth]{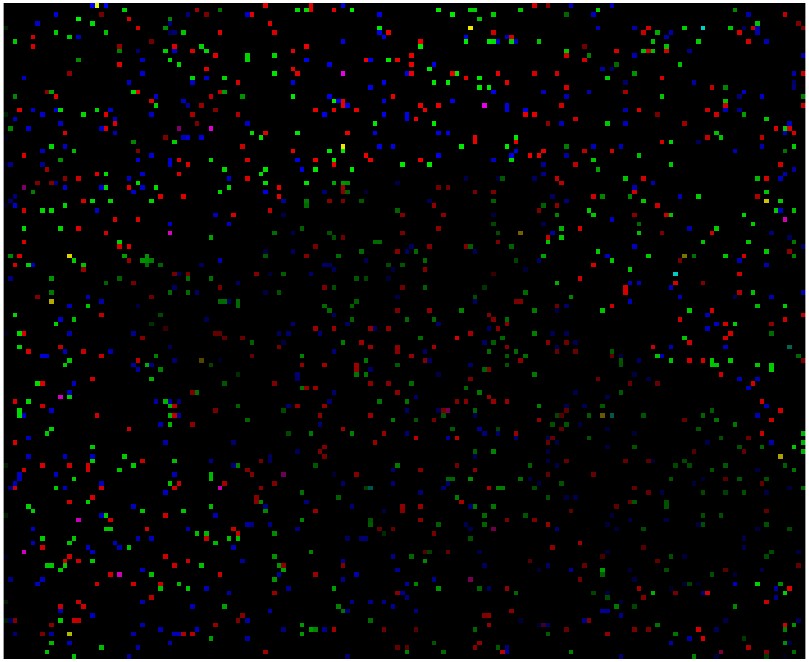}\\
			\includegraphics[width=1\linewidth]{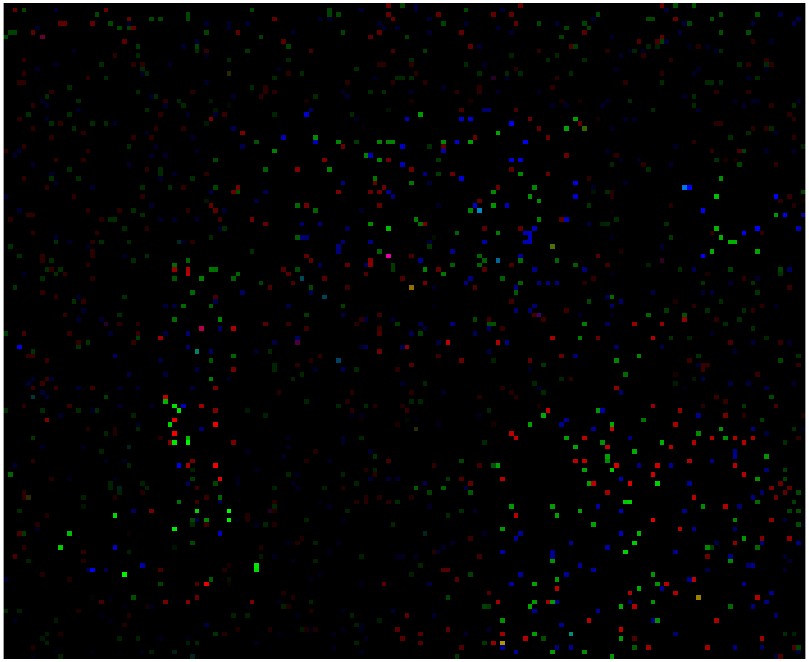}\\
			\includegraphics[width=1\linewidth]{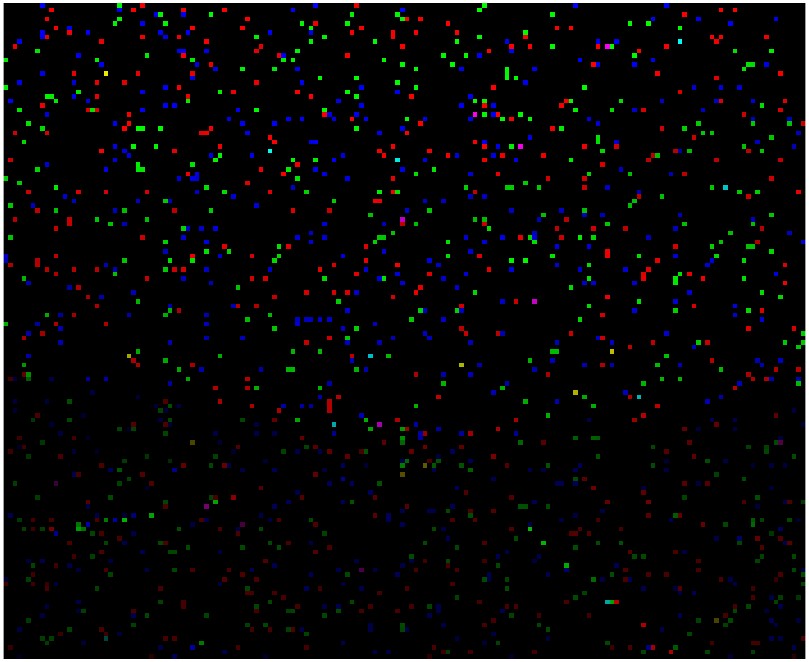}\\
			\includegraphics[width=1\linewidth]{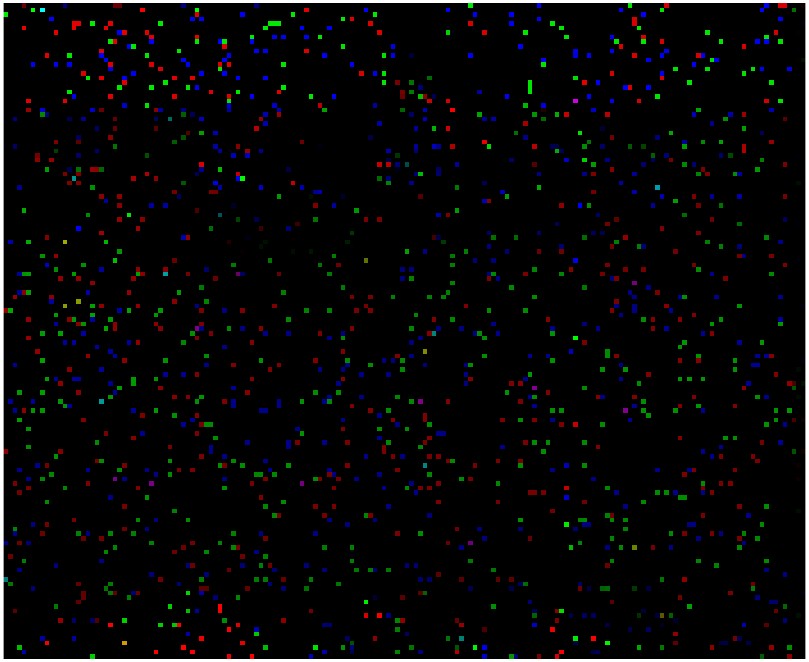} 
	\end{minipage}}\subfigure[TJLC]{
		\begin{minipage}[b]{0.105\linewidth}
			\includegraphics[width=1\linewidth]{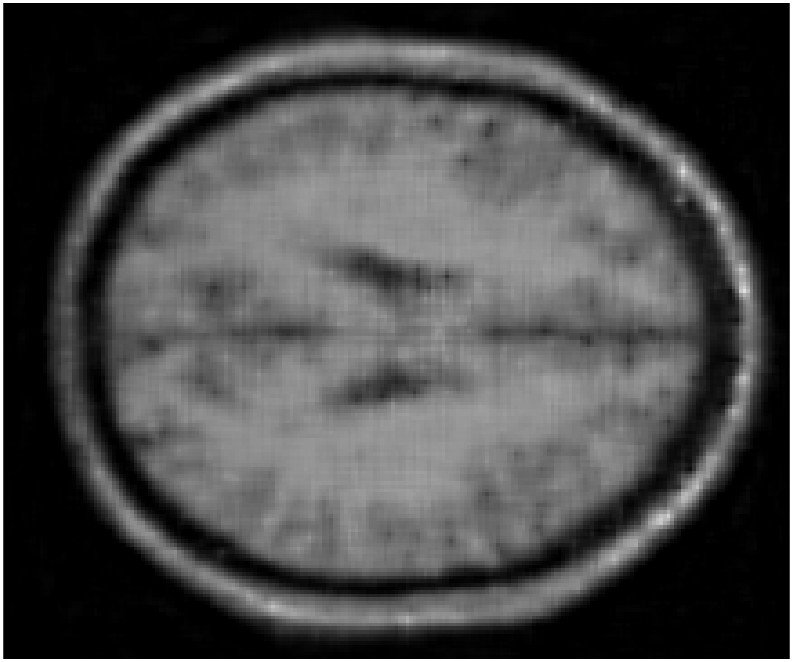}\\
			\includegraphics[width=1\linewidth]{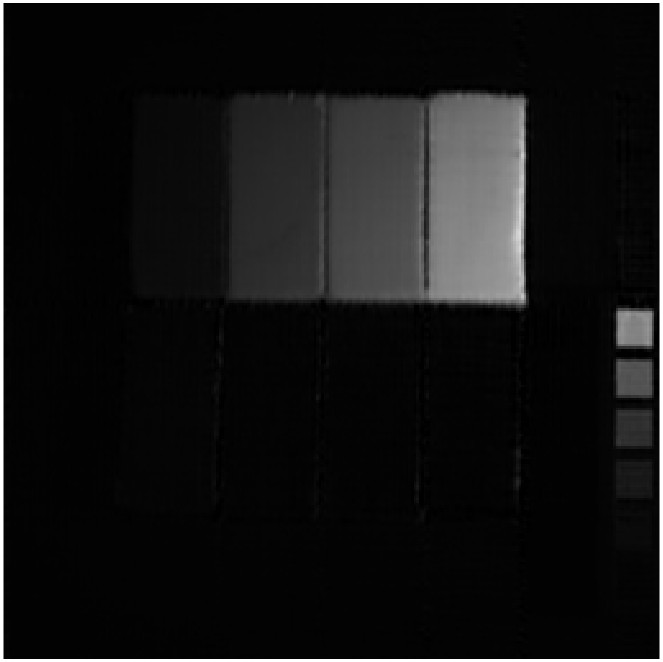}\\
			\includegraphics[width=1\linewidth]{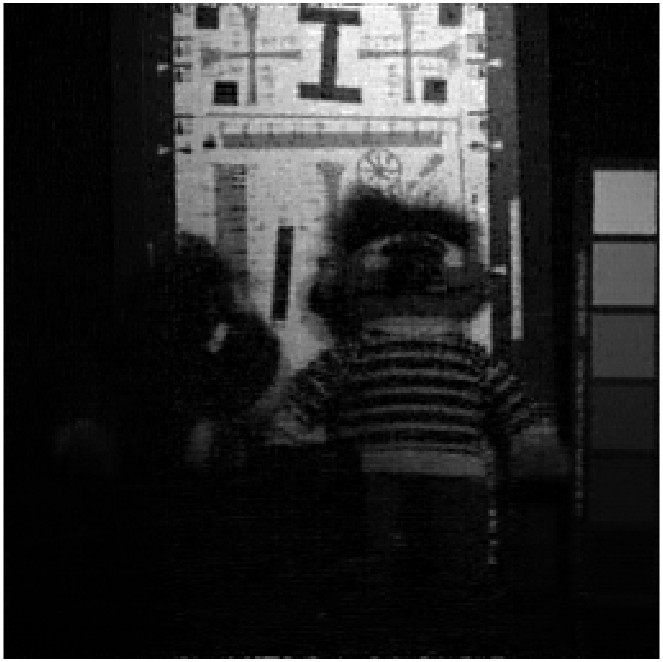}\\
			\includegraphics[width=1\linewidth]{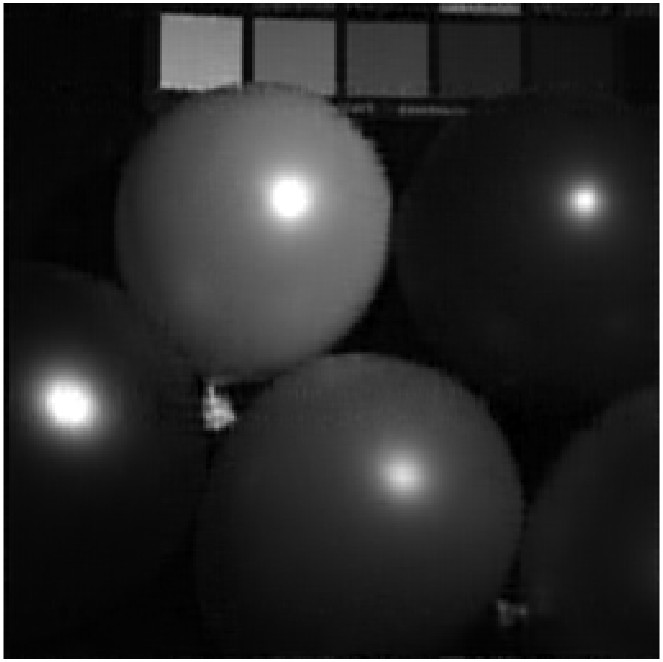}\\
			\includegraphics[width=1\linewidth]{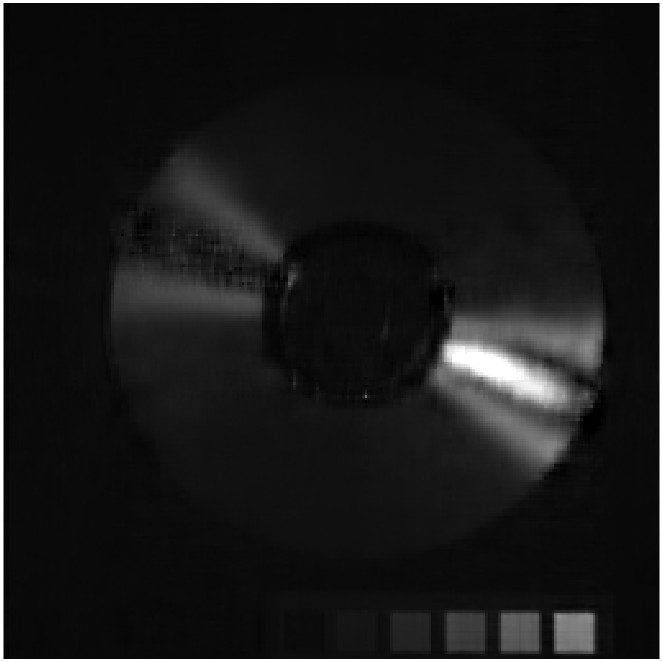}\\
			\includegraphics[width=1\linewidth]{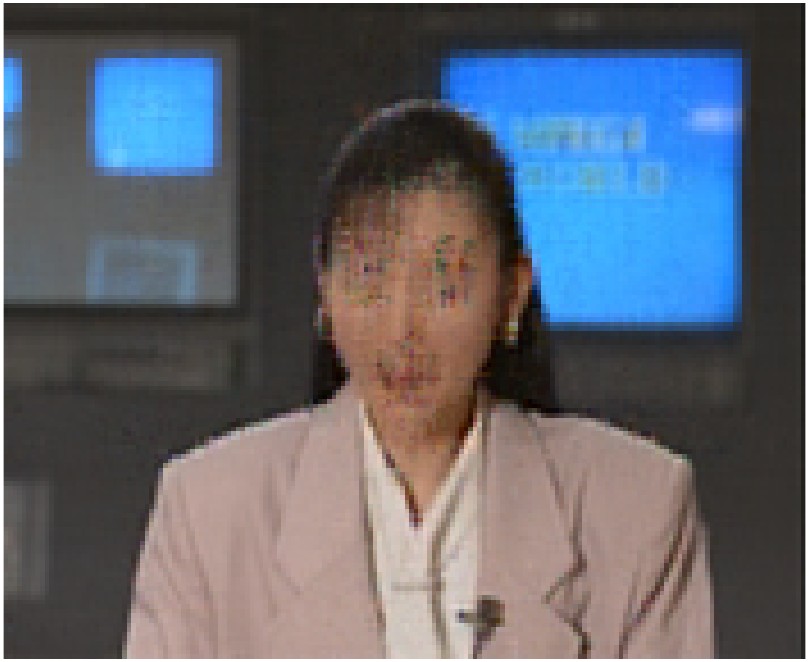}\\
			\includegraphics[width=1\linewidth]{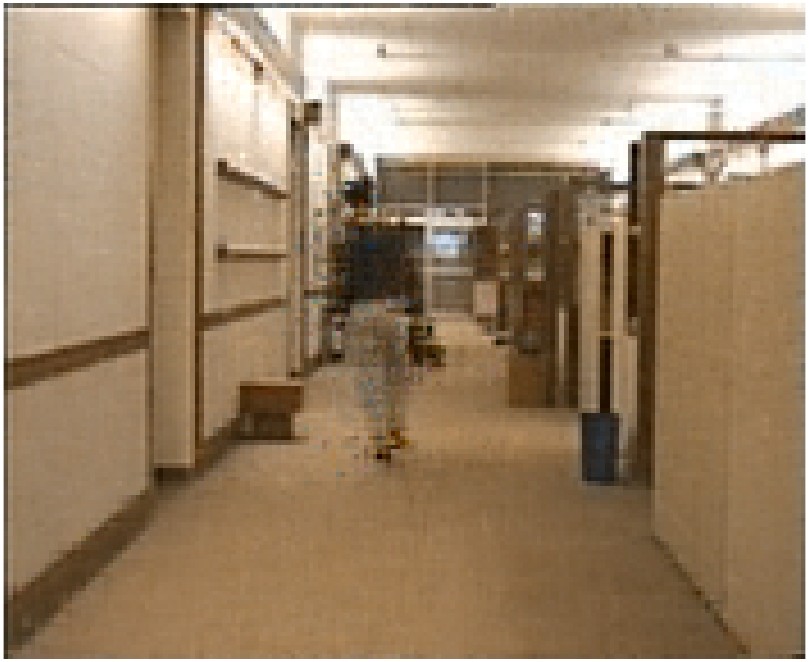}\\
			\includegraphics[width=1\linewidth]{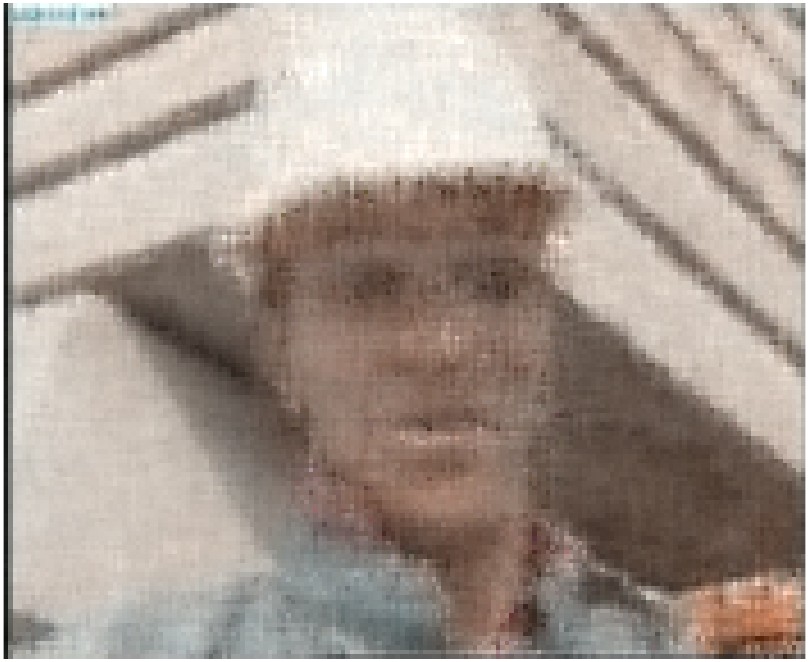}\\
			\includegraphics[width=1\linewidth]{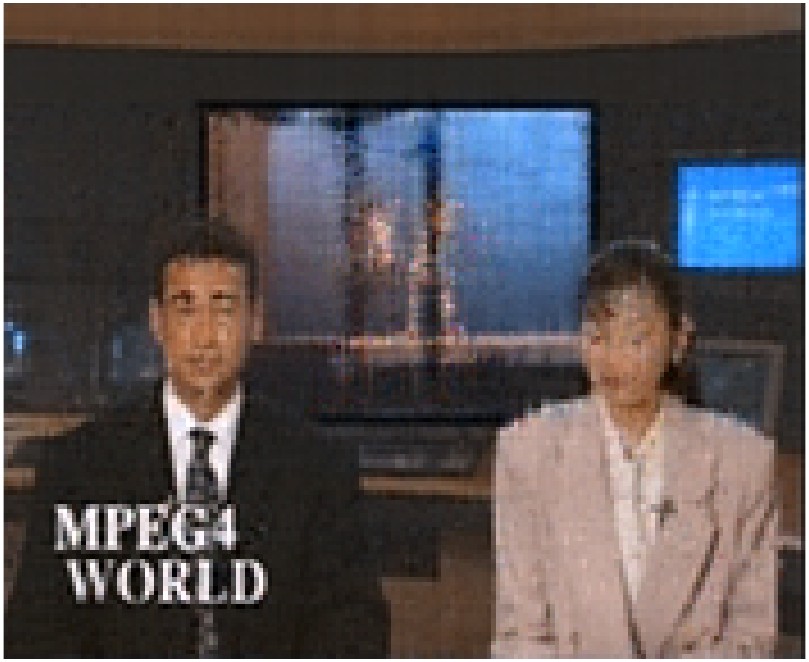}\\
			\includegraphics[width=1\linewidth]{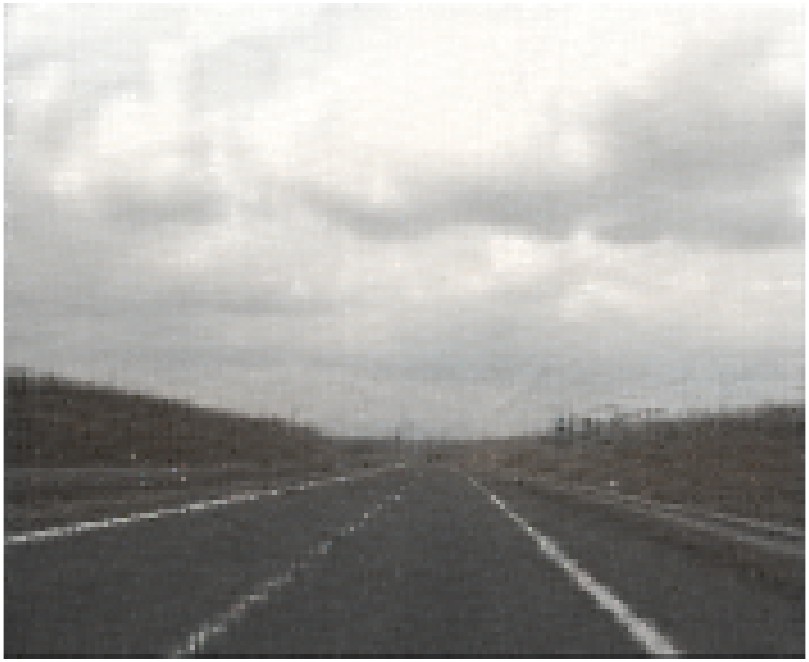}\\
			\includegraphics[width=1\linewidth]{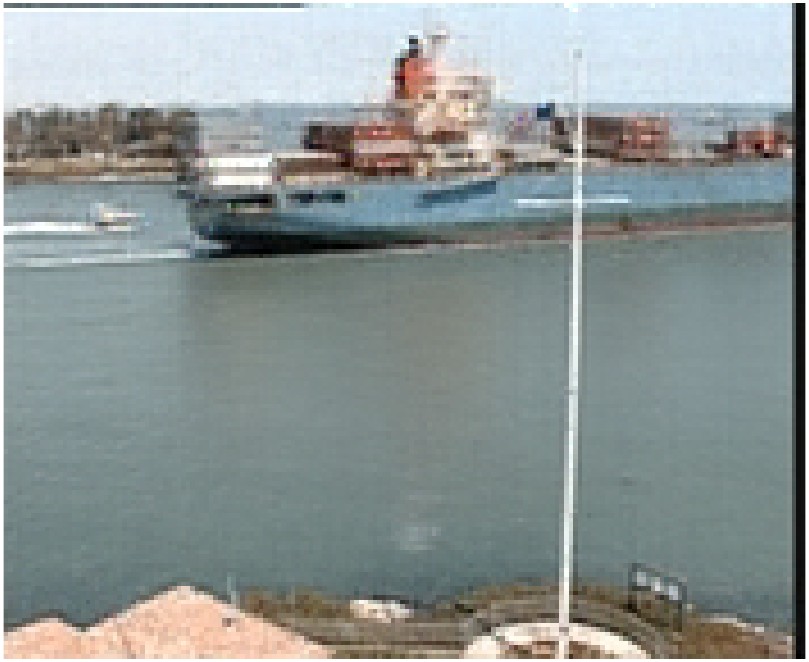} 
	\end{minipage}}\subfigure[Original]{
		\begin{minipage}[b]{0.105\linewidth}
			\includegraphics[width=1\linewidth]{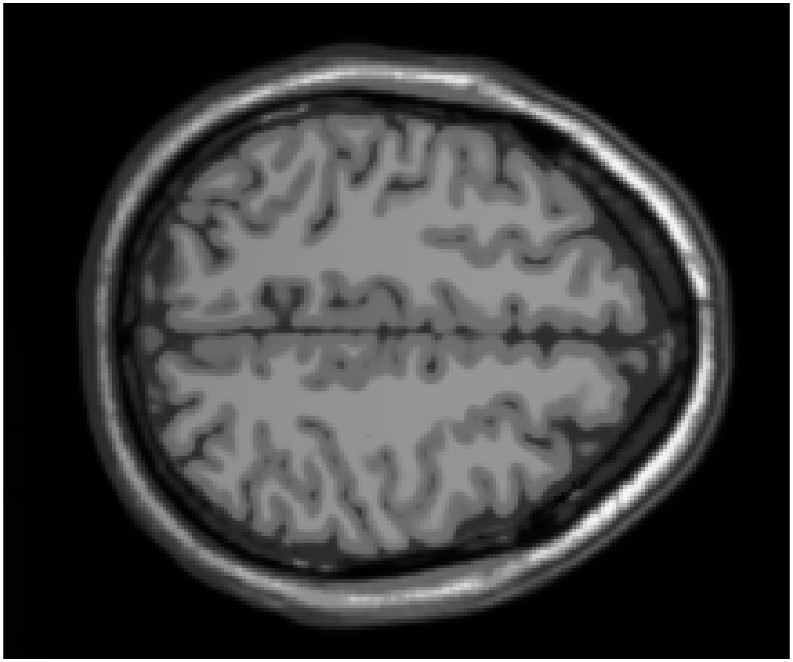}\\
			\includegraphics[width=1\linewidth]{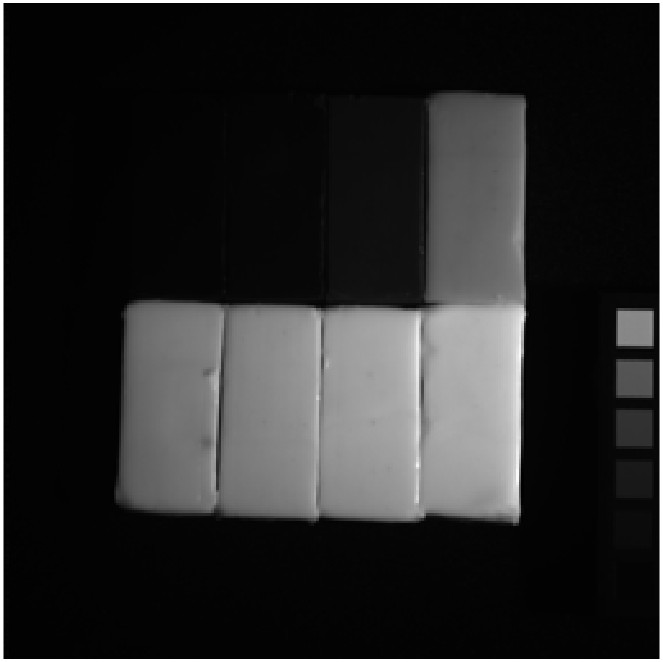}\\
			\includegraphics[width=1\linewidth]{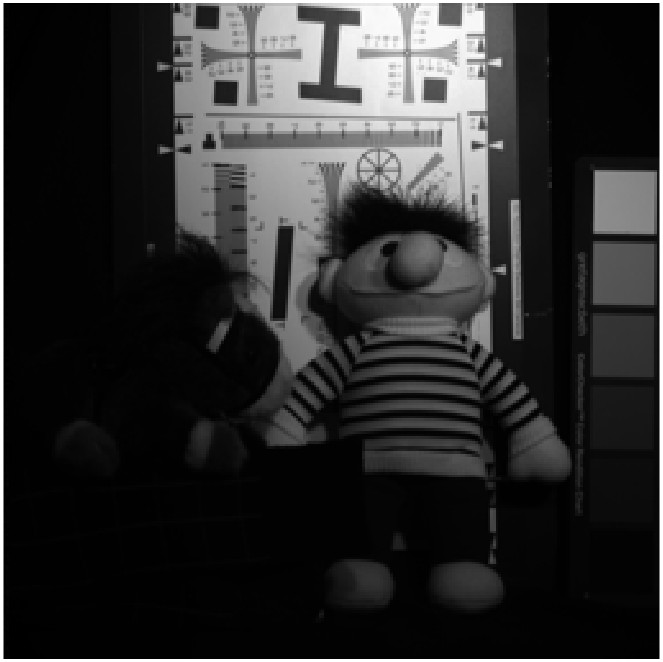}\\
			\includegraphics[width=1\linewidth]{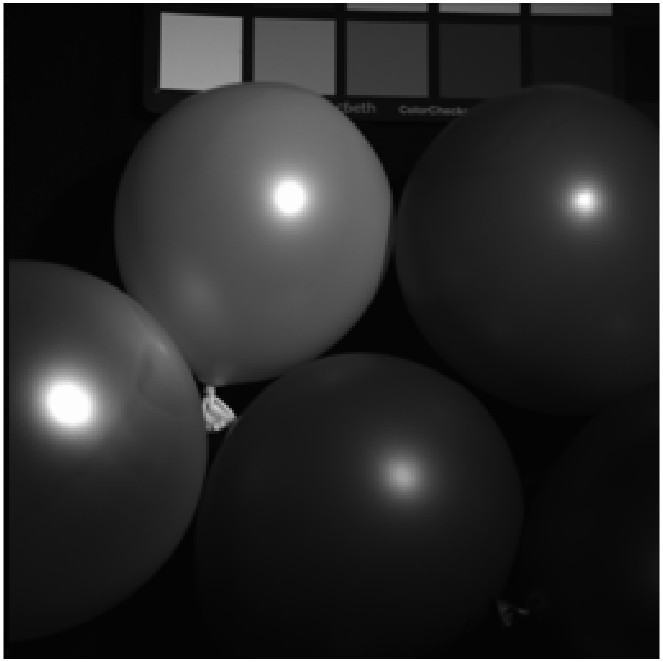}\\
			\includegraphics[width=1\linewidth]{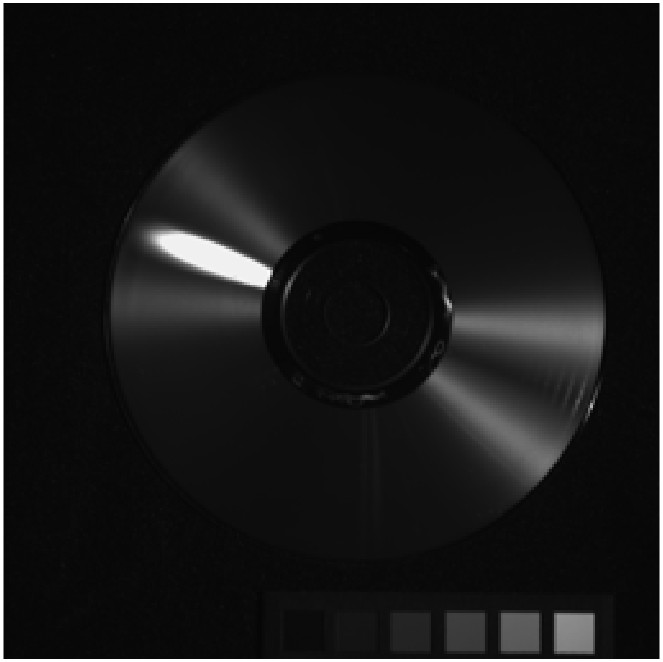}\\
			\includegraphics[width=1\linewidth]{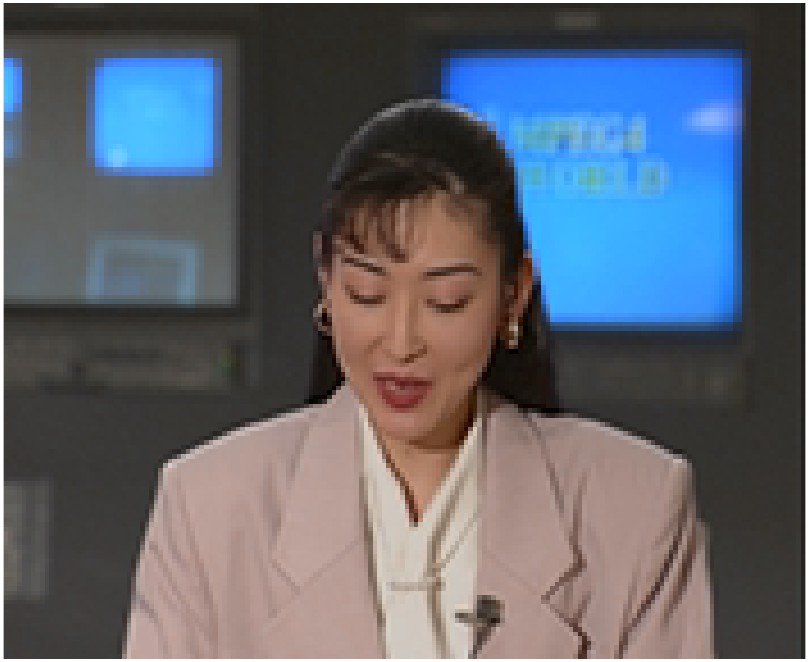}\\
			\includegraphics[width=1\linewidth]{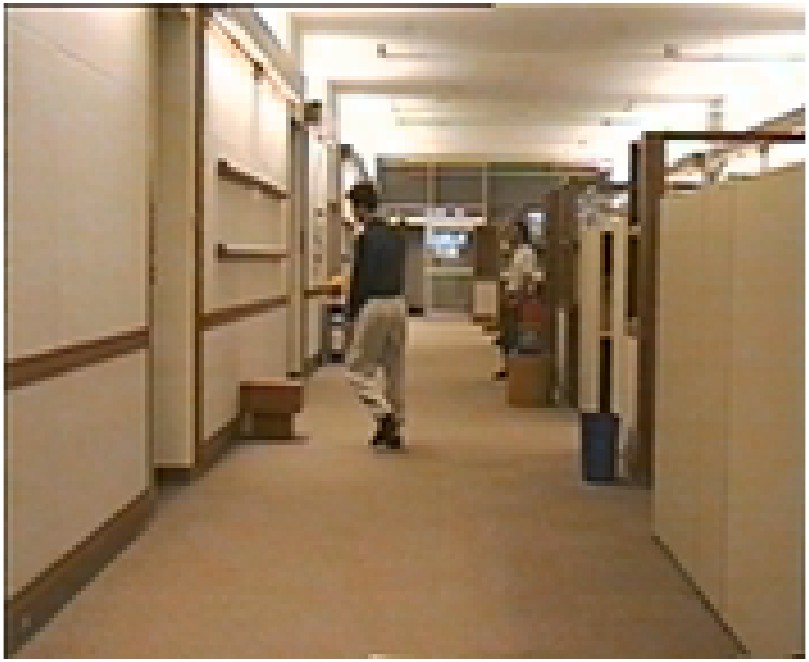}\\
			\includegraphics[width=1\linewidth]{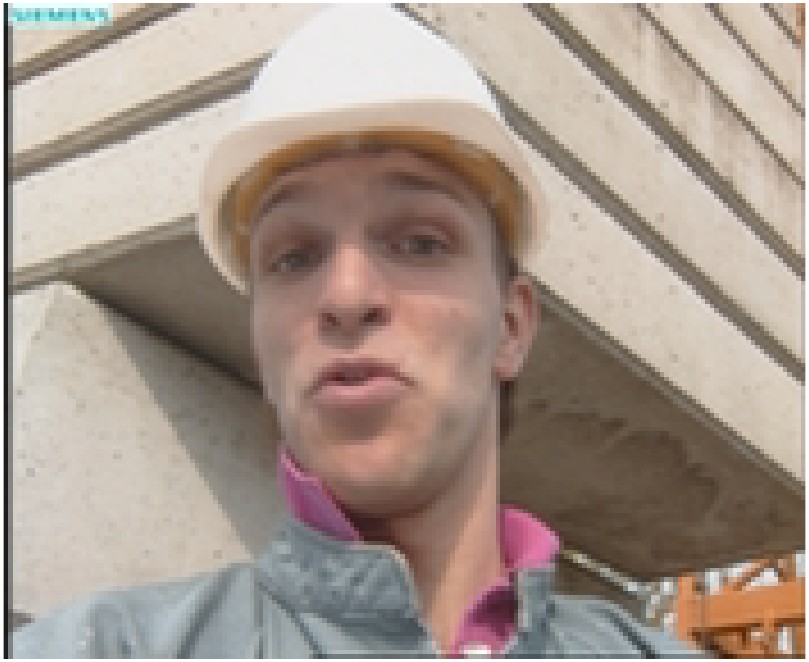}\\
			\includegraphics[width=1\linewidth]{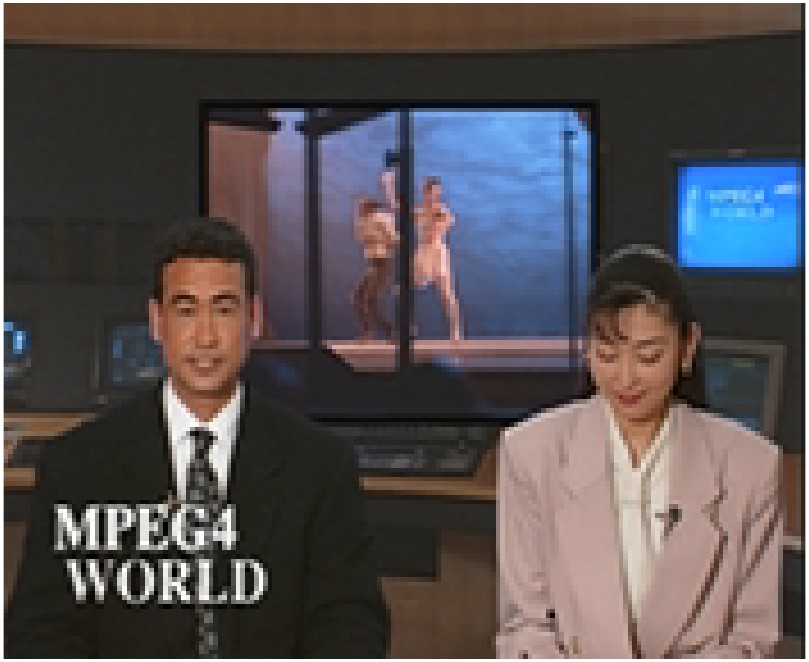}\\
			\includegraphics[width=1\linewidth]{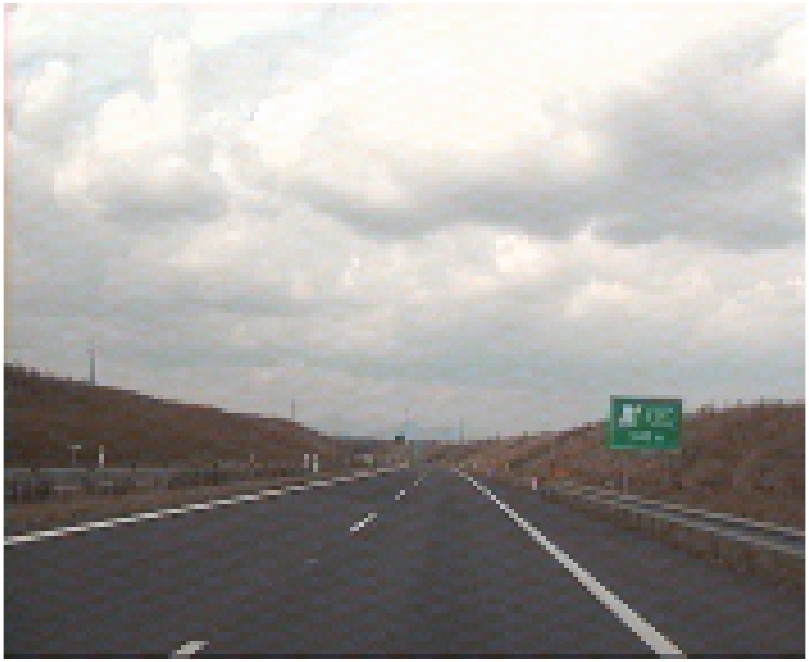}\\
			\includegraphics[width=1\linewidth]{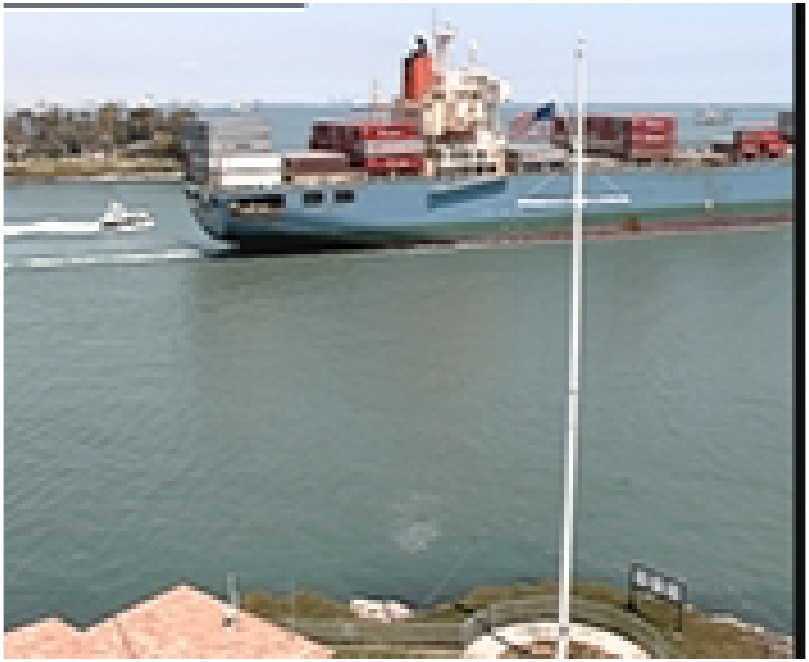} 
	\end{minipage}}\subfigure[MR=99\%]{
		\begin{minipage}[b]{0.105\linewidth}
			\includegraphics[width=1\linewidth]{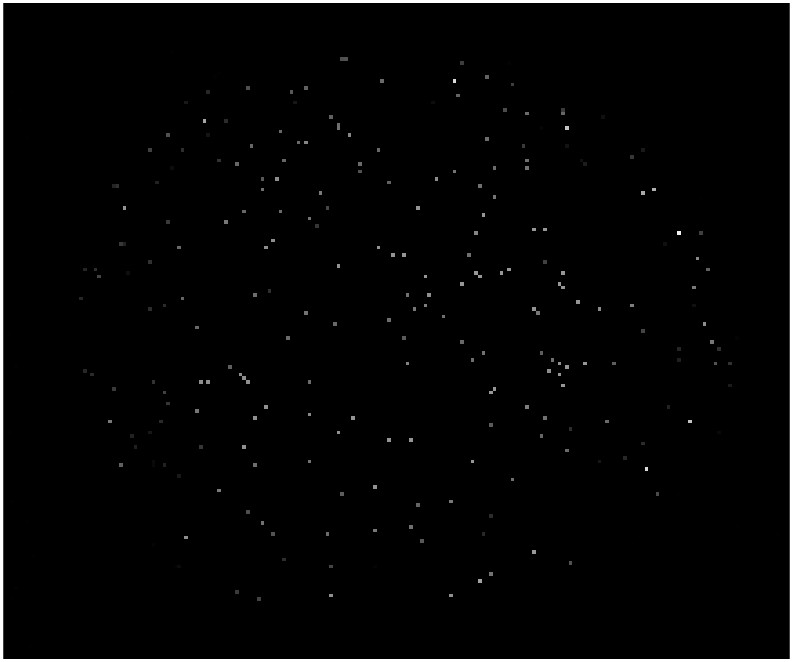}\\
			\includegraphics[width=1\linewidth]{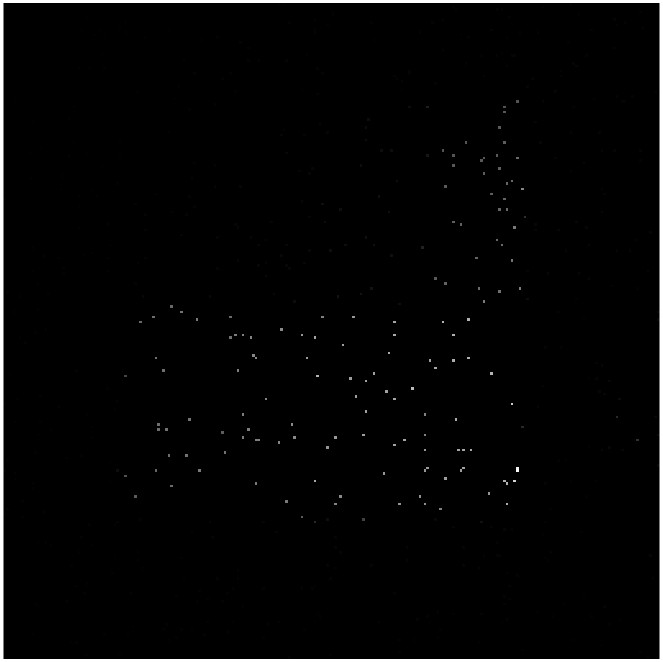}\\
			\includegraphics[width=1\linewidth]{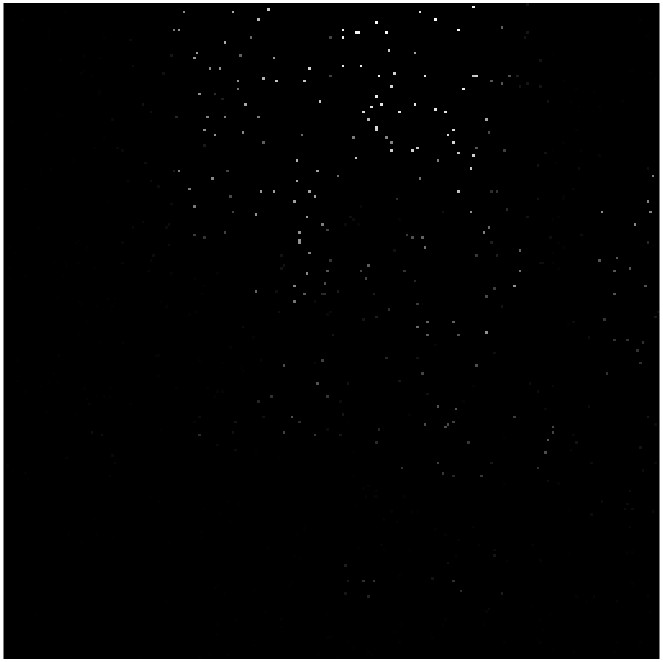}\\
			\includegraphics[width=1\linewidth]{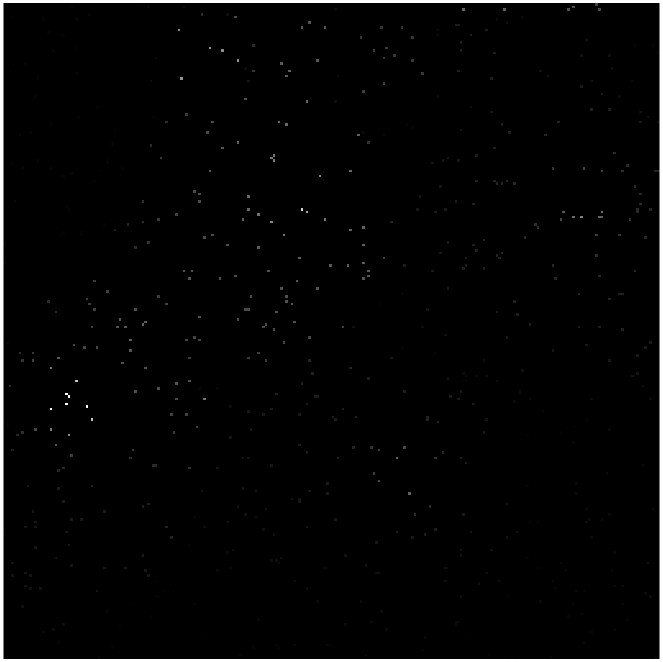}\\
			\includegraphics[width=1\linewidth]{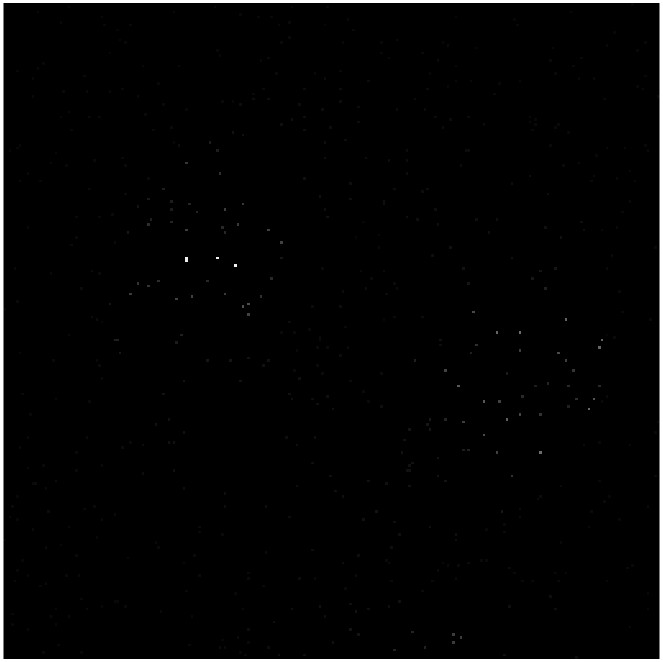}\\
			\includegraphics[width=1\linewidth]{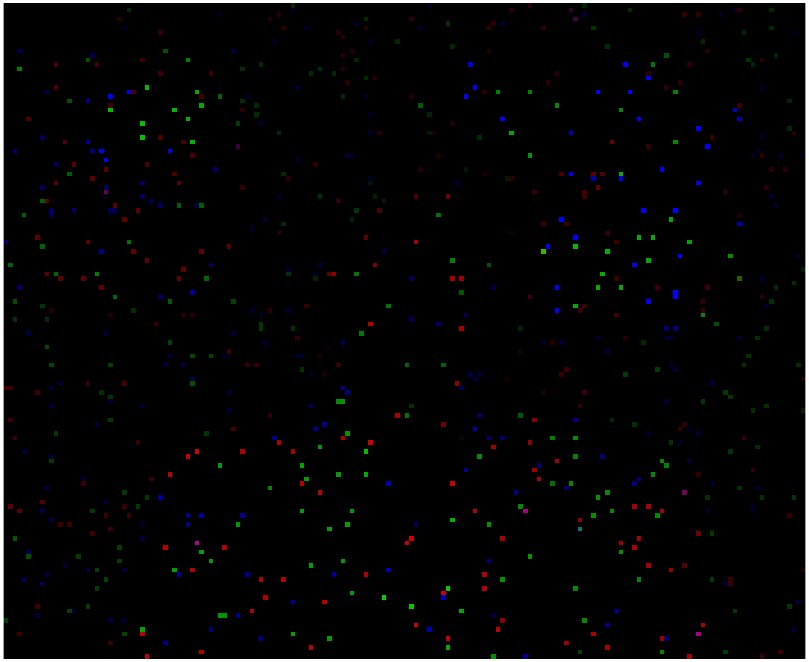}\\
			\includegraphics[width=1\linewidth]{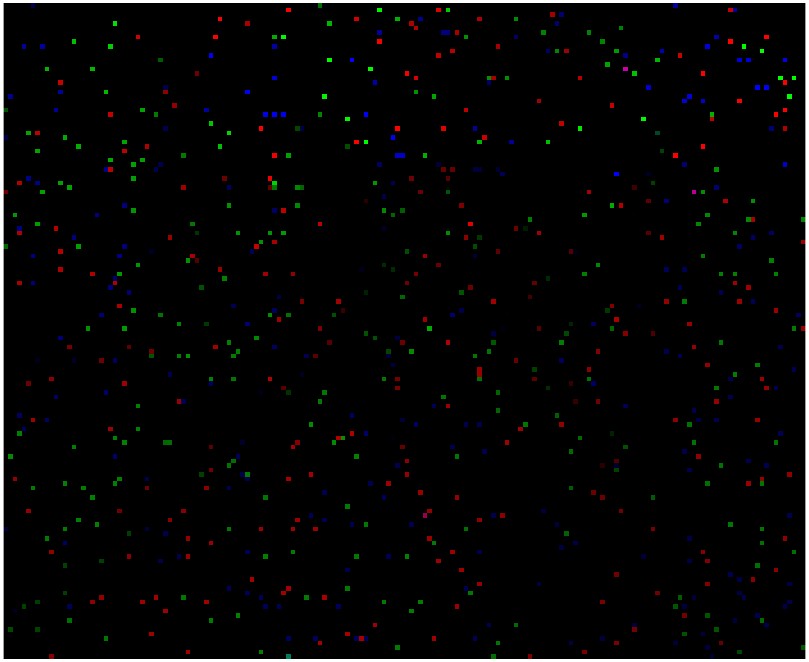}\\
			\includegraphics[width=1\linewidth]{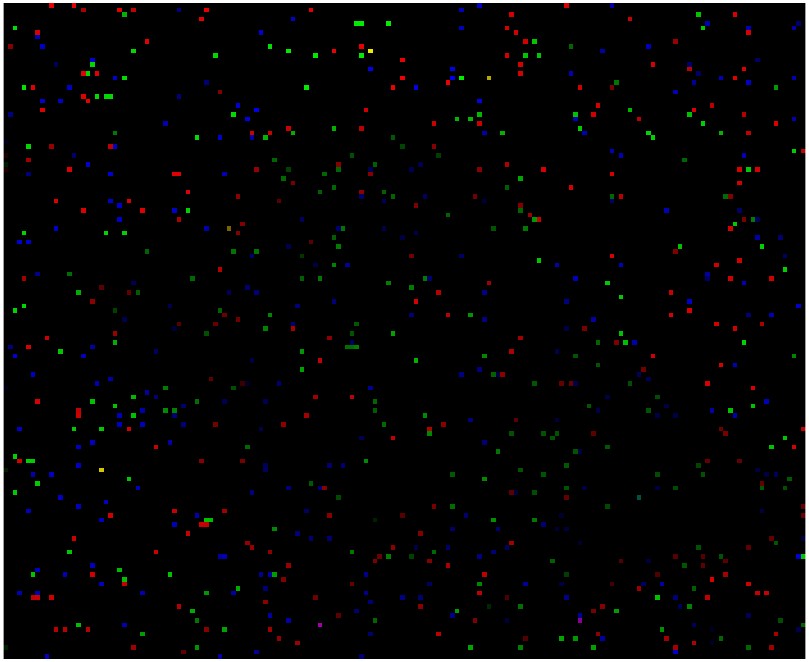}\\
			\includegraphics[width=1\linewidth]{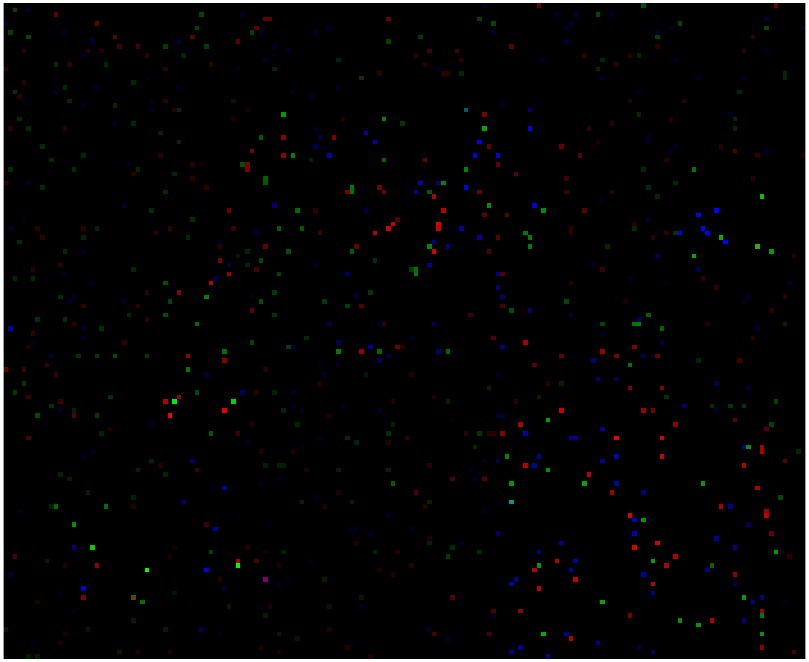}\\
			\includegraphics[width=1\linewidth]{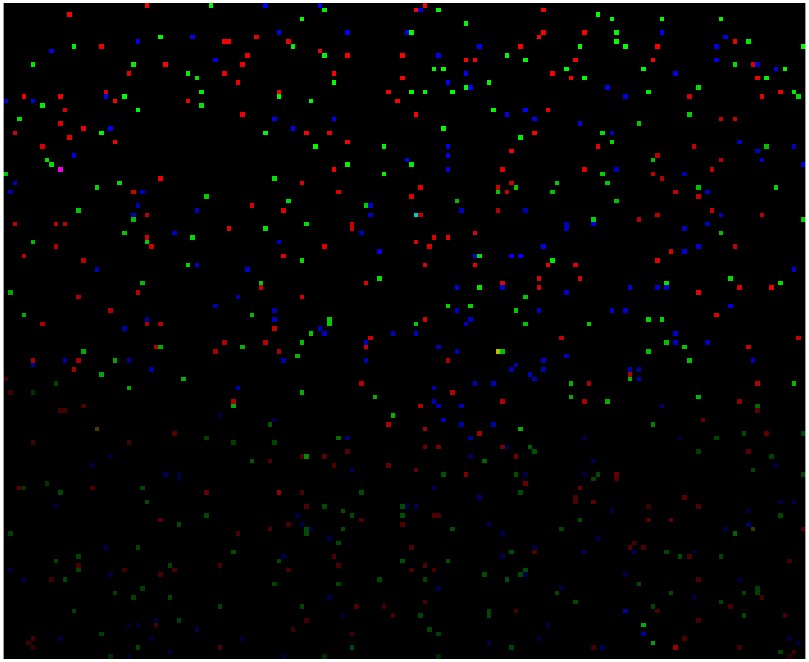}\\
			\includegraphics[width=1\linewidth]{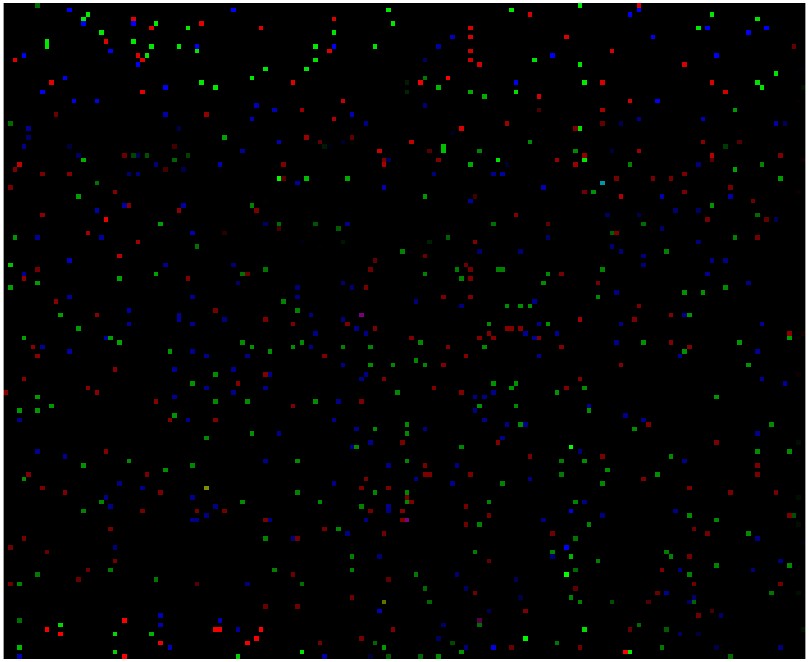} 
	\end{minipage}}\subfigure[TJLC]{
		\begin{minipage}[b]{0.105\linewidth}
		\includegraphics[width=1\linewidth]{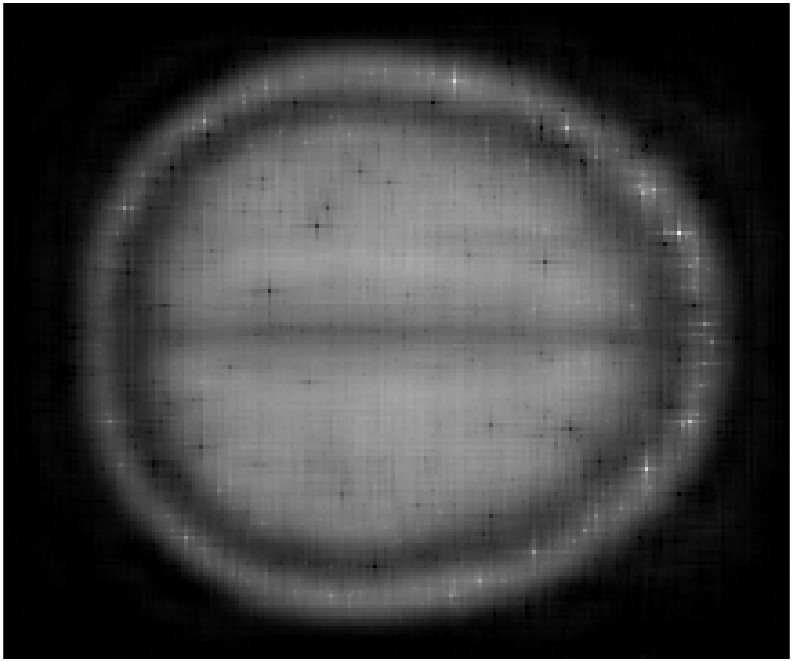}\\
		\includegraphics[width=1\linewidth]{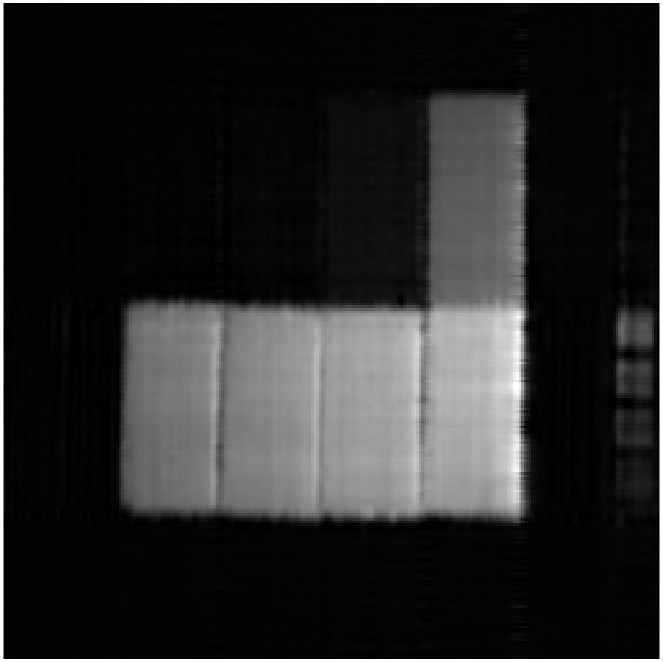}\\
		\includegraphics[width=1\linewidth]{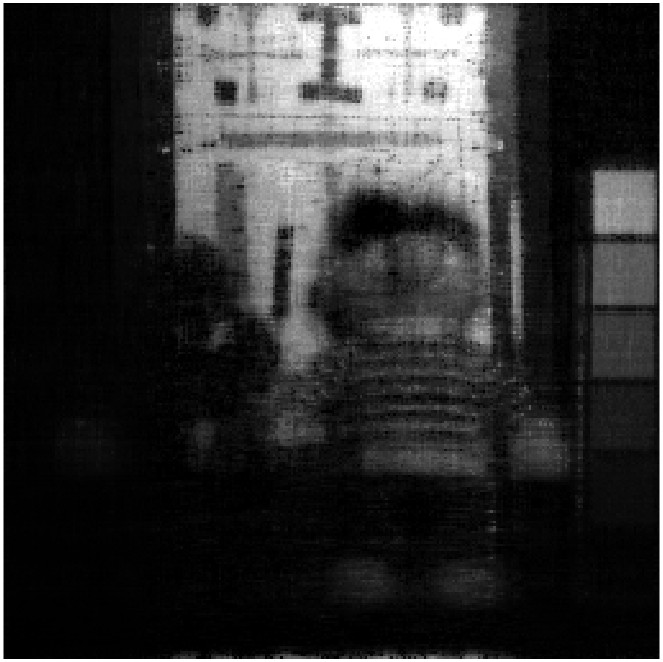}\\
		\includegraphics[width=1\linewidth]{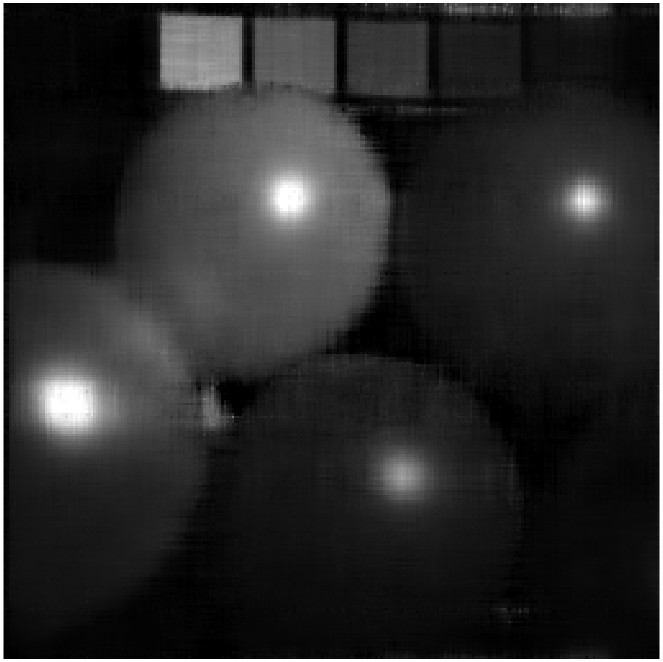}\\
		\includegraphics[width=1\linewidth]{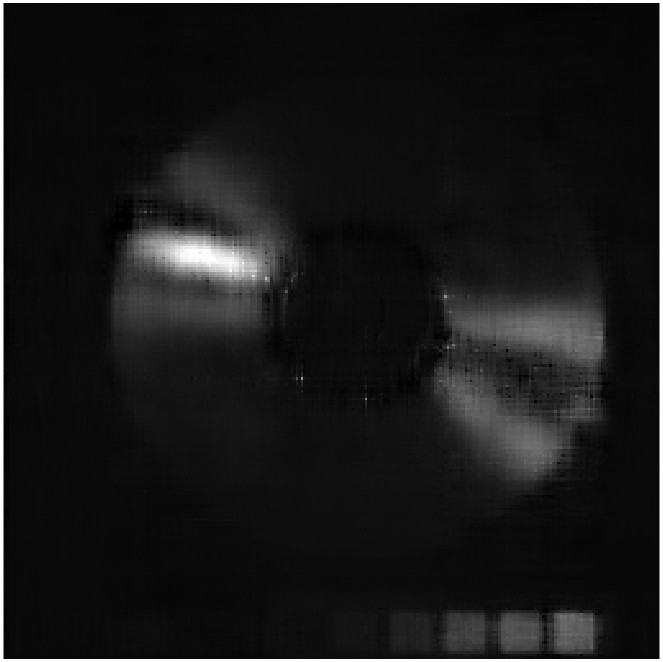}\\
		\includegraphics[width=1\linewidth]{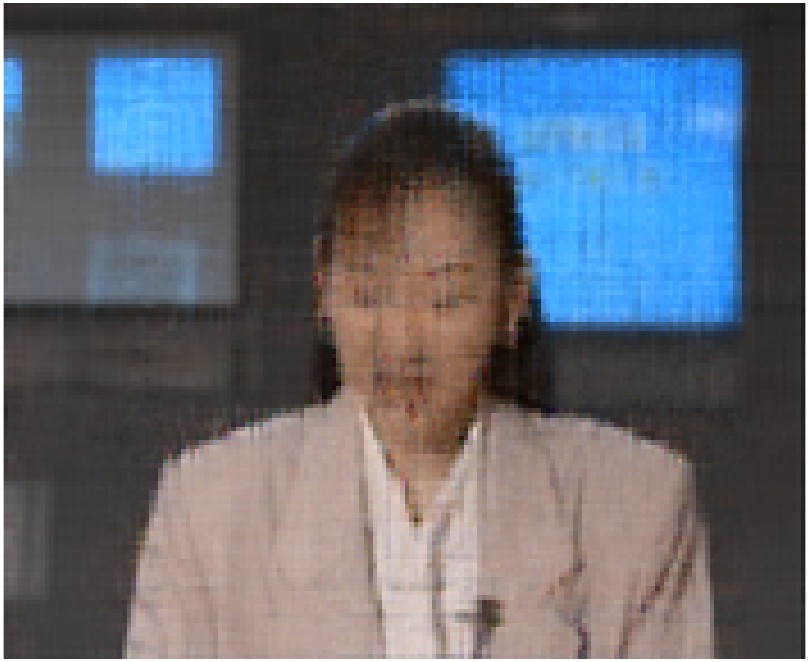}\\
		\includegraphics[width=1\linewidth]{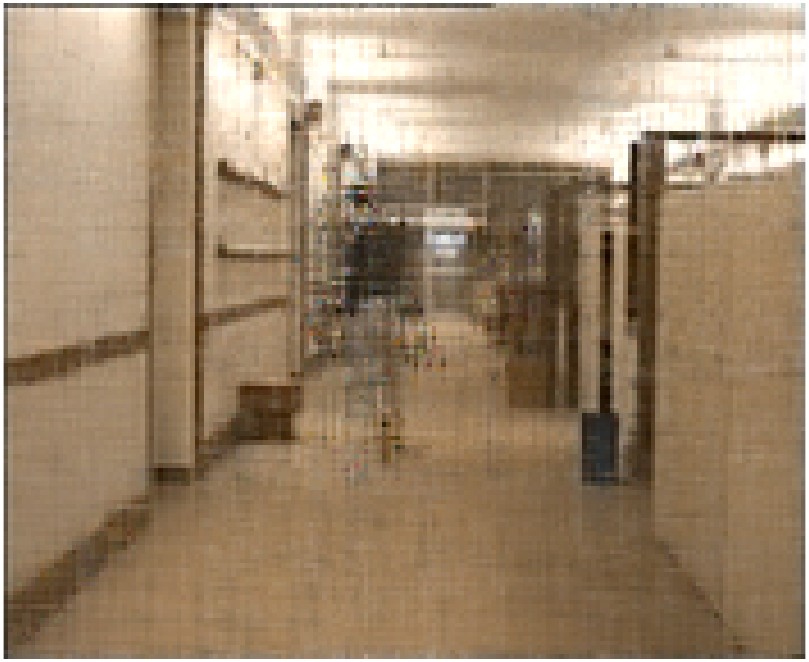}\\
		\includegraphics[width=1\linewidth]{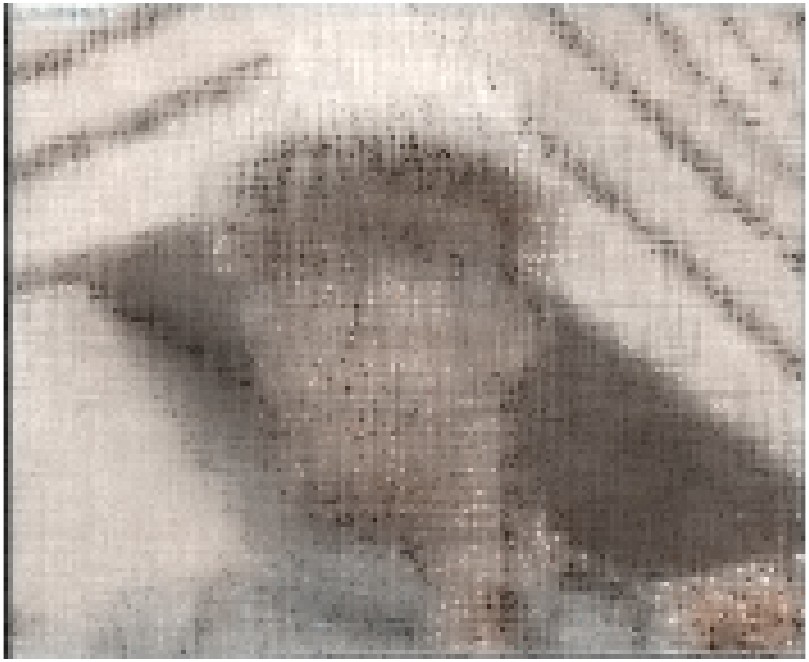}\\
		\includegraphics[width=1\linewidth]{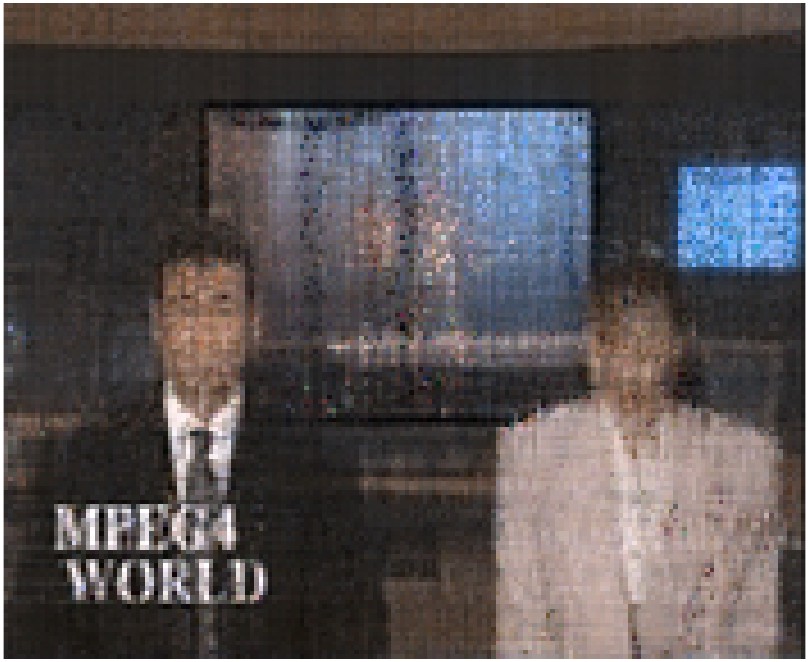}\\
		\includegraphics[width=1\linewidth]{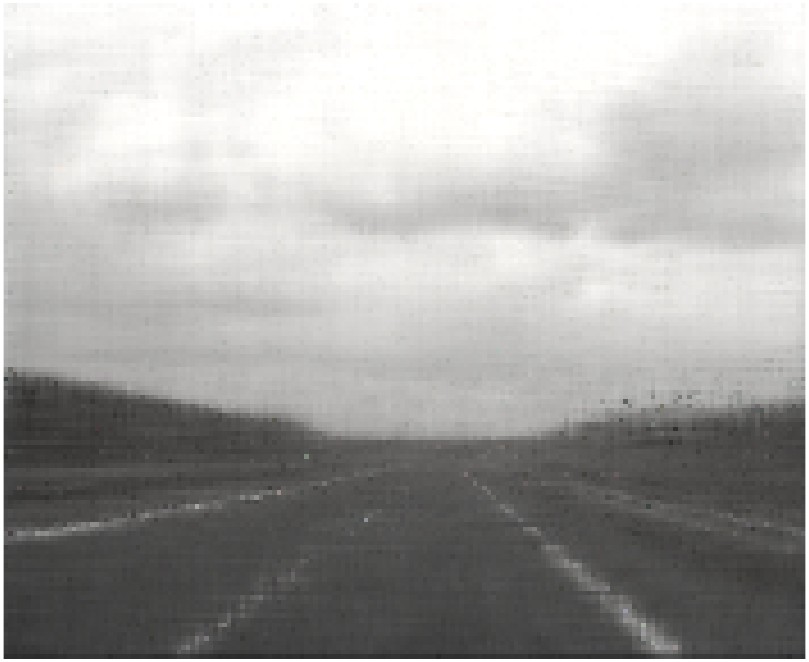}\\
		\includegraphics[width=1\linewidth]{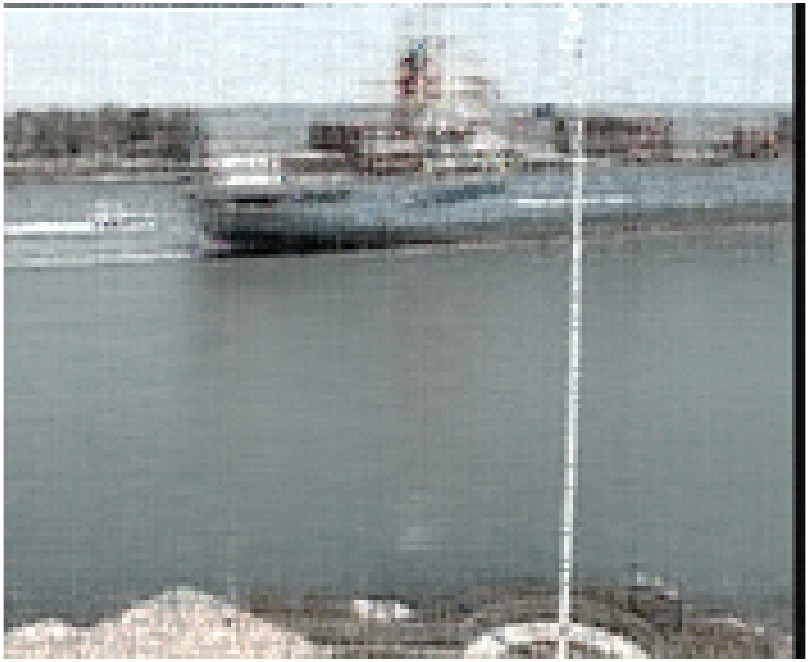} 
	\end{minipage}}
	
	\caption{Visual results for MRI, MSI, and CV datas. }
\label{CVMM}
\end{figure}
To further highlight the strong utilization of observational information by the TJLC method, we conducted experiments with higher missing rates for MRI, MSI, and CV datas. The parameters of the model at this point remain the same as those mentioned in the previous section. Fig. \ref{CVMM} illustrates the visual results for three missing rates: 96\%, 98\%, and 99\%. Specifically, the number of slices for MRI data is 80, 100, and 120 respectively; the number of bands for MSI data is 5, 15, and 25 respectively; and the number of frames for CV data is 5, 20, and 35 respectively. From the figure, it can be observed that even with the missing rate of 99\%, the proposed TJLC method can still perform recovery and identify the overall contours, objects, and scenes in the recovered images. When the missing rate decreases to 98\%, the proposed TJLC method shows a significant improvement in recovery performance. Compared to the missing rate of  99\%, there is a noticeable enhancement in detail recovery in the images at this point. Specifically, the face in the `` foreman " image has evolved from merely showing contour information to displaying discernible detailed facial expressions. At the missing rate of 96\%, the recovery effect of the proposed TJLC method is further significantly enhanced. Compared to the recovery images at the missing rate of 98\%, which still contain a considerable amount of blurry information, the images recovered at the missing rate of 96\% have noticeably regained a large amount of detail, with very little remaining blur. 
\begin{table}[!h]
	\centering
	\caption{The PSNR, SSIM, FSIM and ERGAS values for MRI, MSI, CV datas tested by observed and the TJLC method.}
	\label{HT}
	\resizebox{\textwidth}{!}{%
		\begin{tabular}{|c|c|cccc|cccc|cccc|}
			\hline
			\multicolumn{1}{|l|}{}     & MR       & \multicolumn{4}{c|}{96\%}         & \multicolumn{4}{c|}{98\%}         & \multicolumn{4}{c|}{99\%}         \\ \hline
			Image                       & PQIs     & PSNR   & SSIM  & FSIM  & ERGAS    & PSNR   & SSIM  & FSIM  & ERGAS    & PSNR   & SSIM  & FSIM  & ERGAS    \\ \hline
			\multirow{2}{*}{MRI}       & Observed & 11.354 & 0.308 & 0.522 & 1026.429 & 11.265 & 0.303 & 0.505 & 1037.033 & 11.220 & 0.300 & 0.498 & 1042.367 \\
			& TJLC     & 30.029 & 0.862 & 0.893 & 120.962  & 25.684 & 0.697 & 0.819 & 202.607  & 20.691 & 0.436 & 0.728 & 365.667  \\
			\multirow{2}{*}{clay}      & Observed & 18.674 & 0.244 & 0.770 & 797.808  & 18.581 & 0.222 & 0.790 & 806.173  & 18.537 & 0.211 & 0.804 & 810.500  \\
			& TJLC     & 47.089 & 0.994 & 0.990 & 28.104   & 41.746 & 0.983 & 0.972 & 52.578   & 35.298 & 0.947 & 0.924 & 110.962  \\
			\multirow{2}{*}{chart\_and\_stuffed\_toy}     & Observed & 11.106 & 0.214 & 0.610 & 1135.024 & 11.017 & 0.198 & 0.610 & 1146.743 & 10.972 & 0.189 & 0.617 & 1152.649 \\
			& TJLC     & 35.544 & 0.959 & 0.960 & 67.932   & 30.559 & 0.878 & 0.890 & 121.400  & 25.535 & 0.750 & 0.810 & 217.602  \\
			\multirow{2}{*}{balloons}  & Observed & 13.445 & 0.066 & 0.671 & 863.801  & 13.354 & 0.050 & 0.699 & 872.898  & 13.311 & 0.040 & 0.727 & 877.182  \\
			& TJLC     & 42.126 & 0.990 & 0.985 & 31.938   & 37.747 & 0.972 & 0.960 & 52.759   & 32.506 & 0.922 & 0.922 & 97.637   \\
			\multirow{2}{*}{cd}        & Observed & 17.462 & 0.081 & 0.716 & 720.531  & 17.371 & 0.057 & 0.730 & 728.196  & 17.326 & 0.045 & 0.749 & 731.918  \\
			& TJLC     & 35.568 & 0.972 & 0.959 & 92.203   & 32.470 & 0.953 & 0.939 & 131.314  & 28.849 & 0.914 & 0.902 & 201.202  \\
			\multirow{2}{*}{akiyo}     & Observed & 7.116  & 0.013 & 0.477 & 1134.997 & 7.028  & 0.008 & 0.491 & 1146.611 & 6.983  & 0.005 & 0.516 & 1152.565 \\
			& TJLC     & 35.839 & 0.970 & 0.981 & 41.978   & 31.888 & 0.925 & 0.954 & 65.682   & 27.560 & 0.815 & 0.896 & 107.711  \\
			\multirow{2}{*}{hall}      & Observed & 5.671  & 0.008 & 0.389 & 1189.258 & 5.581  & 0.005 & 0.397 & 1201.700 & 5.537  & 0.003 & 0.423 & 1207.780 \\
			& TJLC     & 32.380 & 0.934 & 0.965 & 56.455   & 29.044 & 0.892 & 0.939 & 81.658   & 25.409 & 0.784 & 0.881 & 124.023  \\
			\multirow{2}{*}{foreman}   & Observed & 4.122  & 0.006 & 0.411 & 1284.712 & 4.032  & 0.003 & 0.421 & 1298.030 & 3.988  & 0.002 & 0.441 & 1304.671 \\
			& TJLC     & 27.328 & 0.818 & 0.886 & 89.754   & 24.203 & 0.685 & 0.827 & 128.062  & 21.131 & 0.523 & 0.768 & 182.270  \\
			\multirow{2}{*}{news}      & Observed & 8.682  & 0.018 & 0.450 & 1060.246 & 8.593  & 0.011 & 0.448 & 1071.163 & 8.549  & 0.006 & 0.457 & 1076.642 \\
			& TJLC     & 30.470 & 0.919 & 0.948 & 88.753   & 26.924 & 0.837 & 0.904 & 132.380  & 22.170 & 0.673 & 0.823 & 227.498  \\
			\multirow{2}{*}{highway}   & Observed & 3.355  & 0.006 & 0.443 & 1362.964 & 3.264  & 0.004 & 0.478 & 1377.222 & 3.221  & 0.002 & 0.517 & 1384.116 \\
			& TJLC     & 31.276 & 0.876 & 0.941 & 55.304   & 29.267 & 0.808 & 0.908 & 69.744   & 27.014 & 0.733 & 0.870 & 90.352   \\
			\multirow{2}{*}{container} & Observed & 4.548  & 0.006 & 0.390 & 1250.345 & 4.458  & 0.003 & 0.406 & 1263.453 & 4.414  & 0.002 & 0.442 & 1269.876 \\
			& TJLC     & 34.724 & 0.961 & 0.979 & 40.321   & 29.341 & 0.906 & 0.943 & 72.791   & 24.449 & 0.800 & 0.875 & 127.235  \\ \hline
		\end{tabular}%
	}
\end{table}

In addition to the visual presentation of results, we also list the quantitative PQIs results for missing rates of 96\%, 98\%, and 99\% in Table \ref{HT}. Through the results in the table, it can be observed that the proposed TJLC method has achieved recovery at a 99\% missing rate, with quantitative metrics such as PSNR and SSIM reaching satisfactory levels. As the amount of observed information increases, i.e., the missing rate decreases to 98\%, the TJLC method exhibits a significant improvement in the recovery of these images. Even with just a 1\% increase in observed information, the improvement in recovery is considerable.

In summary, the proposed TJLC method demonstrates effective recovery even with minimal observed information, as evidenced by its performance at a 99\% missing rate. Moreover, as the amount of observed information increases, both the visual and quantitative results show significant enhancement, highlighting the method's strong utilization of observed information.

\section{Conclusion and future work}
This paper proposes a new LRTC method, namely the TJLC method, which enables a complete tensor to simultaneously possess two low-rank tensor structures, i.e., the tensor Tucker rank and the tensor tubal rank. The proposed tensor joint rank efficiently utilizes this structures, thereby improving the utilization of known element information and ensuring the efficiency of the proposed TJLC method. In addition, the newly proposed tensor logarithmic composite norm not only greatly improves the accuracy of recovery but also solves the problem that the tensor joint rank containing different low-rank tensor structures cannot be directly applied to LRTC problem. Further, the experiments show that the proposed TJLC method can achieve good visual effects and high numerical quantitative results on various data. This further validates the superiority and efficiency of the proposed TJLC method. 

The proposed TJLC method has achieved satisfactory recovery results with observed information as low as 1\%. However, we will delve deeper into how to achieve better results with only 1\% observed information and how to achieve satisfactory recovery with even less observed information. This aspect will be the focus of our further research.



\section*{Acknowledgements}
This work was supported by the National Nature Science Foundation of China under Grant 12071196, Grant 12201267; in part by the Natural Science Foundation of Gansu Province under Grant 22JR5RA559.


\begin{thebibliography}{1}
\expandafter\ifx\csname url\endcsname\relax
  \def\url#1{\texttt{#1}}\fi
\expandafter\ifx\csname urlprefix\endcsname\relax\def\urlprefix{URL }\fi
\expandafter\ifx\csname href\endcsname\relax
  \def\href#1#2{#2} \def\path#1{#1}\fi

\bibitem{12345152009}
T.~G. Kolda, B.~W. Bader, Tensor decompositions and applications, SIAM review
  51~(3) (2009) 455--500.

\bibitem{3120171}
X.~Li, Y.~Ye, X.~Xu, Low-rank tensor completion with total variation for visual
  data inpainting, Proceedings of the AAAI Conference on Artificial
  Intelligence 31~(1) (2017).

\bibitem{2020170}
Y.~Zheng, T.~Huang, X.~Zhao, T.~Jiang, T.~Ji, T.~Ma, Tensor n-tubal rank and
  its convex relaxation for low-rank tensor recovery, Information Sciences 532
  (2020) 170--189.

\bibitem{ZHANG2024110253}
H.~Zhang, H.~Fan, Y.~Li, Tensor recovery based on bivariate equivalent
  minimax-concave penalty, Pattern Recognition 149 (2024) 110253.

\end{thebibliography}


\begin{thebibliography}{10}
	\expandafter\ifx\csname url\endcsname\relax
	\def\url#1{\texttt{#1}}\fi
	\expandafter\ifx\csname urlprefix\endcsname\relax\def\urlprefix{URL }\fi
	\expandafter\ifx\csname href\endcsname\relax
	\def\href#1#2{#2} \def\path#1{#1}\fi
\bibitem{LIAO2023109624}
	T.~Liao, Z.~Wu, C.~Chen, Z.~Zheng, X.~Zhang, Tensor completion via
	convolutional sparse coding with small samples-based training, Pattern Recognit. 141 (2023) 109624.
\bibitem{YANG2022108311}
M.~Yang, Q.~Luo, W.~Li, M.~Xiao, Nonconvex 3D array image data recovery and
pattern recognition under tensor framework, Pattern Recognit. 122 (2022)
108311.
\bibitem{doi:10.1137/22M1531907}
B.-Z. Li, X.-L. Zhao, X.~Zhang, T.-Y. Ji, X.~Chen, M.~K. Ng, A learnable
group-tube transform induced tensor nuclear norm and its application for
tensor completion, SIAM J. Imag. Sci. 16~(3) (2023) 1370--1397.
\bibitem{PAN2023109699}
E.~Pan, Y.~Ma, X.~Mei, F.~Fan, J.~Ma, Hyperspectral image denoising via
spectral noise distribution bootstrap, Pattern Recognit. 142 (2023) 109699.
\bibitem{li2022nonlinear}
B.-Z. Li, X.-L. Zhao, T.-Y. Ji, X.-J. Zhang, T.-Z. Huang, Nonlinear transform induced tensor nuclear norm for tensor completion,  J. Sci. Comput. 92~(3) (2022) 83.
\bibitem{Lu_2019_CVPR}
C.~Lu, X.~Peng, Y.~Wei, Low-rank tensor completion with a new tensor nuclear norm induced by invertible linear transforms, in: 2019 IEEE Conference on Computer Vision and Pattern Recognition, 2019, pp. 5996-6004.
\bibitem{jiang2023robust}
W.~Jiang, J.~Zhang, C.~Zhang, L.~Wang, H.~Qi, Robust low tubal rank tensor completion via factor tensor norm minimization, Pattern Recognit. 135 (2023) 109169.
\bibitem{10145070}
H.~Zhang, H.~Fan, Y.~Li, X.~Liu, C.~Liu, X.~Zhu, Tensor Recovery Based on a Novel Non-Convex Function Minimax Logarithmic Concave Penalty Function, IEEE Trans. Image Process. 32 (2023) 3413--3428.
\bibitem{song2020robust}
G.~Song, M.~K. Ng, X.~Zhang, Robust tensor completion using transformed tensor
singular value decomposition, Numer. Linear Algebra Appl.
27~(3) (2020) e2299.
\bibitem{KONG2023109545}
W.~Kong, F.~Zhang, W.~Qin, J.~Wang, Low-tubal-rank tensor recovery with
multilayer subspace prior learning, Pattern Recognit. 140 (2023) 109545.
\bibitem{5032019109}
J.~Xue, Y.~Zhao, W.~Liao, J.~{Cheung-Wai Chan}, {N}onconvex tensor rank
minimization and its applications to tensor recovery,  Inf. Sci.
503 (2019) 109--128.
\bibitem{ding2022tensor}
W.~Ding, Z.~Sun, X.~Wu, Z.~Yang, J.~Sol{\'e}-Casals, C.~F. Caiafa, Tensor
completion algorithms for estimating missing values in multi-channel audio
signals, Comput. Electr. Eng. 97 (2022) 107561.
\bibitem{harshman1970foundations}
R.~A. Harshman, et~al., Foundations of the PARAFAC procedure: models and
conditions for an ``explanatory" multimodal factor analysis, UCLA Work. Paper. Phonetic. 16 (1970) 1–84.
\bibitem{tucker1966some}
L.~R. Tucker, Some mathematical notes on three-mode factor analysis,
Psychometrika 31~(3) (1966) 279--311. 

\bibitem{doi:10.1137/110837711}
M.~E. Kilmer, K.~Braman, N.~Hao, R.~C. Hoover, Third-order tensors as operators
on matrices: A theoretical and computational framework with applications in
imaging, SIAM J. Matrix Anal. Appl. 34~(1) (2013) 148--172.
\bibitem{6416568}
M.~E. Kilmer, C.~D. Martin, Factorization strategies for third-order tensors, Linear Algebra Appl. 435~(3) (2011) 641--658. 
\bibitem{139201360}
C.~J. Hillar, L.-H. Lim, Most tensor problems are NP-hard, J. ACM 60~(6) (2013) 1--39.
\bibitem{6909886}
Z.~Zhang, G.~Ely, S.~Aeron, N.~Hao, M.~Kilmer, Novel Methods for
Multilinear Data Completion and De-noising Based on Tensor-SVD, in: 2014 IEEE Conference on Computer Vision and Pattern Recognition, 2014, pp. 3842–3849.
\bibitem{9460114}
J.~Xue, Y.~Zhao, S.~Huang, W.~Liao, J.~C.-W. Chan, S.~G. Kong, Multilayer
sparsity-based tensor decomposition for low-rank tensor completion, IEEE Trans. Neural Networks Learn. Syst. 33~(11) (2022) 
6916--6930.  
\bibitem{9694520}
J.~Xue, Y.~Zhao, Y.~Bu, J.~C.-W. Chan, S.~G. Kong, When laplacian scale mixture
meets three-layer transform: A parametric tensor sparsity for tensor
completion, IEEE Trans. Cybern. 52~(12) (2022) 13887--13901. 
\bibitem{YU2023109343}
Q.~Yu, M.~Yang, Low-rank tensor recovery via non-convex regularization,
structured factorization and spatio-temporal characteristics, Pattern Recognit. 137 (2023) 109343.
\bibitem{24120002}
W.~Kong, F.~Zhang, W.~Qin, Q.~Feng, J.~Wang, Low-tubal-rank tensor completion via local and nonlocal knowledge, Inf. Sci. 657 (2024) 120002.

\bibitem{12345152009}
T.~G. Kolda, B.~W. Bader, Tensor decompositions and applications,  SIAM Rev.
51~(3) (2009) 455--500.
\bibitem{2020170}
Y.~Zheng, T.~Huang, X.~Zhao, T.~Jiang, T.~Ji, T.~Ma, Tensor n-tubal rank and
its convex relaxation for low-rank tensor recovery, Inf. Sci. 532
(2020) 170--189.
\bibitem{8606166}
C.~Lu, J.~Feng, Y.~Chen, W.~Liu, Z.~Lin, S.~Yan, Tensor Robust Principal
Component Analysis with a New Tensor Nuclear Norm, IEEE Trans. Pattern Anal. Mach. Intell. 42~(4) (2020) 925--938.
\bibitem{Gu_2014_CVPR}
S.~Gu, L.~Zhang, W.~Zuo, X.~Feng, Weighted nuclear norm minimization with
application to image denoising, in: 2014 IEEE Conference on Computer Vision and Pattern Recognition, 2014, pp. 2862-2869.

\bibitem{mirsky1975trace}
L.~Mirsky, A trace inequality of John von Neumann, Monatshefte f{\"u}r
Mathematik 79~(4) (1975) 303--306.
\bibitem{1284395}
Z.~Wang, A.~Bovik, H.~Sheikh, E.~Simoncelli, Image quality assessment: from
error visibility to structural similarity, IEEE Trans. Image Process. 13~(4) (2004) 600--612.

\bibitem{5705575}
L.~Zhang, L.~Zhang, X.~Mou, D.~Zhang, FSIM: A Feature Similarity
Index for Image Quality Assessment, IEEE Trans. Image Process. 20~(8) (2011) 2378--2386.

\bibitem{2432002352}
L.~Wald, Data Fusion: Definitions and Architectures: Fusion of Images of Different Spatial Resolutions, Paris, France: Presses des MINES, 2002.
\bibitem{6138863}
J.~Liu, P.~Musialski, P.~Wonka, J.~Ye, Tensor Completion for Estimating
Missing Values in Visual Data, IEEE Trans. Pattern Anal. Mach. Intell. 35~(1) (2013) 208--220.
\bibitem{3120171}
X.~Li, Y.~Ye, X.~Xu, Low-rank tensor completion with total variation for visual data inpainting, in Proceedings of the AAAI Conference on Artificial Intelligence 31~(1) (2017).
\bibitem{7782758}
Z.~Zhang, S.~Aeron, Exact Tensor Completion Using t-SVD, IEEE Trans. Signal Process. 65~(6) (2017) 1511--1526.

\bibitem{1122020112680}
T.-X. Jiang, T.-Z. Huang, X.-L. Zhao, L.-J. Deng, Multi-dimensional imaging data recovery via minimizing the partial sum of tubal nuclear norm, J. Comput. Appl. Math. 372 (2020) 112680.
\bibitem{9115254}
T.-X. Jiang, M.~K. Ng, X.-L. Zhao, T.-Z. Huang, Framelet representation of
tensor nuclear norm for third-order tensor completion, IEEE Trans. Image Process. 29 (2020) 7233--7244.
\bibitem{ZHANG2024110253}
H.~Zhang, H.~Fan, Y.~Li, Tensor recovery based on bivariate equivalent
minimax-concave penalty, Pattern Recognit. 149 (2024) 110253.

\end{thebibliography}

\end{document}